%% file: MM_PGO.tex
\documentclass[journal]{IEEEtran}

\usepackage[mathlines,switch]{lineno}
\newcommand*\patchAmsMathEnvironmentForLineno[1]{%
	\expandafter\let\csname old#1\expandafter\endcsname\csname #1\endcsname
	\expandafter\let\csname oldend#1\expandafter\endcsname\csname end#1\endcsname
	\renewenvironment{#1}%
	{\linenomath\csname old#1\endcsname}%
	{\csname oldend#1\endcsname\endlinenomath}}%
\newcommand*\patchBothAmsMathEnvironmentsForLineno[1]{%
	\patchAmsMathEnvironmentForLineno{#1}%
	\patchAmsMathEnvironmentForLineno{#1*}}%
\AtBeginDocument{
	\patchBothAmsMathEnvironmentsForLineno{equation}%
	\patchBothAmsMathEnvironmentsForLineno{align}%
	\patchBothAmsMathEnvironmentsForLineno{flalign}%
	\patchBothAmsMathEnvironmentsForLineno{alignat}%
	\patchBothAmsMathEnvironmentsForLineno{gather}%
	\patchBothAmsMathEnvironmentsForLineno{multline}%
}
\input{header}

\newcommand{\refgarage}{\cref{fig::outlier_garage}}

\newcommand{\refbenchmark}[2]{\cref{fig::benchmark3D,fig::benchmark2D}}

\newcommand{\datasetinfo}{\hyperapp{appendix::L}{L}}

\newcommand{\hyperapp}[2]{Appendix \hyperref[{#1}]{#2}}
\newcommand{\inputapp}{\input{appendix}}
\def\highlight{}

\IEEEoverridecommandlockouts                              

\title{\LARGE \bf Majorization Minimization Methods for Distributed Pose Graph Optimization}

\author{Taosha Fan and Todd D. Murphey
	\thanks{T. Fan is with Meta AI, Pittsburgh, PA 15213, USA and T. D. Murphey is with the Department of Mechanical Engineering, Northwestern University, Evanston, IL 60201, USA. E-mail: {\tt taoshaf@meta.com, t-murphey@northwestern.edu}
		
		This material is partially based upon work supported by the National Science Foundation under awards 1662233 and 1837515. 
		
		The authors thank Yulun Tian for sharing the code of Riemannian block coordinate descent method ($\rbcd$) for distributed PGO.
}
}

\begin{document}
\maketitle
\thispagestyle{empty}
\pagestyle{empty}


\begin{abstract}
We consider the problem of distributed pose graph optimization (PGO) that has important applications in multi-robot simultaneous localization and mapping (SLAM). We propose the majorization minimization (MM) method for distributed PGO ($\mm$) that applies to a broad class of robust loss kernels. The $\mm$ method is guaranteed to converge to first-order critical points under mild conditions. Furthermore, noting that the $\mm$ method is reminiscent of {\highlight proximal methods}, we leverage Nesterov's method and adopt adaptive restarts to accelerate convergence. The resulting accelerated MM methods for distributed PGO---both with a master node in the network ($\ammc$) and without ($\ammd$)---have faster convergence in contrast to the $\mm$ method without sacrificing theoretical guarantees. In particular, the $\ammd$ method, which needs no master node and is fully decentralized, features a novel adaptive restart scheme and has a rate of convergence comparable to that of the $\ammc$ method using a master node to aggregate information from all the nodes. The efficacy of this work is validated through extensive applications to 2D and 3D SLAM benchmark datasets and comprehensive comparisons against existing state-of-the-art methods, indicating that our MM methods converge faster and result in better solutions to distributed PGO. 
\end{abstract}

\vspace{-0.25em}
\section{Introduction}\label{section::intro}
\input{intro}

\vspace{-0.25em}
\section{Related Work}\label{section::work}
\input{related_work}

\vspace{-0.2em}
\section{Notation\protect\footnote{A more complete summary of the notation  is given in \hyperapp{appendix::A}{A}. }}\label{section::notation}

\input{notation}

\section{Problem Formulation}\label{section::problem}
\input{problem}

\section{The Majorization of Loss Kernels}\label{section::loss}
\input{loss_kernels}

\section{The Majorization of Distributed Pose Graph Optimization}\label{section::major_pgo}
\input{major_pgo}

\section{The Majorization Minimization Method for Distributed Pose Graph Optimization}\label{section::mm}
\input{mm_pgo}

\section{The Accelerated Majorization Minimization Method for Distributed Pose Graph Optimization with Master Node}\label{section::ammc}
\input{amm_pgo_c}

\section{The Accelerated Majorization Minimization Method for Distributed Pose Graph Optimization without Master Node}\label{section::amm}
\input{amm_pgo}

\section{Experiments}\label{section::experiemnt}
\input{experiments}

\section{Conclusion and Future Work}\label{section::conclusion}
\input{conclusions}

\vspace{-0.25em}

\bibliographystyle{IEEEtran}
\bibliography{mybib}

\vspace{-1.5em}

\begin{IEEEbiography}[{\includegraphics[width=1in,height=1.25in,clip,keepaspectratio]{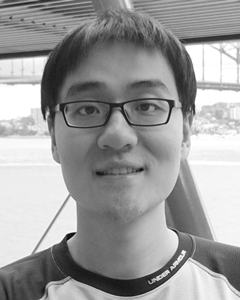}}]{Taosha Fan}
	received his B.E. degree in automotive engineering from Tongji University, Shanghai, China, in 2013, and M.S. degrees in mechanical engineering and mathematics from Johns Hopkins University, Baltimore, MD, USA, in 2015, and Ph.D. degree in mechanical engineering at Northwestern University, Evanston, IL, USA in 2022. His research interests lie at the intersection of robotic control, simulation and estimation. He is currently a research engineer in Meta AI, Pittsburgh, PA, USA.
\end{IEEEbiography}

\begin{IEEEbiography}[{\includegraphics[width=1in,height=1.25in,clip,keepaspectratio]{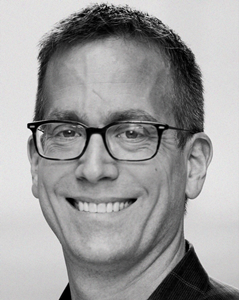}}]{Todd Murphey}
	received his B.S. degree in mathematics from the University of Arizona and the Ph.D. degree in Control and Dynamical Systems from the California Institute of Technology.
	
	He is currently a Professor of mechanical engineering with Northwestern University, Evanston, IL, USA. His laboratory is part of the Center for Robotics and Biosystems. His research interests include robotics, control, machine learning in physical systems, and computational neuroscience
\end{IEEEbiography}

\input{appendix}
\end{document}

%% file: header.tex
\usepackage{amsmath,amscd}
\usepackage{bbm}

\usepackage{dsfont}
\usepackage{bm}
\usepackage{amsthm,array}
\usepackage{amsmath}
\usepackage{amsfonts,latexsym}
\usepackage{graphicx,wrapfig}
\usepackage{subfig}
\usepackage{times}
\usepackage{psfrag,epsfig}
\usepackage{verbatim}
\usepackage{tabularx}
\usepackage[hidelinks]{hyperref}
\usepackage{color}
\usepackage{soul}
\usepackage{indentfirst}
\usepackage{inputenc}
\usepackage[capitalise]{cleveref}
\usepackage{courier}
\usepackage{authblk}
\usepackage{mathtools}
\usepackage{mathtools}
\usepackage{cite}
\usepackage{amssymb}
\usepackage{graphicx}
\usepackage{caption}
\usepackage{dblfloatfix}
\usepackage{multirow}
\usepackage{algorithm,algpseudocode}
\usepackage{ltablex}
\usepackage{relsize}

\usepackage[shortlabels]{enumitem}
\usepackage{xtab}
\usepackage{xfrac}
\usepackage{amssymb}             
\usepackage{soul}
\usepackage{flushend}
\usepackage{chngcntr}
\usepackage{color}
\usepackage{ifthen,version}
\usepackage{float}
\usepackage{pifont}
\newboolean{include-notes}
\setboolean{include-notes}{true}
\newcommand{\tmnote}[1]{\ifthenelse{\boolean{include-notes}}%
  {\textcolor{blue}{\emph{TM says: #1}}}{}}
\newcommand{\tfnote}[1]{\ifthenelse{\boolean{include-notes}}%
  {\textcolor{red}{\emph{PT says: #1}}}{}}

\newcommand{\cmark}{YES}%
\newcommand{\xmark}{$\times$}%

\DeclareMathOperator*{\trace}{\mathrm{trace}}

\makeatletter
\def\mathcenterto#1#2{\mathclap{\phantom{#1}\mathclap{#2}}\phantom{#1}}
\let\old@widetilde\widetilde
\def\widetildeto#1#2{\mathcenterto{#2}{\old@widetilde{\mathcenterto{#1}{#2\,}}}}
\let\old@widehat\widehat
\def\widehatto#1#2{\mathcenterto{#2}{\old@widehat{\mathcenterto{#1}{#2\,}}}}
\def\mathcenterto#1#2{\mathclap{\phantom{#1}\mathclap{#2}}\phantom{#1}}
\let\old@underline\underline
\def\underlineto#1#2{\mathcenterto{#2}{\old@underline{\mathcenterto{#1}{#2\,}}}}
\makeatother

\def\0{\boldsymbol{0}}
\def\1{\boldsymbol{1}}

\def\I{\mathbf{I}}

\def\d{\text{d}}

\def\ratio{\mu}

\def\mm{\mathsf{MM\!\!-\!\!PGO}}
\def\amm{\mathsf{AMM\!\!-\!\!PGO}}
\def\ammd{\mathsf{AMM\!\!-\!\!PGO}^{\#}}
\def\ammc{\mathsf{AMM\!\!-\!\!PGO}^*}

\def\dgs{\mathsf{DGS}}
\def\sesync{\mathsf{SE\!-\!Sync}}
\def\rbcd{\mathsf{RBCD}}
\def\rbcdc{\mathsf{RBCD}{\small++}^*}
\def\rbcdd{\mathsf{RBCD}{\small++}^{\#}}
\def\pcm{\mathsf{PCM}}

\def\0{\boldsymbol{0}}

\def\NN{\mathcal{N}}

\def\I{\mathbf{I}}
\def\d{\text{d}}

\def\transpose{\top} 

\def\R{\mathbb{R}}

\def\VV{\mathcal{V}}

\def\AA{\mathcal{A}}
\def\XX{\mathcal{X}}

\def\se3{\mathfrak{se}(3)}

\def\lG{H}
\def\sF{\overline{F}}
\def\sG{G}

\def\lF{\overline{F}}

\def\diag{\mathrm{diag}}

\def\nM{M}

\def\nH{\mathrm{\Omega}}

\def\nGamma{\mathrm{\Gamma}}
\def\nGammak{\mathrm{\Gamma}^{(\sk)}}
\def\nGammaak{\mathrm{\Gamma}^{\ak}}
\def\lnGamma{\mathrm{\Pi}}
\def\lnGammak{\mathrm{\Pi}^{(\sk)}}
\def\lnGammaak{\mathrm{\Pi}^{\ak}}
\def\lnGammataik{\mathrm{\Pi}_i^{\tau,\ak}}
\def\lnGammaRaik{\mathrm{\Pi}_i^{\kappa,\ak}}
\def\lnGammanaik{\mathrm{\Pi}_i^{\nu,\ak}}
\def\nR{{\widetilde{R}}}
\def\nt{\tilde{t}}
\def\aa{\alpha\alpha}
\def\ab{\alpha\beta}
\def\ba{\beta\alpha}
\def\tai{t_i^\alpha}
\def\Rai{R_i^\alpha}
\def\taik{t_i^{\ak}}
\def\Raik{R_i^{\ak}}

\def\sk{\mathsf{k}}
\def\sK{\mathsf{K}}
\def\skp{\mathsf{k}+1}
\def\skh{\mathsf{k}+\half}
\def\skm{\mathsf{k}-1}
\def\Xk{X^{(\sk)}}
\def\Xkp{X^{(\sk+1)}}
\def\Xkm{X^{(\sk-1)}}
\def\Xkh{X^{(\sk+\shalf)}}
\def\ak{\alpha(\sk)}
\def\abk{\alpha\beta(\sk)}

\def\akp{\alpha(\sk+1)}

\def\akm{\alpha(\sk-1)}
\def\akh{\alpha(\sk+\shalf)}
\def\bk{\beta(\sk)}

\def\bkm{\beta(\sk-1)}

\def\Xa{X^\alpha}
\def\XXa{\XX^\alpha}
\def\Xb{X^\beta}
\def\Xak{X^{\ak}}
\def\Xbk{X^{\bk}}
\def\Xakp{X^{\akp}}

\def\Xakm{X^{\akm}}
\def\Xakh{X^{\akh}}

\def\Yak{Y^{\ak}}
\def\Yk{Y^{(\sk)}}
\def\wabijk{\omega_{ij}^{\ab(\sk)}}
\def\wbajik{\omega_{ji}^{\ba(\sk)}}
\def\nMk{\nM^{(\sk)}}

\def\QQ{\mathcal{Q}}

\def\EE{\mathcal{E}}
\def\aEE{\overrightarrow{\EE}}

\def\aGG{\overrightarrow{\mathcal{G}}}

\def\grad{\mathrm{grad}}

\def\half{\frac{1}{2}}
\def\shalf{\frac{1}{2}}

\long\def\answer#1{}

\long\def\comment#1{}

\newcommand{\ea}{{et al.~}}

\newcommand{\innprod}[2]{\big\langle {#1},{#2}\big\rangle}

\definecolor{ao}{rgb}{1.0, 0, 0.0}

\newtheorem{prop}{Proposition}
\newtheorem{problem}{Problem}

\newtheorem{assumption}{Assumption}
\newtheorem{lemma}{Lemma}
\newtheorem{example}{Example}
\newtheorem{remark}{Remark}

\crefname{prop}{Proposition}{Propositions}
\crefname{assumption}{Assumption}{Assumptions}

\theoremstyle{remark}
\newtheorem*{remark*}{{Remark}}

\newcolumntype{P}[1]{>{\centering\arraybackslash}p{#1}}
\newcolumntype{M}[1]{>{\centering\arraybackslash}m{#1}}

\newcommand{\RNum}[1]{\uppercase\expandafter{\romannumeral #1\relax}}

\hypersetup{
    bookmarks=true,         
    unicode=false,          
    pdftoolbar=true,        
    pdfmenubar=true,        
    pdffitwindow=false,     
    pdfstartview={FitH},    
    colorlinks=true,       
    linkcolor=blue,          
    citecolor=red,        
    filecolor=magenta,      
    urlcolor=cyan           
}

%% file: appendix.tex
\newpage
\section*{Appendix A.\;\; Notation}\label{appendix::A}
\input{appendix_notation}

\section*{Apendix B.\;\;Proof of \cref{prop::upperM}}\label{appendix::B}
\input{appendix_upperM}

\section*{Apendix C.\;\;Proof of \cref{prop::upper}}\label{appendix::C}
\input{appendix_upper}

\section*{Apendix D.\;\;Proof of \cref{prop::G}}\label{appendix::D}
\input{appendix_major}

\section*{Apendix E.\;\;Proof of \cref{prop::mm}}\label{appendix::E}
\input{appendix_mm}

\section*{Apendix F.\;\;Proof of \cref{prop::ammc}}\label{appendix::F}
\input{appendix_ammc}

\section*{Apendix G.\;\;Proof of \cref{prop::FsF}}\label{appendix::G}
\input{appendix_restart}

\section*{Apendix H.\;\;Proof of \cref{prop::amm}}\label{appendix::H}
\input{appendix_amm}

\section*{Apendix I.\;\; The Formulation of $\nGamma^{\ak}$ in \cref{eq::GMa}}\label{appendix::I}
\input{appendix_G}

\section*{Apendix J.\;\; The Formulation of $\lnGamma^{\ak}$ and $\lnGamma_i^{\ak}$ in \cref{eq::lGMa,eq::lGMai}}\label{appendix::J}
\input{appendix_lG}

\section*{Appendix K.\;\; The Closed-Form Solution to \cref{eq::xlGi}}\label{appendix::K}
\input{appendix_sol}

\section*{Appendix L.\;\; Pose Graph Optimization  Datasets}\label{appendix::L}
A variety of 2D and 3D SLAM benchmark datasets \cite{rosen2016se} are used to evaluate distributed PGO methods. As is shown in \cref{table::dataset}, there are seven real-world datasets ({\sf ais2klinik}, 	{\sf CSAIL}, {\sf intel}, {\sf MITb}, {\sf garage}, {\sf cubicle} and {\sf rim}) and five synthetic datasets ({\sf city}, {\sf M3500}, {\sf sphere}, {\sf torus}, and {\sf grid}).
\input{table_dataset}

%% file: appendix_notation.tex
\begin{table}[H]
\centering
\renewcommand{\arraystretch}{1.1}
\begin{tabular}{P{0.09\textwidth}  p{0.27\textwidth}  P{0.06\textwidth} }
\hline
Symbol & \multicolumn{1}{c}{Description} & \multicolumn{1}{c}{Equation}\\
\hline
\multicolumn{3}{c}{Inner Products, Norms and Gradients}\\
\hline
$\;\;\;\,\innprod{\cdot}{\cdot}_M$ & $\innprod{X}{Y}_M\triangleq\trace(X M Y^\transpose)$ & \cref{eq::innerM} \\
$\innprod{\cdot}{\cdot}$ & $\innprod{X}{Y}\triangleq\trace(X Y^\transpose)$ & \cref{eq::inner}\\
$\|\cdot\|$ & The Frobenius norm of matrices and vectors \\
$\;\,\|\cdot\|_2$ & The induced 2-norm of matrices and vectors \\
$\;\;\;\|\cdot\|_M$ & $\|X\|_M\triangleq \sqrt{\trace(X M X^\transpose)}$& \cref{eq::normM}\\
$\nabla F(X)$ & The Euclidean gradient of $F(X)$ &\\
$\grad\, F(X)$ & The Riemannian gradient of $F(X)$ &\\
\hline
\multicolumn{3}{c}{Graph Theory}\\
\hline
$\aGG=(\VV,\,\aEE)$ & The directed graph whose vertices are ordered pairs &\\
$\aEE^{\ab}$ & The set of edges between node $\alpha$ and $\beta$ & \cref{eq::aEE}\\
$\NN_-^\alpha$ & The set of nodes with edges from node $\alpha$& \cref{eq::NN-}\\
$\NN_+^\alpha$ & The set of nodes with edges to node $\alpha$& \cref{eq::NN+}\\
$\NN^\alpha$ & The set of neighbors of node $\alpha$& \cref{eq::NN}\\
\hline
\multicolumn{3}{c}{Poses}\\
\hline
$X_i^\alpha$ & The $i$-th pose of node $\alpha$ & \cref{eq::Xai}\\
$t_i^\alpha$ & The translational part for  $X_i^\alpha$ \\
$R_i^\alpha$ & The rotational part for  $X_i^\alpha$ \\
$X^\alpha$ & The poses of node $\alpha$ & \cref{eq::Xa}\\
$X$ & All the poses of distributed pose graph optimization & \cref{eq::Xall}\\
$\nt_{ij}^{\aa},\;\nt_{ij}^{\ab}$ & The translational part of the noisy intra- and inter-node relative pose measurements  \\
$\nR_{ij}^{\aa},\;\nR_{ij}^{\ab}$ & The rotational part of the noisy intra- and inter-node relative pose measurements  \\
\hline
\multicolumn{3}{c}{Objective Functions}\\
\hline
$F(X)$ & The objective function for distributed pose graph optimization & \cref{eq::F} \\
$F_{ij}^{\aa}(X)$ & The sub-objective function related to intra-node measurements & \cref{eq::Fij2} \\
$F_{ij}^{\ab}(X)$ & The sub-objective function related to inter-node measurements & \cref{eq::Fab2} \\
\hline
\multicolumn{3}{c}{Surrogate Functions}\\
\hline
$\!E_{ij}^{\ab}(X|\Xk)$ & The surrogate function of $F_{ij}^{\ab}(X)$ & \cref{eq::Eab}\\
$\!G(X|\Xk)$ & The surrogate function of $F(X)$ & \cref{eq::G}\\
$\!G^\alpha(\Xa|\Xk)$ & The sub-surrogate function  in $G(X|\Xk)$ that is related to poses of node $\alpha$ & \cref{eq::GMa}\\
$\!H(X|\Xk)$ & The surrogate function of $F(X)$ & \cref{eq::lG}\\
$\!H^\alpha(\Xa|\Xk)$ & The sub-surrogate function in  $H(X|\Xk)$ that is related to poses of node $\alpha$  & \cref{eq::lGMa}\\
$\!H_i^\alpha(\Xa_i|\Xk)$ & The sub-surrogate function in $H(X|\Xk)$ and $H^\alpha(\Xa|\Xk)$ that is related to a single pose $X_i^\alpha$ of node $\alpha$ & \cref{eq::lGMai}\\
\hline
\multicolumn{3}{c}{Nesterov's Acceleration}\\
\hline
$s^{\ak}$, $\lambda^{\ak}$ & The scalars related to the extrapolation of Nesterov's acceleration & \cref{eq::sk,eq::gammak}\\
$Y^{\ak}$ & The extrapolation of $\Xak$ and $\Xakm$ & \cref{eq::Yak} \\
\hline
\multicolumn{3}{c}{Adaptive Restart}\\
\hline
$\lF^{(\sk)}$ & The exponential moving average of $F(X^{(0)}),\, F(X^{(1)}),\,\cdots, F(\Xk)$ & \cref{eq::lFk}\\
$\!\Delta G^\alpha(X|\Xk) $ & The function computing the gap between $F(X)$ and $G(X|\Xk)$ that is related to poses of node $\alpha$ and its neighbors&  \cref{eq::DGa}\\
$\!G^{\ak}$,\, $F^{\ak}$ & The intermediates for the adaptive restart scheme without a master node& \cref{eq::Gakp,eq::Fakp}\\
$\!\lF^{\ak}$& The exponential moving average of $F^{\alpha(0)},\, F^{\alpha(1)},\,\cdots,\, F^{\ak}$  & \cref{eq::lFakp}\\
\hline
\multicolumn{3}{c}{Constant Scalars}\\
\hline
$\xi$ & The  regularizer for $G(X|\Xk)$& \cref{eq::G}\\
$\zeta$ & The  regularizer for $H(X|\Xk)$& \cref{eq::lG}\\
$\eta$ & The exponential moving average parameter for $\lF^{(\sk)}$ and $\lF^{\ak}$ & \cref{eq::lFk,eq::lFakp}\\
$\psi$ and $\phi$ & The parameters for the adaptive restart\\
\hline
\end{tabular}
\end{table}

%% file: appendix_upperM.tex
For any nodes $\alpha,\,\beta\in\AA$, it should be noted that
\begin{multline}\label{eq::FabM}
	\frac{1}{2}\|X\|_{\nM_{ij}^{\ab}}^2 = \frac{1}{2}\|X-\Xk\|_{\nM_{ij}^{\ab}}^2+\\
	\innprod{\Xk\nM_{ij}^{\ab}}{X-\Xk} + \frac{1}{2}\|\Xk\|_{\nM_{ij}^{\ab}}^2
\end{multline}
always holds. Then, we will prove \cref{prop::upperM} considering cases of $\alpha=\beta$ and $\alpha\neq\beta$, respectively. 

\begin{enumerate}[leftmargin=0.45cm]
\item If $\alpha=\beta$, \cref{eq::Faa} indicates $\nabla F_{ij}^{\ab}(\Xk)=\Xk\nM_{ij}^{\ab}$. From \cref{eq::Faa,eq::FabM,eq::wabij}, it is immediate to conclude that \cref{eq::proxFij} holds for any $X$ and $\Xk\in \R^{d\times(d+1)n}$ as long as $\alpha=\beta$.
\vspace{0.25em}
\item If $\alpha\neq\beta$, as a result of \cref{assumption::loss}\ref{assumption::loss_mono},  $\rho(s)$ is a concave function, which suggests
\begin{equation}
	\nonumber
	\rho(s')\leq \rho(s)+\nabla\rho(s)\cdot(s'-s).
\end{equation}
If we let $s=\|\Xk\|_{\nM_{ij}^{\ab}}^2$ and $s'=\|X\|_{\nM_{ij}^{\ab}}^2$, the equation above can be written as
\begin{multline}\label{eq::FabM2}
	\frac{1}{2}\rho\big(\|X\|_{\nM_{ij}^{\ab}}^2\big)\leq \frac{1}{2}\rho\big(\|\Xk\|_{\nM_{ij}^{\ab}}^2\big)+\\
	\frac{1}{2}\nabla\rho\big(\|\Xk\|_{\nM_{ij}^{\ab}}^2\big)\cdot\big(\|X\|_{\nM_{ij}^{\ab}}^2-\|\Xk\|_{\nM_{ij}^{\ab}}^2\big).
\end{multline}
From \cref{eq::FabM}, it can be shown that
\begin{multline}\label{eq::FabM3}
	\frac{1}{2}\|X\|_{\nM_{ij}^{\ab}}^2-\frac{1}{2}\|\Xk\|_{\nM_{ij}^{\ab}}^2 = \frac{1}{2}\|X-\Xk\|_{\nM_{ij}^{\ab}}^2+\\
	\innprod{\Xk\nM_{ij}^{\ab}}{X-\Xk}.
\end{multline}
Then, applying \cref{eq::FabM3} on the right-hand side of \cref{eq::FabM2} results in
\begin{equation}\label{eq::FabM4}
\begin{aligned}
&\frac{1}{2}\rho\big(\|X\|_{\nM_{ij}^{\ab}}^2\big)\\
\leq\; & \frac{1}{2}\nabla\rho\big(\|\Xk\|_{\nM_{ij}^{\ab}}^2\big)\cdot\|X-\Xk\|_{\nM_{ij}^{\ab}}^2+\\
& \nabla\rho\big(\|\Xk\|_{\nM_{ij}^{\ab}}^2\big)\cdot\innprod{\Xk\nM_{ij}^{\ab}}{X-\Xk} +\\
& \frac{1}{2}\rho\big(\|\Xk\|_{\nM_{ij}^{\ab}}\big).
\end{aligned}
\end{equation}
From \cref{eq::Fab,eq::wabij}, we obtain
\begin{equation}
\nonumber
F_{ij}^{\ab}(\Xk)=\frac{1}{2}\rho\big(\|\Xk\|_{\nM_{ij}^{\ab}}^2\big),
\end{equation}
\begin{equation}
\nonumber
\nabla F_{ij}^{\ab}(\Xk) = \nabla\rho\big(\|\Xk\|_{\nM_{ij}^{\ab}}^2\big)\cdot \Xk\nM_{ij}^{\ab},
\end{equation}
\begin{equation}
\nonumber
\wabijk = \nabla\rho\big(\|\Xk\|_{\nM_{ij}^{\ab}}^2\big),
\end{equation}
with which \cref{eq::FabM4} is simplified to
\begin{multline}
\nonumber
\frac{1}{2}\wabijk\|X-\Xk\|_{\nM_{ij}^{\ab}}^2+ \innprod{\nabla F_{ij}^{\ab}(\Xk)}{X-\Xk}+\\
F_{ij}^{\ab}(\Xk)\geq F_{ij}^{\ab}(X).
\end{multline}
Thus, \cref{eq::proxFij} for $\alpha\neq\beta$ as well.
\end{enumerate}
The proof is completed.

%% file: appendix_upper.tex
From \cref{eq::MH,eq::Eab}, we obtain
\vspace{-0.3em}
\begin{multline}\label{eq::EijFij}
E_{ij}^{\ab}(X|\Xk)\geq\frac{1}{2}\wabijk\big\|X-\Xk\big\|_{\nM_{ij}^{\ab}}^2+\\
\innprod{\nabla F_{ij}^{\ab}(\Xk)}{X-\Xk}+F_{ij}^{\ab}(\Xk),
\end{multline}
in which the equality ``$=$'' holds as long as $X=\Xk$. From \cref{prop::upperM}, we obtain
\vspace{-0.3em} 
\begin{multline}\label{eq::gequpper}
	\frac{1}{2}\wabijk\|X-\Xk\|_{\nM_{ij}^{\ab}}^2+\innprod{\nabla F_{ij}^{\ab}(\Xk)}{X-\Xk}+\\
	F_{ij}^{\ab}(\Xk)\geq F_{ij}^{\ab}(X).
\end{multline} 
Then, as a result of \cref{eq::gequpper,eq::EijFij}, it is straightforward to show $E_{ij}^{\ab}(X|\Xk)\geq F_{ij}^{\ab}(X)$ for any $X\in \R^{d\times(d+1)n}$, where the equality ``$=$'' holds as long as $X=\Xk$. The proof is completed.

%% file: appendix_major.tex
\noindent\textbf{Proof of \ref{prop::G1}.\;} 
According to \cref{prop::upper}, $E_{ij}^{\ab}(X|\Xk)$ majorizes $F_{ij}^{\ab}(X)$ and $E_{ij}^{\ab}(X|\Xk)=F_{ij}^{\ab}(\Xk)$ if $X=\Xk$. Then, as a result \cref{eq::F,eq::G}, it can be concluded that $G(X|\Xk)$ majorizes $F(X)$ and $G(X|\Xk)=F(X)$ if $X=\Xk$. The proof is completed.
\vspace{1em}

\noindent\textbf{Proof of \ref{prop::G2}.\;}
From \cref{eq::Hab}, we obtain
\vspace{-0.3em}
\begin{multline}\label{eq::HXab}
\frac{1}{2}\|X-\Xk\|_{\nH_{ij}^{\ab}}^2=\kappa_{ij}^{\alpha\beta}\|R_i^\alpha-R_i^{\ak}\|^2+\\
\tau_{ij}^{\alpha\beta}\|(R_i^\alpha - R_i^{\ak})\nt_{ij}^{\alpha\beta}+t_i^\alpha-t_i^{\ak}\|^2+\\
\kappa_{ij}^{\alpha\beta}\|R_j^\beta-R_j^{\beta(\sk)}\|^2+\tau_{ij}^{\alpha\beta}\|t_j^\beta-t_j^{\beta(\sk)}\|^2.
\end{multline}
From \cref{eq::Fij2}, it is by definition that $F_{ij}^{\ab}(X)$  is a function related with $X^\alpha\in \XX^\alpha$ and $X^\beta\in\XX^\beta$ only, and thus, $\nabla F_{ij}^{\ab}(X)$ is sparse, which suggests
\vspace{-0.3em}
\begin{multline}\label{eq::DFabXX}
\innprod{\nabla F_{ij}^{\ab}(\Xk)}{X-\Xk}=\\
\innprod{\nabla_{\Xa} F_{ij}^{\ab}(\Xk)}{\Xa-\Xak}+\\
\innprod{\nabla_{X^{\beta}} F_{ij}^{\ab}(\Xk)}{X^{\beta}-X^{\beta(\sk)}}.
\end{multline}
In \cref{eq::DFabXX}, $\nabla_{\Xa} F_{ij}^{\ab}(\Xk)$ is the Euclidean gradient of $F_{ij}^{\ab}(X)$ with respect to $\Xa\in\XXa$ at $\Xk\in\XX$.  Substituting \cref{eq::HXab,eq::DFabXX} into \cref{eq::Eab}, we obtain
\vspace{-0.3em}
\begin{multline}\label{eq::Eij3}
E_{ij}^{\ab}(X|\Xk)=\wabijk\cdot\left(\kappa_{ij}^{\alpha\beta}\|R_i^\alpha-R_i^{\ak}\|^2+\right.\\
\tau_{ij}^{\alpha\beta}\|(R_i^\alpha - R_i^{\ak})\nt_{ij}^{\alpha\beta}+t_i^\alpha-t_i^{\ak}\|^2+\\
\left.\kappa_{ij}^{\alpha\beta}\|R_j^\beta-R_j^{\beta(\sk)}\|^2+\tau_{ij}^{\alpha\beta}\|t_j^\beta-t_j^{\beta(\sk)}\|^2\right)+\\
\innprod{\nabla_{\Xa} F_{ij}^{\ab}(\Xk)}{\Xa-\Xak}+\\
\innprod{\nabla_{X^{\beta}} F_{ij}^{\ab}(\Xk)}{X^{\beta}-X^{\beta(\sk)}}+F_{ij}^{\ab}(\Xk).
\end{multline}
In a similar way, $F_{ij}^{\aa}(X)$ in \cref{eq::Fij2} can be rewritten as
\vspace{-0.3em}
\begin{multline}\label{eq::Fij3}
F_{ij}^{\aa}(X)=\frac{1}{2}\kappa_{ij}^{\aa}\|(R_i^\alpha-R_i^{\ak})\nR_{ij}^{\aa} -(R_j^\alpha-R_j^{\alpha(\sk)})\|^2 +\\ 
\frac{1}{2}\tau_{ij}^{\aa}\|(R_i^\alpha-R_i^{\ak}) \nt_{ij}^{\aa}+t_i^\alpha - t_i^{\ak} - (t_j^\alpha-t_j^{\alpha(\sk)})\|^2+\\
\innprod{\nabla_{\Xa} F_{ij}^{\alpha\alpha}(\Xk)}{\Xa-\Xak} + F_{ij}^{\alpha\alpha}(\Xk).
\end{multline}
Substituting \cref{eq::Eij3,eq::Fij3} into \cref{eq::G} and simplifying the resulting equation, we obtain
\begin{equation}\label{eq::Gsum}
G(X|X^{(\sk)})=\sum_{\alpha\in\AA} G^{\alpha}(X^\alpha|\Xk) + F(\Xk)
\end{equation}
where $G^\alpha(\Xa|\Xk)$ is a function that is related with $\Xa\in \XX^\alpha$ only. The proof is completed.

\vspace{1em}

\noindent\textbf{Proof of \ref{prop::G3}.\;}
A tedious but straightforward mathematical manipulation from \cref{eq::Fij3,eq::Eij3,eq::DFabXX,eq::Gsum} indicates that there exists positive-semidefinite matrices $\nGamma^{\ak}\in \R^{(d+1)n_\alpha\times (d+1)n_\alpha}$ such that $G^\alpha(X^\alpha|\Xk)$ in the equation above can be written as
\begin{multline}
	\nonumber
	G^\alpha(X^\alpha|\Xk)=
	\frac{1}{2}\|\Xa-\Xak\|_{\nGamma^{\ak}}^2+\\
	\big\langle\nabla_{\Xa} F(X^{(\sk)}),\,{\Xa-\Xak}\big\rangle,
\end{multline}
in which the formulation of $\nGamma^{\ak}$ is given in Appendix \hyperref[appendix::I]{I}. The proof is completed.

%% file: appendix_mm.tex
\noindent\textbf{Proof of \ref{prop::mm1}.\;} From \cref{assumption::mm} and line~\ref{line::xlG} of \cref{algorithm::mm}, we obtain
\begin{equation}\label{eq::Gacmp}
G^\alpha(\Xakp|\Xk) \leq G^\alpha(\Xakh|\Xk)
\end{equation}
and
\begin{equation}\label{eq::Hacmp}
H^\alpha(\Xakh|\Xk) \leq H^\alpha(\Xak|\Xk).
\end{equation}
From \cref{eq::lG2,eq::G2,eq::Gacmp,eq::Hacmp}, it can be concluded that
\begin{equation}\label{eq::Gcmp}
G(\Xkp|\Xk) \leq G(X^{(\sk+\shalf)}|\Xk)
\end{equation}
and
\begin{equation}\label{eq::Hcmp}
H(\Xkh|\Xk)\leq H(\Xk|\Xk).
\end{equation}
Note that \cref{eq::lGF} suggests
\begin{equation}\label{eq::lGGF1}
F(X^{(\sk+1)}) \leq G(X^{(\sk+1)}|X^{(\sk)})
\end{equation}
and
\begin{equation}\label{eq::lGGF2}
G(X^{(\sk+\shalf)}|X^{(\sk)}) \leq H(X^{(\sk+\shalf)}|X^{(\sk)}).
\end{equation}
Then, \cref{eq::Gcmp,eq::Hcmp,eq::lGGF1,eq::lGGF2} result in
\begin{multline}\label{eq::order}
F(X^{(\sk+1)}) \leq G(X^{(\sk+1)}|X^{(\sk)}) \leq\\
G(X^{(\sk+\shalf)}|X^{(\sk)}) \leq H(X^{(\sk+\shalf)}|X^{(\sk)})\leq \\
H(X^{(\sk)}|X^{(\sk)}) =  F(X^{(\sk)}),
\end{multline}
which indicates that $F(X^{(\sk)})$ is nonincreasing. The proof is completed.

\vspace{0.8em}
\noindent\textbf{Proof of \ref{prop::mm2}.\;} From \cref{prop::mm}\ref{prop::mm1}, it has been proved that $F(X^{(\sk)})$ is nonincreasing. From \cref{eq::obj} and \cref{assumption::loss}, $F(X^{(\sk)})\geq 0$, i.e., $F(X^{(\sk)})$ is bounded below. As a result, there exists $F^{\infty} \in \R $ such that $F(X^{(\sk)})\rightarrow F^{\infty}$. The proof is completed.

\vspace{0.8em}
\noindent\textbf{Proof of \ref{prop::mm3}.\;}
We introduce the following lemmas about $G(X|\Xk)$ and $H(X|\Xk)$ that are used in the proof.
 
 \begin{lemma}\label{lemma::G}
Let  $\nGammak \in \R^{(d+1)n\times (d+1)n}$ be a block diagonal matrix in the form of
 \begin{equation}\label{eq::nG}
 	\nGammak\triangleq\mathrm{diag}\big\{\nGamma^{1(\sk)},\,\cdots,\,\nGamma^{|\AA|(\sk)}\big\}\in \R^{(d+1)n\times (d+1)n}
 \end{equation}
where $\nGammaak$ is given in \cref{eq::GMa}. Then, we have the following results:
\begin{enumerate}[(a)]
\item\label{lemma::Ga} $G(\cdot|\Xk): \R^{d\times(d+1)n}\rightarrow\R$ in \cref{eq::G} is equivalent to
\begin{multline}\label{eq::GM}
	G(X|X^{(\sk)})
	=\frac{1}{2}\|X-X^{(\sk)}\|_{\nGammak}^2+\\
	\innprod{\nabla F(X^{(\sk)})}{{X-X^{(\sk)}}}+ F(X^{(\sk)}),
\end{multline}
where  $\nabla F(X^{(\sk)})\in \R^{d\times(d+1)n}$ is the Euclidean gradient of $F(X)$ at $X^{(\sk)}\in\XX$.
\item\label{lemma::Gb} $\nGammak\succeq\nMk$ where $\nMk$ is given in \cref{eq::M}. 
\item\label{lemma::Gc}$\nGammak$ is bounded, i.e.,  there exists a constant positive-semidefinite matrix $\nGamma\in \R^{(d+1)n\times (d+1)n}$ such that $\nGamma\succeq \nGammak$ holds for any $\sk\geq 0$. 
\end{enumerate}
 \end{lemma}

\begin{proof} The proofs of Lemmas \ref{lemma::G}\ref{lemma::Ga} to \ref{lemma::G}\ref{lemma::Gc} are as the following.

\ref{lemma::Ga} If we substitute \cref{eq::GMa} into \cref{eq::G2}, the result is
\begin{multline}\label{eq::EXk}
	G(X|\Xk)=\sum_{\alpha\in\AA}\Big[\frac{1}{2}\|\Xa-\Xak\|_{\nGamma^{\ak}}^2+\\
	\big\langle\nabla_{\Xa} F(X^{(\sk)}),\,{\Xa-\Xak}\big\rangle\Big] + F(\Xk).
\end{multline}
Also, it is straightforward to show that
\begin{equation}
	\nonumber
	\frac{1}{2}\|X-\Xk\|_{\nGammak}^2=\sum_{\alpha\in\AA}\frac{1}{2}\|\Xa-\Xak\|_{\nGamma^{\ak}}^2,
\end{equation}
where $\nGammak\in\R^{(d+1)n\times(d+1)n}$ is defined as \cref{eq::nG}, and  
\begin{multline}
	\nonumber
	\innprod{\nabla F(\Xk)}{X-\Xk}=\\
	\sum_{\alpha\in\AA}\innprod{\nabla_{\Xa} F(\Xk)}{X^\alpha-\Xak}.
\end{multline}
Thus, \cref{eq::EXk} is equivalent to \cref{eq::GM}, i.e.,
\begin{multline}\label{eq::E3}
	G(X|\Xk) = \frac{1}{2}\|X-\Xk\|_{\nGammak}^2+\\
	\innprod{\nabla F(X^{(\sk)})}{{X-X^{(\sk)}}}+ F(X^{(\sk)}).
\end{multline}
The proof is completed
\vspace{0.2em}

\ref{lemma::Gb} From \cref{eq::G,eq::FaaFab,eq::Eab,eq::E3}, we rewrite $\nGammak\in\R^{(d+1)n\times(d+1)n}$ as
\begin{multline}\label{eq::H}
	\nGammak=\sum_{\alpha\in\AA}\sum_{(i,j)\in \aEE^{\alpha\alpha}}\nM_{ij}^{\aa}+\\
	\sum_{\substack{\alpha,\beta\in\AA,\\\alpha\neq \beta}}\sum_{(i,j)\in \aEE^{\alpha\beta}}\wabijk\cdot \nH_{ij}^{\ab}+\xi\cdot\I,
\end{multline}
in which $\nH_{ij}^{\ab}\succeq\nM_{ij}^{\ab}$ by \cref{eq::MH} and $\xi\geq 0$. Then, as a result of \cref{eq::M,eq::H,eq::MH}, it is straightforward to conclude that
\begin{equation}\label{eq::GMk}
	\nGammak\succeq\nMk+\xi \cdot\I\succeq \nMk.
\end{equation} 
The proof is completed.
\vspace{0.2em}

\ref{lemma::Gc} Let $\nGamma\in\R^{(d+1)n\times(d+1)n}$ be defined as
\begin{equation}\label{eq::HH}
	\nGamma\triangleq\sum_{\alpha\in\AA}\sum_{(i,j)\in \aEE^{\alpha\alpha}}\nM_{ij}^{\aa}+
	\sum_{\substack{\alpha,\beta\in\AA,\\\alpha\neq \beta}}\sum_{(i,j)\in \aEE^{\alpha\beta}} \nH_{ij}^{\ab}+\xi\cdot\I.
\end{equation}
From \cref{assumption::loss}\ref{assumption::loss_drho} and \cref{eq::wabij}, it can be concluded that
\begin{equation}\label{eq::wabijkbnd}
	0\leq \wabijk \leq 1
\end{equation} 
for any $\Xk\in\R^{d\times(d+1)n}$. Furthermore, it is known that $\nH_{ij}^{\ab} \succeq 0$, then \cref{eq::HH,eq::H,eq::wabijkbnd} result in $\nGamma\succeq\nGammak$ for any $\Xk\in\R^{d\times(d+1)n}$. The proof is completed.
\end{proof}

 \begin{lemma}\label{lemma::lG}
	Let  $\lnGammak \in \R^{(d+1)n\times (d+1)n}$ be a block diagonal matrix in the form of
	\begin{equation}\label{eq::lnG}
		\lnGammak\triangleq\mathrm{diag}\big\{\lnGamma^{1(\sk)},\,\cdots,\,\lnGamma^{|\AA|(\sk)}\big\}\in \R^{(d+1)n\times (d+1)n}
	\end{equation}
	where $\lnGammaak$ is given in \cref{eq::lGMa}. Then, we have the following results:
	\begin{enumerate}[(a)]
		\item\label{lemma::lGa} $\lG(\cdot|\Xk): \R^{d\times(d+1)n}\rightarrow\R$ in \cref{eq::lG}  is equivalent to
		\begin{multline}\label{eq::lGM}
			\lG(X|X^{(\sk)})
			=\frac{1}{2}\|X-X^{(\sk)}\|_{\lnGammak}^2+\\
			\innprod{\nabla F(X^{(\sk)})}{{X-X^{(\sk)}}}+ F(X^{(\sk)}),
		\end{multline}
		where  $\nabla F(X^{(\sk)})\in \R^{d\times(d+1)n}$ is the Euclidean gradient of $F(X)$ at $X^{(\sk)}\in\XX$.
		\item\label{lemma::lGb} $\lnGammak\succeq\nGammak\succeq\nMk$ where $\nMk$ and $\nGammak$ are given in \cref{eq::M} and \cref{eq::nG}, respectively. 
		\item\label{lemma::lGc}$\lnGammak$ is bounded, i.e.,  there exists a constant positive-semidefinite matrix $\lnGamma\in \R^{(d+1)n\times (d+1)n}$ such that $\lnGamma\succeq \lnGammak$ holds for any $\sk\geq 0$.
		\item\label{lemma::lGd} $\lG^\alpha(\Xa|\Xk)\geq G^\alpha(\Xa|\Xk)$ for any $\Xa\in\R^{d\times (d+1)n_\alpha}$ where $G^\alpha(\Xa|\Xk)$ is given in \cref{eq::GMa} and the equality holds if  $\Xa=\Xak$.
	\end{enumerate}
\end{lemma}
 \begin{proof}
 The proofs of Lemmas \ref{lemma::lG}\ref{lemma::lGa} to \ref{lemma::lG}\ref{lemma::lGc}  are similar to those of \cref{lemma::G} and the proof of  \cref{lemma::lG}\ref{lemma::lGd} is  immediate from \cref{eq::GMa,eq::lGMa} and \cref{lemma::lG}\ref{lemma::lGb}.
 \end{proof}
 
From \cref{eq::order}, it is known that $F(\Xk)\geq G(\Xkp|\Xk)$, which suggests
\begin{equation}\label{eq::bound}
F(\Xk)- F(\Xkp) \geq
G(\Xkp|\Xk) - F(\Xkp).
\end{equation}
From \cref{eq::proxF,eq::GM}, we obtain
\begin{multline}\label{eq::Fxkp}
F(\Xkp) \leq \frac{1}{2}\|\Xkp-\Xk\|_{\nMk}^2+\\
\innprod{\nabla F(\Xk)}{X-\Xk}+F(\Xk)
\end{multline}
and
\begin{multline}\label{eq::Gxkp}
G(\Xkp|\Xk)=\frac{1}{2}\|\Xkp-\Xk\|_{\nGammak}^2+\\
\innprod{\nabla F(\Xk)}{X-\Xk}+F(\Xk),
\end{multline}
respectively. Substituting \cref{eq::Fxkp,eq::Gxkp} into the right-hand side of \cref{eq::bound}, we obtain
\begin{multline}\label{eq::GF2}
F(X^{(\sk)}) - F(X^{(\sk+1)}) \geq\\
\frac{1}{2}\|\Xkp-\Xk\|_{\nGammak}^2-\frac{1}{2}\|\Xkp-\Xk\|_{\nMk}^2.
\end{multline}
From \cref{eq::GF2,eq::GMk}, there exists a constant scalar $\delta>0$ such that
\begin{equation}\label{eq::dF}
F(X^{(\sk)}) - F(X^{(\sk+1)}) \geq
\frac{\delta}{2}\|X^{(\sk+1)}-X^{(\sk)}\|^2
\end{equation}
as long as $\xi>0$. From \cref{prop::mm}\ref{prop::mm2}, we obtain
\begin{equation}\label{eq::dF0}
F(X^{(\sk)}) - F(X^{(\sk+1)})\rightarrow 0,
\end{equation}
and thus, it can be concluded from \cref{eq::dF,eq::dF0} that
\begin{equation}
\|X^{(\sk+1)}-X^{(\sk)}\| \rightarrow 0.
\end{equation}
The proof is completed.

\vspace{0.8em}
\noindent\textbf{Proof of \ref{prop::mm4}.\;} From \cref{eq::order}, it is known that $F(\Xk)\geq H(\Xkh|\Xk)$ and $G(\Xkh|\Xk)\geq G(\Xkp|\Xk) \geq F(\Xkp)$, which suggests
\begin{multline}
\nonumber
F(\Xk)-F(\Xkp) \geq\\ 
H(\Xkh|\Xk) - G(\Xkh|\Xk).
\end{multline}
From \cref{eq::GM,eq::lGM}, the equation above is equivalent to
\begin{multline}\label{eq::FkFkp}
F(\Xk)-F(\Xkp) \geq \\
\frac{1}{2}\|\Xkh-\Xk\|_{\lnGammak}^2 - \frac{1}{2}\|\Xkh-\Xk\|_{\nGammak}^2.
\end{multline}
A similar procedure to the derivation of \cref{eq::GMk} results in
\begin{equation}\label{eq::lGGk}
\lnGammak \succeq \nGammak + (\zeta-\xi)\cdot\I,
\end{equation}
which suggests there exists a constant scalar $\delta' > 0$ such that
\begin{equation}
\nonumber
\lnGammak \succeq \nGammak + \delta'\cdot\I
\end{equation}
if $\zeta>\xi > 0$. Then, similar to the proof of \cref{prop::mm}\ref{prop::mm3}, we obtain
\begin{equation}\label{eq::DFh}
F(\Xk)-F(\Xkp) \geq \frac{\delta'}{2}\|\Xkh-\Xk\|^2.
\end{equation}
Thus, it can be concluded that
\begin{equation}
	\|\Xkh-\Xk\| \rightarrow 0.
\end{equation}
The proof is completed.

\vspace{0.8em}
\noindent\textbf{Proof of \ref{prop::mm5}.\;} The following lemma about $\nabla F(X)$ is needed in this proof.

\begin{lemma}\label{prop::DF}
If \cref{assumption::loss}\ref{assumption::loss_L} holds, then the Euclidean gradient $\nabla F(\cdot):\R^{d\times(d+1)n}\rightarrow \R^{d\times(d+1)n}$ of $F(X)$ in \cref{eq::F} is Lipschitz continuous, i.e., there exists a constant $\mu > 0$ such that $\|\nabla F(X)-\nabla F(X')\|\leq \mu\cdot\|X-X'\|$.
\end{lemma}

\begin{proof}
From \cref{assumption::loss}\ref{assumption::loss_L}, it is known that $\rho(\|X\|^2)$ has Lipschitz continuous gradient, which suggests that $F_{ij}^{\ab}(X)=\dfrac{1}{2}\rho(\|X\|_{\nM_{ij}^{\ab}}^2)$ in \cref{eq::Fab} has Lipschitz continuous gradient. Note that $F_{ij}^{\aa}(X)=\dfrac{1}{2}\|X\|_{\nM_{ij}^{\aa}}^2$ in \cref{eq::Faa} has Lipschitz continuous gradient as well. Then, from \cref{eq::F}, it can be concluded that $F(X)$ has Lipschitz continuous gradient. The proof is completed.
\end{proof}

It is straightforward to show that the Riemannian gradient $\grad\, F(X)$ takes the form as
\begin{equation}
	\nonumber
	\grad\, F(X) =\begin{bmatrix}
		\grad_{1} F(X) & \cdots  & \grad_{|\AA|} F(X)
	\end{bmatrix}\in T_X \XX.
\end{equation}
In the equation above, $\grad_{\alpha} F(X)$ is the Riemannian gradient of $F(X)$ with respect to $\Xa\in\XX^{\alpha}$ for node $\alpha\in\AA$, and can be written as
\begin{equation}\label{eq::grad}
	\grad_\alpha F(X)=
	\begin{bmatrix}
		\grad_{t^\alpha} F(X) & \grad_{R^\alpha} F(X)
	\end{bmatrix}\in
	T_{X^\alpha}\XX^\alpha
\end{equation}
in which recall that $$T_{X^\alpha}\XX^\alpha\triangleq\R^{d\times n_\alpha}\times T_{R^\alpha} SO(d)^{n_\alpha}.$$
From \cite{absil2009optimization,rosen2016se}, it can be shown that $\grad_{t^\alpha} F(X)$ and $\grad_{R^\alpha} F(X)$ in \cref{eq::grad} are
\begin{equation}\label{eq::grad_t}
	\grad_{t^\alpha} F(X) = \nabla_{t^\alpha} F(X)
\end{equation}
and
\begin{multline}\label{eq::grad_R}
	\grad_{R^\alpha} F(X) = \nabla_{R^\alpha} F(X)-\\
	R^\alpha\, \mathrm{SymBlockDiag}_d^\alpha({R^{\alpha}}^\transpose\nabla_{R^\alpha} F(X)).
\end{multline}
In \cref{eq::grad_R}, $\mathrm{SymBlockDiag}_d^\alpha: \R^{dn_\alpha\times dn_\alpha}\rightarrow \R^{dn_\alpha\times dn_\alpha}$ is a linear operator
\begin{equation}\label{eq::sym}
	\mathrm{SymBlockDiag}_d^\alpha(Z)\triangleq\frac{1}{2}\mathrm{BlockDiag}_d^\alpha(Z+Z^\transpose),
\end{equation} 
in which $\mathrm{BlockDiag}_d^\alpha: \R^{dn_\alpha \times dn_\alpha}\rightarrow \R^{dn_\alpha \times dn_\alpha}$ extracts the $d\times d$-block diagonals of a matrix, i.e.,
\begin{equation}
	\nonumber
	\mathrm{BlockDiag}_d^\alpha(Z)\triangleq\begin{bmatrix}
		Z_{11} & & \\
		&\ddots &\\
		& & Z_{n_\alpha n_\alpha}
	\end{bmatrix}\in \R^{dn_\alpha \times dn_\alpha}.
\end{equation}
As a result of \cref{eq::grad_t,eq::grad_R,eq::sym,eq::grad}, there exists a linear operator
\begin{equation}\label{eq::QQx}
	\QQ_X:\R^{d\times d(n+1)}\rightarrow \R^{d\times d(n+1)}
\end{equation} 
that continuously depends on $X\in \XX$ such that
\begin{equation}\label{eq::gradF}
	\grad\, F(X)=\QQ_X(\nabla F(X)).
\end{equation}
From \cref{eq::lGM}, it is straightforward to show that
\begin{equation}\label{eq::gradG}
	\begin{aligned}
	&\nabla H(\Xkh|\Xk)\\
	=&\nabla F(\Xk)+(\Xkh-\Xk)\lnGammak\\
	=&\nabla F(\Xkh)+(\Xkh-\Xk)\lnGammak\, +\\
	 &\big(\nabla F(\Xk) -\nabla F(\Xkh)\big).
	\end{aligned}
\end{equation}
Note that \cref{eq::gradF} applies to any functions on $\XX$. As a result of \cref{eq::gradF,eq::gradG}, we obtain 
\begin{equation}\label{eq::gradGX}
\begin{aligned}
&\grad\,H(\Xkh|\Xk)\\
=\,&\grad\, F(\Xkh) +\\
& \QQ_{\Xkh}\big((\Xkh-\Xk)\lnGammak\big)+\\
 &\QQ_{\Xkh}\big(\nabla F(\Xk)-\nabla F(\Xkh)\big).
\end{aligned}
\end{equation}
From line~\ref{line::xlG} of \cref{algorithm::mm}, we obtain. 
$$\grad\, H^\alpha(\Xakh|\Xk)=\0.$$
In addition, it is by definition that
\begin{multline}
	\nonumber
	\grad\, H(X|X^{(\sk)})=\\
	\begin{bmatrix}
		\grad\, H^1(X^1|\Xk) & \cdots & \grad\, H^{|\AA|}(X^{|\AA|}|\Xk),
	\end{bmatrix}
\end{multline}
which suggests
\begin{equation}\label{eq::gradG0}
	\grad\,H(\Xkh|\Xk)=\0.
\end{equation}
From \cref{eq::gradG0,eq::gradGX}, we obtain
\begin{multline}
	\nonumber
	\grad\, F(\Xkh)=\QQ_{\Xkh}\big((\Xk-\Xkh)\lnGammak\big)+\\
	\QQ_{\Xkh}\big(\nabla F(\Xkh)-\nabla F(\Xk)\big).
\end{multline}
From the equation above, it can be shown that
\begin{equation}\label{eq::gradFxk_n}
	\begin{aligned}
		&\|\grad\, F(\Xkh)\|\\
		=&\big\|\QQ_{\Xkh}\big((\Xk-\Xkh)\lnGammak\big)+\\
		&\hphantom{\big\|}\QQ_{\Xkh}\big(\nabla F(\Xkh)-\nabla F(\Xk)\big)\big\| \\
		\leq&\big\|\QQ_{\Xkh}\big((\Xk-\Xkh)\lnGammak\big)\big\|+\\
		&\big\|\QQ_{\Xkh}\big(\nabla F(\Xkh)-\nabla F(\Xk)\big)\big\| \\
		\leq& \|\QQ_{\Xkh}\|_2\cdot \|\lnGammak\|_2\cdot\|\Xkh-\Xk\|+\\
		&\|\QQ_{\Xkh}\|_2\cdot\|\nabla F(\Xkh)-\nabla F(\Xk)\|,
	\end{aligned}
\end{equation}
in which $\|\cdot\|_2$ denotes the induced $2$-norm of linear operators. From Lemmas \ref{lemma::lG}\ref{lemma::lGc} and \ref{prop::DF}, there exists a constant positive-semidefinite matrix $\lnGamma\in \R^{(d+1)n\times(d+1)n}$ and constant positive scalar $\mu > 0$ such that $\lnGamma \succeq \lnGammak \succeq 0$ and $\|\nabla F(\Xkh)-\nabla F(\Xk)\|\leq \mu\cdot\|\Xkh-\Xk\|$ for any $\sk\geq 0$, making it possible to upper-bound the right-hand side of \cref{eq::gradFxk_n}: 
\begin{equation}\label{eq::gradFk}
\begin{aligned}
	&\|\grad\, F(\Xkh)\|\\
\leq&\|\QQ_{\Xkh}\|_2\cdot \|\lnGamma\|_2\cdot\|\Xkh-\Xk\|+\\
	&\|\QQ_{\Xkh}\|_2\cdot\mu\cdot\|\Xkh-\Xk\|.
\end{aligned}
\end{equation}
Moreover, \cref{eq::grad_R,eq::grad_t,eq::grad} indicate that  $\QQ_X(\cdot)$ only depends on the rotation $R^\alpha\in SO(d)^{n_\alpha}$ for $\alpha\in\AA$. Since $\QQ_{X}(\cdot)$ is continuous and $SO(d)^{n_\alpha}$ is a compact manifold, $\|\QQ_{\Xkh}\|_2$ is bounded for any $\Xkh\in\XX$.
Thus, there exists a constant scalar $\nu > 0$ such that the right-hand side of \cref{eq::gradFk} can be upper-bounded as
\begin{equation}\label{eq::gradFh}
\|\grad\, F(\Xkh)\|\leq \nu \|\Xkh-\Xk\|.
\end{equation} 
As long as $\zeta>\xi>0$, \cref{eq::DFh,eq::gradFh} result in
\begin{equation}
\nonumber
\|\grad\, F(\Xkh)\|^2 \leq  \frac{2\nu^2}{\delta'}\big(F(\Xk)-F(\Xkp)\big).
\end{equation}
Then, there exists a constant scalar $\epsilon =\frac{\delta'}{\nu^2}>0$ with which the equation above can be rewritten as
\begin{equation}\label{eq::diff_F}
F(\Xk)-F(\Xkp) \geq \frac{\epsilon}{2}\|\grad\, F(\Xkh)\|^2.
\end{equation}
As a result of \cref{eq::diff_F}, we obtain
\begin{multline}\label{eq::diff_F2}
F(X^{(0)}) - F(X^{(\sK+1)}) \geq \frac{\epsilon}{2} \sum_{\sk=0}^{\sK}\|\grad\, F(\Xkh)\|^2\\
\geq \frac{\epsilon(\sK+1)}{2}\min_{0\leq\sk\leq\sK}\|\grad\, F(\Xkh)\|^2.
\end{multline}
From Propositions \ref{prop::mm}\ref{prop::mm1} and \ref{prop::mm}\ref{prop::mm2}, it can be concluded that  $F(\Xkp) \geq F^\infty$ for any $\sk\geq0$, which and \cref{eq::diff_F2} suggest
\begin{equation}
	\nonumber
	\min\limits_{0\leq\sk< \mathsf{K}}\|\grad\, F(\Xkh)\|\leq \sqrt{\frac{2}{\epsilon}\cdot\dfrac{F(X^{(0)})-F^\infty}{{\sK+1}}}.
\end{equation}
The proof is completed.

\vspace{0.8em}
\noindent\textbf{Proof of \ref{prop::mm6}.\;} As a result of Propositions \ref{prop::mm}\ref{prop::mm3} and \ref{prop::mm}\ref{prop::mm4}, it is known that 
\begin{equation}
\nonumber
\|\Xkp-\Xk\|\rightarrow 0
\end{equation}
and
\begin{equation}
\nonumber
\|\Xkh-\Xk\|\rightarrow 0
\end{equation}
as long as $\zeta>\xi>0$. Thus, it can be concluded from \cref{eq::gradFh} that
$$\grad\,F(\Xkh)\rightarrow \0$$
if $\zeta>\xi>0$. In addition, \cref{assumption::loss}\ref{assumption::loss_cont} indicates that $\grad\,F(\Xk)$ is continuous, which suggests
$$\grad\, F(\Xk)\rightarrow \grad\, F(\Xkh).$$
Then, we obtain
$$\grad\, F(\Xk) \rightarrow \0. $$
The proof is completed.

%% file: appendix_ammc.tex
\noindent\textbf{Proof of \ref{prop::ammc1}.\;} In this proof, we will prove $F(\Xk)\leq \sF^{(\sk)}$ and $\sF^{(\skp)}\leq \sF^{(\sk)}$ by induction.
\begin{enumerate}[leftmargin=0.45cm]
\item From lines~\ref{line::alg3::s0}, \ref{line::alg3::lF0}, \ref{line::alg3::lFk} of \cref{algorithm::ammc}, it can be shown that
\begin{equation}\label{eq::sF0011}
	F(X^{(-1)})=F(X^{(0)})=\lF^{(-1)}=\lF^{(0)}.
\end{equation}
\item Suppose $\sk\geq 0$ and $F(\Xk)\leq \lF^{(\sk)}$ holds at $\sk$-th iteration. In terms of $\Xkh$, if the adaptive restart scheme for $\Xkh$  is not triggered, it is immediate to show from line~\ref{line::alg4::restart_s1} of \cref{algorithm::amm_c_x} that
\begin{equation}\label{eq::Fah1}
F(\Xkh) \leq \lF^{(\sk)}.
\end{equation}
On the other hand, if the adaptive restart scheme for $\Xkh$ is triggered, line~\ref{line::alg4::xlG2} of \cref{algorithm::amm_c_x} results in
\begin{equation}\label{eq::Hah}
H^\alpha(\Xakh|\Xk) \leq H^\alpha(\Xak|\Xk)= 0,
\end{equation}
where $H^{\alpha}(\Xak|\Xk) = 0$ is from \cref{eq::lGMa}. Then, \cref{eq::lG2,eq::Hah,eq::lGF} indicate
\begin{equation}\label{eq::Fah2}
F(\Xkh)\leq H(\Xkh|\Xk) \leq F(\Xk) \leq \sF^{(\sk)}.
\end{equation}
Therefore, no matter whether the adaptive restart scheme is triggered or not, we conclude from \cref{eq::Fah1,eq::Fah2} that
\begin{equation}\label{eq::Fah}
F(\Xkh) \leq \lF^{(\sk)}
\end{equation}
always holds.

Furthermore, as a result of lines~\ref{line::alg4::restart_s3} to \ref{line::alg4::restart_e3} of \cref{algorithm::amm_c_x}, we obtain
\begin{equation}\label{eq::FkpF}
F(\Xkp)-\sF^{(\sk)} \leq \phi\cdot\Big(F(\Xkh)-\lF^{(\sk)}\Big) \leq 0.
\end{equation}
From line~\ref{line::alg3::lFk} of \cref{algorithm::ammc} and $F(\Xkp)-\sF^{(\sk)} \leq 0$ in \cref{eq::FkpF}, we obtain
\begin{equation}
\nonumber
F(\Xkp)- \sF^{(\skp)} = (1-\eta)\cdot \big(F(\Xkp) - \sF^{(\sk)} \big)\leq 0
\end{equation}
and
\begin{equation}
\nonumber
\sF^{(\skp)}-\sF^{(\sk)} = \eta\cdot \big(F(\Xkp) - \sF^{(\sk)} \big)\leq 0,
\end{equation}
which suggest $F(\Xkp)\leq \sF^{(\skp)}\leq \sF^{(\sk)}$. 
\item Therefore, it can be concluded that $F(\Xk)\leq \sF^{(\sk)}$ and $\sF^{(\skp)}\leq \sF^{(\sk)}$, which suggests that $\sF^{(\sk)}$ is nonincreasing. The proof is completed.
\end{enumerate}

\vspace{0.8em}
\noindent\textbf{Proof of \ref{prop::ammc2}.\;} From line~\ref{line::alg3::lFk} of \cref{algorithm::ammc}, we obtain
\begin{equation}\label{eq::lFkp}
	\lF^{(\skp)} = (1-\eta)\cdot \lF^{(\sk)} + \eta\cdot F(\Xkp),
\end{equation}
which and \cref{eq::sF0011} suggest that $\lF^{(\sk)}$ is a convex combination of $F(X^{(0)}),\,F(X^{(1)}),\,\cdots F(\Xk)$ as long as $\eta\in(0,\,1]$. Since $F(\Xk)\geq 0$, we obtain $\lF^{(\sk)} \geq 0$ as well, i.e., $\lF^{(\sk)}$ is bounded below. \cref{prop::ammc}\ref{prop::ammc1} indicates that $\lF^{(\sk)}$ is nonincreasing, and thus, there exists $F^{\infty}$ such that $\lF^{(\sk)}\rightarrow F^{\infty}$. Then, it can be still concluded from \cref{eq::lFkp} that $F(\Xk)\rightarrow F^{\infty}$.

\vspace{0.8em}
\noindent\textbf{Proof of \ref{prop::ammc3}.\;}
If $\sk=-1$, line~\ref{line::alg3::s0} of \cref{algorithm::ammc} and \cref{eq::lFk} suggest $X^{(-1)}=X^{(0)}$  and $\lF^{(-1)}=\lF^{(0)}=F(X^{(0)})$, respectively, from which we conclude
\begin{equation}\label{eq::lF01}
	\lF^{(-1)}-\lF^{(0)}=\|X^{(-1)}-X^{(0)}\|^2.
\end{equation}

If $\sk\geq 0$, there exists three possible cases for $\Xkp$:
\begin{enumerate}[leftmargin=0.45cm]
\item\label{cases::prop6::1} If $\Xkp$ is from line~\ref{line::alg4::xG1} of \cref{algorithm::amm_c_x}, then the adaptive restart scheme is not triggered and line~\ref{line::alg4::restart_s1} of \cref{algorithm::amm_c_x} results in
\begin{equation}\label{eq::FlF1}
	 \sF^{(\sk)}-F(\Xkp) \geq \psi\cdot\big\|\Xkp-\Xk\big\|^2.
\end{equation}
\item\label{cases::prop6::2} If $\Xkp$ is from line~\ref{line::alg4::xG2} of \cref{algorithm::amm_c_x}, then the adaptive restart scheme is triggered and \cref{eq::FkFkp} holds as well, from which and \cref{eq::GMk} we obtain
\begin{equation}
\nonumber
F(\Xk)-F(\Xkp)\geq \frac{\xi}{2}\big\|\Xkp-\Xk\big\|^2.
\end{equation}
In the proof of \cref{prop::ammc}\ref{prop::ammc1}, it is known that $\sF^{(\sk)}\geq F(\Xk)$, then the equation above results in
\begin{equation}\label{eq::FlF2}
\sF^{(\sk)}-F(\Xkp)\geq \frac{\xi}{2}\big\|\Xkp-\Xk\big\|^2.
\end{equation}
\item\label{cases::prop6::3} If $\Xkp$ is from line~\ref{line::alg4::xG3} of \cref{algorithm::amm_c_x}, then we obtain $\Xkp=\Xkh$ and $F(\Xkp)=F(\Xkh)$. Then, similar to the derivations of \cref{eq::FlF1,eq::FlF2}, lines~\ref{line::alg4::xlG1} and \ref{line::alg4::xlG2} of \cref{algorithm::amm_c_x} result in
\begin{equation}\label{eq::FlF3}
\sF^{(\sk)}-F(\Xkh) \geq \psi\cdot\big\|\Xkh-\Xk\big\|^2
\end{equation}
and
\begin{equation}\label{eq::FlF4}
\sF^{(\sk)}-F(\Xkh)\geq \frac{\zeta}{2}\big\|\Xkh-\Xk\big\|^2,
\end{equation}
respectively, from which and $\Xkp=\Xkh$ and $F(\Xkp)=F(\Xkh)$ we obtain either
\begin{equation}
\nonumber
\sF^{(\sk)}-F(\Xkp) \geq \psi\cdot\big\|\Xkp-\Xk\big\|^2
\end{equation}
or
\begin{equation}
\nonumber
\sF^{(\sk)}-F(\Xkp)\geq \frac{\zeta}{2}\big\|\Xkp-\Xk\big\|^2.
\end{equation}
\end{enumerate}
Therefore, if $0<\eta\leq 1$, $0<\psi\leq 1$, $\xi>0$, $\zeta>0$, it can be shown from cases \ref{cases::prop6::1}) to \ref{cases::prop6::3}) above that there exists a constant scalar $\sigma>0$ such that
\begin{equation}\label{eq::sFF0}
	\sF^{(\sk)}-F(\Xkp) \geq  \frac{\sigma}{2}\|\Xkp-\Xk\|^2
\end{equation}
holds for any $\sk\geq 0$. In addition, note that \cref{eq::lFkp} is equivalent to
\begin{equation}\label{eq::lFklFkp}
	\lF^{(\sk)}-\lF^{(\skp)}= \eta\cdot\Big(\sF^{(\sk)}-F(\Xkp)\Big).
\end{equation} 
From \cref{eq::sFF0,eq::lFklFkp}, we conclude that there exists $\delta > 0$ such that
\begin{multline}\label{eq::sFF01}
\sF^{(\sk)}-\sF^{(\skp)}=\eta\cdot\Big(\sF^{(\sk)}-F(\Xkp)\Big) \geq \\
\frac{\eta\sigma}{2}\|\Xkp-\Xk\|^2 \geq \frac{\delta}{2}\|\Xkp-\Xk\|^2
\end{multline}
for any $\sk\geq 0$. 

As a result of \cref{eq::lF01,eq::sFF01}, we further obtain
\begin{equation}\label{eq::sFF1}
\sF^{(\sk)}-\sF^{(\skp)}\geq \frac{\delta}{2}\|\Xkp-\Xk\|^2
\end{equation}
for any $\sk\geq -1$. Recall $\lF^{(\sk)}\rightarrow F^\infty$ from \cref{prop::ammc}\ref{prop::ammc2}, which suggests
\begin{equation}\label{eq::sFF2}
	\sF^{(\sk)} - \sF^{(\skp)}\rightarrow 0.
\end{equation}
Therefore, \cref{eq::sFF1,eq::sFF2} indicate that
\begin{equation}
	\nonumber
	\|X^{(\sk+1)}-X^{(\sk)}\| \rightarrow 0.
\end{equation}
The proof is completed.

\vspace{0.8em}
\noindent\textbf{Proof of \ref{prop::ammc4}.\;} Note that \cref{eq::FlF3,eq::FlF4} suggest that there exists $\sigma'>0$ such that
\begin{equation}\label{eq::FlF6}
	\sF^{(\sk)}-F(\Xkh)\geq \frac{\sigma'}{2}\big\|\Xkh-\Xk\big\|^2
\end{equation}
always holds. From lines~\ref{line::alg4::restart_s3} to \ref{line::alg4::restart_e3} of \cref{algorithm::amm_c_x}, we obtain 
\begin{multline}\label{eq::lFkFkp0}
\lF^{(\sk)}-F(\Xkp) \geq \phi\cdot\Big(\lF^{(\sk)} - F(\Xkh)\Big)\geq\\
\frac{\phi\sigma'}{2}\big\|\Xkh-\Xk\big\|^2,
\end{multline}
where the last inequality is due to \cref{eq::FlF6}.
From \cref{eq::lFklFkp,eq::lFkFkp0}, it can be shown that
\begin{multline}\label{eq::FlF5}
\lF^{(\sk)}-\lF^{(\skp)}= \eta\cdot\Big(\sF^{(\sk)}-F(\Xkp)\Big)\geq\\
\frac{\eta\phi\sigma'}{2}\big\|\Xkh-\Xk\big\|^2.
\end{multline} 
The equation above suggests that there exists a constant scalar $\delta'>0$ such that
\begin{equation}\label{eq::sFF3}
\lF^{(\sk)}-\lF^{(\skp)}\geq \frac{\delta'}{2}\|\Xkh-\Xk\|^2.
\end{equation}
Note that \cref{prop::ammc}\ref{prop::ammc2} results in $\sF^{(\sk)} - \sF^{(\skp)}\rightarrow 0$. Thus, similar to the proof of \cref{prop::ammc}\ref{prop::ammc3}, it can be concluded from \cref{eq::sFF3} that
\begin{equation}
	\nonumber
	\|\Xkh-\Xk\| \rightarrow 0
\end{equation}
if $\zeta>\xi>0$. The proof is completed.

\vspace{0.8em}
\noindent\textbf{Proof of \ref{prop::ammc5}.\;} For any $\sk\geq 0$,  there are two possible cases about $\Xakh\in\XXa$:
\begin{enumerate}[leftmargin=0.45cm]
\item If $\Xakh\in\XXa$ is from line~\ref{line::alg4::xlG1} of \cref{algorithm::amm_c_x}, we obtain
\begin{equation}\label{eq::minH1}
	\Xakh\gets\arg\min\limits_{\Xa\in\XX^\alpha }\lG^\alpha(\Xa|\Yk).
\end{equation}
From \cref{eq::lGMa}, it can be shown that
\begin{equation}
	\nonumber
	\begin{aligned}
		&\nabla \lG^\alpha(\Xakh|\Yk)\\
		=&\nabla_{\Xa} F(\Yk)+(\Xakh-\Yak)\lnGamma^{\ak}\\
		=&\nabla_{\Xa} F(\Xkh)+(\Xakh-\Yak)\lnGamma^{\ak}\, +\\
		&\big(\nabla_{\Xa} F(\Yk) -\nabla_{\Xa} F(\Xkh)\big).
	\end{aligned}
\end{equation}
The equation above suggests
\begin{equation}\label{eq::gradHa}
\begin{aligned}
&\grad\, \lG^\alpha(\Xakh|\Yk)\\
=&\grad_{\alpha} F(\Xkh)+\\
&\QQ_{\Xakh}^\alpha\big((\Xakh-\Yak)\lnGamma^{\ak}\big) +\\
&\QQ_{\Xakh}^\alpha\big(\nabla_{\Xa} F(\Yk) -\nabla_{\Xa} F(\Xkh)\big),
\end{aligned}
\end{equation}
where $\grad_\alpha F(X)$ is the Riemannian gradient of $F(X)$ with respect to $\Xa \in\XX^\alpha$, and $\QQ_{\Xa}^\alpha: \R^{d\times dn_\alpha}\rightarrow \R^{d\times dn_\alpha}$ is a linear operator that extracts the $\alpha$-th block of $\QQ_{X}(\cdot)$ in \cref{eq::QQx}. Since $\Xakh$ is an optimal solution to \cref{eq::minH1}, we obtain
\begin{equation}\label{eq::gradHa0}
\grad\, H^\alpha(\Xakh|\Yk)=\0.
\end{equation}
From \cref{eq::gradHa,eq::gradHa0}, a straightforward mathematical manipulation indicates
\begin{equation}\label{eq::gradnFh1}
	\begin{aligned}
		&\grad_{\alpha} F(\Xkh) \\
		=&\QQ_{\Xakh}^\alpha\big((\Yak-\Xakh)\lnGamma^{\ak}\big)+\\
		&\QQ_{\Xakh}^\alpha\big(\nabla_{\Xa} F(\Xkh)-\nabla_{\Xa} F(\Yk)\big).
	\end{aligned}
\end{equation}
From \cref{eq::gradnFh1}, we further obtain
\begin{equation}
\nonumber
\begin{aligned}
&\|\grad_{\alpha} F(\Xkh)\|\\
\leq& \|\QQ_{\Xakh}^\alpha\|_2\cdot\|(\Yak-\Xakh)\lnGamma^{\ak}\|+\\
& \|\QQ_{\Xakh}^\alpha\|_2\cdot\|\nabla_{\Xa} F(\Xkh)-\nabla_{\Xa} F(\Yk)\|\\
\leq&\|\QQ_{\Xakh}^\alpha\|_2\cdot\|(\Yk-\Xkh)\lnGammak\|+\\
&\|\QQ_{\Xakh}^\alpha\|_2\cdot\|\nabla F(\Xkh)-\nabla F(\Yk)\|\\
\leq&\|\QQ_{\Xakh}^\alpha\|_2\cdot\|\lnGammak\|_2\cdot\|\Xkh-\Yk\|+\\
&\|\QQ_{\Xakh}^\alpha\|_2\cdot\|\nabla F(\Xkh)-\nabla F(\Yk)\|,
\end{aligned}
\end{equation}
where $\|\cdot\|_2$ denotes the induced 2-norm of linear operators. Recall from Lemmas \ref{lemma::lG}\ref{lemma::lGc} and \ref{prop::DF} that $\lnGamma\succeq\lnGammak$ and $\|\nabla F(\Xkh)-\nabla F(\Yk)\|\leq \mu\cdot\|\Xkh-\Yk\|$. Therefore, the right-hand side of the equation above can be further upper-bounded as
\begin{equation}\label{eq::gradak}
\begin{aligned}
	&\|\grad_\alpha F(\Xkh)\|\\
\leq&\|\QQ_{\Xakh}^\alpha\|_2\cdot \|\lnGamma\|_2\cdot\|\Xkh-\Yk\|+\\
	&\|\QQ_{\Xakh}^\alpha\|_2\cdot\mu\cdot\|\Xkh-\Yk\|.
\end{aligned}
\end{equation}
Similar to \cref{eq::gradFk,eq::gradFh}, $\|\QQ_{\Xakh}^\alpha\|_2$ in \cref{eq::gradak} is bounded as well. Therefore, there exists a constant scalar $\nu^\alpha >0$ such that the right-hand side of \cref{eq::gradak} can be upper-bounded:
\begin{equation}\label{eq::gradFa1}
	\|\grad_{\alpha} F(\Xkh)\|\leq \nu^\alpha \|\Xkh-\Yk\|.
\end{equation}
Recall that $\Yk\in\R^{d\times(d+1)n}$ results from line~\ref{line::alg3::Yk} of \cref{algorithm::ammc}:
\begin{equation}\label{eq::Y}
	\Yk = X^{(\sk)}+\big(X^{(\sk)}-X^{(\sk-1)}\big) \lambda^{(\sk)}.
\end{equation}
In \cref{eq::Y}, $\lambda^{(\sk)}\in\R^{(d+1)n\times(d+1)n}$ is a diagonal matrix
\begin{equation} 
	\nonumber
	\lambda^{(\sk)}\triangleq\diag\{\lambda^{1(\sk)}\cdot\I^1,\,\cdots,\,\lambda^{|\AA|(\sk)}\cdot\I^{|\AA|}\}\in \R^{(d+1)n\times(d+1)n},
\end{equation}
where $\lambda^{\ak}\in\R$ is given by line~\ref{line::alg3::sk} of \cref{algorithm::ammc} and $\I^\alpha\in \R^{(d+1)n_\alpha\times (d+1)n_\alpha}$ is the identity matrix. From \cref{eq::Y,eq::gradFa1}, it can be shown that
\begin{equation}\label{eq::gradnFbnd}
\begin{aligned}
	&\|\grad_\alpha F(\Xkh)\|\\
\leq&\nu^\alpha \|\Xkh-\Xk-\big(\Xk-\Xkm\big) \lambda^{(\sk)}\|\\
\leq&\nu^\alpha\|\big(\Xk\!-\Xkm\big) \lambda^{(\sk)}\|+\nu^\alpha \|\Xkh\!-\!\Xk\|\\
\leq&\nu^\alpha\|\lambda^{(\sk)}\|_2\cdot\|\Xk-\Xkm\|+\\
	&\nu^\alpha \|\Xkh-\Xk\|.
\end{aligned}
\end{equation} 
From line~\ref{line::alg3::sk} of \cref{algorithm::ammc}, we obtain $s^{\alpha(\sk)}\geq1$, and thus,
\begin{equation}
	\nonumber
	\lambda^{\ak}=\frac{\sqrt{4{s^{\alpha(\sk)}}^2+1}-1}{2s^{\alpha(\sk)}}=\frac{2s^{\alpha(\sk)}}{\sqrt{4{s^{\alpha(\sk)}}^2+1}+1}\in(0,\,1),
\end{equation}
which suggests $\|\lambda^{(\sk)}\|_2\in(0,\,1)$. Then, we upper-bound the right-hand side of \cref{eq::gradnFbnd} using $\|\lambda^{(\sk)}\|_2\in(0,\,1)$:
\begin{multline}\label{eq::gradanFF1}
	\|\grad_{\alpha} F(\Xkh)\|\leq \nu^\alpha\|\Xk-\Xkm\|+\\\nu^\alpha\|\Xkh-\Xk\|.
\end{multline}
\item If $\Xakh\in\XXa$ is from line~\ref{line::alg4::xlG2} of \cref{algorithm::amm_c_x}, we obtain
\begin{equation}\label{eq::minH2}
	\Xakh\gets\arg\min\limits_{\Xa\in\XX^\alpha }\lG^\alpha(\Xa|\Xk).
\end{equation}
A similar procedure to the derivation of \cref{eq::gradFa1} yields
\begin{equation}
	\nonumber
	\|\grad_{\alpha} F(\Xkh)\|\leq \nu^\alpha \|\Xkh-\Xk\|,
\end{equation}
where $\nu^\alpha>0$ is the same as that in \cref{eq::gradFa1}. Thus, we obtain
\begin{multline}\label{eq::gradanFF2}
	\|\grad_{\alpha} F(\Xkh)\leq \nu^\alpha\|\Xkh-\Xk\|\leq
	\\  \nu^\alpha\|\Xk-\Xkm\|+\nu^\alpha\|\Xkh-\Xk\|.
\end{multline}
\end{enumerate}
Therefore, as long as $\sk\geq 0$, it can be concluded from \cref{eq::gradanFF1,eq::gradanFF2} that
\begin{multline}\label{eq::gradanF2}
\|\grad_{\alpha} F(\Xkh)\|\leq \nu^\alpha\|\Xk-\Xkm\|+\\ \nu^\alpha\|\Xkh-\Xk\|
\end{multline}
holds for any node $\alpha\in\AA$. 

If $\sk\geq 0$, as a result of \cref{eq::gradanF2}, there exists a constant scalar $\nu\triangleq\sum_{\alpha\in\AA}\nu^\alpha>0$ such that
\begin{equation}
\nonumber
\begin{aligned}
	&\|\grad\, F(\Xkh)\|\\
\leq&\,\sum_{\alpha\in\AA}\|\grad_\alpha F(\Xkh)\|\\
\leq&\,\sum_{\alpha\in\AA}\nu^\alpha\cdot\big(\|\Xk-\Xkm\| + \|\Xkh-\Xk\|\big)\\
=&\, \nu\|\Xk-\Xkm\| + \nu\|\Xkh-\Xk\|\\
\leq&\sqrt{2}\nu\sqrt{ \|\Xk-\Xkm\|^2 + \|\Xkh-\Xk\|^2},
\end{aligned}
\end{equation}
which is equivalent to
\begin{multline}\label{eq::gradF2}
\|\grad\,F(\Xkh)\|^2\leq 2\nu^2\cdot\|\Xk-\Xkm\|^2 + \\
2\nu^2\cdot\|\Xkh-\Xk\|^2.
\end{multline}
Note that \cref{eq::sFF1,eq::sFF3} hold as long as $\zeta>\xi>0$, from which we might upper-bound $\|\Xk-\Xkm\|^2$ and $\|\Xkh-\Xk\|^2$ in \cref{eq::gradF2} and obtain
\begin{multline}\label{eq::gradFF2}
	\|\grad\,F(\Xkh)\|^2\leq\frac{4\nu^2}{\delta}\big(\lF^{(\skm)}-\lF^{(\sk)}\big) +\\
	 \frac{4\nu^2}{\delta'}\big(\lF^{(\sk)}-\lF^{(\skp)}\big).
\end{multline}
Recall from \cref{prop::ammc}\ref{prop::amm1} that 
\begin{equation}\label{eq::sFkmk}
\sF^{(\skm)}-\sF^{(\sk)}\geq 0
\end{equation}
and 
\begin{equation}\label{eq::sFkkp}
\sF^{(\sk)}-\sF^{(\skp)}\geq 0.
\end{equation}
Then, if we let $\epsilon \triangleq \min\{\frac{\delta}{2\nu^2},\,\frac{\delta'}{2\nu^2}\}>0$, \cref{eq::gradFF2,eq::sFkmk,eq::sFkkp} lead to
\begin{equation}\label{eq::gradF3}
	\lF^{(\skm)}-\lF^{(\skp)} \geq \frac{\epsilon}{2}\|\grad\,F(\Xkh)\|^2.
\end{equation}
A telescoping sum of \cref{eq::gradF3} over $\sk$ from $0$ to $\sK$ yields
\begin{equation}
	\nonumber
	\lF^{(-1)} + \lF^{(0)} - \lF^{(\sk)}-\lF^{(\skp)}\geq \\
	\frac{\epsilon}{2} \sum_{\sk=0}^{\sK}\|\grad\,F(\Xkh)\|^2,
\end{equation}
and thus,
\begin{multline}\label{eq::lFbnd2}
	\lF^{(-1)} + \lF^{(0)} - \lF^{(\sk)}-\lF^{(\skp)}\geq \\
	\frac{\epsilon(\sK+1)}{2}\min_{0\leq\sk\leq\sK}\|\grad\, F(\Xkh)\|^2.
\end{multline}
From lines~\ref{line::alg3::lF0}, \ref{line::alg3::lFk} of \cref{algorithm::ammc}, we obtain
\begin{equation}\label{eq::lF0}
	\lF^{(-1)}=\lF^{(0)}=F(X^{(0)}),
\end{equation} 
and Propositions \ref{prop::ammc}\ref{prop::ammc1} and \ref{prop::ammc}\ref{prop::ammc2} indicate
\begin{equation}\label{eq::Finf}
	\lF^{(\sk)}\geq \lF^{(\skp)}\geq F^{\infty}.
\end{equation}
As a result of \cref{eq::lF0,eq::Finf,eq::lFbnd2}, it can be concluded that
\begin{multline}
	\nonumber
	F(X^{(0)})-F^\infty \geq \frac{\epsilon(\sK+1)}{4}\min_{0\leq\sk\leq\sK}\|\grad\, F(\Xkh)\|^2,
\end{multline}
which is equivalent to
\begin{equation}
	\min\limits_{0\leq\sk< \mathsf{K}}\|\grad\, F(\Xkh)\|\leq 2\sqrt{\frac{1}{\epsilon}\cdot\dfrac{F(X^{(0)})-F^\infty}{{\sK+1}}}.
\end{equation}
The proof is completed.

\vspace{0.5em}
\noindent\textbf{Proof of \ref{prop::ammc6}.\;} From Propositions \ref{prop::ammc}\ref{prop::ammc3} and \ref{prop::ammc}\ref{prop::ammc4}, we obtain
\begin{equation}
	\nonumber
	\|\Xkp-\Xk\|\rightarrow 0
\end{equation}
and
\begin{equation}
	\nonumber
	\|\Xkh-\Xk\|\rightarrow 0
\end{equation}
as long as $\zeta>\xi>0$, from which and \cref{eq::gradF2}, it is trivial to show that
\begin{equation}\label{eq::gradFkh}
	\grad\, F(\Xkh)\rightarrow \0.
\end{equation}
In addition, note that $\grad\,F(X)$ is continuous by \cref{assumption::loss}\ref{assumption::loss_cont}, which suggests
\begin{equation}\label{eq::gradFkkh}
	\grad\,F(\Xk)\rightarrow\grad\,F(\Xkh).
\end{equation}
From \cref{eq::gradFkh,eq::gradFkkh}, it can be concluded that
\begin{equation}
	\nonumber
	\grad\, F(\Xk) \rightarrow \0.
\end{equation}
The proof is completed.

%

%% file: appendix_restart.tex
\noindent\textbf{Proof of \ref{prop::FsF1}.\;} We will prove $F(\Xk)=\sum_{\alpha\in\AA} F^{\ak}$ by induction.
\begin{enumerate}[leftmargin=0.45cm]
\item From \cref{eq::F}, it can be shown that
\begin{align}
&F(X^{(0)})\nonumber\\
=&\sum_{\alpha\in\AA}\sum_{(i,j)\in \aEE^{\alpha\alpha}}F_{ij}^{\aa}(X^{(0)})+
\sum_{\substack{\alpha,\beta\in\AA,\\\alpha\neq \beta}}\sum_{(i,j)\in \aEE^{\alpha\beta}} F_{ij}^{\ab}(X^{(0)})\nonumber\\
=&\sum_{\alpha\in\AA}\bigg(\sum_{(i,j)\in \aEE^{\alpha\alpha}}F_{ij}^{\aa}(X^{(0)}) +\nonumber\\
 &\quad\quad\quad\frac{1}{2}\sum_{\beta\in\NN_-^{\alpha}}\sum_{(i,j)\in \aEE^{\alpha\beta}} F_{ij}^{\ab}(X^{(0)})+\label{eq::FF0}\\
 &\quad\quad\quad\frac{1}{2}\sum_{\beta\in\NN_+^{\alpha}}\sum_{(i,j)\in \aEE^{\beta\alpha}} F_{ij}^{\ba}(X^{(0)})\bigg)\nonumber\\
=&\sum_{\alpha\in\AA} F^{\alpha(0)},\nonumber
\end{align}
where the last equality results from \cref{eq::Fa0}.
\vspace{0.25em}
\item Suppose $F(\Xk)=\sum_{\alpha\in\AA} F^{\ak}$ holds at $\sk$-th iteration. As a result of \cref{eq::Gakp}, we obtain
\begin{equation}\label{eq::GGa}
\begin{aligned}
&\sum_{\alpha\in\AA}G^{\akp}\\
=&\sum_{\alpha\in\AA} G^\alpha(\Xakp|\Xk) + \sum_{\alpha\in\AA}F^{\ak}\\
=&\sum_{\alpha\in\AA} G^\alpha(\Xakp|\Xk) + F(\Xk)\\
=& G(\Xkp|\Xk),
\end{aligned}
\end{equation}
where the second and third equality are due to $F(\Xk)=\sum_{\alpha\in\AA} F^{\ak}$ and \cref{eq::G2}, respectively. From \cref{eq::DGa} and
\begin{equation}
\nonumber
\|\Xkp-\Xk\|^2=\sum_{\alpha\in\AA}\|\Xakp-\Xak\|^2,
\end{equation}
it is straightforward to show
\begin{multline}\label{eq::DG}
\sum_{\alpha\in\AA} \Delta G^{\alpha}(\Xakp|\Xk)=\\
\sum_{\substack{\alpha,\beta\in\AA,\\\alpha\neq \beta}}\sum_{(i,j)\in \aEE^{\ab}} \left(F_{ij}^{\ab}(X)-E_{ij}^{\ab}(X|\Xk)\right)-\\
\frac{\xi}{2} \big\|\Xkp-\Xk\big\|^2.
\end{multline}
In addition, \cref{eq::Fakp,eq::GGa} suggest
\begin{equation}\label{eq::FGG}
\begin{aligned}
&\sum_{\alpha\in\AA}F^{\akp}\\
=&\sum_{\alpha\in\AA}G^{\akp}+\sum_{\alpha\in\AA} \Delta G^{\alpha}(\Xakp|\Xk)\\
=&G(\Xkp|\Xk) + \sum_{\alpha\in\AA} \Delta G^{\alpha}(\Xakp|\Xk).
\end{aligned}
\end{equation}
Substituting \cref{eq::G,eq::DG} into \cref{eq::FGG} yields.
\begin{multline}
\nonumber
\sum_{\alpha\in\AA}F^{\akp}= \sum_{\alpha\in\AA}\sum_{(i,j)\in \aEE^{\alpha\alpha}}F_{ij}^{\aa}(\Xkp)+\\
\sum_{\substack{\alpha,\beta\in\AA,\\\alpha\neq \beta}}\sum_{(i,j)\in \aEE^{\alpha\beta}} F_{ij}^{\ab}(\Xkp).
\end{multline}
We simplify the equation above with \cref{eq::F} and obtain
\begin{equation}
F(\Xkp)=\sum_{\alpha\in\AA}F^{\akp}.
\end{equation}
\item Therefore, it can be concluded that $F(\Xk)=\sum_{\alpha\in\AA} F^{\ak}$ holds for any $\sk\geq 0$.
\end{enumerate}
\vspace{0.5em}

\noindent\textbf{Proof of \ref{prop::FsF2}.\;} We will prove $\sF^{(\sk)} = \sum_{\alpha\in\AA}\sF^{\ak}$ by induction.
\begin{enumerate}[leftmargin=0.45cm]
\item Recall from \cref{eq::sFa0,eq::FF0,eq::lFk} that $\sF^{(0)}=F(X^{(0)})$, $\sF^{\alpha(0)}=F^{\alpha(0)}$ and $F(X^{(0)})=\sum_{\alpha\in\AA} F^{\alpha(0)}$, which immediately yields
\begin{equation}
\sF^{(0)}=F(X^{(0)}).
\end{equation}
\item Suppose $\sF^{(\sk)}=\sum_{\alpha\in\AA} \sF^{\ak}$ holds at $\sk$-th iteration. As a result of \cref{eq::lFk}, we obtain
\begin{equation}
\nonumber
\sF^{(\skp)} = (1-\eta)\cdot\sF^{(\sk)}+\eta\cdot F(\Xkp).
\end{equation}
Note that \cref{prop::FsF}\ref{prop::FsF1} suggests $F(\Xkp)=\sum_{\alpha\in\AA}F^{\akp}$. Apply $\sF^{(\sk)}=\sum_{\alpha\in\AA} \sF^{\ak}$ and $F(\Xkp)=\sum_{\alpha\in\AA}F^{\akp}$ on the right-hand side of the equation above results in
\begin{multline}
 \sF^{(\skp)}=(1-\eta)\cdot\sum_{\alpha\in\AA}\sF^{\ak} +\\
  \eta\cdot\sum_{\alpha\in\AA}F^{\akp}=\sum_{\alpha\in\AA}\sF^{\akp},
\end{multline}
where the last equality is due to \cref{eq::lFakp}.
\item Therefore, it can be concluded that $\sF^{(\sk)}=\sum_{\alpha\in\AA} \sF^{\ak}$ holds for any $\sk\geq 0$.  The proof is completed.
\end{enumerate}

\vspace{0.5em}

\noindent\textbf{Proof of \ref{prop::FsF4}.\;} \cref{prop::upper} indicates $E_{ij}^{\ab}(X|\Xkm)\geq F_{ij}^{\ab}(X)$ and $E_{ij}^{\ba}(X|\Xkm)\geq F_{ij}^{\ba}(X)$, from which and \cref{eq::DGa} we obtain 
\begin{equation}\label{eq::DGal}
	\Delta G^\alpha(\Xa|\Xkm)\leq 0
\end{equation}
as long as $\xi\geq 0$. From \cref{eq::DGal,eq::Fakp}, it is immediate to conclude 
\begin{equation}
F^{\akp}\leq G^{\akp}
\end{equation}
for any $\sk\geq 0$. If $G^{\akp}\leq \sF^{\ak}$, the equation above further suggests
\begin{equation}\label{eq::FGlF}
F^{\akp}\leq G^{\akp}\leq \sF^{\ak}.
\end{equation}
From \cref{eq::lFakp}, we obtain
\begin{equation}\label{eq::sFakp}
\sF^{\akp} = (1-\eta)\cdot \sF^{\ak} + \eta\cdot F^{\akp},
\end{equation}
where note that $\eta\in(0,1]$. Thus, we conclude that $\sF^{\akp}$ is a convex combination of $\sF^{\ak}$ and $F^{\akp}$, which and \cref{eq::FGlF} lead to
\begin{equation}
F^{\akp}\leq \sF^{\akp}\leq \sF^{\ak}.
\end{equation}
The proof is completed.

%% file: appendix_amm.tex
\noindent\textbf{Proof of \ref{prop::amm1}.\;}  We will prove $F^{\ak}\leq \lF^{\ak}$ and $\lF^{\akp}\leq \lF^{\ak}$ by induction, from which it can be shown that $\sF^{(\skp)}\leq \sF^{(\sk)}$.
\begin{enumerate}[leftmargin=0.5cm]
	\item From lines~\ref{line::alg5::Fakp2}, \ref{line::alg5::lFa0}, \ref{line::alg5::Fakp1}, \ref{line::alg5::lFakp} of \cref{algorithm::amm}, it can be shown that 
	\begin{equation}
		\nonumber
		F^{\alpha(-1)}=F^{\alpha(0)}=\lF^{\alpha(-1)}=\lF^{\alpha(0)}.
	\end{equation}
	\item Suppose $\sk\geq 0$ and $F^{\ak}\leq \lF^{\ak}$ holds at $\sk$-th iteration. In terms of $\Xakh$, if the adaptive restart scheme  for $\Xakh$ is not triggered, it is immediate to show from line~\ref{line::alg6::restart1} of \cref{algorithm::amm_x} that
	\begin{equation}\label{eq::GGHFk2}
		\sG^{\akh}  \leq \lF^{\ak}.
	\end{equation}
	On the other hand, if the adaptive restart scheme for $\Xakh$ is triggered, line~\ref{line::alg6::Xakh2} of \cref{algorithm::amm_x} results in
	\begin{multline}\label{eq::GHHk}
	 G^{\alpha}(\Xakh|\Xk) \leq	H^{\alpha}(\Xakh|\Xk)\leq\\  H^{\alpha}(\Xak|\Xk) = 0,
	\end{multline}
	where the first inequality and the last equality are due to \cref{lemma::lG}\ref{lemma::lGd} and \cref{eq::lGMa}, respectively. From \cref{eq::Gakp,eq::GHHk}, we obtain
	\begin{multline}\label{eq::GGHFk}
	\!\!\!G^{\akh} = G^\alpha(\Xakh|\Xk) + F^{\ak}\leq\\
	 H^\alpha(\Xakh|\Xk) + F^{\ak}\leq F^{\ak}\leq \sF^{\ak}.
	\end{multline}
	Therefore, no matter whether the adaptive restart scheme is triggered or not, we conclude from \cref{eq::GGHFk,eq::GGHFk2} that
	\begin{equation}\label{eq::GHlF}
		G^{\akh}\leq  \lF^{\ak}
	\end{equation}
	always holds. Furthermore, if $\Xakp$ and $G^{\akp}$ result from line~\ref{line::alg6::Xakp3} of \cref{algorithm::amm_x}, we obtain $\Xakp=\Xakh$ and $G^{\akp}=G^{\akh}$, which and \cref{eq::GHlF} yield
	\begin{equation}\label{eq::GlFbnd2}
		G^{\akp}=G^{\akh}\leq \lF^{\ak}.
	\end{equation}
	Otherwise, line~\ref{line::alg6::check1} of \cref{algorithm::amm_x} and \cref{eq::GHlF} suggest
	\begin{equation}\label{eq::GlFbnd}
		G^{\akp} - \lF^{\ak} \leq \phi\cdot \big(\sG^{\akh}-\lF^{\ak}\big)\leq 0.
	\end{equation}
	Then, \cref{eq::GlFbnd,eq::GlFbnd2} suggest
	\begin{equation}\label{eq::GkpFk}
		G^{\akp} \leq \lF^{\ak}
	\end{equation}
	always holds, from which and \cref{prop::FsF}\ref{prop::FsF4} we conclude
	\begin{equation}\label{eq::FlFak}
		F^{\akp}\leq \lF^{\akp}\leq \lF^{\ak}.
	\end{equation}
	\item Therefore, it can be concluded that $F^{\ak}\leq \lF^{\ak}$ and  $\lF^{\akp}\leq \lF^{\ak}$.
\end{enumerate}
Summing both sides of  $\sF^{\akp}\leq \sF^{\ak}$ over all the nodes $\alpha$ and implementing Propositions \ref{prop::FsF}\ref{prop::FsF1} and \ref{prop::FsF}\ref{prop::FsF2} yields
\begin{equation}
	\nonumber
	\lF^{(\skp)}=\sum_{\alpha\in\AA}\lF^{\akp} \leq \sum_{\alpha\in\AA}\lF^{\ak}=\lF^{(\sk)},
\end{equation}
which suggests that $\lF^{(\sk)}$ is nonincreasing. The proof is completed.

\vspace{0.8em}
\noindent\textbf{Proof of \ref{prop::amm2}.\;} Recalling that $F(\Xk)\geq 0$ holds  by definition for any $\sk\geq 0$ and $\lF^{(\sk)}$ is the exponential moving average of $F(X^{(0)}),\,F(X^{(1)}),\,\cdots,\, F(\Xk)$, we obtain $\sF^{(\sk)}\geq 0$, i.e., $\sF^{(\sk)}$ is bounded below. In addition, \cref{prop::amm}\ref{prop::amm1} indicates that $\lF^{(\sk)}$ is nonincreasing, and thus, there exists $F^{\infty}$ such that $\lF^{(\sk)}\rightarrow F^{\infty}$, from which and \cref{eq::lFk} it can be concluded that $F(\Xk)\rightarrow F^{\infty}$ as well. The proof is completed.

\vspace{0.8em} 
\noindent\textbf{Proof of \ref{prop::amm3}.\;} From \cref{eq::G2}, it can be shown that $G(\Xkp|\Xk)$ takes the form as
\begin{equation}
	\nonumber
	\begin{aligned}
	G(\Xkp|\Xk) =& \sum_{\alpha\in\AA}G^\alpha(\Xak|\Xk) + F(\Xk)\\
	=&\sum_{\alpha\in\AA}G^\alpha(\Xak|\Xk) +\!\sum_{\alpha\in\AA} F^{\ak},
	\end{aligned}
\end{equation}
where the last equality is due to \cref{prop::FsF}\ref{prop::FsF1}. Applying \cref{eq::Gakp} on the equation above results in
\begin{equation}\label{eq::Gk}
G(\Xkp|\Xk)=\sum_{\alpha\in\AA} G^{\akp}.
\end{equation}
Recalling $G^{\akp}\leq \lF^{\ak}$ from \cref{eq::GkpFk} and $\sum_{\alpha\in\AA}\lF^{\ak} = \lF^{(\sk)}$ from \cref{prop::FsF}\ref{prop::FsF2}, we obtain
\begin{equation}\label{eq::GlFk}
	G(\Xkp|\Xk) = \sum_{\alpha\in\AA} G^{\akp} \leq \sum_{\alpha\in\AA} \lF^{\ak} = \lF^{(\sk)}.
\end{equation}
From \cref{eq::GlFk}, it can be shown that
\begin{equation}\label{eq::lFFG}
	\lF^{(\sk)} - F(\Xkp)\geq G(\Xkp|\Xk) - F(\Xkp).
\end{equation}
Substitute \cref{eq::Gxkp,eq::Fxkp} into the right-hand side of \cref{eq::lFFG} and simplify the resulting equation:
\begin{multline}\label{eq::lFkFkp2}
	\lF^{(\sk)} - F(\Xkp)\geq \frac{1}{2}\|\Xkp-\Xk\|_{\nGammak}^2-\\
	\frac{1}{2}\|\Xkp-\Xk\|_{\nMk}^2.
\end{multline}
From \cref{eq::lFkFkp2,eq::lFklFkp}, we obtain
\begin{multline}
	\nonumber
	\lF^{(\sk)}-\lF^{(\skp)}= \eta\cdot\big(\lF^{(\sk)}-F^{(\skp)}\big)\geq\\
	\frac{\eta}{2}\|\Xkp-\Xk\|_{\nGammak}^2-\frac{\eta}{2}\|\Xkp-\Xk\|_{\nMk}^2.
\end{multline}
From \cref{eq::GMk}, note that $\nGammak-\nMk\geq \xi\cdot\I$, and thus, there exists $\delta > 0$ such that the equation above is reduced to
\begin{equation}\label{eq::lFkFkp}
	\lF^{(\sk)} - \lF^{(\skp)} \geq \frac{\delta}{2}\|\Xkp-\Xk\|^2
\end{equation}
as long as $\xi>0$. Furthermore, \cref{prop::amm}\ref{prop::amm2} yields $$\lF^{(\sk)}-\lF^{(\skp)}\rightarrow 0,$$ from which and \cref{eq::lFkFkp}, we obtain
\begin{equation}\label{eq::XkpXk}
	\|\Xkp-\Xk\|\rightarrow 0.
\end{equation}
The proof is completed.

\vspace{0.8em} 
\noindent\textbf{Proof of \ref{prop::amm4}.\;} In terms of $\Xakh$, there are two possible cases:
\begin{enumerate}[leftmargin=0.45cm]
\item If $\Xakh$ is from line~\ref{line::alg6::Xakh1} of \cref{algorithm::amm_x}, then line~\ref{line::alg6::restart1} of \cref{algorithm::amm_x} indicates
\begin{equation}\label{eq::lFG1}
\sF^{\ak}-G^{\akh}\geq  \psi\cdot\|\Xakh-\Xak\|^2.
\end{equation}
\item If $\Xakh$ is from line~\ref{line::alg6::Xakh2} of \cref{algorithm::amm_x}, then note that \cref{eq::GGHFk} holds for any $\sk\geq 0$, which suggests
\begin{multline}
\nonumber
\sF^{\ak}-G^{\akh}\geq H^\alpha(\Xakh|\Xk)-\\
G^\alpha(\Xakh|\Xk).
\end{multline}
Recalling the definitions of $G^\alpha(\Xakh|\Xk)$ and $H^\alpha(\Xakh|\Xk)$ in \cref{eq::GMa,eq::lGMa}, we rewrite the equation above as
\begin{multline}\label{eq::lFG}
\sF^{\ak}-G^{\akh}\geq \frac{1}{2}\|\Xakh-\Xak\|_{\lnGammaak}^2-\\
\frac{1}{2}\|\Xakh-\Xak\|_{\nGammaak}^2.
\end{multline}
Similar to $\nGammak$ and $\lnGammak$ in \cref{eq::lGGk}, we obtain
\begin{equation}\label{eq::lGGak}
\lnGammaak \succeq \nGammaak + (\zeta-\xi)\cdot\I.
\end{equation}
Applying \cref{eq::lGGak} on the right-hand side of \cref{eq::lFG} indicates
\begin{multline}\label{eq::lFG2}
	\sF^{\ak}-G^{\akh}\geq \frac{\zeta-\xi}{2}\|\Xakh-\Xak\|^2.
\end{multline}
\end{enumerate}
Then, as a result of \cref{eq::lFG1,eq::lFG2}, there exists a constant scalar $\sigma'>0$ such that
\begin{equation}\label{eq::lFG3}
	\sF^{\ak}-G^{\akh}\geq \frac{\sigma'}{2}\|\Xakh-\Xak\|^2
\end{equation}
if $\zeta>\xi>0$. In addition, lines~\ref{line::alg6::check1} to \ref{line::alg6::check2} of \cref{algorithm::amm_x} results in
\begin{multline}\label{eq::lFGak2}
\sF^{\ak} - \sG^{\akp} \geq \phi\cdot \Big(\sF^{\ak} - \sG^{\akh}\Big)\geq\\
\frac{\phi\sigma'}{2}\|\Xakh-\Xak\|^2,
\end{multline}
where the last inequality is from \cref{eq::lFG3}. Summing both sides of \cref{eq::lFGak2} over all the nodes $\alpha\in\AA$ and simplifying the resulting equation with \cref{prop::FsF}\ref{prop::FsF2} and \cref{eq::Gk}, we obtain
\begin{equation}
\nonumber
\sF^{(\sk)} - G(\Xkp|\Xk) \geq \frac{\phi\sigma'}{2}\|\Xkh-\Xk\|^2.
\end{equation}
Recalling $G(\Xkp|\Xk)\geq F(\Xkp)$  from \cref{prop::G}\ref{prop::G2}, the equation above indicates
\begin{multline}\label{eq::lFF}
\sF^{(\sk)} - F(\Xkp) \geq \sF^{(\sk)} - G(\Xkp|\Xk) \geq\\ \frac{\phi\sigma'}{2}\|\Xkh-\Xk\|^2.
\end{multline}
From \cref{eq::lFF,eq::lFklFkp}, it is immediate to show
\begin{multline}
\nonumber
\sF^{(\sk)} - \sF^{(\skp)} \geq \eta\cdot\Big(\sF^{(\sk)}-F(\Xkp)\Big)\geq \\
\frac{\eta\phi\sigma'}{2}\|\Xkh-\Xk\|^2.
\end{multline}
Therefore, there exists a constant scalar $\delta'>0$ such that
\begin{equation}
\sF^{(\sk)} - \sF^{(\skp)} \geq \frac{\delta'}{2}\|\Xkh-\Xk\|^2.
\end{equation}
Since $\sF^{(\sk)} \rightarrow F^{\infty}$, it can be concluded that
\begin{equation}
	\|\Xkh-\Xk\|\rightarrow 0
\end{equation}
if $\zeta>\xi>0$. The proof is completed.
\vspace{0.8em} 

\noindent\textbf{Proof of \ref{prop::amm5} and \ref{prop::amm6}.\;} The proofs of Propositions \ref{prop::amm}\ref{prop::amm5} and \ref{prop::amm}\ref{prop::amm6} are almost the same as these of Propositions \ref{prop::ammc}\ref{prop::ammc5} and \ref{prop::ammc}\ref{prop::ammc6}, which are thus omitted due to space limitation.

%% file: appendix_G.tex
From \cref{eq::GMa,eq::Eij3,eq::Fij3,eq::G}, it can be concluded that
\begingroup
\allowdisplaybreaks
\begin{equation}\label{eq::FaH}
	\begin{aligned}
	&\frac{1}{2}\|X^\alpha-X^{\ak}\|_{\nGamma^{\ak}}^2\\
=&\sum_{i=1}^{n_\alpha}\Big[\frac{1}{2}\xi\|R_i^\alpha-R_i^{\ak}\|^2 + \frac{1}{2}\xi\|t_i^\alpha-t_i^{\ak}\|^2\Big]+\hspace{5em}\\
&\sum_{(i,j)\in \aEE^{\aa}}\Big[\frac{1}{2}\kappa_{ij}^{\aa}\|(R_i^\alpha-R_i^{\ak})\nR_{ij}^{\aa} -(R_j^\alpha-R_j^{\alpha(\sk)})\|^2 +\hspace{17em}\\ 
&\;\,\frac{1}{2}\tau_{ij}^{\aa}\|(R_i^\alpha-R_i^{\ak}) \nt_{ij}^{\aa}+t_i^\alpha - t_i^{\ak} - (t_j^\alpha-t_j^{\alpha(\sk)})\|^2\Big]+\\
&\sum_{\substack{\beta\in \NN_-^\alpha}}\sum_{(i,j)\in \aEE^{\ab}}\wabijk\Big[\kappa_{ij}^{\ab}\|R_i^\alpha-R_i^{\ak}\|^2 +\\[-0.75em]
&\quad\quad\quad\quad\quad\quad\quad\;\;\tau_{ij}^{\ab}\|\big(R_i^\alpha-R_i^{\ak}\big) \nt_{ij}^{\ab}+t_i^\alpha-t_i^{\ak}\|^2\Big]+\\[0.3em]
&\sum_{\substack{\beta\in \NN_+^\alpha}}\sum_{(j,i)\in \aEE^{\beta\alpha}}\wbajik\Big[\kappa_{ji}^{\beta\alpha}\|R_i^\alpha -R_{i}^{\ak}\|^2+\\
&\quad\quad\quad\quad\hspace{10em}\tau_{ji}^{\beta\alpha}\|t_i^\alpha - t_{i}^{\ak}\|^2\Big].
	\end{aligned} 
\end{equation}
\endgroup
For notational clarity, we introduce
$$\aEE_{i-}^{\ab}\triangleq \{(i,\,j)|(i,\,j)\in \aEE^{\ab}\},$$
$$\aEE_{i+}^{\ab}\triangleq \{(i,\,j)|(j,\,i)\in \aEE^{\beta\alpha}\},$$
$$\NN_{i-}^{\alpha}\triangleq \{\beta\in \AA|\exists(i,\,j)\in \aEE^{\ab}\text{ and }\beta\neq\alpha\},$$
$$\NN_{i+}^{\alpha}\triangleq \{\beta\in \AA|\exists(j,\,i)\in \aEE^{\beta\alpha}\text{ and }\beta\neq\alpha\},$$ $$\EE_i^{\ab}\triangleq\aEE_{i-}^{\ab}\cup\aEE_{i+}^{\ab},$$
$$\NN_i^\alpha\triangleq\NN_{i-}^\alpha\cup\NN_{i+}^\alpha.$$ 
In addition, we define $\kappa_{ji}^{\beta\alpha}\triangleq\kappa_{ij}^{\ab}$, $\tau_{ji}^{\beta\alpha}\triangleq\tau_{ij}^{\ab}$, $\wabijk\triangleq\wbajik$, and
\begin{equation}\label{eq::kappak}
\kappa_{ij}^{\ab(k)}\triangleq \wabijk\cdot\kappa_{ij}^{\ab},
\end{equation}
\begin{equation}\label{eq::tauk}
\tau_{ij}^{\ab(k)}\triangleq \wabijk\cdot\tau_{ij}^{\ab}
\end{equation}
for any $(i,\,j)\in\aEE^{\ab}$. 

Then, \cref{eq::FaH} indicates that $\nGamma^{\ak}\in \R^{(d+1)n_\alpha\times (d+1)n_\alpha}$ in \cref{eq::GMa} takes the form as
\begin{equation}\label{eq::nGamma}
\nGamma^{\ak}=\begin{bmatrix}
\nGamma^{\tau,\ak}\hphantom{{}^\transpose} & \nGamma^{\nu,\ak}\\
\nGamma{\vphantom{H}^{\nu,\ak}}^\transpose & \nGamma^{\kappa,\ak}
\end{bmatrix},
\end{equation}
in which $\nGamma^{\tau,\ak}\in \R^{n_\alpha\times n_\alpha}$, $\nGamma^{\nu,\ak}\in\R^{n_\alpha\times dn_\alpha}$ and $\nGamma^{\kappa,\ak}\in\R^{dn_\alpha\times dn_\alpha}$ are sparse matrices that are defined as
\begingroup
\allowdisplaybreaks
\begin{align*}
&\Big[\nGamma^{\tau,\ak}\Big]_{ij}=\begin{cases}
\xi + \sum\limits_{e\in \EE_i^{\aa}}\tau_{e}^{\aa}+\\
\quad\quad\;\;\sum\limits_{\substack{\beta\in \NN_i^\alpha}}\sum\limits_{e\in \EE_i^{\ab}}2\tau_e^{\abk}, & i=j,\\
-\tau_{ij}^{\aa}, & (i,j)\in \aEE^{\aa},\\
-\tau_{ji}^{\aa}, & (j,i)\in \aEE^{\aa},\\
0,& \text{otherwise,}
\end{cases}\\[0.8em]
&\Big[\nGamma^{\nu,\ak}\Big]_{ij}=\begin{cases}
\sum\limits_{e\in \EE_{i-}^{\aa}}\tau_{e}^{\aa}\nt{_{e}^{\aa}}^\transpose+\\
\;\sum\limits_{\substack{\beta\in \NN_{i-}^\alpha}}\sum\limits_{e\in \EE_{i-}^{\ab}}2\tau_e^{\abk}\nt{\vphantom{t}_{e}^{\ab}}^\transpose, & i=j,\\
-\tau_{ji}^{\aa}\nt{_{ji}^{\aa}}^\transpose, & (j,i)\in \aEE^{\aa},\\
0,& \text{otherwise,}
\end{cases}\\[0.8em]
&\Big[\nGamma^{\kappa,\ak}\Big]_{ij}=\begin{cases}
\xi\cdot\I+\sum\limits_{e\in \EE_i^{\aa}}\kappa_{e}^{\aa}\cdot\I+\\
\quad\sum\limits_{e\in \EE_{i-}^{\aa}}\tau_{e}^{\aa}\cdot\nt_{e}^{\aa}\nt{\vphantom{t}_{e}^{\aa}}^\transpose+\\
\sum\limits_{\substack{\beta\in \NN_i^\alpha}}\big(\sum\limits_{e\in \EE_i^{\ab}}2\kappa_e^{\abk}\cdot\I+\\
\quad\sum\limits_{e\in \EE_{i-}^{\ab}}\;2\tau_{e}^{\abk}\cdot\nt_{e}^{\ab}\nt{\vphantom{t}_{e}^{\ab}}^\transpose\big), & i=j,\\
-\kappa_{ij}^{\aa}\cdot\nR_{ij}^{\aa}, & (i,j)\in \aEE^{\aa},\\
-\kappa_{ji}^{\aa}\cdot\nR{\vphantom{R}_{ij}^{\aa}}^\transpose, & (j,i)\in \aEE^{\aa},\\
0,& \text{otherwise.}
\end{cases}
\end{align*}
\endgroup

%% file: appendix_lG.tex
From \cref{eq::Eij3,eq::lG}, it can be concluded that
\begingroup
\allowdisplaybreaks
\begin{align}\label{eq::lnGammai}
&\frac{1}{2}\|X-\Xk\|_{\lnGamma_i^\ak}^2\nonumber\\
=&\frac{1}{2}\zeta\|R_i^\alpha-R_i^{\ak}\|^2 + \frac{1}{2}\zeta\|t_i^\alpha-t_i^{\ak}\|^2+\nonumber\\
&\sum_{(i,j)\in\aEE_{i-}^{\aa}}\Big[\kappa_{ij}^{\aa}\|R_i^\alpha-R_i^{\ak}\|^2+\nonumber\\[-0.75em]
&\quad\quad\quad\quad\;\;\tau_{ij}^{\aa}\|\big(R_i^\alpha-R_i^{\ak}\big) \nt_{ij}^{\aa}+t_i^\alpha-t_i^{\ak}\|^2\Big]+\nonumber\\
&\sum_{(i,j)\in \aEE_{i+}^{\aa}}\Big[\kappa_{ji}^{\aa}\|R_i^\alpha -R_{i}^{\ak}\|^2+\tau_{ji}^{\aa}\|t_i^\alpha - t_{i}^{\ak}\|^2\Big]+\nonumber\\
&\sum_{\substack{\beta\in \NN_{i-}^\alpha}}\sum_{(i,j)\in \aEE_{i-}^{\ab}}\wabijk\Big[\kappa_{ij}^{\ab}\|R_i^\alpha-R_i^{\ak}\|^2 +\nonumber\\[-0.75em]
&\quad\quad\quad\quad\quad\quad\quad\;\;\tau_{ij}^{\ab}\|\big(R_i^\alpha-R_i^{\ak}\big) \nt_{ij}^{\ab}+t_i^\alpha-t_i^{\ak}\|^2\Big]+\nonumber\\
&\sum_{\substack{\beta\in \NN_{i+}^\alpha}}\sum_{(j,i)\in \aEE_{i+}^{\beta\alpha}}\wbajik\Big[\kappa_{ji}^{\beta\alpha}\|R_i^\alpha -R_{i}^{\ak}\|^2+\nonumber\\
&\quad\quad\quad\quad\hspace{10em}\tau_{ji}^{\beta\alpha}\|t_i^\alpha - t_{i}^{\ak}\|^2\Big].
\end{align}
\endgroup

From \cref{eq::lnGammai}, it can be shown that $\lnGamma_i^{\ak}\in\R^{(d+1)\times(d+1)}$ in \cref{eq::lGMai} can be written as
\begin{equation}\label{eq::lnGammaki}
\lnGamma_i^{\ak} = \begin{bmatrix}
\lnGamma_i^{\tau,\ak}\hphantom{{}^\transpose} & \lnGamma_i^{\nu,\ak}\\
\lnGamma{\vphantom{H}_i^{\nu,\ak}}^\transpose & \lnGamma_i^{\kappa,\ak}
\end{bmatrix},
\end{equation}
in which $\lnGamma_i^{\tau,\ak}\in\R$, $\lnGamma_i^{\nu,\ak}\in\R^{1\times d}$ and $\lnGamma_i^{\kappa,\ak}\in\R^{d\times d}$ are defined as
\begin{equation}\label{eq::lnGamma_ti}
\lnGamma_i^{\tau,\ak}=\zeta + \sum\limits_{e\in \EE_i^{\aa}}2\tau_{e}^{\aa}+\sum\limits_{\substack{\beta\in \NN_i^\alpha}}\sum\limits_{e\in \EE_i^{\ab}}2\tau_e^{\abk},
\end{equation}
\begin{equation}\label{eq::lnGamma_vi}
\lnGamma_i^{\nu,\ak}=\sum\limits_{e\in \EE_{i-}^{\aa}}2\tau_{e}^{\aa}\nt{_{e}^{\aa}}^\transpose+\sum\limits_{\substack{\beta\in \NN_{i-}^\alpha}}\sum\limits_{e\in \EE_{i-}^{\ab}}2\tau_e^{\abk}\nt{\vphantom{t}_{e}^{\ab}}^\transpose,
\end{equation}
\begin{multline}\label{eq::lnGamma_ri}
\lnGamma_i^{\kappa,\ak}=\zeta\cdot\I+\sum\limits_{e\in \EE_i^{\aa}}2\kappa_{e}^{\aa}\cdot\I+\sum\limits_{e\in \EE_{i-}^{\aa}}2\tau_{e}^{\aa}\cdot\nt_{e}^{\aa}\nt{\vphantom{t}_{e}^{\aa}}^\transpose+\\
\sum\limits_{\substack{\beta\in \NN_i^\alpha}}\big(\sum\limits_{e\in \EE_i^{\ab}}2\kappa_e^{\abk}\cdot\I+
\sum\limits_{e\in \EE_{i-}^{\ab}}\;2\tau_{e}^{\abk}\cdot\nt_{e}^{\ab}\nt{\vphantom{t}_{e}^{\ab}}^\transpose\big),
\end{multline}
in which $\kappa_{ij}^{\abk}$ and $\tau_{ij}^{\abk}$ are given in \cref{eq::kappak,eq::tauk}, respectively. 

Similar to $\nGamma^{\ak}\in \R^{(d+1)n_\alpha\times (d+1)n_\alpha}$ in \cref{eq::nGamma},  $\lnGamma^{\ak}\in \R^{(d+1)n_\alpha\times (d+1)n_\alpha}$ in \cref{eq::lGMa} also takes the form as
\begin{equation}\label{eq::lnGamma}
\lnGamma^{\ak}=\begin{bmatrix}
	\lnGamma^{\tau,\ak}\hphantom{{}^\transpose} & \lnGamma^{\nu,\ak}\\
	\lnGamma{\vphantom{H}^{\nu,\ak}}^\transpose & \lnGamma^{\kappa,\ak}
\end{bmatrix},
\end{equation}
in which $\lnGamma^{\tau,\ak}\in \R^{n_\alpha\times n_\alpha}$, $\lnGamma^{\nu,\ak}\in\R^{n_\alpha\times dn_\alpha}$ and $\lnGamma^{\kappa,\ak}\in\R^{dn_\alpha\times dn_\alpha}$ are sparse matrices.
Following \cref{eq::lGMa,eq::lGMsum,eq::lGMai}, it is straightforward to show that
\begin{equation}
	\nonumber
	\frac{1}{2}\|\Xa-\Xak\|_{\lnGamma^{\ak}}^2 = \sum_{i=1}^{n_\alpha} \frac{1}{2}\|\Xa_i-\Xak_i\|_{\lnGamma_i^{\ak}}^2.
\end{equation}
From the equation above, $\lnGamma^{\tau,\ak}\in \R^{n_\alpha\times n_\alpha}$, $\lnGamma^{\nu,\ak}\in\R^{n_\alpha\times dn_\alpha}$ and $\lnGamma^{\kappa,\ak}\in\R^{dn_\alpha\times dn_\alpha}$ in \cref{eq::lnGamma} are defined as
\begin{equation}
\nonumber
\Big[\lnGamma^{\tau,\ak}\Big]_{ij}=\begin{cases}
	\lnGamma_i^{\tau,\ak}, & i=j,\\
	0,& \text{otherwise,}
\end{cases}
\end{equation}
\begin{equation}
\nonumber
\Big[\nGamma^{\nu,\ak}\Big]_{ij}=\begin{cases}
	\lnGamma_i^{\nu,\ak}, & i=j,\\
	0,& \text{otherwise,}
\end{cases}
\end{equation}
\begin{equation}
\nonumber
\Big[\nGamma^{\kappa,\ak}\Big]_{ij}=\begin{cases}
	\lnGamma_i^{\kappa,\ak}, & i=j,\\
	0,& \text{otherwise},
\end{cases}
\end{equation}
in which $\lnGamma_i^{\tau,\ak}\in\R$, $\lnGamma_i^{\nu,\ak}\in\R^{1\times d}$ and $\lnGamma_i^{\kappa,\ak}\in\R^{d\times d}$ are given in \cref{eq::lnGamma_ti,eq::lnGamma_vi,eq::lnGamma_ri}.

%% file: appendix_sol.tex
Substituting \cref{eq::lnGammaki} into \cref{eq::xlGi} and rewriting the resulting equation in terms of $\Xa_i=\begin{bmatrix}
t_i^\alpha & R_i^\alpha
\end{bmatrix}\in\R^{d\times(d+1)}$ leads to
\begin{multline}
\nonumber
\big(t_i^{\akh},\,R_i^{\akh}\big)=\\
\arg\min_{\substack{t_i^\alpha\in\R^d,\, R_i^\alpha\in SO(d)}}\; \frac{1}{2} \|\tai-\taik\|_{\lnGammataik}^2 +\\
\lnGammanaik(\Rai-\Raik)^\transpose(\tai-\taik) +\\
\frac{1}{2}\|\Rai-\Raik\|_{\lnGammaRaik}^2+\\
\innprod{\nabla_{\tai}F(\Xk)}{\tai-\taik} +\\ \innprod{\nabla_{\Rai}F(\Xk)}{\Rai-\Raik}.
\end{multline}
For notational simplicity, the equation above is simplified to
\begin{multline}\label{eq::opti}
\min_{\substack{t_i^\alpha\in\R^d,\, R_i^\alpha\in SO(d)}}\; \frac{1}{2} \|\tai\|_{\lnGammataik}^2 +
\lnGammanaik{\Rai}^\transpose\tai +\\
\frac{1}{2}\|\Rai\|_{\lnGammaRaik}^2+  \innprod{\gamma_{i}^{\tau,\ak}}{\tai} + \innprod{\gamma_{i}^{\kappa,\ak}}{\Rai},
\end{multline}
in which
\begin{equation}
\gamma_{i}^{\tau,\ak}=\nabla_{\tai}F(\Xk)-\taik\lnGammataik\in\R^d
\end{equation}
and
\begin{equation}
\gamma_{i}^{\kappa,\ak}=\nabla_{\Rai}F(\Xk)-\Raik\lnGammaRaik\in\R^{\d\times d}.
\end{equation}
Recalling from \cref{eq::lnGammaki} that $\lnGammataik\in\R$ and $\lnGammataik>0$,  we obtain that
\begin{equation}\label{eq::tsol}
\tai = -\Rai{\lnGammanaik}^\transpose{\lnGammataik}^{-1}-\gamma_{i}^{\tau,\ak}{\lnGammataik}^{-1}
\end{equation}
minimizes \cref{eq::opti} if $\Rai\in SO(d)$ is given. Then, substituting \cref{eq::tsol} into \cref{eq::opti} yields
\begin{equation}\label{eq::opti2}
R_i^{\akh}=\arg\min_{\Rai\in SO(d)} \frac{1}{2}\|\Rai\|_{\Xi_i^{\ak}}^2 - \innprod{\upsilon_i^{\ak}}{\Rai},
\end{equation}
in which
\begin{equation}
\Xi_i^{\ak}=\lnGammaRaik-{\lnGammanaik}^\transpose{\lnGammataik}^{-1}\lnGammanaik\in\R^{d\times d}
\end{equation}
and
\begin{equation}
\upsilon_i^{\ak} = \gamma_{i}^{\tau,\ak}{\lnGammataik}^{-1}\lnGammanaik - \gamma_{i}^{\kappa,\ak} \in\R^{d\times d}.
\end{equation}
If we apply ${\Rai}^\transpose\Rai=\I$ on \cref{eq::opti2}, then \cref{eq::xlGi} is equivalent to
\begin{equation}\label{eq::opti4}
R_i^{\akh} = \arg\max_{\Rai\in SO(d)} \innprod{\upsilon_i^{\ak}}{\Rai}.
\end{equation} 
Thus, \cref{eq::xlGi} is reduced to an optimization problem on $SO(d)$, which has a closed-form solution as follows.

Following \cite{umeyama1991least}, if $\upsilon_i^{\ak}\in \R^{d\times d}$ admits a singular value decomposition $\upsilon_i^{\ak}=U_i^{\ak}\mathrm{\Sigma}_i^{\ak} {V_i^{\ak}}^\transpose$ in which $U_i^{\ak}$ and $V_i^{\ak}\in O(d)$ are orthogonal (not necessarily special orthogonal) matrices, and $\mathrm{\Sigma}_i^{\ak}=\diag\{\sigma_1^{\ak},\,\sigma_2^{\ak},\,\cdots,\,\sigma_d^{\ak}\}\in \R^{d\times d}$ is a diagonal matrix, and $\sigma_1^{\ak}\geq\sigma_2^{\ak}\geq\cdots\geq \sigma_d^{\ak}\geq 0$ are singular values of $\upsilon_i^{\ak}$, then the optimal solution to \cref{eq::opti4} is
\begin{equation}\label{eq::Rsol}
	R_i^{\akh}=\begin{cases}
		U_i^{\ak}\mathrm{\Lambda}^+ {V_i^{\ak}}^\transpose, & \det\big(U_i^{\ak}{V_i^{\ak}}^\transpose\big) > 0,\\[0.3em]
		U_i^{\ak}\mathrm{\Lambda}^- {V_i^{\ak}}^\transpose, & \text{otherwise},
	\end{cases}
	\vspace{-0.25em}
\end{equation}
in which $\mathrm{\Lambda}^+=\diag\{1,\,1,\,\cdots,\,1\}\in\R^{d\times d}$ and $\mathrm{\Lambda}^-=\diag\{1,\,1,\,\cdots,\,-1\}\in\R^{d\times d}$. If $d = 2$, the equation above is equivalent to the polar decomposition of $2\times 2$ matrices, and if $d=3$, there are fast algorithms for singular value decomposition of $3\times 3$ matrices \cite{mcadams2011computing}. As a result, \cref{eq::opti4} can be efficiently solved in the case of $SO(2)$ and $SO(3)$, both of which are commonly used in SLAM.

As long as $R_i^{\akh}\in SO(d)$ is known, $t_i^{\akh}\in\R^d$ can be exactly recovered using \cref{eq::tsol}:
\begin{multline}
t_i^{\akh} = -R_i^{\akh}{\lnGammanaik}^\transpose{\lnGammataik}^{-1}-\\ \gamma_{i}^{\tau,\ak}{\lnGammataik}^{-1}.
\end{multline}
Therefore, $X_i^{\akh}=\begin{bmatrix}
t_i^{\akh} & R_i^{\akh}
\end{bmatrix}\in \R^{d\times(d+1)}$ is exactly solved, whose computation only involves matrix multiplication and singular value decomposition.

%% file: table_dataset.tex
\begin{table}[H]
	\centering
	\setlength{\tabcolsep}{0.3em}
	\renewcommand{\arraystretch}{1.2}
	\caption{2D and 3D SLAM benchmark datasets.}\label{table::dataset}
	\begin{tabular}{|P{0.07\textwidth}|P{0.055\textwidth}|P{0.065\textwidth}|P{0.11\textwidth}|P{0.1\textwidth}|}
		\hline			
		\multirow{2}{*}{Dataset} & \multirow{2}{*}{2D/3D} & \multirow{2}{*}{\# Poses} & \multirow{2}{*}{\# Measurements}  & Real-World Dataset \\
		\hline
		\hline
		{\sf ais2klinik} & 2D &15115 & 16727 & \cmark\\
		\hline
		{\sf city} & 2D & 10000 & 20687 & \xmark\\
		\hline
		{\sf CSAIL} & 2D & 1045 & 1172 & \cmark \\
		\hline
		{\sf M3500} & 2D & 3500 & 5453 & \xmark \\
		\hline
		{\sf intel} & 2D & 1728 & 2512 & \cmark\\
		\hline
		{\sf MITb} & 2D & 808 & 827 & \cmark \\
		\hline
		{\sf sphere} & 3D & 2500 & 4949 & \xmark\\
		\hline
		{\sf torus} & 3D & 5000 & 9048 & \xmark\\
		\hline
		{\sf grid} & 3D & 8000 & 22236 & \xmark\\
		\hline
		{\sf garage} & 3D & 1661 & 6275 & \cmark\\
		\hline
		{\sf cubicle} & 3D & 5750 & 16869 & \cmark\\
		\hline
		{\sf rim} & 3D & 10195 & 29743 & \cmark\\
		\hline
	\end{tabular}
\vspace{-1em}
\end{table}

%% file: intro.tex
Pose graph optimization (PGO) is a nonlinear and nonconvex optimization problem estimating unknown poses from noisy relative pose measurements. PGO associates each pose  with a vertex and each relative pose measurement with an edge such that the optimization problem is well represented through a graph. PGO has important applications in a number of areas, including but not limited to robotics \cite{cadena2016past,rosen4advances,thrun2005probabilistic}, autonomous driving \cite{geiger2012we}, and computational biology \cite{singer2011angular,singer2011three}. Recent advances \cite{rosen2014rise,rosen2016se,rosen2020scalable,dellaert2012factor,fan2019cpl,fan2020cpl,carlone2016planar,carlone2015lagrangian,grisetti2009nonlinear,briales2017cartan} suggest that PGO can be well solved using iterative optimization. Nevertheless, the aforementioned techniques \cite{rosen2014rise,rosen2016se,rosen2020scalable,dellaert2012factor,carlone2016planar,carlone2015lagrangian,grisetti2009nonlinear,fan2019cpl,briales2017cartan,fan2020cpl}  are difficult to distribute across a network due to communication and computational limitations, and  are only applicable to small- and medium-sized problems with at most tens of thousands poses.  In addition, their centralized pipelines are equivalent to using a master node to aggregate information from the entire network, and thus, fail to meet  privacy requirements one may wish to impose \cite{li2019coordinated,zhang2019complete}.

In multi-robot simultaneous localization and mapping (SLAM)  \cite{cunningham2010ddf,aragues2011multi,cunningham2013ddf,saeedi2016multiple,dong2015distributed,lajoie2020door,cieslewski2018data,tchuiev2020distributed,chang2020kimera,tian2021kimera}, each robot estimates not only its own poses but those of the others as well to build an environment map. Even though such a problem can be solved by PGO,  communication between  robots is restricted and multi-robot SLAM has more unknown poses than single-robot SLAM. Thus, instead of using centralized PGO \cite{rosen2014rise,rosen2016se,rosen2020scalable,dellaert2012factor,carlone2016planar,fan2020cpl,fan2019cpl,carlone2015lagrangian,grisetti2009nonlinear,briales2017cartan}, it is more reasonable to formulate this large-sized estimation problem involving multiple robots as distributed PGO---each robot in multi-robot SLAM is represented as a node and two nodes (robots) are said to be neighbors if there exists a noisy relative pose measurement between them (a more detailed description of distributed PGO can be found in \cref{section::problem}). In most cases, it is assumed that inter-node communication only occurs between neighboring nodes and most of these iterative optimization methods are infeasible  due to the expensive communication cost of solving linear system and performing line search \cite{rosen2014rise,rosen2016se,rosen2020scalable,dellaert2012factor,carlone2016planar,carlone2015lagrangian,grisetti2009nonlinear,fan2019cpl,briales2017cartan,fan2020cpl}, which renders distributed PGO  more challenging than centralized PGO.

In this paper, we propose majorization minimization (MM) methods \cite{hunter2004tutorial,sun2016majorization} for distributed PGO. As the name would suggest, MM methods have two steps. First, in the majorization step, we construct a surrogate function that majorizes the objective function, i.e., the surrogate function is greater to the objective function except for the current iterate where both of them attain the same value. Then, in the minimization step, we minimize the surrogate function instead of the original objective function to improve the iterate. MM methods remain difficult, albeit straightforward, in practical use, e.g., the surrogate function, whose construction and minimization can not be more difficult than solving the optimization problem itself, is usually unknown, and MM methods might fail to converge and suffer from slow convergence. Thus, the implementation of MM methods on large-scale, complicated and nonconvex optimization problems like distributed PGO is nontrivial, and inter-node communication requirements impose extra restrictions making it more so. All of these issues  are  addressed both theoretically and empirically in the rest of this paper.

This paper extends the preliminary results in \cite{fan2019proximal,fan2020mm}, where we developed MM methods for centralized and distributed PGO that are guaranteed to converge to first-order critical points. In \cite{fan2019proximal,fan2020mm}, we also introduced and elaborated on the use of Nesterov's method \cite{nesterov1983method,nesterov2013introductory} and adaptive restart \cite{o2015adaptive} for the first time to accelerate the convergence of PGO. Beyond the initial results in \cite{fan2019proximal,fan2020mm}, this paper presents completely redesigned MM methods for distributed PGO and provides more comprehensive theoretical and  empirical results. In particular, our MM methods in this paper are capable of handling a broad class of robust loss kernels, no longer require each iteration to attain a local optimal solution to the surrogate function for the convergence guarantees, and adopt a novel adaptive restart scheme for distributed PGO without a master node to make full use of Nesterov's acceleration.

In summary, the contributions of this paper are the follows:
\begin{enumerate}
\item We derive a class of surrogate functions that suit well with MM methods for distributed PGO. These surrogate functions apply to a broad class of robust loss kernels in robotics and computer vision.
\item We develop MM methods for distributed PGO that are guaranteed to converge to first-order critical points under mild conditions. Our MM methods for distributed PGO  implement a novel update rule such that each iteration does not have to minimize the surrogate function to a local optimal solution.
\item We leverage Nesterov's methods and adaptive restart to accelerate MM methods for distributed PGO and achieve significant improvement in convergence without any compromise of theoretical guarantees.
\item We present a decentralized adaptive restart scheme to make full use of Nesterov's acceleration such that accelerated MM methods for distributed PGO without a master node are almost as fast as those requiring a master node.
\end{enumerate}

The rest of this paper is organized as follows. \cref{section::work} reviews the state-of-the-art methods for distributed PGO. \cref{section::notation} introduces mathematical notation and preliminaries that are used in this paper. \cref{section::problem} formulates the problem of distributed PGO. \cref{section::loss,section::major_pgo} present surrogate functions for individual loss terms and the overall distributed PGO, respectively, which are fundamental to our MM methods. \cref{section::mm,section::amm,section::ammc} present unaccelerated and accelerated MM methods for distributed PGO that are guaranteed to converge to first-order critical points, which are the major contributions of this paper. \cref{section::experiemnt} implements our MM methods for distributed PGO on a number of simulated and real-world SLAM datasets and make extensive comparisons against existing state-of-the-art methods \cite{choudhary2017distributed,tian2019distributed}. \cref{section::conclusion} concludes this paper and discusses future work.

%% file: related_work.tex
In the last decade, multi-robot SLAM has been becoming increasingly popular, which promotes the development of distributed PGO \cite{choudhary2017distributed,tian2019distributed,tron2014distributed,eric2020geod}.

Choudhary \ea \cite{choudhary2017distributed} present a two-stage algorithm that implements either Jacobi Over-Relaxation or  Successive Over-Relaxation as distributed linear system solvers. Similar to centralized methods, \cite{choudhary2017distributed} first evaluates the chordal initialization \cite{carlone2015initialization} and then improves the initial guess with a single Gauss-Newton step. However, one step of Gauss-Newton method in most cases can not lead to sufficient convergence for distributed PGO. In addition, no line search is performed in \cite{choudhary2017distributed} due to the communication limitation, and thus, the behaviors of the single Gauss-Newton step is totally unpredictable and might result in bad solutions.

Tian \ea \cite{tian2019distributed} present the distributed certifiably correct PGO using Riemannian block coordinate descent method, which is later generalized to asynchronous and parallel distributed PGO \cite{tian2020asynchronous}. Specially, their method makes use of Riemannian staircase optimization to solve the semidefinite relaxation of distributed PGO and is guaranteed to converge to global optimal solutions under moderate measurement noise. Following our previous works \cite{fan2020mm,fan2019proximal}, they implement Nesterov's method for acceleration as well. Contrary to our MM methods, a major drawback of \cite{tian2019distributed} is that their method has to precompute red-black coloring assignment for block aggregation and keep part of the blocks in idle for estimate updates. In addition, although several strategies for block selection (e.g., greedy/importance sampling) and Nesterov's acceleration (e.g., adaptive/fixed restarts) are adopted in \cite{tian2019distributed} to improve the convergence, most of them are either inapplicable without a master node or at the sacrifice of computational efficiency and theoretical guarantees. In contrast, our MM methods are much faster (see Section~\ref{section::experiemnt}) but have no such restrictions for acceleration. More recently, Tian \ea  further apply Riemannian block coordinate descent method to distributed PGO with robust loss kernels \cite{tian2021kimera}. However, they solve robust distributed PGO by trivially updating the weights using graduated nonconvexity \cite{yang2020graduated} and no formal proofs of convergence are provided. Again, { this is  in contrast to the work} presented here that has provable convergence to first-order critical points for a broad class of robust loss kernels.

Tron and Vidal \cite{tron2014distributed} present a consensus-based method for distributed PGO using Riemannian gradient. The authors derive a condition for convergence guarantees related with the stepsize of the method and the degree of the pose graph. Nonetheless, their method estimates rotation and translation separately, fails to handle robust loss kernels, and needs extra computation to find the convergence-guaranteed stepsize. 

Cristofalo \ea \cite{eric2020geod} present a novel distributed PGO method using  Lyapunov theory and multi-agent consensus. Their method is guaranteed to converge if the pose graph has certain topological structures. However, \cite{eric2020geod} updates rotations without exploiting the translational measurements and only applies to pairwise consistent PGO with nonrobust loss kernels.

In comparison to these aforementioned techniques, our MM methods have the mildest conditions (not requiring any specific pose graph structures, any extra computation for preprocessing, any master nodes for information aggregation, etc.) to converge to first-order critical points, apply to a broad class of robust loss kernels in robotics and computer vision, and manage to implement decentralized acceleration with convergence guarantees. Most importantly, as is shown in \cref{section::experiemnt}, our MM methods outperform existing state-of-the-art methods in terms of both efficiency and accuracy on a variety of SLAM benchmark datasets.

%% file: notation.tex
\textbf{Miscellaneous Sets.}\; $\R$ denotes the sets of real numbers; $\R^+$ denotes the sets of nonnegative real numbers; $\R^{m\times n}$ and $\R^n$ denote the sets of $m\times n$ matrices and $n\times 1$ vectors, respectively. $SO(d)$ denotes the set of special orthogonal groups and  $SE(d)$ denotes the set of special Euclidean groups.  $|\cdot|$ denotes the cardinality of a set.

\textbf{Matrices.}\; For a matrix $X\in \R^{m\times n}$, $[X]_{ij}$ denotes the $(i,\,j)$-th entry or $(i,\,j)$-th block of $X$, and $[X]_i$ denotes the $i$-th entry or $i$-th block of $X$. For symmetric matrices $X,\, Y\in \R^{n\times n}$, $X\succeq Y$ (or $Y\preceq X$) and $X\succ Y$ (or $Y\prec X$) mean that $X-Y$ is positive (or negative) semidefinite and definite, respectively.

\textbf{Inner Products and Norms.}\; For a matrix $M\in\R^{n\times n}$, $\innprod{\cdot}{\cdot}_{M}:\R^{m\times n}\times \R^{m\times n}\rightarrow \R$ denotes the function
\vspace{-0.25em}
\begin{equation}\label{eq::innerM}
	\innprod{X}{Y}_M\triangleq \trace(XMY^\transpose)
\end{equation}
where $X,\,Y\in \R^{m\times n}$. If $M$ is the identity matrix, $\innprod{\cdot}{\cdot}_M$ is also represented as $\innprod{\cdot}{\cdot}:\R^{m\times n}\times\R^{m\times n}\rightarrow \R$ such that
\vspace{-0.15em}
\begin{equation}\label{eq::inner}
\innprod{X}{Y}\triangleq\trace(XY^\transpose).
\vspace{-0.1em}
\end{equation}
For a positive semidefinite matrix $M\in\R^{n\times n}$, $\|\cdot\|_M:\R^{m\times n}\rightarrow\R^+$ denotes the function
\vspace{-0.15em}
\begin{equation}\label{eq::normM}
\|X\|_M\triangleq\sqrt{\trace(X M X^\transpose)}
\vspace{-0.1em}
\end{equation}
where $X\in \R^{m\times n}$.  Also, $\|\cdot\|$ denotes the Frobenius norm of matrices and vectors, and $\|\cdot\|_2$ denotes the induced $2$-norms of matrices and linear operators.  

\textbf{Riemannian Geometry.}\; If $F(\cdot):\R^{m\times n}\rightarrow\R $ is a function, $\mathcal{M}\subset \R^{m\times n}$ is a Riemannian manifold and $X\in \mathcal{M}$, then {\highlight $\nabla F(X)$ and $\mathrm{grad}\, F(X)$ denote the Euclidean and Riemannian gradients, respectively. }
\vspace{0.15em}

\text{\textbf{Graph Theory.}}\quad  PGO is represented as a directed graph $\aGG=(\VV,\,\aEE)$ where $\VV$ and $\EE$ are the sets of vertices and edges, respectively \cite{rosen2016se}. In distributed PGO, each vertex is described as  an ordered pair $(\alpha,\,i)\in\VV$ where $\alpha$ is the node index and $i$ the local index of the vertex within node $\alpha$. For any  nodes $\alpha$ and $\beta$ in distributed PGO, $\aEE^{\ab}$ denotes the set of edges between nodes $\alpha$ and $\beta$: 
\vspace{-0.15em}
\begin{equation}\label{eq::aEE}
	\aEE^{\ab}\triangleq\{(i,\,j)|((\alpha,\,i),\,(\beta,\,j))\in\aEE\};
\vspace{-0.15em}
\end{equation} 
and $\NN_-^\alpha$  denotes the set of nodes with edges from node $\alpha$:
\vspace{-0.15em}
\begin{equation}\label{eq::NN-}
	\NN_-^\alpha\triangleq\{\beta|\aEE^{\ab}\neq \emptyset\text{ and }\alpha\neq\beta\};
\vspace{-0.15em}
\end{equation} 
and $\NN_+^\alpha$ denotes  the set of nodes with edges to node $\alpha$:
\vspace{-0.15em}
\begin{equation}\label{eq::NN+}
	\NN_+^\alpha\triangleq\{\beta|\aEE^{\ba}\neq \emptyset\text{ and }\alpha\neq\beta\};
\vspace{-0.15em}
\end{equation} 
and $\NN^\alpha$ denotes the set of nodes with edges  from or to node $\alpha$:
\vspace{-0.15em}
\begin{equation}\label{eq::NN}
\!	\NN^\alpha\!\triangleq \NN_-^\alpha \cup \NN_+^\alpha \triangleq \{\beta|\aEE^{\ab}\!\neq\! \emptyset\text{ or }\! \aEE^{\ba}\!\neq\! \emptyset\text{ and }\alpha \!\neq\! \beta\}. 
\vspace{-0.1em}
\end{equation}

\textbf{Optimization.}\quad For optimization variables $X$, $X^\alpha$, $R^\alpha$, $t^\alpha$, etc., the notation $\Xk$, $X^{\ak}$, $R^{\ak}$, $t^{\ak}$, etc. denotes the $\sk$-th iterate of corresponding optimization variables.

%% file: problem.tex
\subsection{Distributed Pose Graph Optimization}
In distributed PGO \cite{tron2014distributed,choudhary2017distributed,tian2019distributed}, we are given $|\AA|$ nodes $\AA\triangleq\{1,\,2,\,\cdots,\, |\AA|\}$ and each node $\alpha\in\AA$ has $n_\alpha$ poses $g_{1}^\alpha$, $g_{2}^\alpha$, $\cdots$, $g_{n_\alpha}^\alpha\in SE(d)$. Let $g_{(\cdot)}^\alpha\triangleq(t_{(\cdot)}^\alpha,\,R_{(\cdot)}^\alpha)$ where $t_{(\cdot)}^\alpha\in\R^d$ is the translation and $R_{(\cdot)}^\alpha\in SO(d)$ the rotation.  We consider the problem of estimating unknown poses $g_{1}^\alpha$, $g_{2}^\alpha$, $\cdots$, $g_{n_\alpha}^\alpha\in SE(d)$ for all the nodes $\alpha\in\AA$ given intra-node noisy measurements $\tilde{g}_{ij}^{\alpha\alpha}\triangleq(\nt_{ij}^{\alpha\alpha},\,\nR_{ij}^{\alpha\alpha})\in SE(d)$ of the relative pose
\vspace{-0.25em}
\begin{equation}\label{eq::gijaa}
g_{ij}^{\alpha\alpha} \triangleq \big({g_i^\alpha}\big)^{-1} g_j^\alpha\in SE(d)
\end{equation}
within a single node $\alpha$, and inter-node noisy measurements $\tilde{g}_{ij}^{\alpha\beta}\triangleq(\nt_{ij}^{\alpha\beta},\,\nR_{ij}^{\alpha\beta})\in SE(d)$ of the relative pose
\begin{equation}\label{eq::gijab}
g_{ij}^{\alpha\beta} \triangleq \big({g_i^\alpha}\big)^{-1} g_j^\beta\in SE(d)
\end{equation}
between different nodes $\alpha\neq\beta$.  In \cref{eq::gijaa,eq::gijab}, note that $\nt_{ij}^{\alpha\alpha}$ and $\nt_{ij}^{\alpha\beta}\in\R^d$ are translational measurements, and $\nR_{ij}^{\alpha\alpha}$ and $\nR_{ij}^{\alpha\beta}\in SO(d)$ are rotational measurements.


Following \cref{eq::aEE,eq::NN-,eq::NN+,eq::NN}, we represent distributed PGO as  a directed graph  $\aGG=(\VV,\,\aEE)$ such that  unknown pose $g_i^\alpha\in SE(d)$ and noisy measurement $\tilde{g}_{ij}^{\ab}\in SE(d)$ have one-to-one correspondence to vertex $(\alpha,\,i)\in\VV$   and  directed edge $((\alpha,\,i),\,(\beta,\,j))\in \aEE$, respectively.  We refer nodes $\alpha$ and $\beta\in\AA$ as neighbors as long as either $\aEE^{\ab}\neq\emptyset$ or $\aEE^{\ba}\neq\emptyset$. Then,  $\NN_-^\alpha$ and $\NN_+^\alpha$ are the sets of neighbors with a directed edge from and to node $\alpha$, respectively, and $\NN^\alpha$ is the set of neighbors with a directed edge connected to node $\alpha$.

In the rest of this paper, we make the following assumption that each node can communicate with its neighbors and the network topology is unchanged during optimization. These assumptions  are common in distributed PGO \cite{choudhary2017distributed,tian2019distributed,tron2014distributed,eric2020geod}.
\begin{assumption}\label{assumption::neighbor}
	Each node $\alpha$ can communicate with its neighbors $\beta\in\NN^\alpha$ and the network topology is fixed.
\end{assumption}
\vspace{-1.25em}

\subsection{Loss Kernels}
In practice, it is inevitable that there exist inter-node measurements that are outliers resulting from false loop closures. These outliers adversely affect the overall  performance of distributed PGO. To address this issue, it is popular to use non-trivial loss kernels---e.g., Huber and Welsch losses---to enhance the robustness of distributed PGO \cite{agarwal2013robust,carlone2018convex,barron2019cvpr}.

In this paper, we make the following assumption that applies to a broad class of loss kernels $\rho(\cdot):\R^+\rightarrow\R$ in robotics and computer vision.
\vspace{-0.25em}
\begin{assumption}\label{assumption::loss}
	The loss kernel $\rho(\cdot):\R^+\rightarrow\R$ satisfies the following properties:
	\begin{enumerate}[(a)]
		\item $\rho(s)\geq 0$ for any $s\in\R^+$ and the equality ``$=$'' holds if and only if $s=0$;
		\item $\rho(\cdot):\R^+\rightarrow\R$ is continuously differentiable;\label{assumption::loss_cont}
		\item $\rho(\cdot):\R^+\rightarrow \R$ is a concave function;\label{assumption::loss_mono}
		\item $0\leq\nabla\rho(s)\leq 1$ for any $s\in\R^+$ and $\nabla\rho(0)=1$;\label{assumption::loss_drho}
		\item $\varphi(\cdot):\R^{m\times n}\rightarrow \R$ with $\varphi(X)\triangleq\rho(\|X\|^2)$ has Lipschitz continuous gradient, i.e., there exists $\mu>0$ such that $\|\nabla\varphi(X)-\nabla\varphi(X')\|\leq \mu\cdot\|X-X'\|$ for any $X,\,X'\in\R^{m\times n}$. \label{assumption::loss_L}
	\end{enumerate}
\end{assumption}

In the following, we present some examples of loss kernels (see \cref{fig::robust_loss}) satisfying \cref{assumption::loss}.

\begin{example}[Trivial Loss]\label{example::trivial}
	\begin{equation}\label{eq::trivial}
		\rho(s) =s.
	\end{equation}
\end{example}
\begin{example}[Huber Loss]\label{example::huber}
	\begin{equation}\label{eq::huber}
		\rho(s) =\begin{cases}
			s, & |s|\leq a,\\
			2\sqrt{a|s|}-a, & |s|\geq a
		\end{cases}
	\end{equation}
	where $a>0$.
\end{example}
\begin{example}[Welsch Loss]\label{example::GM}
	\begin{equation}\label{eq::geman}
		\rho(s) = a - a\exp\left(-\frac{s}{a}\right)
	\end{equation}
	where $a>0$.
\end{example}

\begin{figure}
	\centering
	\includegraphics[width=0.3\textwidth]{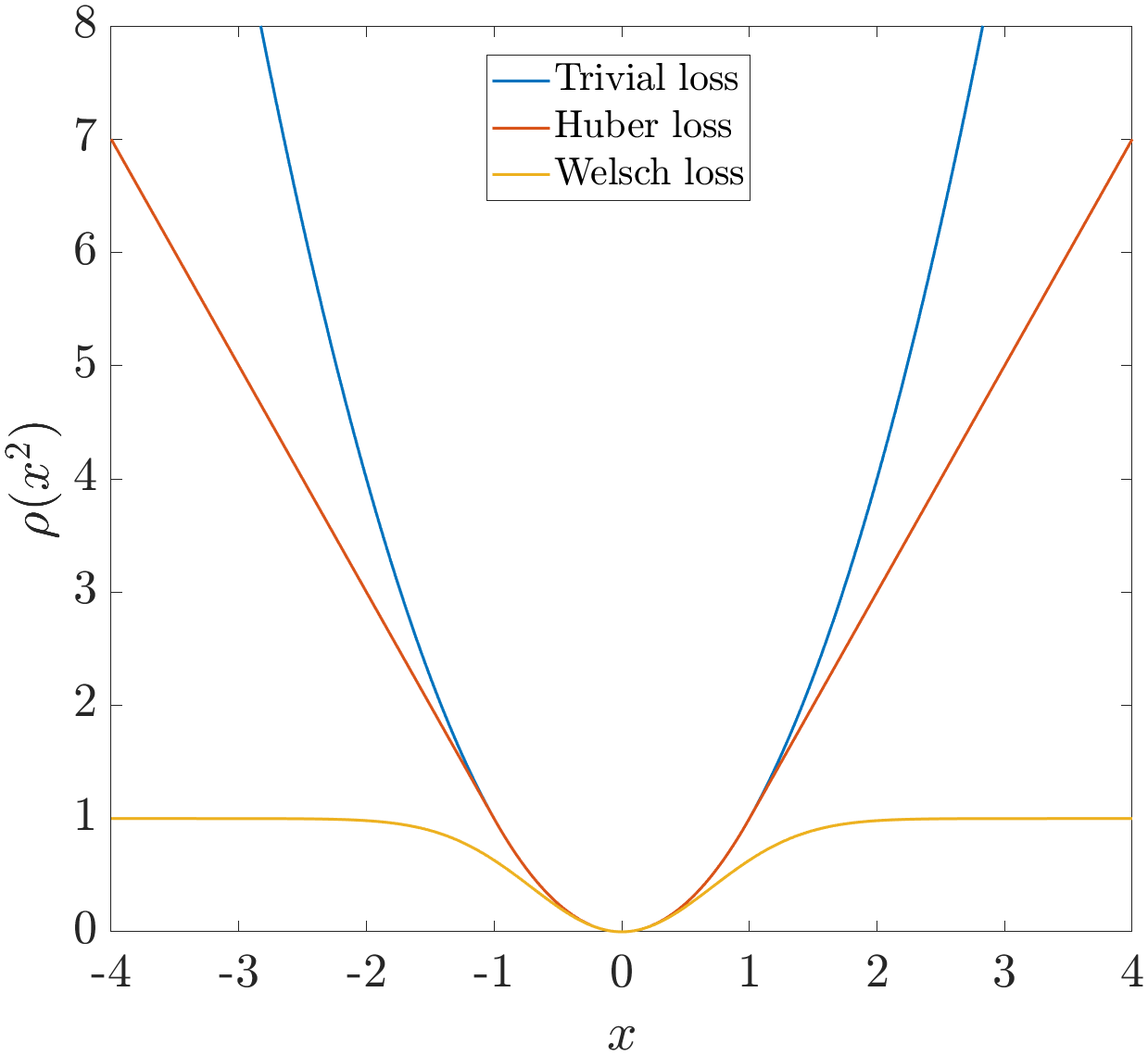}
	\caption{$\rho(x^2)$ for the trivial, Huber, Welsch loss kernels.}\label{fig::robust_loss}
	\vspace{-1em}
\end{figure}

\vspace{-0.5em}
\subsection{Objective Function}
Recall that each node $\alpha\in\AA$ has $n_\alpha$ unknown poses $g_{1}^\alpha$, $g_{2}^\alpha$, $\cdots$, $g_{n_\alpha}^\alpha\in SE(d)$. For notational simplicity, we define $\XXa$ and $\XX$ as
\begin{equation}
\nonumber
\XXa\triangleq \R^{d\times n_\alpha}\times SO(d)^{n_\alpha}
\end{equation}
and
\begin{equation}
	\nonumber
	\XX \triangleq  \XX^1\times\cdots\times \XX^{|\AA|} \subset \R^{d\times (d+1)n},
\end{equation}
respectively, where $n\triangleq\sum_{\alpha\in\AA} n_\alpha$. Furthermore, we represent $g_i^\alpha\in SE(d)$, i.e., the $i$-th pose of node $\alpha\in\AA$, as a $d\times(d+1)$ matrix
\begin{equation}\label{eq::Xai}
	X_i^\alpha \triangleq \begin{bmatrix}
		t_i^\alpha & R_i^\alpha
	\end{bmatrix}\in SE(d)\subset \R^{d\times(d+1)},
\end{equation}
represent $(g_{1}^\alpha$, $g_{2}^\alpha$, $\cdots$, $g_{n_\alpha}^\alpha)\in SE(d)^{n_\alpha}$, i.e., all the poses of node $\alpha\in\AA$, as an element of $\XXa$ as well as a $d\times (d+1)n_\alpha$ matrix
\begin{equation}\label{eq::Xa}
	X^\alpha\triangleq \begin{bmatrix}
		t^\alpha & R^\alpha
	\end{bmatrix}\in \XXa \subset \R^{d\times (d+1)n_\alpha},
\end{equation}
where
$$
\vphantom{\Big\{}t^\alpha\triangleq\begin{bmatrix}
	t_1^\alpha & \cdots & t_{n_\alpha}^\alpha
\end{bmatrix}\in \R^{d\times n_\alpha}
$$
and
$$
R^\alpha\triangleq\begin{bmatrix}
	R_1^\alpha & \cdots & R_{n_\alpha}^\alpha
\end{bmatrix}\in SO(d)^{n_\alpha}\subset\R^{d\times dn_\alpha},
$$
and represent $\{(g_{1}^\alpha$, $g_{2}^\alpha$, $\cdots$, $g_{n_\alpha}^\alpha)\}_{\alpha\in\AA}\in SE(d)^{n}$, i.e., all the poses of distributed PGO, as an element of $\XX$ as well as a $d\times(d+1)n$ matrix
\begin{equation}\label{eq::Xall}
	X\triangleq\begin{bmatrix}
		X^1 & \cdots & X^{|\AA|}
	\end{bmatrix}\in \XX\subset \R^{d\times(d+1)n}.
\end{equation}

\begin{remark}
	$\XXa$ and $\XX$ are by definition homeomorphic to $SE(d)^{n_\alpha}$ and $SE(d)^n$, respectively. Thus, $\Xa\in\XXa$ and $X\in\XX$ are sufficient to represent elements of $SE(d)^{n_\alpha}$ and $SE(d)^n$.
\end{remark}

Following \cite{rosen2016se,fan2020mm,fan2019proximal}, distributed PGO can be formulated as an optimization problem on $X=\begin{bmatrix}
	X^1 & \cdots & X^{|\AA|}
\end{bmatrix}\in \XX$:
\begin{problem}[Distributed Pose Graph Optimization]
\begin{equation}\label{eq::pgo}
	\min_{X\in \XX} F(X).
\end{equation}
The objective function $F(X)$ in \cref{eq::pgo} is defined as
\begin{multline}\label{eq::obj}
F(X)\triangleq \sum_{\alpha\in\AA}\sum_{(i,j)\in \aEE^{\alpha\alpha}}\frac{1}{2}\Big[\kappa_{ij}^{\alpha\alpha}\|R_i^\alpha \nR_{ij}^{\alpha\alpha} -R_j^\alpha\|^2 +\\ \tau_{ij}^{\alpha\alpha}\|R_i^\alpha \nt_{ij}^{\alpha\alpha}+t_i^\alpha - t_j^\alpha\|^2\Big]+\\
\sum_{\substack{\alpha,\beta\in\AA,\\
		\alpha\neq \beta}}\sum_{(i,j)\in \aEE^{\alpha\beta}}\frac{1}{2}\Big[\rho\left(\kappa_{ij}^{\alpha\beta}\|R_i^\alpha \nR_{ij}^{\alpha\beta} -R_j^\beta\|^2\right. +\\ 
\left.\tau_{ij}^{\alpha\beta}\|R_i^\alpha \nt_{ij}^{\alpha\beta}+t_i^\alpha - t_j^\beta\|^2\right)\Big],
\end{multline}
where $\kappa_{ij}^{\alpha\alpha}$, $\tau_{ij}^{\alpha\alpha}$, $\kappa_{ij}^{\alpha\beta}$, $\tau_{ij}^{\alpha\beta}$ are the weights and $\rho(\cdot):\R^+\rightarrow\R$ is the loss kernel. 
\end{problem}

For notational simplicity, $F(X)$ in \cref{eq::obj} can be also rewritten as  
\begin{multline}\label{eq::F}
	F(X)= \sum_{\alpha\in\AA}\sum_{(i,j)\in \aEE^{\alpha\alpha}}F_{ij}^{\aa}(X)+\\
	\sum_{\substack{\alpha,\beta\in\AA,\\\alpha\neq \beta}}\sum_{(i,j)\in \aEE^{\alpha\beta}} F_{ij}^{\ab}(X)
	\vspace{-0.5em}
\end{multline}
where
\vspace{-0.5em}
\begin{subequations}\label{eq::Faaabij}
\begin{multline}\label{eq::Fij2}
	F_{ij}^{\aa}(X)\triangleq\frac{1}{2}\kappa_{ij}^{\aa}\|R_i^\alpha\nR_{ij}^{\aa} -R_j^\alpha\|^2 +\\ 
	\frac{1}{2}\tau_{ij}^{\aa}\|R_i^\alpha \nt_{ij}^{\aa}+t_i^\alpha - t_j^\alpha\|^2,
\vspace{-0.5em}
\end{multline}
\begin{multline}\label{eq::Fab2}
	F_{ij}^{\ab}(X)\triangleq\frac{1}{2}\rho\Big(\kappa_{ij}^{\alpha\beta}\|R_i^\alpha\nR_{ij}^{\alpha\beta} -R_j^\beta\|^2 +\\ 
	\frac{1}{2}\tau_{ij}^{\alpha\beta}\|R_i^\alpha \nt_{ij}^{\alpha\beta}+t_i^\alpha - t_j^\beta\|^2\Big).
\end{multline}
\end{subequations}
Note that $F_{ij}^{\aa}(X)$ and $F_{ij}^{\ab}$ corresponds to intra- and inter-node measurements, respectively.

In the next sections, we will present MM methods for distributed PGO, which is the major contribution of this paper.

%% file: loss_kernels.tex
In this section, we  present surrogate functions majorizing the loss kernels $\rho(\cdot)$. The resulting surrogate functions lead to an intermediate upper bound of distributed PGO while attaining the same value as the original objective function at each iterate.

It is straightforward to show that there exists sparse and positive semidefinite matrices $\nM_{ij}^{\alpha\beta}\in\R^{(d+1)n\times(d+1)n}$ for either $\alpha=\beta$ or $\alpha\neq\beta$ such that
\begin{multline}\label{eq::Mab}
\frac{1}{2}\|X\|_{\nM_{ij}^{\ab}}^2=\frac{1}{2}\kappa_{ij}^{\alpha\beta}\|R_i^\alpha\nR_{ij}^{\alpha\beta} -R_j^\beta\|^2 +\\ 
\frac{1}{2}\tau_{ij}^{\alpha\beta}\|R_i^\alpha \nt_{ij}^{\alpha\beta}+t_i^\alpha - t_j^\beta\|^2.
\end{multline}
Then, in terms of intra-node measurements with $\alpha=\beta$ and inter-node measurements with $\alpha\neq\beta$, $F_{ij}^{\aa}(X)$ and $F_{ij}^{\ab}$ take the form of
\begin{subequations}\label{eq::FaaFab}
\begin{equation}\label{eq::Faa}
	F_{ij}^{\aa}(X)\triangleq\frac{1}{2}\|X\|_{\nM_{ij}^{\aa}}^2,
\end{equation}
\begin{equation}\label{eq::Fab}
	F_{ij}^{\ab}(X)\triangleq\frac{1}{2}\rho\big(\|X\|_{\nM_{ij}^{\ab}}^2\big).
\end{equation}
\end{subequations}
From \cref{eq::Fij2,eq::Fab2}, we obtain an upper bound of $F_{ij}^{\aa}(X)$ and $F_{ij}^{\ab}(X)$ as the following proposition states.

\begin{prop}\label{prop::upperM}
	Let $\Xk=\begin{bmatrix}
	X^{1(\sk)} & \cdots & X^{|\AA|(\sk)}
	\end{bmatrix}\in \XX$ with $\Xak\in \XX^\alpha$ be an iterate of \cref{eq::pgo}. If $\rho(\cdot):\R^+\rightarrow \R$ is a loss kernel that satisfies \cref{assumption::loss}, then we obtain
	\begin{multline}\label{eq::proxFij}
	\frac{1}{2}\wabijk\|X-\Xk\|_{\nM_{ij}^{\ab}}^2+\innprod{\nabla F_{ij}^{\ab}(\Xk)}{X-\Xk}+\\
	F_{ij}^{\ab}(\Xk)\geq F_{ij}^{\ab}(X)
	\end{multline}
	for any $X$ and $\Xk\in \R^{d\times(d+1)n}$, in which $\wabijk\in\R$ is defined as
	\begin{equation}\label{eq::wabij}
	\wabijk\triangleq\begin{cases}
		1, & \alpha =\beta,\\
		\nabla\rho\big(\|\Xk\|_{\nM_{ij}^{\ab}}^2\big), & \alpha\neq\beta.
	\end{cases}
	\end{equation} 
	In \cref{eq::proxFij}, the equality ``$=$'' holds as long as $X=\Xk$.
\end{prop}
\begin{proof}
	See \hyperapp{appendix::B}{B}.
\end{proof}

Note that $F(X)$, as is shown in \cref{eq::F}, is equivalent to the sum of all $F_{ij}^{\aa}(X)$ and $F_{ij}^{\ab}(X)$. Then, an immediate upper bound of $F(X)$ resulting from \cref{prop::upperM} is
\vspace{-0.25em}
\begin{multline}\label{eq::proxF}
\frac{1}{2}\big\|X-\Xk\big\|_{\nMk}^2+\innprod{\nabla F(\Xk)}{X-\Xk}+\\
F(\Xk)\geq F(X)
\end{multline} 
in which $\nMk\in\R^{(d+1)n\times(d+1)n}$ is a positive semidefinite matrix that is defined as
\begin{multline}\label{eq::M}
\nMk\triangleq \sum_{\alpha\in\AA}\sum_{(i,j)\in \aEE^{\alpha\alpha}}\nM_{ij}^{\aa}+\\
\sum_{\substack{\alpha,\beta\in\AA,\\\alpha\neq \beta}}\sum_{(i,j)\in \aEE^{\alpha\beta}}\wabijk\cdot \nM_{ij}^{\ab}\in\R^{(d+1)n\times(d+1)n}.
\end{multline}
In addition, the equality ``$=$'' in \cref{eq::proxF} holds as long as $X=\Xk$. 

\begin{remark}
If the loss kernel  $\rho(\cdot)$ is non-trivial, $\wabijk$ is a function of $\Xk$ as defined in \cref{eq::wabij}, and $\nMk$ is a positive semidefinite matrix depending on $\Xk$ as well.
\end{remark}

It is obvious that \cref{eq::proxF} has $X^\alpha\in\XXa$ of different nodes coupled with each other, and as a result, is difficult to be used for distributed PGO. In spite of that, as is shown in the next sections, \cref{eq::proxF} is still useful for the development and analysis of our MM methods for distributed PGO.

%% file: major_pgo.tex
In this section, following a similar procedure to our previous works \cite{fan2019proximal,fan2020mm}, we  present surrogate functions $G(X|\Xk)$ and $\lG(X|\Xk)$ that majorize the objective function $F(X)$. The surrogate functions $G(X|\Xk)$ and $\lG(X|\Xk)$ decouple unknown poses of different nodes, and thus, are critical to our MM methods for distributed PGO.


\subsection{The Majorization of $F_{ij}^{\ab}(X)$}
For any matrices $B,\,C$ and $P\in \R^{m\times n}$, it can be shown that
\begin{equation}\label{eq::inequality}
\frac{1}{2}\|B-C\|_{\nM_{ij}^{\ab}}^2 \leq \|B-P\|_{\nM_{ij}^{\ab}}^2 + \|C-P\|_{\nM_{ij}^{\ab}}^2
\end{equation}
as long as $\nM_{ij}^{\ab}\in\R^{n\times n}$ is positive semidefinite, where ``$=$'' holds if
$$P=\frac{1}{2}B+\frac{1}{2}C.$$
If we let $P=\0$, \cref{eq::inequality} becomes
\begin{equation}\label{eq::inequality2}
\frac{1}{2}\|B-C\|_{\nM_{ij}^{\ab}}^2 \leq \|B\|_{\nM_{ij}^{\ab}}^2 + \|C\|_{\nM_{ij}^{\ab}}^2,
\end{equation}
{\highlight which holds for any $B$ and $C\in \R^{m\times n}$.} Applying \cref{eq::inequality2} on the right-hand side of \cref{eq::Mab}, we obtain
\begin{equation}\label{eq::inequality3}
\begin{aligned}
\frac{1}{2}\|X\|_{\nM_{ij}^{\ab}}^2\leq&\;\kappa_{ij}^{\alpha\beta}\|R_i^\alpha\nR_{ij}^{\alpha\beta}\|^2+\kappa_{ij}^{\alpha\beta}\|R_j^\beta\|^2+\\
&\;\tau_{ij}^{\alpha\beta}\|R_i^\alpha \nt_{ij}^{\alpha\beta}+t_i^\alpha\|^2+\tau_{ij}^{\alpha\beta}\|t_j^\beta\|^2\\
=&\kappa_{ij}^{\alpha\beta}\|R_i^\alpha\|^2+\kappa_{ij}^{\alpha\beta}\|R_j^\beta\|^2+\\
&\tau_{ij}^{\alpha\beta}\|R_i^\alpha \nt_{ij}^{\alpha\beta}+t_i^\alpha\|^2+\tau_{ij}^{\alpha\beta}\|t_j^\beta\|^2,
\end{aligned}
\end{equation}
where the last equality is due to $\big(\nR_{ij}^{\ab}\big)^\transpose\nR_{ij}^{\ab}=\nR_{ij}^{\ab}\big(\nR_{ij}^{\ab}\big)^\transpose=\I$. Furthermore, there exists a positive semidefinite matrix $\nH_{ij}^{\alpha\beta}\in \R^{(d+1)n\times (d+1)n}$ such that the right-hand side of \cref{eq::inequality3} can be rewritten as
\begin{equation}\label{eq::Hab}
\begin{aligned}
\frac{1}{2}\|X\|_{\nH_{ij}^{\ab}}^2
=&\kappa_{ij}^{\alpha\beta}\|R_i^\alpha\|^2+\kappa_{ij}^{\alpha\beta}\|R_j^\beta\|^2+\\
 &\tau_{ij}^{\alpha\beta}\|R_i^\alpha \nt_{ij}^{\alpha\beta}+t_i^\alpha\|^2+\tau_{ij}^{\alpha\beta}\|t_j^\beta\|^2,
\end{aligned}
\end{equation}
where $\nH_{ij}^{\ab}$ is a block diagonal matrix decoupling unknown poses of different nodes. Replacing the right-hand side of \cref{eq::inequality3} with \cref{eq::Hab} results in
\begin{equation}
\nonumber
 \frac{1}{2}\|X\|_{\nM_{ij}^{\ab}}^2 \leq \frac{1}{2}\|X\|_{\nH_{ij}^{\ab}}^2
\end{equation}
for any $X\in\R^{d\times(d+1)n}$, which suggests 
\begin{equation}\label{eq::MH}
\nH_{ij}^{\ab}\succeq\nM_{ij}^{\ab}.
\end{equation}
With $\nH_{ij}^{\ab}\in\R^{(d+1)n\times(d+1)n}$ in \cref{eq::MH,eq::Hab}, we define $E_{ij}^{\ab}(\cdot|\Xk):\R^{d\times(d+1)n}\rightarrow \R$:
\begin{multline}\label{eq::Eab}
E_{ij}^{\ab}(X|\Xk)\triangleq \frac{1}{2}\wabijk\|X-\Xk\|_{\nH_{ij}^{\ab}}^2+\\
\innprod{\nabla F_{ij}^{\ab}(\Xk)}{X-\Xk}+ F_{ij}^{\ab}(\Xk),
\end{multline}
where $\wabijk$ is given in \cref{eq::wabij}. From the equation above, it can be concluded that $E_{ij}^{\ab}(X|\Xk)$ majorizes $F_{ij}^{\ab}(X)$  as the following proposition states, which is important for the construction of surrogate functions for distributed PGO.

\begin{prop}\label{prop::upper}
	Given any nodes $\alpha,\,\beta\in \AA$ with either $\alpha = \beta$ or $\alpha\neq\beta$, if $\rho(\cdot):\R^+\rightarrow \R$ is a loss kernel that satisfies \cref{assumption::loss}, then we obtain
	\begin{equation}\label{eq::EF}
	E_{ij}^{\ab}(X|\Xk)\geq F_{ij}^{\ab}(X).
	\end{equation} 
	for any $X\in \R^{d\times(d+1)n}$. In the equation above, the equality ``$=$'' holds if $X=\Xk$.
\end{prop}
\begin{proof}
	See \hyperapp{appendix::C}{C}.
\end{proof}

\subsection{The Majorization of $F(X)$}
From \cref{prop::upper}, it is straightforward to construct surrogate functions majorizing $F(X)$ in \cref{eq::obj} as the following proposition states.

\begin{prop}\label{prop::G}

	Let $X^{(\sk)}=\begin{bmatrix}
	X^{1(\sk)} & \cdots & X^{|\AA|(\sk)}
	\end{bmatrix}\in \XX$ with $X^{\alpha(\sk)}\in \XX^\alpha$ be an iterate of $X\in \XX$ in \cref{eq::pgo}. Suppose $G(\cdot|\Xk): \R^{d\times(d+1)n}\rightarrow\R$ is a function:
	\begin{multline}\label{eq::G}
		G(X|\Xk)\triangleq \sum_{\alpha\in\AA}\sum_{(i,j)\in \aEE^{\alpha\alpha}}F_{ij}^{\aa}(X)+\\
		\sum_{\substack{\alpha,\beta\in\AA,\\
				\alpha\neq \beta}}\sum_{(i,j)\in \aEE^{\alpha\beta}} E_{ij}^{\ab}(X|\Xk)+\frac{\xi}{2}\big\|X-\Xk\big\|^2
	\end{multline}
	where $\xi\in\R$ and $\xi\geq 0$. Then, we have the following results:
{\highlight
	\begin{enumerate}[(a)]
	\item \label{prop::G1}  For any $X\in\R^{d\times(d+1)}$ and $X^{(\sk)}\in\XX$, 	
	\begin{equation}\label{eq::GF}
		G(X|X^{(\sk)})\geq F(X)
	\end{equation}
	where the equality ``$=$'' holds if  $X=X^{(\sk)}$.
	\item \label{prop::G2} $G(X|\Xk)$ is equivalent to
	\begin{equation}\label{eq::G2}
		G(X|\Xk)=\sum_{\alpha\in\AA} G^{\alpha}(\Xa|\Xk)  + F(\Xk),
	\end{equation}
	where $G^\alpha(\Xa|X^{(\sk)})$ is a function of $\Xa\in\XXa$ within a single node $\alpha$. 
	\item\label{prop::G3}  For any node $\alpha\in\AA$, there exists positive-semidefinite matrices $\nGamma^{\ak}\in \R^{(d+1)n_\alpha\times (d+1)n_\alpha}$ such that 
	\begin{multline}\label{eq::GMa}
	G^\alpha(\Xa|X^{(\sk)}) \triangleq \frac{1}{2}\|\Xa-\Xak\|_{\nGamma^{\ak}}^2+\\
	\big\langle\nabla_{\Xa} F(X^{(\sk)}),\,{\Xa-\Xak}\big\rangle
	\end{multline}
	where $\nabla_{\Xa} F(X^{(\sk)})$ is the Euclidean gradient of $F(X)$ with respect to $\Xa\in \XXa$ at $\Xk\in\XX$.
	\end{enumerate}
}
\end{prop}
\begin{proof}
	See \hyperapp{appendix::D}{D}.
\end{proof}

Following \cref{prop::G}, we might further majorize $F(X)$ as well as $G(X|\Xk)$ by applying \cref{eq::EF} to \cref{eq::G} and replacing $F_{ij}^{\aa}(X|\Xk)$ with $E_{ij}^{\aa}(X|\Xk)$, which results in the following proposition.

\begin{prop}\label{prop::lG}
Let $X^{(\sk)}=\begin{bmatrix}
	X^{1(\sk)} & \cdots & X^{|\AA|(\sk)}
\end{bmatrix}\in \XX$ with $X^{\alpha(\sk)}\in \XX^\alpha$ be an iterate of $X\in \XX$ in \cref{eq::pgo}, and $X_i^{\ak}=\begin{bmatrix}
t^{\alpha(\sk)} & R^{\alpha(\sk)}
\end{bmatrix}\in SE(d)$ the corresponding iterate of $X_i^\alpha \in  SE(d)$. Suppose $\lG(\cdot|\Xk): \R^{d\times(d+1)n}\rightarrow\R$ is a function:
\vspace{-0.25em}
\begin{multline}\label{eq::lG}
\lG(X|\Xk)= \sum_{\alpha\in\AA}\sum_{(i,j)\in \aEE^{\alpha\alpha}}E_{ij}^{\aa}(X|\Xk)+\\
\sum_{\substack{\alpha,\beta\in\AA,\\
		\alpha\neq \beta}}\sum_{(i,j)\in \aEE^{\alpha\beta}} E_{ij}^{\ab}(X|\Xk) + \frac{\zeta}{2} \big\|X-\Xk\big\|^2
\end{multline}
where $\zeta\in\R$ and $\zeta\geq\xi\geq 0$. Note that $\xi\in \R$ is given in \cref{eq::G}. Then,  we have the following results:
{\highlight
\begin{enumerate}[(a)]
	\item\label{prop::lG1}  For any $X\in\R^{d\times(d+1)}$ and $X^{(\sk)}\in\XX$, 	
	\begin{equation}\label{eq::lGF}
		\lG (X|X^{(\sk)})\geq G(X|X^{(\sk)})\geq F(X)
	\end{equation}
	where the equality ``$=$'' holds if  $X=X^{(\sk)}$.
	\item\label{prop::lG2} $\lG(X|\Xk)$ is equivalent to
	\begin{equation}\label{eq::lG2}
		\begin{aligned}
			\lG(X|\Xk)&=\sum_{\alpha\in\AA} H^{\alpha}(\Xa|\Xk)  + F(\Xk)\\
							&=\sum_{\alpha\in\AA}\sum_{i=1}^{n_\alpha} \lG_i^{\alpha}(\Xa_i|\Xk) + F(\Xk)
		\end{aligned}
	\end{equation}
	where $\lG^\alpha(\Xa|X^{(\sk)})$ is a function of $\Xa\in\XXa$ within a single node $\alpha$, and  $\lG_i^\alpha(\Xa_i|X^{(\sk)})$ is a  function of a single pose $\Xa_i\in SE(d)\subset \R^{d\times (d+1)}$. 
	
	\item\label{prop::lG3} For any node $\alpha\in\AA$ and $i\!\in\!\{1,\,\cdots,\, n_\alpha\}$, there exists positive-semidefinite matrices $\lnGamma^{\ak}\in \R^{(d+1)n_\alpha\times (d+1)n_\alpha}$ and $\lnGamma_i^{\ak}\in \R^{(d+1)\times (d+1)}$ such that
	\vspace{-0.25em}
	\begin{multline}\label{eq::lGMa}
		\lG^\alpha(\Xa|X^{(\sk)}) =\frac{1}{2}\|\Xa-\Xak\|_{\lnGamma^{\ak}}^2+\\
		\big\langle\nabla_{\Xa} F(X^{(\sk)}),\,{\Xa-\Xak}\big\rangle,
	\end{multline}
	\vspace{-1.25em}
	\begin{multline}\label{eq::lGMai}
		\lG_i^\alpha(X_i^\alpha|X^{(\sk)}) =\frac{1}{2}\|X_i^\alpha-X_i^{\ak}\|_{\lnGamma_i^{\ak}}^2+\\
		\big\langle\nabla_{X_i^{\alpha}} F(X^{(\sk)}),\,{X_i^{\alpha}-X_i^{\ak}}\big\rangle
	\end{multline}
	 where $\nabla_{\Xa} F(X^{(\sk)})$ and $\nabla_{\Xa_i} F(X^{(\sk)})$ are the Euclidean gradients of $F(X)$ with respect to $\Xa\in\XXa$ and $\Xa_i\in SE(d)$ at $\Xk\in\XX$, respectively. 
\end{enumerate}
}
\end{prop}
\begin{proof}
The proof is similar to that of \cref{prop::G}.
\end{proof}

\begin{remark}
As a result of \cref{eq::lG2}, $\lG^\alpha(\Xa|\Xk)$ can be rewritten as the sum of $\lG_i^\alpha(\Xa_i|\Xk)$, i.e.,
\begin{equation}\label{eq::lGMsum}
	\lG^\alpha(\Xa|X^{(\sk)})=\sum_{i=1}^{n_\alpha} \lG_i^{\alpha}(\Xa_i|\Xk).
\end{equation}
Note that $\lG_i^\alpha(\Xa_i|\Xk)$ in \cref{eq::lG2,eq::lGMsum} relies on a single pose $X_i^\alpha\in SE(d)\subset \R^{d\times (d+1)}$. This will be later exploited  in \cref{section::mm,section::amm,section::ammc} to improve the computational efficiency of distributed PGO.
\end{remark}

\begin{remark}
Propositions \cref{prop::G,prop::lG}  indicate that $G(X|\Xk)$ and $H(X|\Xk)$ not only majorize $F(X)$ but also decouple  poses from different nodes through  $G^\alpha(\Xa|\Xk)$ and $H^\alpha(\Xa|\Xk)$, making it possible to implement majorization minimization methods for distributed PGO.
\end{remark}

\begin{remark}
	\highlight
	$\zeta$ and $\xi$  in \cref{eq::G,eq::lG} are important for the convergence analysis of MM methods for distributed PGO in \cref{section::amm,section::ammc,section::mm} . It is  recommended to be set $\zeta>\xi>0$ but close to zero such that $G(X|\Xk)$ and $\lG(X|\Xk)$ are tighter upper bounds of $F(X)$ and yield faster convergence.
\end{remark}

Recall that $G^\alpha(\Xa|\Xk)$ and $\lG^\alpha(\Xa|\Xk)$ in \cref{eq::GMa,eq::lGMa} rely on $\nGammaak$, $\lnGammaak$,
$\nabla_{\Xa} F(\Xk)$, which---according to \cref{eq::Faaabij,eq::F}---are only related to $F_{ij}^{\aa}(X)$, $F_{ij}^{\ab}(X)$, $F_{ji}^{\beta\alpha}(X)$.
Also, \cref{eq::Faaabij} indicates that $F_{ij}^{\aa}(X)$ depends on $X_i^\alpha$ and $X_j^\alpha$, while $F_{ij}^{\ab}(X)$ and $F_{ji}^{\beta\alpha}(X)$ on $X_i^\alpha$ and $X_j^\beta$. Therefore, $\nGammaak$, $\lnGammaak$,
$\nabla_{\Xa} F(\Xk)$ can be evaluated  as long as node $\alpha$ have access to $\Xb$ from its neighbor $\beta$. Furthermore, $G^\alpha(\Xa|\Xk)$ and  $\lG^\alpha(\Xa|X^{(\sk)})$  can be constructed in a distributed setting with one inter-node communication round between neighboring nodes $\alpha$ and $\beta$.

In the next sections, we will present MM methods for distributed PGO using $G(X|\Xk)$ and $\lG(X|\Xk)$ that are guaranteed to converge to first-order critical points.

%% file: mm_pgo.tex
{\highlight In distributed optimization, MM methods are one of the most popular first-order optimization methods  \cite{hunter2004tutorial,sun2016majorization}. As mentioned before, MM methods solve an optimization problem by iteratively minimizing an upper bound of the objective function such that the objective value is nonincreasing. Recall that $G(X|\Xk)$ and $H(X|\Xk)$ majorize $F(X)$ and decouple poses from different nodes; see \cref{prop::G,prop::lG}, respectively. Therefore, we might make use of MM methods where distributed PGO is reduced to independent optimization problems that can be solved in parallel.  In \cref{section::mm::update}, we propose MM methods for distributed PGO using $G(X|\Xk)$ and $H(X|\Xk)$. Then, in \cref{section::mm::algorithm}, we present the algorithm and prove that the proposed method is guaranteed to converge to first-order critical points.}

\vspace{-0.75em}

\subsection{Update Rule}\label{section::mm::update}
According to \cref{prop::G,prop::lG}, $G(X|\Xk)$ and $\lG(X|\Xk)$ are surrogate functions majorizing $F(X)$:
\begin{equation}\label{eq::lGGFeq}
	\lG(X|\Xk) \geq G(X|\Xk) \geq F(X),
\end{equation}
\begin{equation}\label{eq::lGGFneq}
\lG(\Xk|\Xk) = G(\Xk|\Xk) = F(\Xk).
\end{equation}
Following the notion of MM methods \cite{hunter2004tutorial}, we propose the following update rule:
\begin{equation}\label{eq::minlG}
	\Xkh\gets\arg\min_{X\in \XX} \lG(X|\Xk),
\end{equation}
\begin{equation}\label{eq::minG}
	\Xkp\gets\arg\min_{X\in \XX} G(X|\Xk).
\end{equation}
Here, $\Xkh$  in \cref{eq::minlG} is first solved and used to initialize $\Xkp$ in \cref{eq::minG}. Also, \cref{eq::lG2,eq::G2} indicate that \cref{eq::minlG,eq::minG} are equivalent to  $|\AA|$ independent optimization problems of $\Xa\in\XX^\alpha$ within a single node $\alpha$:
\begin{equation}\label{eq::xlG}
	\Xakh\gets\arg\min_{\Xa\in \XXa}\lG^\alpha(\Xa|\Xk),\quad \forall\alpha\in\AA,
\end{equation}
\begin{equation}\label{eq::xG}
	\Xakp\gets\arg\min_{\Xa\in \XXa}G^\alpha(\Xa|\Xk),\quad \forall\alpha\in\AA,
\end{equation}
where  $\Xakh$ in \cref{eq::xlG} is the initial guess to solve $\Xakp$ in \cref{eq::xG}. We remark that \cref{eq::xlG,eq::xG} can be  solved within a single node $\alpha\in\AA$. Recalling from \cref{eq::lGMsum} that $H^\alpha(\Xa|\Xk)=\sum_{i=1}^{n_\alpha}H_i^{\alpha}(X_i^{\alpha}|\Xk)$, we further reduce \cref{eq::xlG} to $n\triangleq\sum_{\alpha\in\AA} n_\alpha$ independent optimization problems on a single pose $\Xa_i\in SE(d)$:
\vspace{-0.5em}
\begin{multline}\label{eq::xlGi}
	\Xakh_i\gets\arg\min_{X_i^\alpha\in SE(d)}\lG_i^\alpha(\Xa_i|\Xk),\\
	\quad\quad\forall\alpha\in\AA\text{ and } i\in \{1,\,\cdots,\,n_\alpha\}.
\end{multline}
In \hyperapp{appendix::K}{K}, we have shown that \cref{eq::xlGi} admits an efficient closed-form solution involving only matrix multiplication and singular value decomposition \cite{umeyama1991least}.

From \cref{eq::lGGFeq,eq::lGGFneq}, we conclude that \cref{eq::minlG,eq::minG,eq::xlG,eq::xG} result in
\begin{subequations}
	\begin{equation}\label{eq::FlG}
		F(\Xk)=\lG(\Xk|\Xk)\geq\\ \lG(\Xkh|\Xk)\geq F(\Xkh),
	\end{equation}
	\begin{equation}\label{eq::FG}
		F(\Xk)=G(\Xk|\Xk)\geq\\ G(\Xkp|\Xk)\geq F(\Xkp)
	\end{equation}
\end{subequations}
which indicate  $F(\Xkh)\leq F(\Xk)$ and $F(\Xkp)\leq F(\Xk)$. Therefore,  \cref{eq::minlG,eq::minG,eq::xlG,eq::xG} are a reasonable update rule for distributed PGO.  In particular, we remark that \cref{eq::minlG,eq::minG,eq::xlG,eq::xG} combine the strengths of our previous work \cite{fan2019proximal,fan2020mm}. Even though \cref{eq::minlG,eq::xlG}  are motivated by \cite{fan2019proximal},  \cref{eq::minG,eq::xG} make better use of the  information within a single node, and thus, take fewer iterations. In contrast to \cite{fan2020mm}, since \cref{eq::minlG,eq::xlG} have an efficient closed-form solution to \cref{eq::xlGi} that yields sufficient improvement,  the time-consuming local minimization of \cref{eq::minG,eq::xG} is avoided  as long as $\Xkp$ and $\Xakp$ are initialized with $\Xkh$ and $\Xakh$. Most importantly, as is shown later in \cref{prop::mm}, the proposed  updated rule of \cref{eq::minlG,eq::minG,eq::xlG,eq::xG} has provable convergence to first-order critical points.

\subsection{Algorithm}\label{section::mm::algorithm}

\setlength{\textfloatsep}{7pt}

\begin{algorithm}[t]
	\caption{The $\mm$ Method}
	\label{algorithm::mm}
	\begin{algorithmic}[1]
		\State\textbf{Input}: An initial iterate $X^{(0)}\in \XX$ and $\zeta \geq \xi \geq 0$.
		\State\textbf{Output}: A sequence of iterates $\{\Xk\}$ and $\{\Xkh\}$.\vspace{0.2em} 
		\vspace{0.1em}
		\For{$\sk\gets 0,\,1,\,2,\,\cdots$}
		\vspace{0.1em}
		\For{node $\alpha\gets 1,\,\cdots,\, |\AA|$}
		\vspace{0.1em}
		\State  retrieve $\Xbk$ from $\beta\in\NN_{\alpha}$ \label{line::alg1::comm}
		\vspace{0.1em}
		\State evaluate $\nGammaak$, $\lnGammaak$, $\nabla_{\Xa} F(\Xk)$  \label{line::grad_F}
		\vspace{0.1em}
		\State $\Xakh\gets \arg\min\limits_{\Xa\in\XXa }\lG^\alpha(\Xa|\Xk)$ using \cref{algorithm::lG}\label{line::xlG}
		\vspace{0.1em}
		\State $\Xakp\gets$ improve $\arg\min\limits_{\Xa\in\XXa }G^\alpha(\Xa|\Xk)$\label{line::mm_opt} with $\Xakh$ as the initial guess \label{line::xG}
		\EndFor
		\EndFor
	\end{algorithmic}
\end{algorithm}

\begin{algorithm}[t]
	\caption{Solve $\Xakh\gets \arg\min\limits_{\Xa\in \XXa}\lG^\alpha(\Xa|\Xk)$}
	\label{algorithm::lG}
	\begin{algorithmic}[1]
		\State\textbf{Input}: $\Xak$, $\lnGammaak$, $\nabla_{\Xa} F(\Xk)$.
		\vspace{0.1em}
		\State\textbf{Output}: $\Xakh$.\vspace{0.2em} 
		\For{$i\gets 1,\,\cdots,\, n_\alpha$}\label{line::xlG1}
		\vspace{0.1em}
		\State $\Xakh_i \gets \arg\min\limits_{\Xa_i\in SE(d)} \lG_i^\alpha(\Xa_i|\Xk)$ using \hyperapp{appendix::K}{K}\label{line::alg2::xak}
		\vspace{0.1em}
		\EndFor
		\vspace{0.1em}
		\State retrieve $\Xakh$ from $\Xakh_i$ in which $i = 1,\,\cdots,\,n_\alpha$\label{line::xlG2}
	\end{algorithmic}
\end{algorithm}

The proposed update rule results in the $\mm$ method for distributed PGO (\cref{algorithm::mm}). The outline of  $\mm$  is as follows:
\begin{enumerate}[leftmargin=0.45cm]
\item In line~\ref{line::alg1::comm} of \cref{algorithm::mm}, each node $\alpha$ performs one inter-node communication round to retrieve $\Xbk$ from its neighbors $\beta\in\NN_{\alpha}$. Note that that no other inter-node communication is required.
\item In lines~\ref{line::grad_F} of \cref{algorithm::mm}, each node $\alpha$ evaluates $\nGamma$, $\lnGamma$, $\nabla_{\Xa} F(\Xk)$ with $\Xak$ and $\Xbk$ where $\beta\in\NN^\alpha$ are neighbors of node $\alpha$. 
\item In line~\ref{line::xlG} of \cref{algorithm::mm}, we obtain the intermediate solution $\Xakh$ using \cref{algorithm::lG}. We have proved that the resulting $\Xakh$ is already sufficient to guarantee the convergence to first-order critical points.
\item In line~\ref{line::alg2::xak} of \cref{algorithm::lG}, there exists an exact and efficient closed-form solution to $\Xakh$ using \hyperapp{appendix::K}{K}.

\item In line~\ref{line::xG} of \cref{algorithm::mm}, we use $X^{\alpha(\sk+\shalf)}$ to initialize \cref{eq::xG}, and improve the final solution $\Xakp$ through iterative optimization such that $G^{\alpha}(\Xakp|\Xk)\leq G^{\alpha}(X^{\alpha(\sk+\shalf)}|\Xk)$. Note that $\Xakp$ does not have to be a local optimal solution to \cref{eq::xG}, nevertheless, $\Xakp$ is still expected to have a faster convergence than $X^{\alpha(\sk+\shalf)}$.
\end{enumerate}

\begin{remark}
	\highlight
	In line~\ref{line::alg1::comm} of \cref{algorithm::mm}, node $\alpha$ does not retrieve all the poses in $\Xbk$---only poses related to inter-node measurements are needed and exchanged. This also applies to line~\ref{line::alg4::comm} of \cref{algorithm::amm_c_x}, and lines~\ref{line::alg5::comm1} and \ref{line::alg5::comm} of \cref{algorithm::amm} in the following sections.
\end{remark}

\begin{remark}
	$\mm$  (\cref{algorithm::mm}) requires no inter-node communication except for  line~\ref{line::alg1::comm} of \cref{algorithm::mm} that is used to evaluate $\nGammaak$, $\lnGammaak$, $\nabla_{\Xa} F(\Xk)$, which, as mentioned before, can be distributed with one inter-node communication round between neighboring nodes $\alpha$ and $\beta$ without introducing any additional computation.
\end{remark}

Since  $\Xakh$ in \cref{eq::xlG} has a closed-form solution that can be efficiently computed, and \cref{eq::xG} does not require $\Xakp$ to be a local optimal solution, the overall computational efficiency of the $\mm$ method is significantly improved in contrast to \cite{fan2019proximal,fan2020mm}. More importantly, the $\mm$ method still converges to first-order critical points as long as the following assumption holds.

\begin{assumption}\label{assumption::mm}
For $\Xakp$ and $X^{\alpha(\sk+\shalf)}$, it is assumed that
\begin{equation}
G^{\alpha}(\Xakp|\Xk)\leq G^{\alpha}(X^{\alpha(\sk+\shalf)}|\Xk)
\end{equation}
for each node $\alpha=1,\,2\,\cdots,\, |\AA|$.
\end{assumption}

Note that \cref{assumption::mm} can be satisfied with ease as long as line~\ref{line::xG} of \cref{algorithm::mm} is initialized with $\Xakh$. Then, we have the following proposition about the convergence of $\mm$  (\cref{algorithm::mm}).

\begin{prop}\label{prop::mm}
If \cref{assumption::loss,assumption::neighbor,assumption::mm} hold, then for a sequence of $\{\Xk\}$ and $\{\Xkh\}$ generated by  \cref{algorithm::mm},  we have
\begin{enumerate}[(a)]
\item\label{prop::mm1} $F(\Xk)$ is nonincreasing as $\sk\rightarrow\infty$;
\item\label{prop::mm2}  $F(\Xk)\rightarrow F^\infty$ as $\sk\rightarrow\infty$;
\item\label{prop::mm3} $\|\Xkp-\Xk\|\rightarrow 0$ as $\sk\rightarrow\infty$ if $\xi>0$;
\item\label{prop::mm4} $\|\Xkh-\Xk\|\rightarrow 0$ as $\sk\rightarrow\infty$ if $\zeta > \xi > 0$;
\item\label{prop::mm5}  if $\zeta > \xi > 0$, then there exists $\epsilon > 0$ such that 
\begin{equation}
	\nonumber
	\min\limits_{0\leq\sk< \mathsf{K}}\|\grad\, F(\Xkh)\|\leq \sqrt{\frac{2}{\epsilon}\cdot\dfrac{F(X^{(0)})-F^\infty}{{\sK+1}}}
\end{equation}
for any $\mathsf{K}\geq 0$;
\item\label{prop::mm6}  if $\zeta > \xi> 0$, then $\grad\, F(\Xk)\rightarrow \0$ and $\grad\, F(\Xkh)\rightarrow \0$ as $\sk\rightarrow \infty$.
\end{enumerate}
\end{prop}
\begin{proof}
See \hyperapp{appendix::E}{E}.
\end{proof}

\begin{remark}
In contrast to other distributed PGO algorithms \cite{tian2019distributed,tron2014distributed,choudhary2017distributed,eric2020geod},  $\mm$  has the mildest conditions for convergence and apply to a broad class of loss kernels.
\end{remark}

%% file: amm_pgo_c.tex
{\highlight
	In the last several decades, a number of accelerated first-order optimization methods have been proposed \cite{nesterov1983method,nesterov2013introductory}. Even though most of them were originally developed for convex optimization, it has been recently found that these accelerated methods also work well for nonconvex optimization \cite{ghadimi2016accelerated,jin2018accelerated,li2015accelerated}. In our previous works \cite{fan2019proximal,fan2020mm}, we proposed to use Nesterov's method to accelerate  distributed PGO, which yield much faster convergence. Since  $\mm$  is a first-order optimization method, it is possible to exploit Nesterov's method for acceleration. In \cref{section::ammc::nesterov}, we implement Nesterov's method to accelerate MM methods for distributed PGO. Then, in \cref{section::ammc::adaptive}, we introduce the adaptive restart scheme \cite{o2015adaptive}  to guarantee the convergence if a master node exists. Lastly, in \cref{section::ammc::algorithm}, we propose the accelerated MM method for distributed PGO with master and  prove that such a method converges to first-order critical points.
}  

\vspace{-0.5em}
\subsection{Nesterov's Method}\label{section::ammc::nesterov}
According to \cref{prop::G,prop::lG}, \cref{eq::minG,eq::minlG} are {\highlight proximal operators} of  $F(X)$, making it possible to implement Nesterov's method \cite{nesterov1983method,nesterov2013introductory} for acceleration and resulting in the following update rule for $\Xakh$ and $\Xakp$:
\begin{equation}\label{eq::sk}
	s^{\akp}=\tfrac{\sqrt{4{s^{\ak}}^2+1}+1}{2},
\end{equation}
\begin{equation}\label{eq::gammak}
	\lambda^{\ak}= \tfrac{s^{\ak}-1}{s^{\akp}},
\end{equation}
\begin{equation}\label{eq::Yak}
	Y^{\ak}= X^{\ak}+\lambda^{\ak}\cdot\big(\Xak-\Xakm\big),
\end{equation}
\begin{equation}\label{eq::HYak}
	\Xakh=\arg\min\limits_{\Xa\in\XX^\alpha }\lG^\alpha(\Xa|\Yk),
\end{equation}
\begin{equation}\label{eq::GYak}
	\Xakp=\arg\min\limits_{\Xa\in\XX^\alpha }G^\alpha(\Xa|\Yk).
\end{equation}
In \cref{eq::HYak,eq::GYak}, $G^\alpha(\cdot|\Yk):\XX^\alpha\rightarrow\R$ and $\lG^\alpha(\cdot|\Yk):\XX^\alpha\rightarrow\R$ are surrogate functions at $\Yak$:
\begin{multline}\label{eq::GXY}
	G^\alpha(X^\alpha|Y^{(\sk)}) =\frac{1}{2}\|\Xa-\Yak\|_{\nGamma^{\ak}}^2+\\
	\big\langle\nabla_{\Xa} F(\Yk),\,{\Xa-\Yak}\big\rangle,
\end{multline}
\vspace{-2em}
\begin{multline}\label{eq::lGXY}
	\lG^\alpha(X^\alpha|Y^{(\sk)}) =\frac{1}{2}\|\Xa-\Yak\|_{\lnGamma^{\ak}}^2+\\
	\big\langle\nabla_{\Xa} F(\Yk),\,{\Xa-\Yak}\big\rangle,
\end{multline}
where $\nGamma^{\ak}$ and $\lnGamma^{\ak}$  are the same as these in $G^\alpha(\cdot|\Xk)$ and $\lG^\alpha(\cdot|\Xk)$ in \cref{eq::GMa,eq::lGMa}. 

The key insight of Nesterov's method is to exploit the momentum $\Xak-\Xakm$ for acceleration, which is essentially governed by \cref{eq::sk,eq::gammak,eq::Yak}. Note that Nesterov's method using \cref{eq::sk,eq::gammak,eq::Yak,eq::HYak,eq::GYak} is  equivalent to \cref{eq::xG,eq::xlG} when $s^{\ak}=1$ and $\lambda^{\ak}=0$, and then increasingly affected by the momentum as $s^{\ak}$ and $\lambda^{\ak}$ increase.

Nesterov's method is known to converge quadratically for convex optimization while the unaccelerated MM method only has linear convergence  \cite{nesterov1983method,nesterov2013introductory}. Even though distributed PGO is a nonconvex optimization problem, similar to \cite{fan2019proximal,fan2020mm},  \cref{eq::gammak,eq::Yak,eq::GYak,eq::HYak,eq::sk} using  Nesterov's method for acceleration  empirically have significant speedup while introducing  almost no extra computation or communication compared to the $\mm$ method. Thus, it is preferable to adopt Nesterov's method to accelerate distributed PGO.

\vspace{-0.5em}
\subsection{Adaptive Restart}\label{section::ammc::adaptive}
In spite of faster convergence, Nesterov's accelerated distributed PGO using \cref{eq::sk,eq::gammak,eq::Yak,eq::GYak,eq::HYak} is no longer nonincreasing, and might fail to converge due to the nonconvexity of PGO. Fortunately, such a problem can be remedied with an adaptive restart scheme \cite{o2015adaptive,fan2020mm,fan2019proximal} as the following.

Let $\sF^{(\sk)}$ be an exponential moving averaging of $F(X^{(0)})$, $F(X^{(1)})$, $\cdots$, $F(\Xk)$ :
\begin{equation}\label{eq::lFk}
	\sF^{(\sk)}\triangleq \begin{cases}
		F(X^{(0)}),& \sk=0,\\
		(1-\eta)\cdot\sF^{(\skm)}+\eta\cdot F(\Xk),&\text{otherwise}
	\end{cases}
\end{equation}
where $\eta\in(0,\,1]$.  {\highlight Following \cite{fan2019proximal,li2015accelerated,zhang2004nonmonotone}, the adaptive restart scheme guarantees the convergence by keeping $F(\Xkp)\leq \lF^{(\sk)}$. Even though it is not obvious, $F(\Xkp)\leq \lF^{(\sk)}$  can be achieved with the following steps}:
\begin{enumerate}[leftmargin=0.45cm]
\item  Update $\Xkh$ and $\Xkp$ by solving \cref{eq::GYak,eq::HYak} for each node $\alpha\in\AA$;
\item If $F(\Xkh)>\sF^{(\sk)}$, update $\Xkh$ again by solving  \cref{eq::xlG} for each node $\alpha\in\AA$;
\item If $F(\Xkp)>\sF^{(\sk)}$, update $\Xkp$ again by solving  \cref{eq::xG}  and reduce $s^{\akp)}$  for each node $\alpha\in\AA$.  
\end{enumerate}
{\highlight Due to space limitation, the complete analysis of the adaptive restart scheme  is left in  \hyperapp{appendix::F}{F} where more details are presented.}

\begin{remark}
Since $\eta\in(0,\,1]$ in \cref{eq::lFk}, then $\sF^{(\skp)}\leq \sF^{(\sk)}$ as long as $F(\Xkp)\leq \sF^{(\sk)}$.  In \hyperapp{appendix::F}{F}, we have proved that $F(\Xkp)\leq \sF^{(\sk)}$ hods if  $\Xkh$ and $\Xkp$ are updated with  \cref{eq::xG,eq::xlG} for each node $\alpha\in\AA$. Therefore,  the adaptive restart scheme above results in a nonincreasing sequence of $\sF^{(\sk)}$.  Furthermore,  \hyperapp{appendix::F}{F} indicates that such an adaptive restart scheme is sufficient to guarantee the convergence to first-order critical points under mild conditions.
\end{remark}

Note that one has to aggregate information across the network to evaluate and compare $\sF^{(\sk)}$, $F(\Xkh)$, $F(\Xkp)$ using \cref{eq::lFk,eq::F}. Thus, a master node capable of communicating with each node $\alpha\in\AA$ is required. In the rest of this section, we make the following assumption about the existence of such a master node.

\begin{assumption}\label{assumption::master}
	There is a master node to retrieve $\Xak$ and $\Xakh$ from each node $\alpha\in\AA$ and evaluate $\sF^{(\sk)}$, $F(\Xkh)$, $F(\Xkp)$.  
\end{assumption}

\subsection{Algorithm}\label{section::ammc::algorithm}
\input{algorithm_amm_c}
\input{algorithm_amm_c_x}

Implementing Nesterov's method  and the adaptive restart scheme, we obtain the $\ammc$ method for distributed PGO (\cref{algorithm::ammc}), where ``$*$'' indicates  the existence of a master node. 

The outline of  $\ammc$  is as follows:
\begin{enumerate}[leftmargin=0.45cm]
\item In lines~\ref{line::alg3::sk}, \ref{line::alg3::Yk} of \cref{algorithm::ammc}, each node $\alpha$ computes $\Yk$ for Nesterov's acceleration that is related with $s^{\ak}\in[1,\,\infty)$ and $\lambda^{\ak}\in[0,\,1)$.
\vspace{0.25em}
\item  In line~\ref{line::alg4::comm} of \cref{algorithm::amm_c_x}, each node $\alpha$ performs one inter-node communication round to retrieve $\Xbk$ and $Y^{\bk}$ from its neighbors $\beta\in\NN^\alpha$.
\item  In line~\ref{line::alg3::mcomm1} of \cref{algorithm::ammc} and lines~\ref{line::alg4::mcomm1}, \ref{line::alg4::mcomm2}, \ref{line::alg4::mcomm3} of \cref{algorithm::amm_c_x}, each node $\alpha$ performs one inter-node communication round to send $\Xakh$ and $\Xakp$ to the master node.
\vspace{0.2em}
\item In lines~\ref{line::alg4::DFaYk} of \cref{algorithm::amm_c_x}, each node $\alpha$ evaluates $\nGammak$, $\lnGammaak$, $\nabla_{X^{\alpha}} F(\Xk)$, $\nabla_{X^{\alpha}} F(\Yk)$ using $\Xak$, $\Yak$, $\Xbk$, $Y^{\bk}$ where $\beta\in\NN^\alpha$ are neighbors of node $\alpha$.
\item In lines~\ref{line::alg3::lF0}, \ref{line::alg3::lFk} of \cref{algorithm::ammc} and lines~\ref{line::alg4::Fxk}, \ref{line::alg4::Fxkh}, \ref{line::alg4::Fxkp} of \cref{algorithm::amm_c_x},  the master node evaluates  $\sF^{(\sk)}$, $F(\Xkh)$, $F(\Xkp)$ that are used for adaptive restart.
\vspace{0.2em}
\item In lines~\ref{line::alg4::restart_s1} to  \ref{line::alg4::restart_e2} of \cref{algorithm::amm_c_x}, the master node performs adaptive restart  to keep $F(\Xkh)\leq \sF^{(\sk)}$ and $F(\Xkp)\leq \sF^{(\sk)}$, which yields a nonincreasing sequence of $\lF^{(\sk)}$ to guarantee the convergence. 
\vspace{0.25em}
\item In lines~\ref{line::alg4::xG1}, \ref{line::alg4::xG2} of \cref{algorithm::amm_c_x}, note that $\Xakp$ does not have to be a local optimal solution to \cref{eq::xG}.
\vspace{0.25em}
\item In lines~\ref{line::alg4::restart_s3} to \ref{line::alg4::restart_e3} of \cref{algorithm::amm_c_x}, $F(\Xkp)$ is guaranteed to yield sufficient improvement over $\sF^{(\sk)}$ compared to $F(\Xkh)$. 
\end{enumerate}
In spite of acceleration, $\ammc$  converges to first-order critical points  as the following proposition states.

\begin{prop}
{\highlight If \cref{assumption::mm,assumption::loss,assumption::neighbor,assumption::master} hold, $\psi>0$ and $\phi>0$,}  then for a sequence of $\{\Xk\}$ and $\{\Xkh\}$ generated by \cref{algorithm::ammc}, we have
\begin{enumerate}[(a)]
	\item\label{prop::ammc1}  $\lF^{(\sk)}$  is nonincreasing;
	\item\label{prop::ammc2}  $F(X^{(\sk)})\rightarrow F^\infty$ and $\lF^{(k)}\rightarrow F^\infty$  as $\sk\rightarrow\infty$;
	\item\label{prop::ammc3} $\|\Xkp-\Xk\|\rightarrow 0$ as $\sk\rightarrow\infty$ if $\xi>0$ and $\zeta>0$;
	\item\label{prop::ammc4} $\|\Xkh-\Xk\|\rightarrow 0$ as $\sk\rightarrow\infty$ if $\zeta\geq\xi>0$;
	\item\label{prop::ammc5} if $\zeta \geq\xi>0$, then there exists $\epsilon > 0$ such that 
	\begin{equation}
		\nonumber
		\min\limits_{0\leq\sk< \mathsf{K}}\|\grad\, F(\Xkh)\|\leq 2\sqrt{\frac{1}{\epsilon}\cdot\dfrac{F(X^{(0)})-F^\infty}{{\sK+1}}}
	\end{equation}
	for any $\mathsf{K}\geq 0$;
	\item\label{prop::ammc6} if $\zeta > \xi> 0$, then 
	$\grad\, F(\Xk)\rightarrow \0$ and
	$\grad\, F(\Xkh)\rightarrow \0$ as $\sk\rightarrow \infty$.
\end{enumerate}
\label{prop::ammc}
\end{prop}
\begin{proof}
	See \hyperapp{appendix::F}{F}.
\end{proof}

\begin{remark}
If $\eta=1$  in \cref{eq::lFk}, $F(\Xk)=\sF^{(\sk)}$, and $F(\Xk)$ is also nonincreasing according to \cref{prop::ammc}\ref{prop::ammc1}. While $F(\Xk)$  might fail to be nonincreasing, we still recommend to choose $\eta\ll 1$  that empirically yields fewer adaptive restarts and faster convergence for distributed PGO.
\end{remark}

\begin{remark}
$\psi>0$ and $\phi>0$ in  \cref{algorithm::amm_c_x} guarantee that $F(\Xkh)$ and $F(\Xkp)$ yield sufficient improvement over $\sF^{(\sk)}$ in terms of $\|\Xkh-\Xk\|$ and $\|\Xkp-\Xk\|$, and are recommended to set close to zero to avoid unnecessary adaptive restarts and make full use of Nesterov's  acceleration.
\end{remark}


%% file: algorithm_amm_c.tex
\begin{algorithm}[t]
	\caption{The $\ammc$ Method}
	\label{algorithm::ammc}
	\begin{algorithmic}[1]
		\State\textbf{Input}: An initial iterate $X^{(0)}\in \XX$, and $\zeta \geq \xi \geq 0$, and $\eta\in(0,\,1]$, and $\psi>0$, and $\phi>0$.
		\State\textbf{Output}: A sequence of iterates $\{X^{(\sk)}\}$ and $\{\Xkh\}$. 	
		\vspace{0.1em}
		\For{node $\alpha\leftarrow 1,\,\cdots,\, |\AA|$}
		\vspace{0.1em}
		\State $X^{\alpha(-1)}\gets X^{\alpha(0)}$ and $s^{\alpha(0)}\gets 1$\label{line::alg3::s0}
		\vspace{0.1em}
		\State send $X^{\alpha(0)}$ to the master node \label{line::alg3::mcomm1}
		\EndFor
		\State evaluate $F(X^{(0)})$ using \cref{eq::obj} at the master node \label{line::alg3::F0} 
		\vspace{0.1em}
		\State $\sF^{(-1)}\leftarrow F(X^{(0)})$ at the master node  \label{line::alg3::lF0}
		\vspace{0.1em}
		\For{$\sk\gets0,\,1,\,2,\,\cdots$}
				\vspace{0.2em}
		\For{node $\alpha\gets 1,\,\cdots,\, |\AA|$}
		\vspace{0.1em}
		\State $s^{\akp}\gets\tfrac{\sqrt{4{s^{\ak}}^2+1}+1}{2}$,\; $\lambda^{\ak}\gets \tfrac{s^{\ak}-1}{s^{\akp}}$\label{line::alg3::sk}
		\vspace{0.15em}
		\State $Y^{\ak}\gets X^{\ak}+\lambda^{\ak}\cdot\big(\Xak-\Xakm\big)$\label{line::alg3::Yk}
		\vspace{0.1em}
		\EndFor
		\vspace{0.1em}
		\State $\sF^{(\sk)}\! \leftarrow \!(1\!-\eta)\cdot\sF^{(\skm)}\!+\eta\cdot F(\Xk)$ at the master node \label{line::alg3::lFk}
		\vspace{-0.75em}
		\State update $\Xkh$ and $\Xkp$ using \cref{algorithm::amm_c_x}
		\EndFor
	\end{algorithmic}
\end{algorithm}

%% file: algorithm_amm_c_x.tex
\begin{algorithm}[t]
	\caption{Updates for the $\ammc$ Method}
	\label{algorithm::amm_c_x}
	\begin{algorithmic}[1]
		\For{node $\alpha\gets 1,\,\cdots,\, |\AA|$}
		\vspace{0.25em}
		\State  retrieve $\Xbk$ and $Y^{\bk}$ from $\beta\in\NN_{\alpha}$\label{line::alg4::comm}
		\vspace{0.1em}
		\State  evaluate $\nGammaak$, $\lnGammaak$, $\nabla_{\Xa} F(\Xk)$, $\nabla_{\Xa} F(\Yk)$ \label{line::alg4::DFaYk}
		\vspace{0.1em}
		\State $\Xakh\gets \arg\min\limits_{X^\alpha\in\XX^\alpha }\lG^\alpha(X^\alpha|\Yk)$ using \cref{algorithm::lG} \label{line::alg4::xlG1}
		\vspace{0.1em}
		\State $\Xakp\gets$ improve $\arg\min\limits_{X^\alpha\in\XX^\alpha }G^\alpha(X^\alpha|\Yk)$ with $\Xakh$ as the initial guess \label{line::alg4::xG1}
		\vspace{0.1em}
		\State  send $\Xakh$ and $\Xakp$ to the master node \label{line::alg4::mcomm1}
		\EndFor
		\vspace{0.1em}
		\State evaluate $F(\Xkh)$ and $F(\Xkp)$ using \cref{eq::obj} \label{line::alg4::Fxk} at the master node
		\vspace{0.1em}
		\If{$F(\Xkh)>\sF^{(\sk)}-\psi\cdot\|\Xkh-\Xk\|^2$}\label{line::alg4::restart_s1}
		\For{node $\alpha\gets1,\,\cdots,\, |\AA|$}
		\vspace{0.1em}
		\State $\Xakh\gets \arg\min\limits_{X^\alpha\in\XX^\alpha }\lG^\alpha(X^\alpha|\Xk)$ using \cref{algorithm::lG} \label{line::alg4::xlG2}
		\vspace{0.1em}
		\State  send $\Xakh$ to the master node \label{line::alg4::mcomm2}
		\EndFor
		\vspace{0.1em}
		\State evaluate $F(\Xkh)$ using \cref{eq::obj}\label{line::alg4::Fxkh} at the master node
		\EndIf\label{line::alg4::restart_e1}
		\vspace{0.1em}
		\If{$F(\Xkp)>\sF^{(\sk)}-\psi\cdot\|\Xkp-\Xk\|^2$}\label{line::alg4::restart_s2}
		\For{node $\alpha\gets1,\,\cdots,\, |\AA|$}
		\vspace{0.1em}
		\State $\Xakp\gets$ improve $\arg\min\limits_{X^\alpha\in\XX^\alpha }\!G^\alpha(X^\alpha|\Xk)$ with $\Xakh$ as the initial guess \label{line::alg4::xG2}
		\vspace{0.1em}
		\State $s^{\akp}\gets \max\{\tfrac{1}{2}s^{\akp},\,1\}$
		\vspace{0.1em}
		\State  send $\Xakp$ to the master node \label{line::alg4::mcomm3}
		\EndFor
		\vspace{0.1em}
		\State evaluate $F(\Xkp)$ using \cref{eq::obj}\label{line::alg4::Fxkp} at the master node
		\EndIf\label{line::alg4::restart_e2}
		\vspace{0.1em}
		\If{$\lF^{(\sk)}-F(\Xkp) < \phi\cdot\Big(\lF^{(\sk)}-F(\Xkh)\Big)$}\label{line::alg4::restart_s3}
		\vspace{0.1em}
		\State $\Xkp\gets \Xkh$ and $F(\Xkp)\gets F(\Xkh)$\label{line::alg4::xG3}
		\vspace{0.1em}
		\EndIf\label{line::alg4::restart_e3}
	\end{algorithmic}
\end{algorithm}

%% file: amm_pgo.tex
{
	\highlight
	The adaptive restart is essential for the convergence of accelerated MM methods. In  $\ammc$  (\cref{algorithm::ammc}), the adaptive restart scheme needs a master node to evaluate $F(\Xkp)$ and $\sF^{(\sk)}$ and guarantee the convergence. On the other hand, if there is no master node, the adaptive restart scheme requires substantial amount of inter-node communication, making  $\ammc$  unscalable for large-scale distributed PGO. Recently, we developed an adaptive restart scheme for distributed PGO that does not require a master node while generating convergent iterates \cite{fan2020mm}.  Nevertheless, the adaptive restart scheme in \cite{fan2020mm}  is conservative and suffers from unnecessary restarts that hinder acceleration and yield slower convergence. Thus, we need to redesign the adaptive restart scheme for distributed PGO without master node to maximize the performance of accelerated MM methods. To address this issue, in \cref{subsection::amm::restart}, we develop a novel adaptive restart scheme that requires no master node and is fully decentralized. Then, in \cref{section::amm::algorithm}, we propose the accelerated MM method for distributed PGO without master that has provable convergence to first-order critical points. In particular, the resulting accelerated MM method, which needs no master node and is fully decentralized, empirically has no loss of computational efficiency in contrast to  $\ammc$  with master node; see \cref{section::experiemnt} for more details. 
}

\vspace{-0.75em}
\subsection{Adaptive Restart}\label{subsection::amm::restart}
{ 
Recall that $\ammc$'s adaptive restart scheme guarantees the convergence by keeping $F(\Xkp)\leq \sF^{(\sk)}$, where the master node only evaluates and compares $F(\Xkp)$ and  $\sF^{(\sk)}$. This suggests that if we could achieve $F(\Xkp)\leq \sF^{(\sk)}$ without evaluating and comparing $F(\Xkp)$ and  $\sF^{(\sk)}$, no master node will be needed. We also  note that if there is a sequence of $\{F^{\ak}\}$ and $\{\lF^{\ak}\}$  for each node $\alpha$ such that
\begin{equation}\label{eq::Fsum}
F(\Xk) = \sum_{\alpha\in\AA} F^{\ak},
\end{equation}
\begin{equation}\label{eq::lFsum}
	\lF^{(\sk)} = \sum_{\alpha\in\AA} \lF^{\ak},
\end{equation}
\begin{equation}\label{eq::FlF}
	F^{\akp} \leq \lF^{\ak},
\end{equation}
then $F(\Xkp)=\sum_{\alpha\in\AA}F^{\akp}\leq \sum_{\alpha\in\AA}\sF^{\ak}=\sF^{(\sk)}$. Therefore,  the sequence above of $\{F^{\ak}\}$ and $\{\lF^{\ak}\}$  is sufficient to  keep $F(\Xkp)\leq \sF^{(\sk)}$  despite that $F(\Xkp)$ and $\lF^{(\sk)}$ are not explicitly evaluated and compared. More importantly, an adaptive restart scheme without master node can be developed with the sequence.  In rest of this section, we will construct  $\{F^{\ak}\}$ and $\{\lF^{\ak}\}$  satisfying \cref{eq::FlF,eq::Fsum,eq::lFsum},  which further results in the adaptive restart scheme for distributed PGO without a master node. }

{ For notational simplicity, we define $\Delta G^\alpha(X|\Xk):\XX\rightarrow\R$ related to the majorization gap of  $G(X|\Xk)$ over $F(X)$:
	\vspace{-0.5em}
\begin{multline}\label{eq::DGa}
	\Delta G^\alpha(X|\Xk)\triangleq -
	\frac{\xi}{2} \big\|\Xa-\Xak\big\|^2 + \\
	\frac{1}{2}\sum_{\beta\in\NN_-^{\alpha}}\sum_{(i,j)\in \aEE^{\ab}}\left(F_{ij}^{\ab}(X)-E_{ij}^{\ab}(X|\Xk)\right)+\\
	\frac{1}{2}\sum_{\beta\in\NN_+^{\alpha}}\sum_{(i,j)\in \aEE^{\ba}}\left(F_{ij}^{\ba}(X)-E_{ij}^{\ba}(X|\Xk)\right)
\end{multline}
where $F_{ij}^{\ab}(X)$, $F_{ij}^{\ba}(X)$, $E_{ij}^{\ab}(X|\Xk)$, $E_{ij}^{\ba}(X|\Xk)$  are given in \cref{eq::Faaabij,eq::Eab}. From $\Delta G^\alpha(X|\Xk)$ in \cref{eq::DGa}, we recursively define $F^{\ak}$, $\sF^{\ak}$, $G^{\ak}$ according to:
\begin{enumerate}[leftmargin=0.45cm]
\item If $\sk=-1$, each node $\alpha$ initializes $F^{\alpha(-1)}$ and $\sF^{\alpha(-1)}$ with
\vspace{-0.75em}
\begin{multline}\label{eq::Fa0}
\!\!\!\!\!\!\!\!\!\!\!\!\!	F^{\alpha(-1)} \!\triangleq \!\!\!\!\sum_{(i,j)\in \aEE^{\alpha\alpha}}\!\! F_{ij}^{\aa}(X^{(0)})+	\frac{1}{2}
\!\sum_{\beta\in\NN_-^{\alpha}}\sum_{(i,j)\in \aEE^{\ab}}\!\! F_{ij}^{\ab}(X^{(0)}) +\\
	\frac{1}{2}\sum_{\beta\in\NN_+^{\alpha}}\sum_{(i,j)\in \aEE^{\ba}}\!\! F_{ij}^{\ba}(X^{(0)}),
\end{multline}
\begin{equation}\label{eq::sFa0}
	\sF^{\alpha(-1)}\triangleq F^{\alpha(-1)}.
\end{equation}
\item If $\sk\geq 0$, each node $\alpha$ recursively updates $G^{\ak}$, $F^{\ak}$ and $\sF^{\ak}$  according to
\begin{equation}\label{eq::Gakp}
	G^{\ak} \triangleq G^\alpha(\Xak|\Xkm) + F^{\akm},
\end{equation}
\begin{equation}\label{eq::Fakp}
	F^{\ak} \triangleq G^{\ak} +  \Delta G^\alpha(\Xk|\Xkm),
\end{equation}
\begin{equation}\label{eq::lFakp}
	\sF^{\ak}  \triangleq (1-\eta)\cdot\sF^{\akm} + \eta\cdot F^{\ak}
\end{equation}
where $\eta\in(0,\,1]$.\\[-1em]  
\end{enumerate}
In \hyperapp{appendix::G}{G}, we have proved that such a sequence of $\{F^{\ak}\}$ and $\{\lF^{\ak}\}$ satisfies \cref{eq::FlF,eq::Fsum,eq::lFsum} as long as $\sG^{\akp} \leq \sF^{\ak}$, which yields the following proposition.}

{
\begin{prop}\label{prop::FsF}
For $G^{\ak}$, $F^{\ak}$, $\sF^{\ak}$ in \cref{eq::sFa0,eq::Fa0,eq::Gakp,eq::Fakp,eq::lFakp}, we have
\begin{enumerate}[(a)]
\item\label{prop::FsF1} $F(\Xk)=\sum_{\alpha\in\AA} F^{\ak}$ where $F(\Xk)$ is given in \cref{eq::obj};\\[-0.8em]
\item\label{prop::FsF2} $\sF^{(\sk)}=\sum_{\alpha\in\AA}\sF^{\ak}$  where $\sF^{(\sk)}$ is given in \cref{eq::lFk};\\[-0.8em]
\item\label{prop::FsF4} $F^{\akp}\leq\sF^{\akp} \leq \sF^{\ak}$ if $\sG^{\akp} \leq \sF^{\ak}$.
\end{enumerate}
\end{prop}
\begin{proof}
	See \hyperapp{appendix::G}{G}.
\end{proof}

It can be concluded from Propositions \ref{prop::FsF}\ref{prop::FsF1} and \ref{prop::FsF}\ref{prop::FsF2} that the resulting $\{F^{\ak}\}$ and $\{\lF^{\ak}\}$ satisfies \cref{eq::Fsum,eq::lFsum}, and \cref{prop::FsF}\ref{prop::FsF4} indicates that \cref{eq::FlF} holds if $\sG^{\akp} \leq \sF^{\ak}$. In \hyperapp{appendix::H}{H}, we have also proved that the following steps  are sufficient to lead to $\sG^{\akp} \leq \sF^{\ak}$:
\begin{enumerate}[leftmargin=0.45cm]
	\item  Update $\Xakp$ by solving \cref{eq::GYak} at node $\alpha$;
	\item Compute $G^{\akp}$ with \cref{eq::Gakp} at node $\alpha$;
	\item If $G^{\akp}>\sF^{\akp}$, update $\Xakp$ again by solving  \cref{eq::xG}  and reduce $s^{\akp)}$  at node $\alpha\in\AA$. 
\end{enumerate}
Then, we not only obtain a sequence of $\{F^{\ak}\}$ and $\{\lF^{\ak}\}$  satisfying \cref{eq::FlF,eq::Fsum,eq::lFsum}, but also an adaptive restart scheme using $G^{\ak}$, $F^{\ak}$, $\sF^{\ak}$ to keep $F(\Xkp)\leq \sF^{(\sk)}$. Note that $F(\Xkp)$ and $\lF^{(\sk)}$ are neither evaluated nor compared. Instead, we evaluate and compare $G^{\akp}$ and $\lF^{\ak}$ independently at each node $\alpha$. Moreover, according to  \cref{eq::Faaabij,eq::GMa,eq::DGa},  it is tedious but straightforward to show that $G^{\ak}$, $F^{\ak}$, $\sF^{\ak}$ in \cref{eq::Fa0,eq::sFa0,eq::Fakp,eq::lFakp,eq::Gakp} can be computed with one inter-node communication round between node $\alpha$ and its neighbors $\beta\in\NN_{\alpha}$. We emphasize that such an adaptive restart scheme differs from those in $\ammc$ and \cite{fan2019proximal,li2015accelerated,zhang2004nonmonotone} that have a master node to  evaluate and compare $F(\Xkp)$ and $\sF^{(\sk)}$. In contrast, the resulting adaptive restart scheme keeps $F(\Xkp)\leq \lF^{(\sk)}$ but needs no master node,  and thus, is well-suited for distributed PGO without master node.
}

\begin{remark}
	\highlight
	$F_{ij}^{\ab}(X)-E_{ij}^{\ab}(X|\Xk)$ and $F_{ji}^{\ba}(X)-E_{ji}^{\ba}(X|\Xk)$ are majorization gaps for inter-node measurements related to nodes $\alpha$ and $\beta$.  According to \cref{eq::DGa}, $\Delta G^{\alpha}(\Xa|\Xk)$ takes half of  $F_{ij}^{\ab}(X)-E_{ij}^{\ab}(X|\Xk)$ and $F_{ji}^{\ba}(X)-E_{ji}^{\ba}(X|\Xk)$ for inter-node measurements in node $\alpha$.  Then, \cref{eq::Fakp} uses $\Delta G^\alpha(\Xa|\Xk)$ to compute $F^{\ak}$ where majorization gaps $F_{ij}^{\ab}(X)-E_{ij}^{\ab}(X|\Xk)$ and $F_{ji}^{\ba}(X)-E_{ji}^{\ba}(X|\Xk)$ are evenly redistributed between nodes $\alpha$ and $\beta$. 
\end{remark}

\input{algorithm_amm}

\vspace{-0.5em}
\subsection{Algorithm}\label{section::amm::algorithm}
With the adaptive restart scheme using $G^{\ak}$, $F^{\ak}$, $\sF^{\ak}$ to keep $F(\Xkp)\leq \sF^{(\sk)}$, we obtain the $\ammd$ method (\cref{algorithm::amm}) for distributed PGO, where ``$\#$'' indicates that no master node is needed.

The outline of  $\ammd$  is similar to  $\ammc$  and the key difference is the adaptive restart scheme:

\begin{enumerate}[leftmargin=0.45cm]
	
\item In lines~\ref{line::alg5::comm1}, \ref{line::alg5::comm} of \cref{algorithm::amm}, each node $\alpha$ performs one inter-node communication round to retrieve $\Xbk$ and $Y^{\bk}$ from its neighbors $\beta\in\NN^\alpha$. We remark that no other inter-node communication is required.

\item In lines~\ref{line::alg5::Fakp2}, \ref{line::alg5::lFa0}, \ref{line::alg5::Fakp1}, \ref{line::alg5::lFakp} of \cref{algorithm::amm} and lines~\ref{line::alg6::Gakh1}, \ref{line::alg6::Gakp1}, \ref{line::alg6::Gakh2}, \ref{line::alg6::Gakp2} of \cref{algorithm::amm_x}, each node $\alpha$  evaluates $F^{\ak}$, $\sF^{\ak}$, $G^{\akh}$, $G^{\akp}$ that are used for adaptive restart. Note that $\Xbk$ and $X^{\bkm}$ from node $\alpha$'s neighbors  $\beta\in\NN^\alpha$ are needed.

\item In lines~\ref{line::alg6::restart1} to \ref{line::alg6::restart4} of \cref{algorithm::amm_x}, each node $\alpha$ performs independent adaptive restart such that $G^{\akh}\leq \lF^{\ak}$ and $G^{\akp}\leq \lF^{\ak}$, which also results in $F(\Xkp)\leq \sF^{(\sk)}$ and a nonincreasing sequence of $\sF^{(\sk)}$ for distributed PGO without master node. 
\item In lines~\ref{line::alg6::check1} to \ref{line::alg6::check2} of \cref{algorithm::amm_x}, $G^{\akp}$ is guaranteed to yield sufficient improvement over $\lF^{\ak}$ compared to $G^{\akh}$.
\end{enumerate}
Furthermore, $\ammd$  converges to first-order critical points as the following propositions states. 

\input{algorithm_amm_x}

\begin{prop}\label{prop::amm}
	{\highlight If \cref{assumption::mm,assumption::loss,assumption::neighbor} hold, $\psi>0$ and $\phi>0$,} then for a sequence of  $\{X^{(\sk)}\}$ and $\{X^{(\skh)}\}$  generated by \cref{algorithm::amm}, we have
	\begin{enumerate}[(a)]
		\item\label{prop::amm1}  $\lF^{(\sk)}$  is nonincreasing;
		\item\label{prop::amm2} $F(X^{(\sk)})\rightarrow F^\infty$ and $\lF^{(k)}\rightarrow F^\infty$  as $\sk\rightarrow\infty$;
		\item\label{prop::amm3} $\|\Xkp-\Xk\|\rightarrow 0$ as $\sk\rightarrow\infty$ if $\zeta>\xi>0$;
		\item\label{prop::amm4} $\|\Xkh-\Xk\|\rightarrow 0$ as $\sk\rightarrow\infty$ if $\zeta>\xi>0$;
		\item\label{prop::amm5} if $\zeta \geq\xi>0$, then there exists $\epsilon > 0$ such that 
		\begin{equation}
			\nonumber
			\min\limits_{0\leq\sk< \mathsf{K}}\|\grad\, F(\Xkh)\|\leq 2\sqrt{\frac{1}{\epsilon}\cdot\dfrac{F(X^{(0)})-F^\infty}{{\sK+1}}}
		\end{equation}
		for any $\mathsf{K}\geq 0$;
		\item\label{prop::amm6} if $\zeta \geq\xi>0$, then $\grad\, F(\Xk)\rightarrow \0$ and $\grad\, F(\Xkh)\rightarrow \0$ as $\sk\rightarrow \infty$.
	\end{enumerate}
\end{prop}
\begin{proof}
	See \hyperapp{appendix::H}{H}.
\end{proof}

In spite of no master node, \cref{prop::amm} indicates that  $\ammd$  has provable convergence as long as each node $\alpha\in\AA$ can communicate with its neighbors $\beta\in \NN^\alpha$. Thus,  $\ammd$  eliminates the bottleneck of communication for distributed PGO without master node. In addition,  $\ammd$  also reduces unnecessary adaptive restarts compared to \cite{fan2020mm}, and thus makes better of Nesterov's acceleration and has faster convergence.  

%% file: algorithm_amm.tex
\begin{algorithm}[t]
	\caption{The $\ammd$ Method}
	\label{algorithm::amm}
	\begin{algorithmic}[1]
		\State\textbf{Input}: An initial iterate $X^{(0)}\in \XX$, and $\eta\in(0,\,1]$, and $\zeta \geq \xi \geq 0$, and $\psi>0$, and $\phi>0$.
		 \vspace{0.1em}
		\State\textbf{Output}: A sequence of iterates $\{X^{(\sk)}\}$ and $\{\Xkh\}$.\vspace{0.2em} 	
		\vspace{0.1em}
		\For{node $\alpha\gets1,\,\cdots,\, |\AA|$}
		\State $X^{\alpha(-1)}\gets X^{\alpha(0)}$ and $s^{\alpha(0)}\gets 1$\label{line::alg5::s0}
		\vspace{0.1em}
		\State retrieve $X^{\beta(-1)}$ and $X^{\beta(0)}$ from $\beta\in\NN_{\alpha}$\label{line::alg5::comm1}
		\vspace{0.1em}
		\State evaluate $F^{\alpha(-1)}$ using \cref{eq::Fa0}\label{line::alg5::Fakp2}
		\vspace{0.1em}
		\State evaluate $\sF^{\alpha(-1)}$ using \cref{eq::sFa0}\label{line::alg5::lFa0}
		\vspace{0.1em}
		\State $G^{\alpha(0)} \leftarrow G^\alpha(X^{\alpha(0)}|X^{(-1)}) + F^{\alpha(-1)}$\label{line::alg5::G}
		\vspace{0.1em}
		\EndFor
		\For{$\sk\gets0,\,1,\,2,\,\cdots$}
		\For{node $\alpha\gets 1,\,\cdots,\, |\AA|$}
		\vspace{0.1em}
		\State $s^{\akp}\gets\tfrac{\sqrt{4{s^{\ak}}^2+1}+1}{2}$,\; $\lambda^{\ak}\gets \tfrac{s^{\ak}-1}{s^{\akp}}$\label{line::alg5::sk}
		\vspace{0.1em}
		\State $Y^{\ak}\gets X^{\ak}+\lambda^{\ak}\cdot\big(\Xak-X^{\akm}\big)$\label{line::alg5::Yk}
		\vspace{0.1em}
		\State retrieve $\Xbk$ and $Y^{\bk}$ from $\beta\in\NN_{\alpha}$\label{line::alg5::comm}
		\vspace{0.1em}
		\State $F^{\ak}\leftarrow G^{\ak} + \Delta G^{\alpha}(\Xk|\Xkm)$ using \cref{eq::DGa,eq::Gakp} \label{line::alg5::Fakp1}
		\vspace{0.1em}
		\State $\sF^{\ak}\gets (1-\eta)\cdot \sF^{\akm} + \eta\cdot F^{\ak}$\label{line::alg5::lFakp}
		\vspace{0.1em}
		\State 	update $\Xakh$ and $\Xakp$ using \cref{algorithm::amm_x}
		\EndFor
		\EndFor
	\end{algorithmic}
\end{algorithm}

%% file: algorithm_amm_x.tex
\begin{algorithm}[t]
	\caption{Updates for the $\ammd$ Method}
	\label{algorithm::amm_x}
	\begin{algorithmic}[1]
		\State evaluate $\wabijk$ and $\wbajik$ using \cref{eq::wabij} \label{line::alg6::wabij}
		\vspace{0.1em}
				\State  evaluate $\nGammaak$, $\lnGammaak$, $\nabla_{\Xa} F(\Xk)$, $\nabla_{\Xa} F(\Yk)$ in \cref{eq::GMa,eq::lGMa} \label{line::alg6::DFak}
		\vspace{0.1em}
		\State $\Xakh\gets\arg\min\limits_{\Xa\in\XX^\alpha }\lG^\alpha(\Xa|Y^{(\sk)})$ using \cref{algorithm::lG}\label{line::alg6::Xakh1}
		\State $\sG^{\akh}\gets G^\alpha(\Xakh|\Xk) + F^{\ak}$\label{line::alg6::Gakh1}
		\vspace{0.1em}
		\State $\Xakp\gets$ improve $\arg\min\limits_{\Xa\in\XX^\alpha }G^\alpha(\Xa|\Yk)$ with $\Xakh$ as the initial guess\label{line::alg6::Xakp1}
		\vspace{0.1em}
		\State $\sG^{\akp}\gets G^\alpha(\Xakp|\Xk) + F^{\ak} $ \label{line::alg6::Gakp1}
		\vspace{0.1em}
		\If{$\sG^{\akh} > \sF^{\ak}-\psi\cdot\|\Xakh-\Xak\|^2$}\label{line::alg6::lGX1}\label{line::alg6::restart1}
		\vspace{0.1em}
		\State$\Xakh\gets\arg\min\limits_{X^\alpha\in\XX^\alpha }\lG^\alpha(X^\alpha|\Xk)$\label{line::alg6::Xkh} using \cref{algorithm::lG}\label{line::alg6::Xakh2}
		\vspace{0.1em}
		\State $\sG^{\akh}\gets G^\alpha(\Xakh|\Xk) + F^{\ak}$\label{line::alg6::Gakh2}
		\vspace{0.1em}
		\EndIf\label{line::alg6::restart2}
		\vspace{0.1em}
		\If{$\sG^{\akp} > \sF^{\ak}$}\label{line::alg6::restart3}
		\vspace{0.1em}
		\State $X^{\alpha(\sk+1)}\!\gets$ improve $\arg\min\limits_{X^\alpha\in\XX^\alpha }G^\alpha(X^\alpha|\Xk)$ with $\Xakh$ as the initial guess\label{line::alg6::Xakp2}
		\vspace{0.1em}
		\State $\sG^{\akp}\gets G^\alpha(\Xakp|\Xk) + F^{\ak}$\label{line::alg6::Gakp2}
		\vspace{0.1em}
		\State $s^{\akp}\gets \max\{\tfrac{1}{2}s^{\akp},\,1\}$
		\vspace{0.1em}
		\EndIf\label{line::alg6::restart4}
		\vspace{0.1em}
		\If{$\lF^{\ak}-\sG^{\akp} < \phi\cdot\Big(\lF^{\ak}-\sG^{\akh}\Big)$}\label{line::alg6::check1}
		\State $\Xakp \gets \Xakh$ and $\sG^{\akp}\gets\sG^{\akh}$\label{line::alg6::Xakp3}
		\EndIf\label{line::alg6::check2}
	\end{algorithmic}
\end{algorithm}

%% file: experiments.tex
In this section, we evaluate the performance of our MM methods ($\mm$, $\ammc$ and $\ammd$) for distributed PGO on the simulated \textsf{\small Cube} datasets and a number of 2D and 3D SLAM benchmark datasets \cite{rosen2016se}.  In terms of $\mm$, $\ammc$ and $\ammd$, $\eta$, $\xi$, $\zeta$, $\psi$ and $\phi$ in \cref{algorithm::amm,algorithm::mm,algorithm::ammc} are $5\times10^{-4}$, $1\times10^{-10}$, $1.5\times 10^{-10}$, $1\times 10^{-10}$ and $1\times10^{-6}$, respectively, for all the experiments. In addition, $\mm$, $\ammc$ and $\ammd$ can  take at most one iteration when solving \cref{eq::xG,eq::GYak} to improve the estimates.   All the experiments have been performed on a laptop with an Intel Xeon(R) CPU E3-1535M v6 and 64GB of RAM running Ubuntu 20.04.

\vspace{-0.5em}

\subsection{\textsf{\small Cube} Datasets}
\input{fig_cube}

\input{fig_cube_results}

In this section, we test and evaluate our MM methods for distributed PGO on $20$ simulated \textsf{\small Cube} datasets (see \cref{fig::cube})  with $5$, $10$ and $50$ robots. In the experiment, a simulated {\sf\small Cube} dataset has $12 \times 12 \times 12$ cube grids with $1$ m side length, a path of $3600$ poses along the rectilinear edge of the cube grid, odometric measurements between all the pairs of sequential poses, and loop-closure measurements between nearby but non-sequential poses that are randomly available with a probability of $0.1$. We generate the odometric and loop-closure measurements according to the noise models in \cite{rosen2016se} with an expected translational RMSE  of $0.02$ m and an expected angular RMSE of $0.02\pi$ rad. The centralized chordal initialization \cite{carlone2015initialization} is implemented such that distributed PGO with different number of robots have the same initial estimate. The maximum number of iterations is $1000$.

We evaluate the convergence of $\mm$, $\ammc$ and $\ammd$ in terms of the relative suboptimality gap and Riemannian gradient norm. For reference, we also make comparisons against $\amm$ \cite{fan2020mm}. Note that $\amm$ is the original accelerated MM method for distributed PGO whose adaptive restart scheme is conservative and might  prohibit Nesterov's acceleration.

\textbf{Relative Suboptimality Gap.} We implement the certifiably-correct $\sesync$ \cite{rosen2016se} to get the globally optimal objective value $F^*$  of distributed PGO with the trivial loss kernel (\cref{example::trivial}), making it possible to compute the relative suboptimality gap $(F-F^*)/F^*$ where $F$ is the objective value for each iteration. The results are in \cref{fig::cube_f}.

\textbf{Riemannian Gradient Norm.}  We also compute the Riemannian gradient norm of distributed PGO with the trivial, Huber  and Welsch loss kernels in \cref{example::huber,example::GM,example::trivial} for evaluation. Note that it is difficult to find globally optimal solutions to distributed PGO if Huber and Welsch loss kernels are used. The results are in \cref{fig::cube_g_huber,fig::cube_g_gm,fig::cube_g_trivial}.

In \cref{fig::cube_g_gm,fig::cube_g_huber,fig::cube_f,fig::cube_g_trivial}, it can be seen that $\mm$, $\ammc$, $\ammd$ and $\amm$ have a faster convergence if the number of robots (nodes) decreases. This is expected since $G(X|X^{(k)})$ and $H(X|X^{(k)})$ in \cref{eq::G,eq::lG} result in tighter approximations for distributed PGO with fewer robots (nodes). In addition, \cref{fig::cube_g_gm,fig::cube_g_huber,fig::cube_g_trivial} suggest that the convergence rate of $\mm$, $\ammc$, $\ammd$ and $\amm$ also relies on the type of loss kernels. Nevertheless, $\ammc$, $\ammd$ and $\amm$ accelerated by Nesterov's method outperform  unaccelerated $\mm$  by a large margin for any number of robots and any types of loss kernels, which means that Nesterov's method improves the convergence of distributed PGO. In particular, Figs. \ref{fig::cube_f}(a), \ref{fig::cube_g_trivial}(a), \ref{fig::cube_g_huber}(a), \ref{fig::cube_g_gm}(a) indicate that $\ammd$  with $50$ robot still converges faster than $\mm$ with $5$ robots despite that the later has a much smaller number of robots. Therefore, we conclude that Nesterov's method accelerates the convergence of distributed PGO.

We emphasize the convergence comparisons of Nesterov's accelerated $\ammc$, $\ammd$ and $\amm$ that merely differ from each other by the adaptive restart schemes---$\ammc$ has an additional master node to aggregate information from all the robots (nodes), whereas $\ammd$ and $\amm$ are restricted to one inter-node communication round per iteration among neighboring robots (nodes). Notwithstanding limited local communication, as is shown in \cref{fig::cube_f,fig::cube_g_huber,fig::cube_g_gm}, $\ammd$ has a convergence rate comparable to that of $\ammc$ using a master node while being significantly faster than $\amm$.  {\highlight In particular,  $\ammd$ reduces adaptive restarts by $80\%$ to $95\%$ compared to $\amm$ on the \textsf{\small Cube} datasets}, and thus, is expected to make better use of Nesterov's acceleration. Since $\ammd$ and $\amm$ differ in the adaptive restart schemes, we attribute the faster convergence of $\ammd$ to its redesigned adaptive restart scheme. These results suggest that $\ammd$ is advantageous over other methods for very large-scale distributed PGO where computational and communicational efficiency are equally important.

\vspace{-0.5em}

\vspace{-0.5em}
\subsection{Benchmark Datasets}\label{subsection::experiment::benchmark}
In this section, we evaluate our MM methods ($\mm$, $\ammc$ and $\ammd$) for distributed PGO on a number of 2D and 3D SLAM benchmark datasets \cite{rosen2016se} (see \datasetinfo). We use the trivial loss kernel and there are no outliers such that the globally optimal solution can be exactly computed with $\sesync$ \cite{rosen2016se}. For each dataset, we also make comparisons against $\sesync$ \cite{rosen2016se}, distributed Gauss-Seidel ($\dgs$) \cite{choudhary2017distributed} and the Riemannian block coordinate descent ($\rbcd$) \cite{tian2019distributed} method, all of which are the state-of-the-art algorithms for centralized and distributed PGO. The $\sesync$ and $\dgs$ methods use the recommended settings in \cite{choudhary2017distributed,rosen2016se}. We implement two Nesterov's accelerated variants of $\rbcd$  \cite{tian2019distributed}, i.e., one with greedy selection rule and adaptive restart ($\rbcdc$) and the other with uniform selection rule and fixed restart ($\rbcdd$)\footnote{In the experiments, we run $\rbcdd$ \cite{tian2019distributed} with fixed restart frequencies of 30, 50 and 100 iterations for each dataset and report the best results.}. As mentioned before, $\ammc$ and $\ammd$ can take at most one iteration when updating $\Xakp$ using \cref{eq::GYak,eq::xG}, which is similar to $\rbcdc$ and $\rbcdd$. An overview of the aforementioned methods is given in \cref{table::cmp_method}.

\input{table_cmp_method}

\input{table_cmp_dataset}

\textbf{Number of Iterations.} First, we examine the convergence of $\mm$, $\ammc$, $\ammd$, $\dgs$ \cite{choudhary2017distributed}, $\rbcdc$ \cite{tian2019distributed}  and $\rbcdd$ \cite{tian2019distributed} w.r.t. the number of iterations. The distributed PGO has 10 robots and all the methods are initialized with distributed Nesterov's accelerated chordal initialization \cite{fan2020mm}.

The objective values of each method with 100, 250 and 1000 iterations are reported in \cref{table::comp} and the reconstruction results using $\ammd$ are shown in \refbenchmark{1}{2}. For almost all the benchmark datasets, $\ammc$ and $\ammd$ outperform the other methods ($\mm$, $\dgs$, $\rbcdc$ and $\rbcdd$). While $\rbcdc$ and $\rbcdd$ have similar performances in four relatively  easy datasets---{\sf CSAIL}, {\sf sphere}, {\sf torus} and {\sf grid}---$\ammc$ and $\ammd$ achieve much better results in the other more challenging datasets in particular if there are no more than 250 iterations. As discussed later, $\ammc$ and $\ammd$ have faster convergence to more accurate estimates without any extra computation and communication in contrast to $\rbcdc$ and $\rbcdd$. Last but not the least, \cref{table::comp} demonstrates that the accelerated $\ammc$ and $\ammd$ converge significantly faster than the unaccelerated $\mm$, which further validates the usefulness of Nesterov's method.

\input{fig_benchmark}

\input{fig_succ_iter}

We also compute the performance profiles \cite{dolan2002benchmarking} based on the number of iterations. Given a  tolerance $\Delta\in(0,\,1]$, the objective value threshold $F_{\Delta}(p)$ of a PGO problem $p$
is 
\vspace{-0.15em}
\begin{equation}\label{eq::Fdel}
F_{\Delta}(p) = F^* + \Delta\cdot\big(F^{(0)}-F^*\big)
\vspace{-0.15em}
\end{equation}
where  $F^{(0)}$ and $F^*$ are the initial and globally optimal objective values, respectively. Let $I_{\Delta}(p)$ denote the minimum number of iterations that  a PGO method takes to reduce the objective value to $F_\Delta(p)$, i.e.,
\vspace{-0.15em}
\begin{equation}
\nonumber
I_{\Delta}(p)\triangleq\min_{\sk}\big\{\sk\geq 0| F^{(\sk)}\leq F_{\Delta}(p)\big\}
\vspace{-0.15em}
\end{equation}
where $F^{(\sk)}$ is the objective value at iteration $\sk$. Then, for a problem set $\mathcal{P}$, the performance profiles of a PGO method is the percentage of problems solved w.r.t. the number of iterations $\sk$:
\vspace{-0.25em}
\begin{equation}
\nonumber
\substack{\text{\normalsize percentage of problems solved}\\ \text{\normalsize at iteration $\sk$}} \triangleq \frac{\big|\{p\in\mathcal{P}|I_{\Delta}(p)\leq \sk\}\big|}{|\mathcal{P}|}.
\end{equation}

The performance profiles based on the number of iterations over a variety of 2D and 3D SLAM benchmark datasets (see \datasetinfo) are shown in \cref{fig::succ_iter}. The tolerances evaluated are $\Delta=1\times10^{-2}$, $5\times10^{-3}$, $1\times10^{-3}$ and $1\times10^{-4}$. We report the performance of $\mm$, $\ammc$, $\ammd$, $\dgs$ \cite{choudhary2017distributed}, $\rbcdc$ \cite{tian2019distributed} and $\rbcdd$ \cite{tian2019distributed} for distributed PGO with 10 robots (nodes). As expected, $\ammc$ and $\ammc$ dominates the other methods ($\mm$, $\dgs$, $\rbcdc$ and $\rbcdd$) in terms of the convergence for all the tolerances $\Delta$, which means that  both of them are better choices for distributed PGO.

In \cref{table::comp} and \cref{fig::succ_iter}, we emphasize that $\ammd$ requiring no master node achieves comparable performance to that of $\ammc$ using a master node, and is a lot better than all the other methods with a master node ($\rbcdc$) and without ($\mm$, $\dgs$ and $\rbcdd$). Even though $\rbcdc$ and $\rbcdd$ are similarly accelerated with Nesterov's method, we remark that $\rbcdd$ without a master node suffers a great performance drop compared to $\rbcdc$, and more importantly, $\rbcdd$ has no convergence guarantees to first-order critical points. These results reverify that $\ammd$ is more suitable for very large-scale distributed PGO with limited local communication. 

 Note that all of $\mm$, $\ammc$, $\ammd$, $\dgs$ \cite{choudhary2017distributed}, $\rbcdc$ \cite{tian2019distributed} and $\rbcdd$ \cite{tian2019distributed} have to exchange poses of inter-node measurements with the neighbors, and thus, need almost the same amount of communication per iteration. However, \cref{fig::succ_iter} indicates that $\ammc$ and $\ammd$ have much faster convergence in terms of the number of iterations, which also means less communication for the same level of accuracy. In addition, $\rbcdc$ and $\rbcdd$ have to keep part of the nodes in idle during optimization and rely on red-black coloring for block aggregation and random sampling for block selection, which induce additional computation and communication. In contrast, neither $\ammc$ nor $\ammd$ has any extra practical restrictions except \cref{assumption::loss,assumption::mm,assumption::neighbor,assumption::master}.

\input{fig_succ_time}

\textbf{Optimization Time.} We evaluate the optimization time of $\ammc$ and $\ammd$ with different numbers of robots (nodes) against the centralized baseline $\sesync$ \cite{rosen2016se}. To improve the  time efficiency of our methods, $\Xakp$ in \cref{eq::xG,eq::GYak}  uses the same rotation as $\Xakh$ in \cref{eq::xlG,eq::HYak} and merely updates the translation. Due to the different numbers of robots (nodes), the centralized chordal initialization \cite{carlone2015initialization} is used for all the runs.

Similar to the number of iterations, we use the performance profiles to evaluate $\ammc$ and $\ammd$ in terms of the optimization time. Recall from \cref{eq::Fdel} the objective value threshold $F_{\Delta}(p)$ where $p$ is the PGO problem and $\Delta\in(0,\,1]$ is the tolerance. Since the average optimization time per node is directly related with the speedup, we measure the efficiency of a distributed PGO method with $N$ nodes by computing the average optimization time $T_{\Delta}(p,N)$ that each node takes to reduce the objective value to $F_{\Delta}(p)$:
\vspace{-0.25em}
\begin{equation}
\nonumber
T_{\Delta}(p,N)=\tfrac{T_\Delta(p)}{N},
\vspace{-0.25em}
\end{equation}
where $T_\Delta(p)$ denotes the total optimization time of all the $N$ nodes. We remark that the centralized optimization method has $N=1$ node and $T_{\Delta}(p,N)=T_{\Delta}(p)$. Let $T_{\sesync}$ denote the optimization time that $\sesync$ needs to find the globally optimal solution. The performance profiles assume a distributed PGO method solves problem $p$  for some $\ratio\in[0,\,+\infty)$ if $T_{\Delta}(p,N)\leq \ratio\cdot T_{\sesync}$. Note that $\ratio$ is the scaled average optimization time per node and $\sesync$ solves problem $p$ globally at $\ratio=1$. Then, as a result of \cite{dolan2002benchmarking}, the performance profiles evaluate the speedup of distributed PGO methods for a given optimization problem set $\mathcal{P}$ using  the percentage of problems solved w.r.t. the scaled average optimization time per node $\ratio\in[0,\,+\infty)$: 
\vspace{-0.25em}
\begin{equation}
	\nonumber
	\substack{\text{\normalsize percentage of problems}\\ \text{\normalsize  solved at $\ratio$}} \triangleq\! \frac{\big|\{p\in\mathcal{P}|T_{\Delta}(p,\,N)\!\leq\! \ratio\cdot T_{\sesync}\}\big|}{|\mathcal{P}|}.
\end{equation}

\cref{fig::succ_time} shows the performance profiles based on the scaled average optimization time per node. The tolerances evaluated are $\Delta=1\times10^{-2}$, $1\times10^{-3}$, $1\times10^{-4}$ and $1\times10^{-5}$. We report the performance of $\ammc$ and $\ammd$ with $10$, $25$ and $100$ robots (nodes). For reference, we also evaluate the performance profile of the centralized PGO baseline $\sesync$ \cite{rosen2016se}.  As the results demonstrate, $\ammc$ and $\ammd$ are significantly faster than $\sesync$ \cite{rosen2016se} in most cases for modest accuracies of $\Delta=1\times 10^{-2}$ and $\Delta=1\times 10^{-3}$, for which the only challenging case is the {\sf CSAIL} dataset, whose chordal initialization is already very close to the globally optimal solution.  In spite of the performance decline for smaller tolerances of $\Delta=1\times 10^{-4}$ and $\Delta=1\times 10^{-5}$, $\ammc$ and $\ammd$ with 100 robots (nodes) still achieve a $2.5\sim 20\mathrm{x}$ speedup of optimization time over $\sesync$ for more than $70\%$ of the benchmark datasets, not to mention that the average optimization time per node of $\ammc$ and $\ammd$ decreases with the number of robots (nodes). {\highlight Note that the communication overhead is not considered in the experiments.} Nevertheless  \cref{fig::succ_time} indicates that $\ammc$ and $\ammd$ are promising  as fast parallel backends for very large-scale PGO and real-time multi-robot SLAM.

In summary, $\ammc$ and $\ammd$ achieve the state-of-the-art performance for distributed PGO and enjoy a significant multi-node speedup compared to the centralized baseline \cite{rosen2016se} for modestly but sufficiently accurate estimates.

\subsection{Robust Distributed PGO} 

In this section, we evaluate the robustness of $\ammd$ against the outlier inter-node loop closures. Similar to \cite{chang2020kimera,lajoie2020door}, we first use the distributed pairwise consistent measurement set maximization algorithm ($\pcm$) \cite{mangelson2018pairwise} to reject spurious inter-node loop closures and then solve the resulting distributed PGO using $\ammd$ with the trivial, Huber  and Welsch loss kernels in \cref{example::GM,example::trivial,example::huber} . 

We implement $\ammd$ on the 2D {\sf intel} and 3D {\sf garage} datasets (see \datasetinfo) with 10 robots (nodes). For each dataset, we add false inter-node loop closures with uniformly random rotation and translation errors in the range of $[0,\,\pi]$ rad and $[0,\,5]$ m, respectively. In addition, after the initial outlier rejection using the $\pcm$ algorithm \cite{mangelson2018pairwise}, we initialize $\ammd$ with distributed Nesterov's accelerated chordal initialization \cite{fan2020mm} for all the loss kernels.

The absolute trajectory errors (ATE)  of $\ammd$ w.r.t. different {\highlight outlier ratios} of inter-node loop closures are  in \cref{fig::outlier}. The ATEs are computed against the outlier-free results of $\sesync$ \cite{rosen2016se}  and  averaged over 10 Monte Carlo runs. 

\input{fig_outlier}

In \cref{fig::outlier}(a), $\pcm$ \cite{mangelson2018pairwise} rejects most of the outlier inter-node loop closure for the {\sf intel} dataset and $\ammd$ solves the distributed PGO problems regardless of the loss kernel types and {\highlight outlier ratios}. Note that $\ammd$ with the Welsch loss kernel has larger ATEs (avg. $0.057$ m) against $\sesync$ \cite{rosen2016se} than those with the trivial and Huber loss kernels (avg. $0.003$ m), and we argue that this is related to the loss kernel types. The ATEs are evaluated based on $\sesync$ using the trivial loss kernel, which is in fact identical/similar to distributed PGO with the trivial and Huber loss kernels but different from that with the Welsch loss kernel. Thus, the estimates from the trivial and Huber loss kernels are expected to be more close to those of $\sesync$, which result in smaller ATEs compared to the Welsch loss kernel  if no outliers.

\input{fig_garage}

For the more challenging {\sf garage} dataset, as is shown in \cref{fig::outlier}(b), $\pcm$ fails for {\highlight outlier ratios} over $0.4$, and further, distributed PGO with the trivial and Huber loss kernels results in ATEs as large as $65$ m. In contrast, distributed PGO with the Welsch loss kernel still successfully estimates the poses with an average ATE of $2.5$ m despite the existence of  outliers---note that the {\sf garage} dataset has a trajectory over $7$ km. For the {\sf garage} dataset, a qualitative comparison of distributed PGO with different loss kernels is also presented in \refgarage, where the Welsch loss kernel still has the best performance. The results are not surprising since the Welsch loss kernel is known to be more robust against outliers than the other two loss kernels \cite{barron2019cvpr}.

The results above indicate that our MM methods can be applied to distributed PGO in the presence of outlier inter-node loop closures when combined with robust loss kernels like Welsch and other outlier rejection techniques like $\pcm$ \cite{mangelson2018pairwise}. In addition, we emphasize again that our MM methods have provable convergence to first-order critical points for a broad class of robust loss kernels, whereas the convergence guarantees of existing distributed PGO methods \cite{choudhary2017distributed,tian2019distributed,tron2014distributed,eric2020geod} are restricted to the trivial loss kernel.

%% file: fig_cube.tex
\begin{figure}[b]
	\centering
	\includegraphics[width=0.175\textwidth]{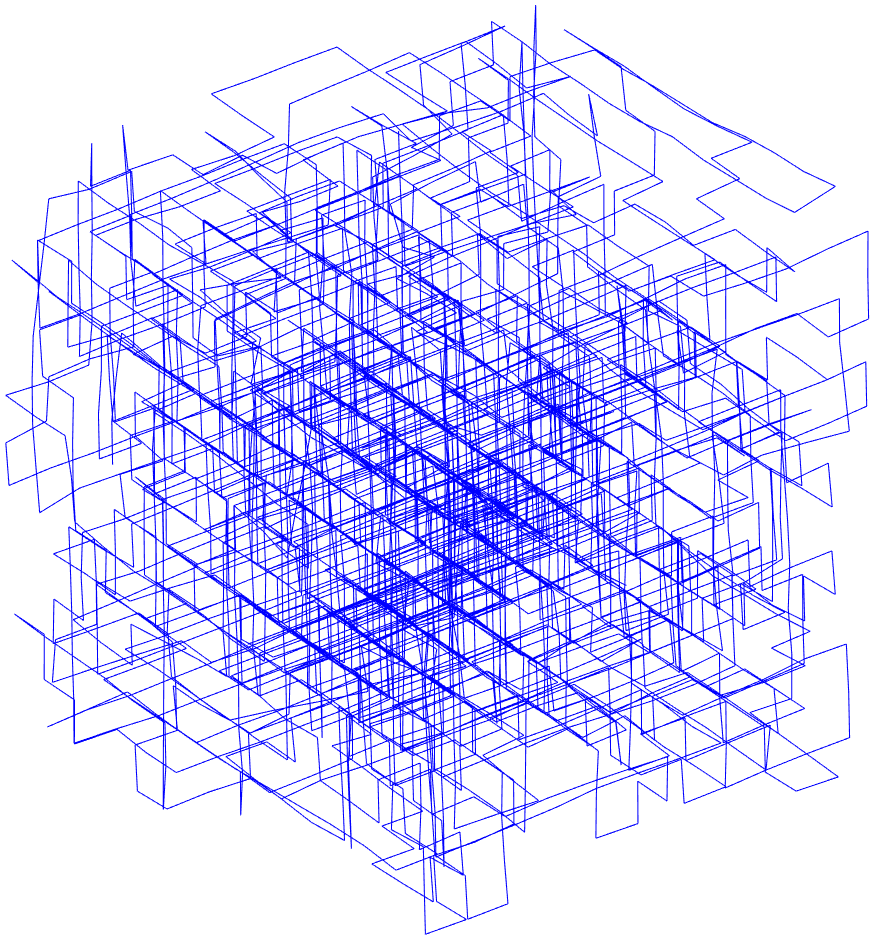}
	\caption{A {\sf\small Cube} dataset has $12\times12\times12$ grids of side length of $1$ m, $3600$ poses, loop closure probability of  $0.1$, an translational RSME of $ 0.02$ m and an angular RSME of $0.02\pi$ rad.}\label{fig::cube}
	\vspace{-0.75em}
\end{figure}

%% file: fig_cube_results.tex
\begin{figure*}[t]
	\centering
	\begin{tabular}{cccc}
		\hspace{-0.5em}\subfloat[][$\ammd$ vs. $\mm$ ]{\includegraphics[trim =0mm 0mm 0mm 0mm,width=0.24\textwidth]{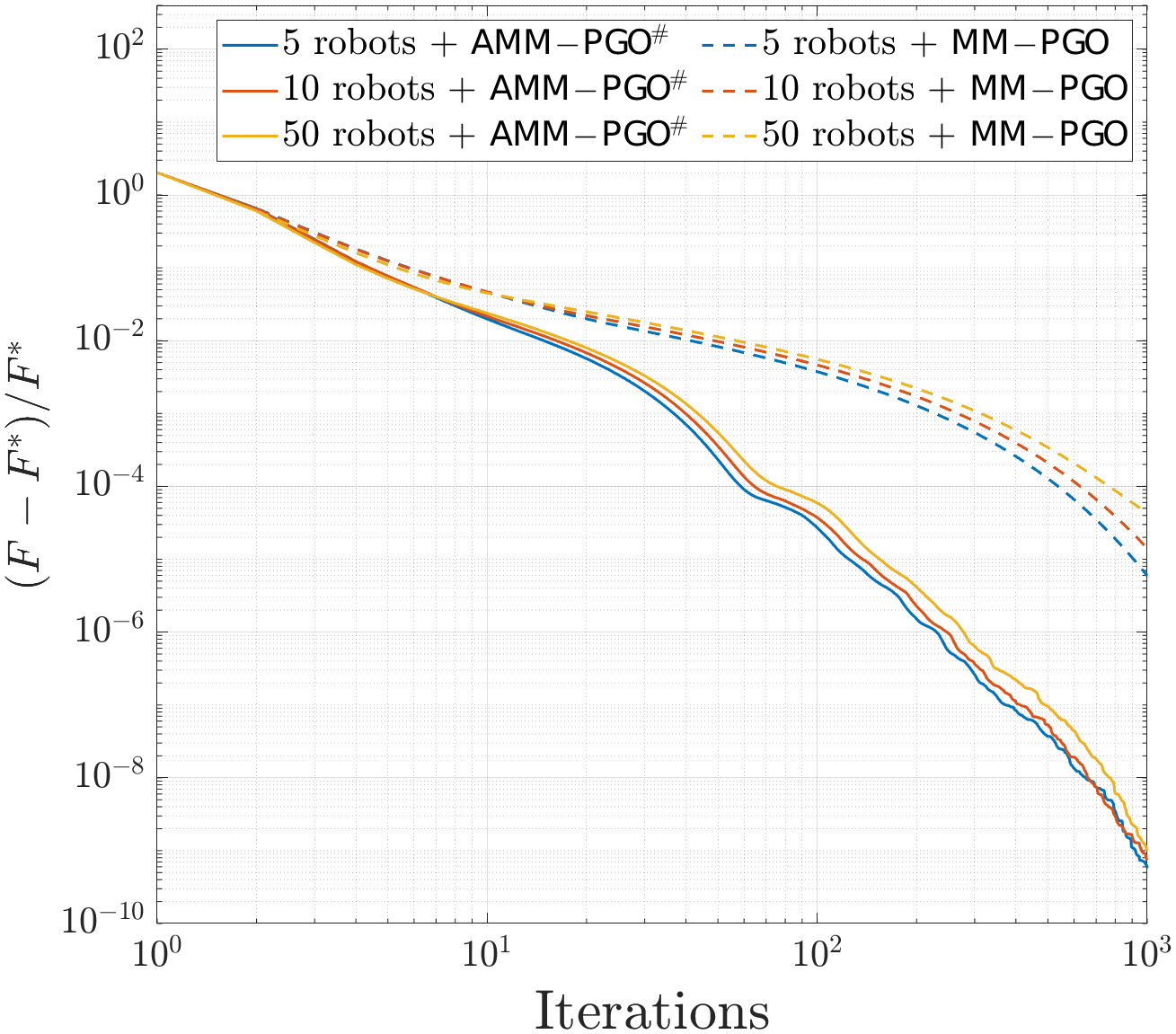}} &
		\hspace{-0.6em}\subfloat[][5 robots]{\includegraphics[trim =0mm 0mm 0mm 0mm,width=0.24\textwidth]{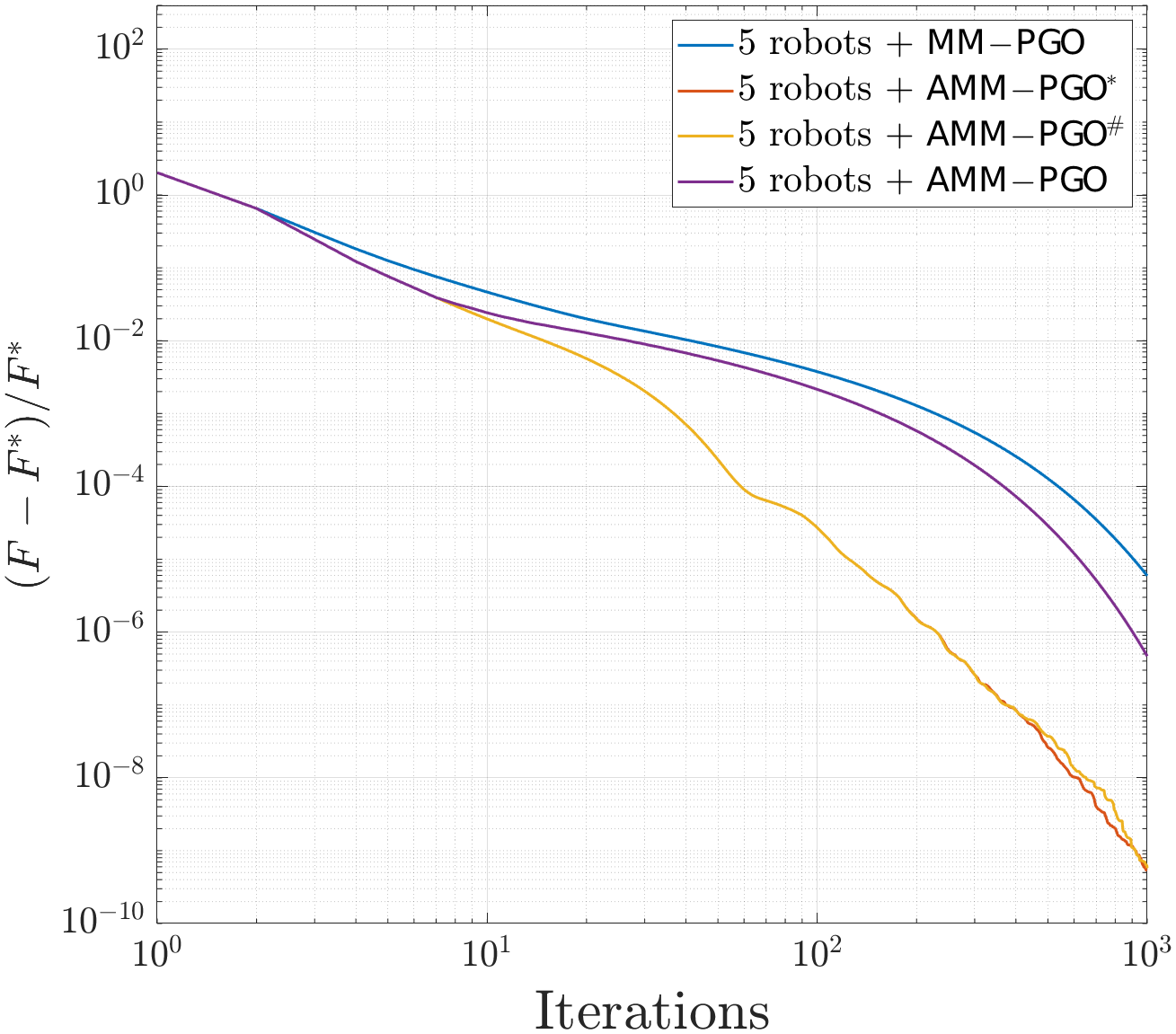}} &
		\hspace{-0.6em}\subfloat[][10 robots]{\includegraphics[trim =0mm 0mm 0mm 0mm,width=0.24\textwidth]{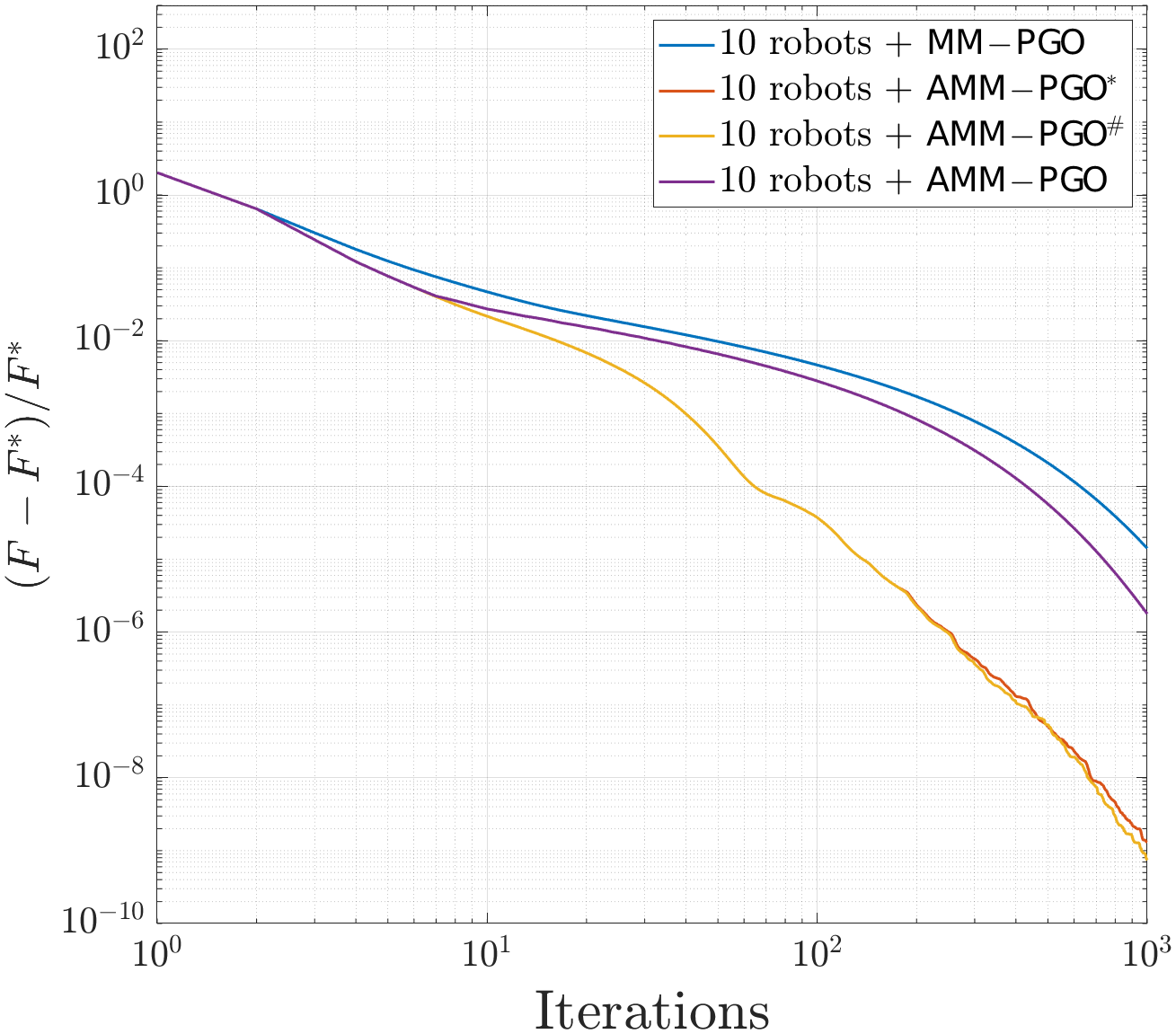}}&
		\hspace{-0.6em}\subfloat[][50 robots]{\includegraphics[trim =0mm 0mm 0mm 0mm,width=0.24\textwidth]{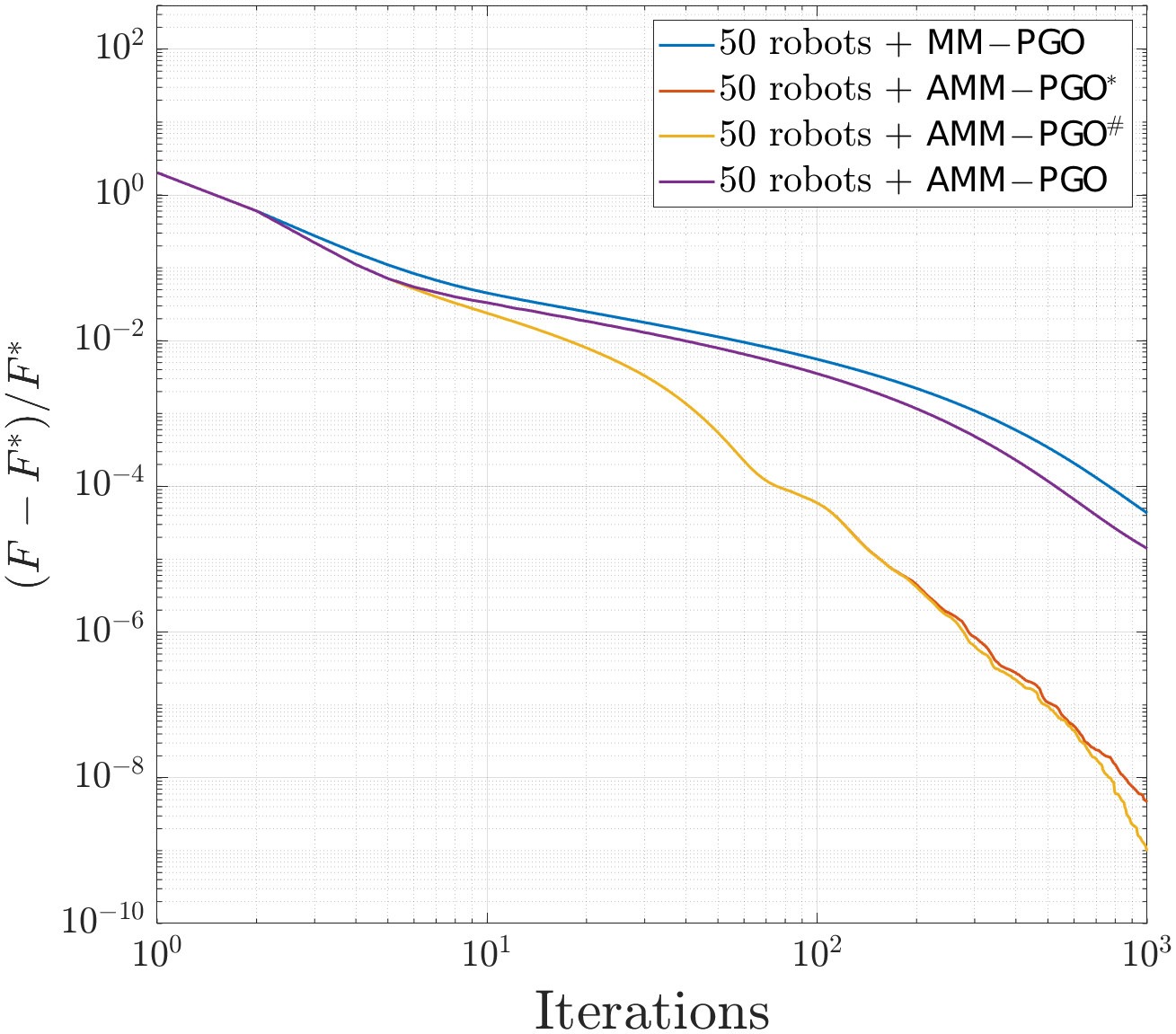}}
	\end{tabular}
	\caption{The relative suboptimality gaps of  $\mm$, $\ammc$, $\ammd$ and $\amm$ \cite{fan2020mm}  for distributed PGO with the \textbf{trivial loss kernel} on $5$, $10$ and $50$ robots. The results are averaged over $20$  Monte Carlo runs.}\label{fig::cube_f} 
	\vspace{-0.2em}
	
		\begin{tabular}{cccc}
		\hspace{-0.5em}\subfloat[][$\ammd$ vs. $\mm$]{\includegraphics[trim =0mm 0mm 0mm 0mm,width=0.24\textwidth]{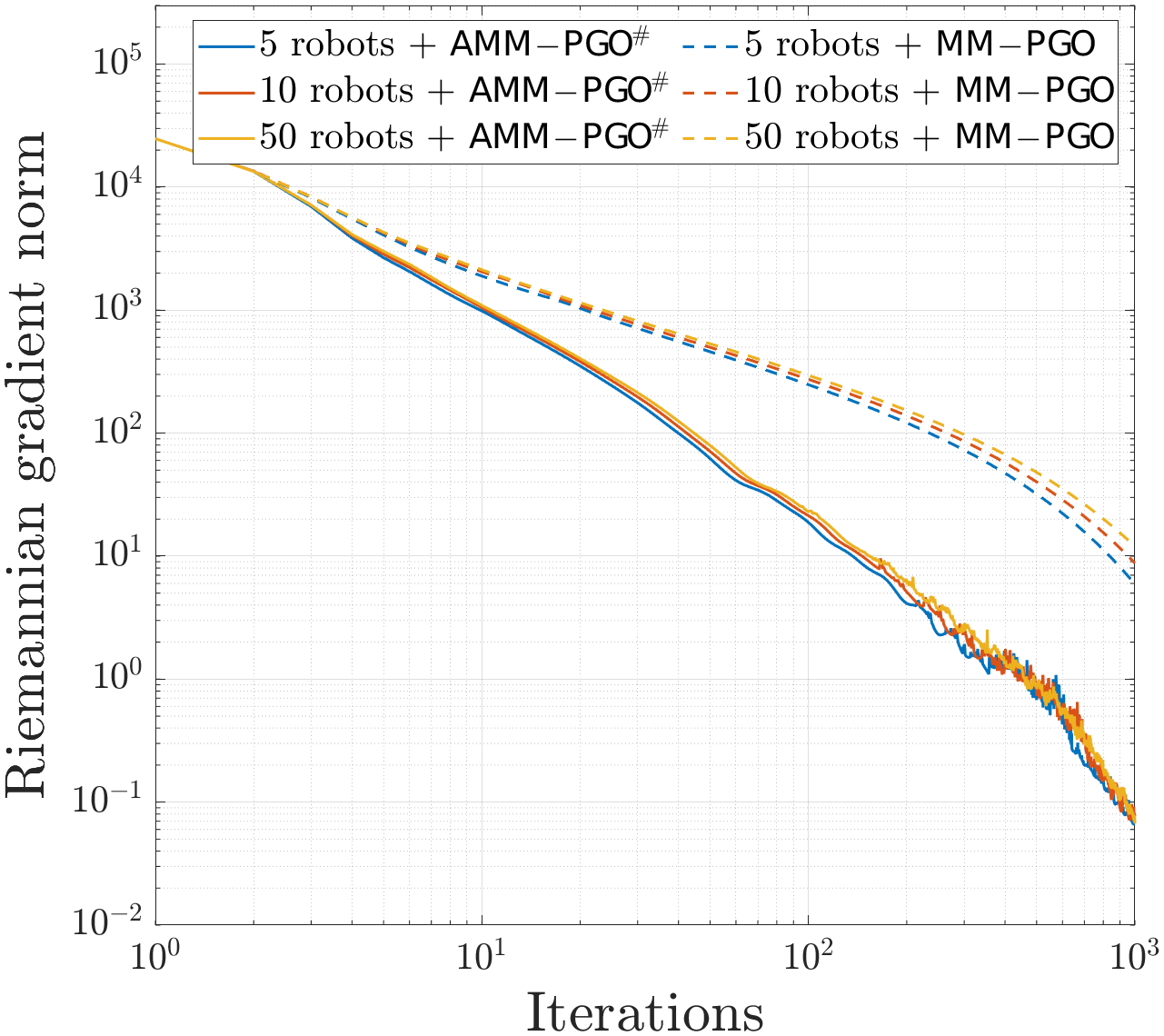}} &
		\hspace{-0.6em}\subfloat[][5 robots]{\includegraphics[trim =0mm 0mm 0mm 0mm,width=0.24\textwidth]{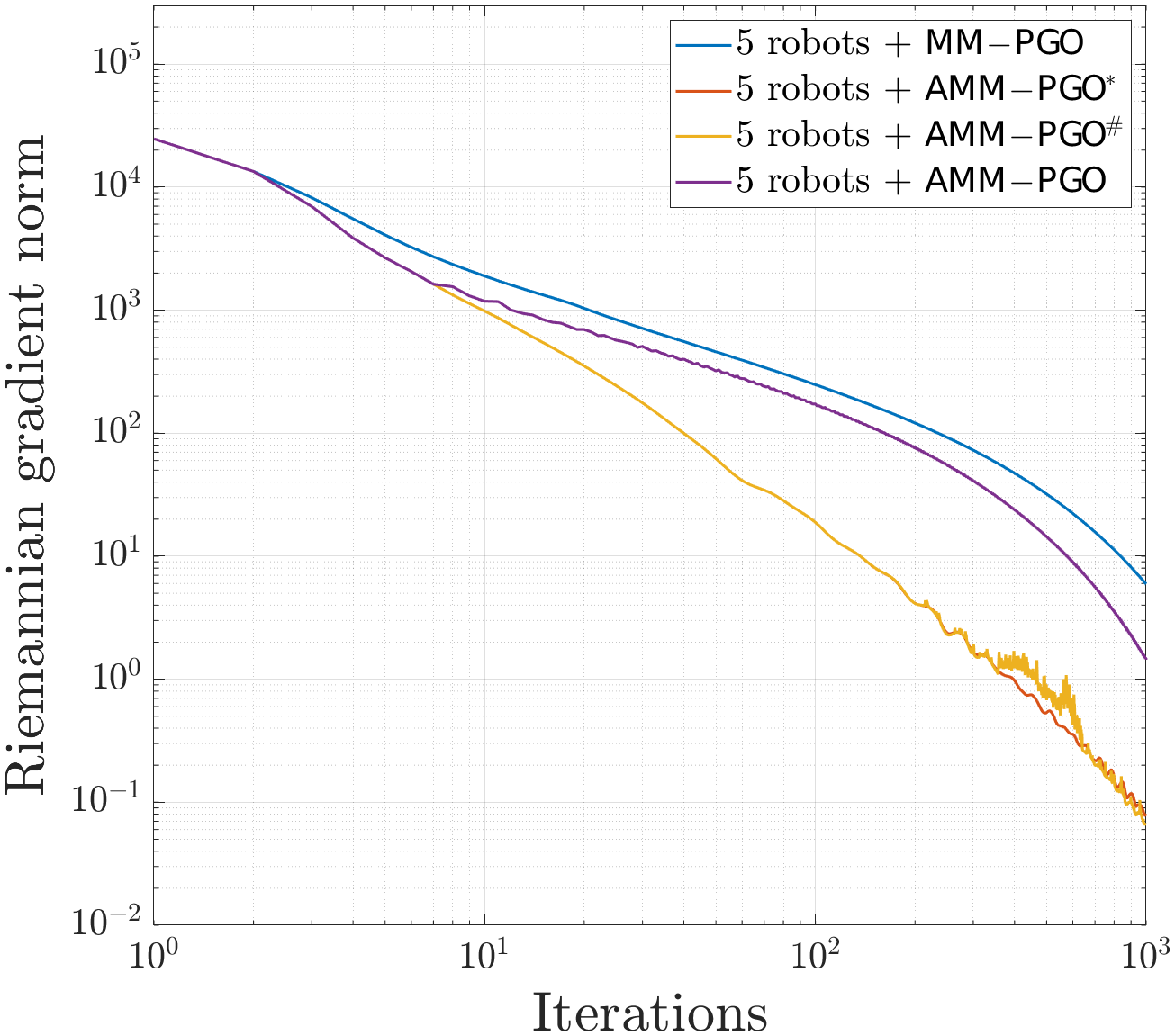}} &
		\hspace{-0.6em}\subfloat[][10 robots]{\includegraphics[trim =0mm 0mm 0mm 0mm,width=0.24\textwidth]{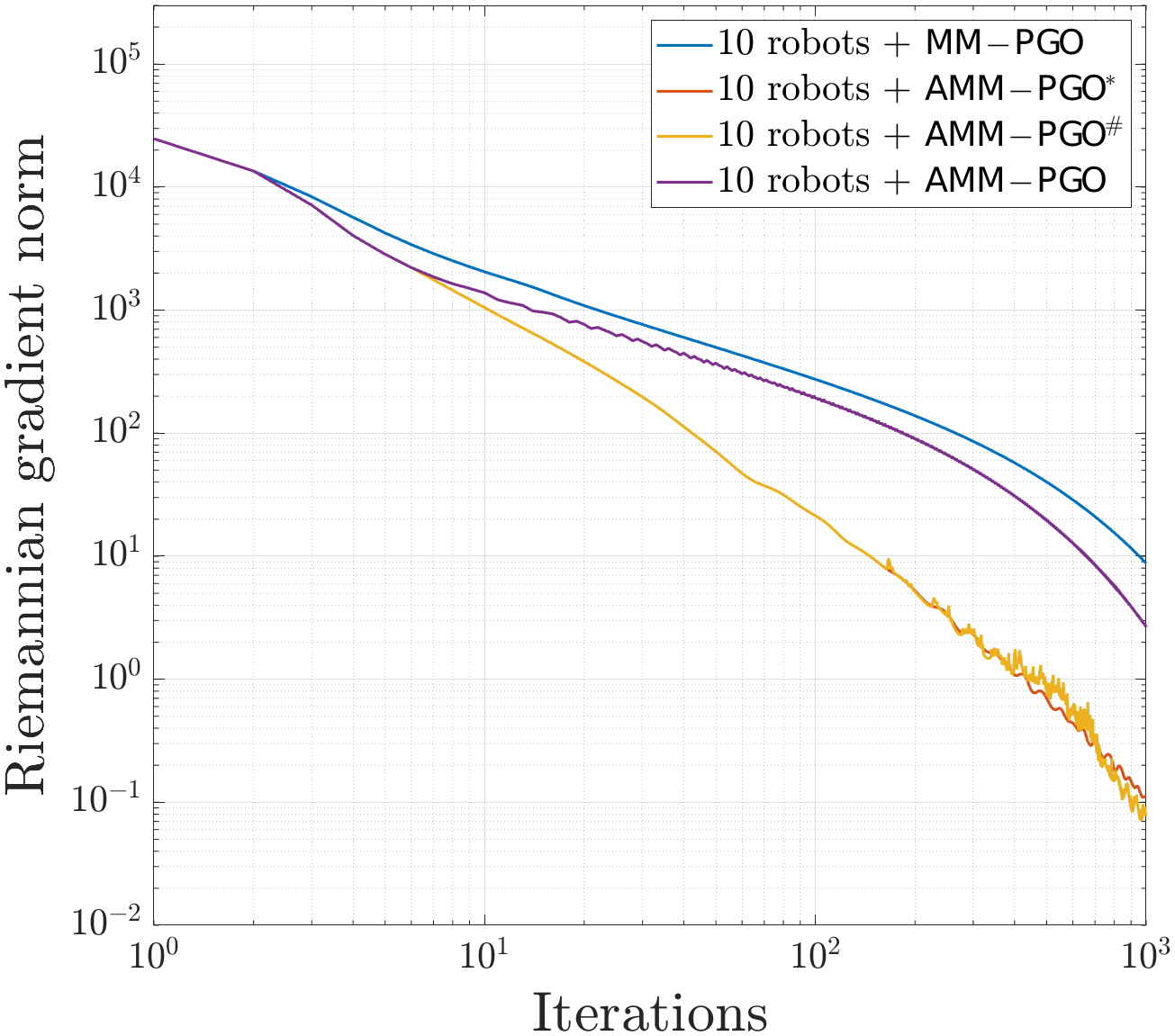}}&
		\hspace{-0.6em}\subfloat[][50 robots]{\includegraphics[trim =0mm 0mm 0mm 0mm,width=0.24\textwidth]{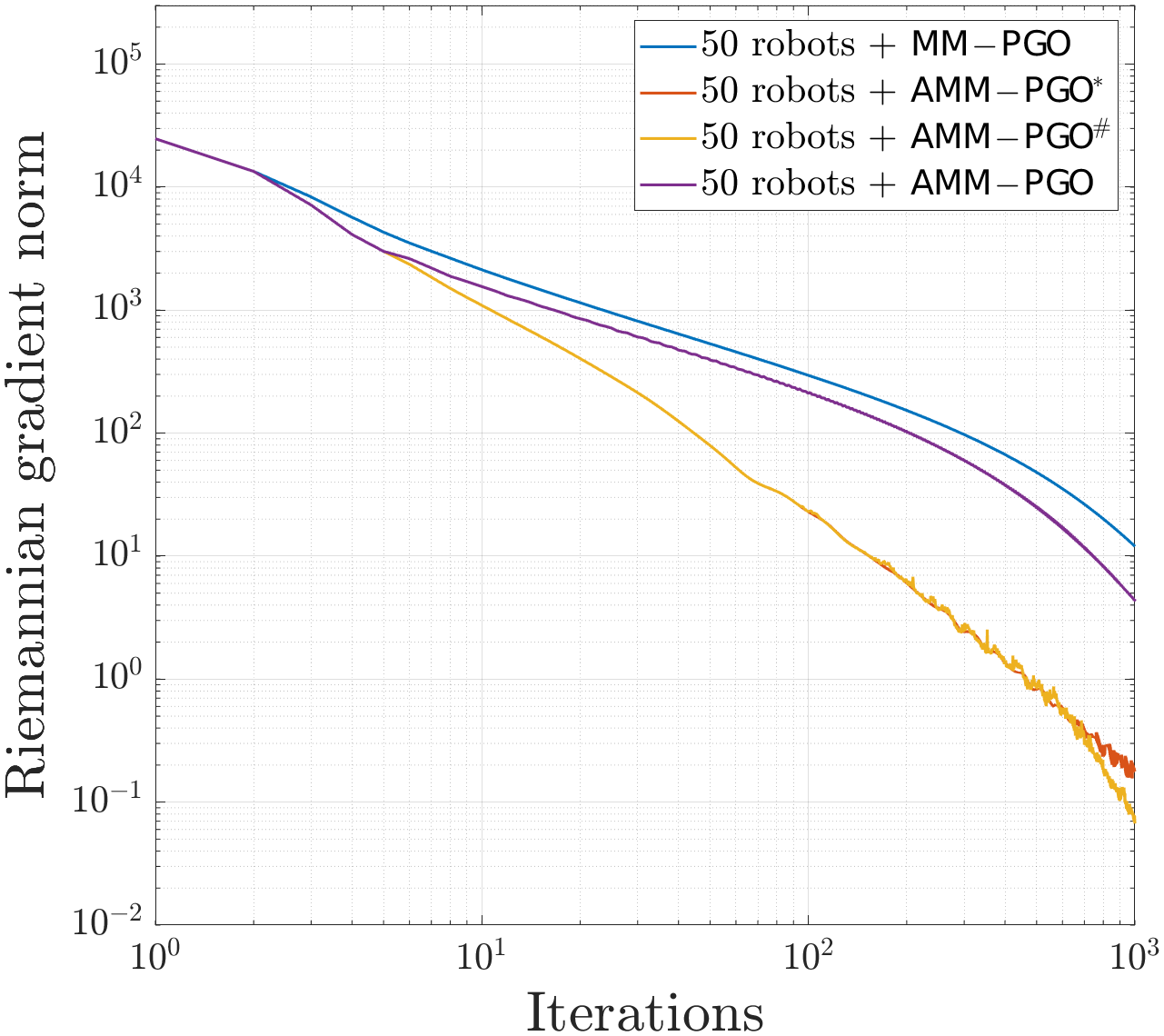}}
	\end{tabular}
	\caption{The Riemannian gradient norms of  $\mm$, $\ammc$, $\ammd$ and $\amm$ \cite{fan2020mm}  for distributed PGO with the \textbf{trivial loss kernel} on $5$, $10$ and $50$ robots. The results are averaged over $20$  Monte Carlo runs.}\label{fig::cube_g_trivial} 
	\vspace{-1em}
\end{figure*}
	
	\begin{figure*}[t]
		\centering
	\begin{tabular}{cccc}
		\hspace{-0.5em}\subfloat[][$\ammd$ vs. $\mm$]{\includegraphics[trim =0mm 0mm 0mm 0mm,width=0.24\textwidth]{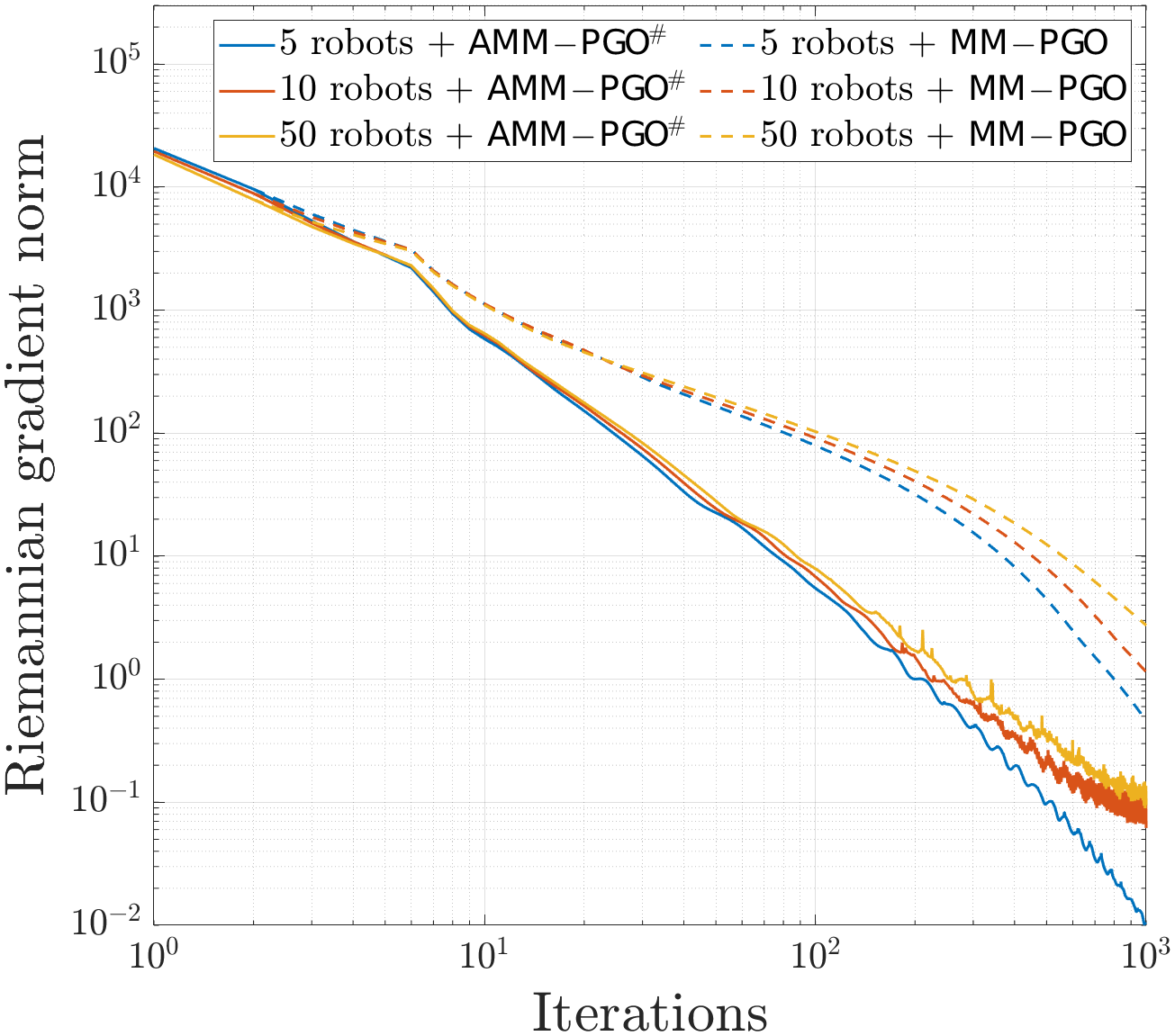}} &
		\hspace{-0.6em}\subfloat[][5 robots]{\includegraphics[trim =0mm 0mm 0mm 0mm,width=0.24\textwidth]{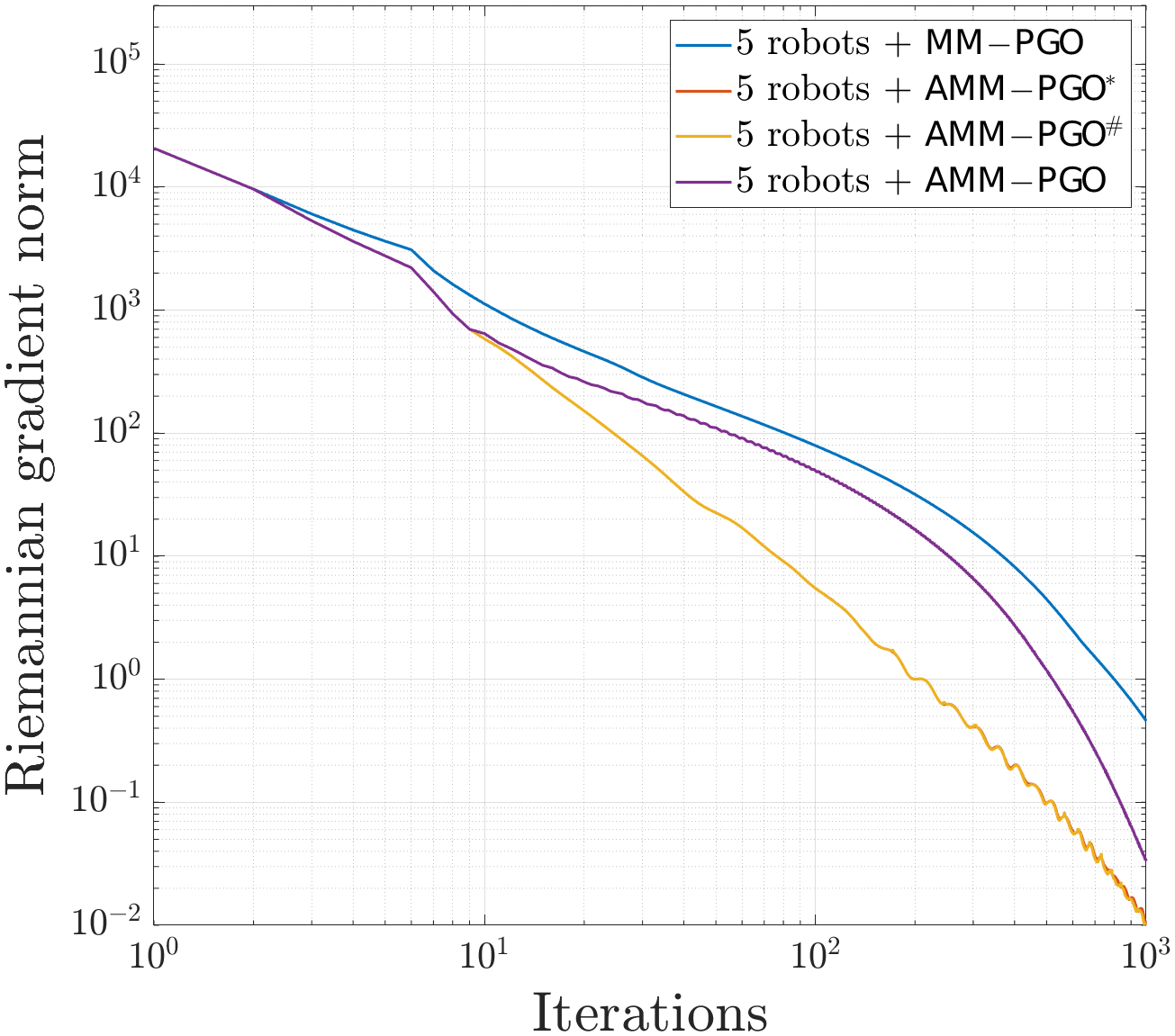}} &
		\hspace{-0.6em}\subfloat[][10 robots]{\includegraphics[trim =0mm 0mm 0mm 0mm,width=0.24\textwidth]{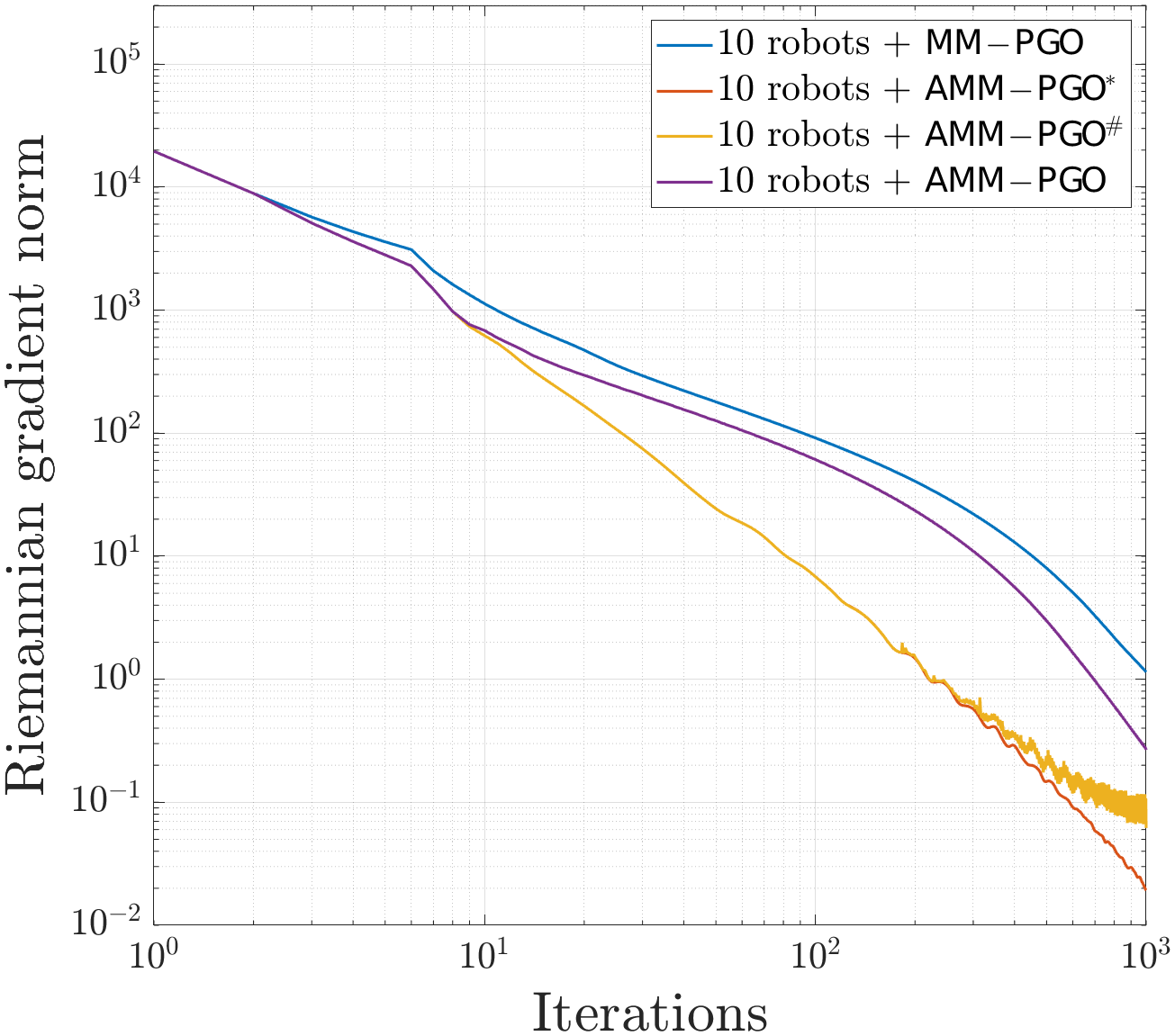}}&
		\hspace{-0.6em}\subfloat[][50 robots]{\includegraphics[trim =0mm 0mm 0mm 0mm,width=0.24\textwidth]{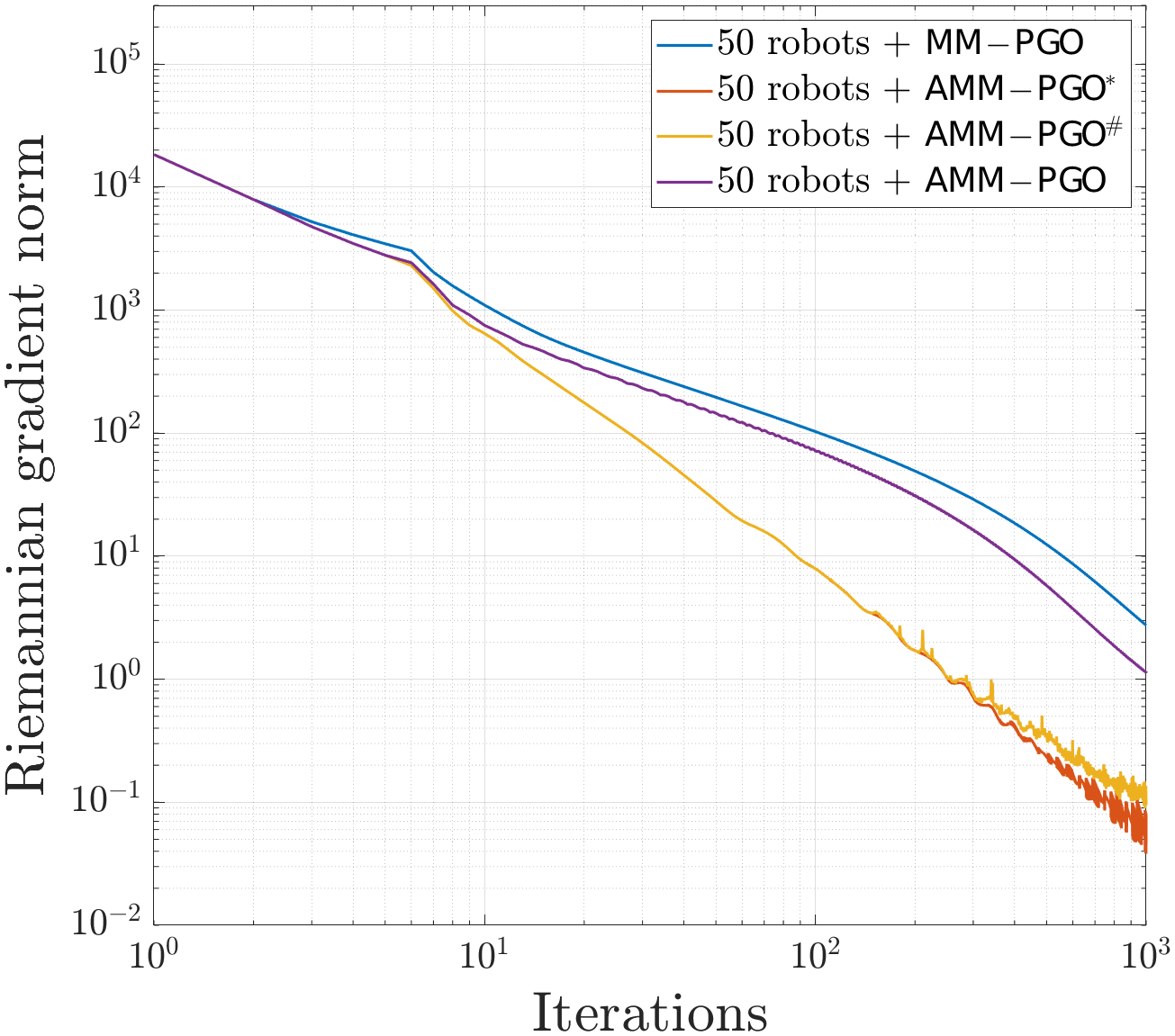}}
	\end{tabular}
	\caption{The Riemannian gradient norms of  $\mm$, $\ammc$, $\ammd$ and $\amm$ \cite{fan2020mm}  for distributed PGO with the \textbf{Huber loss kernel} on $5$, $10$ and $50$ robots. The results are averaged over $20$  Monte Carlo runs.}\label{fig::cube_g_huber} 
	\vspace{-0.2em}

	\begin{tabular}{cccc}
		\hspace{-0.5em}\subfloat[][$\ammd$ vs. $\mm$]{\includegraphics[trim =0mm 0mm 0mm 0mm,width=0.24\textwidth]{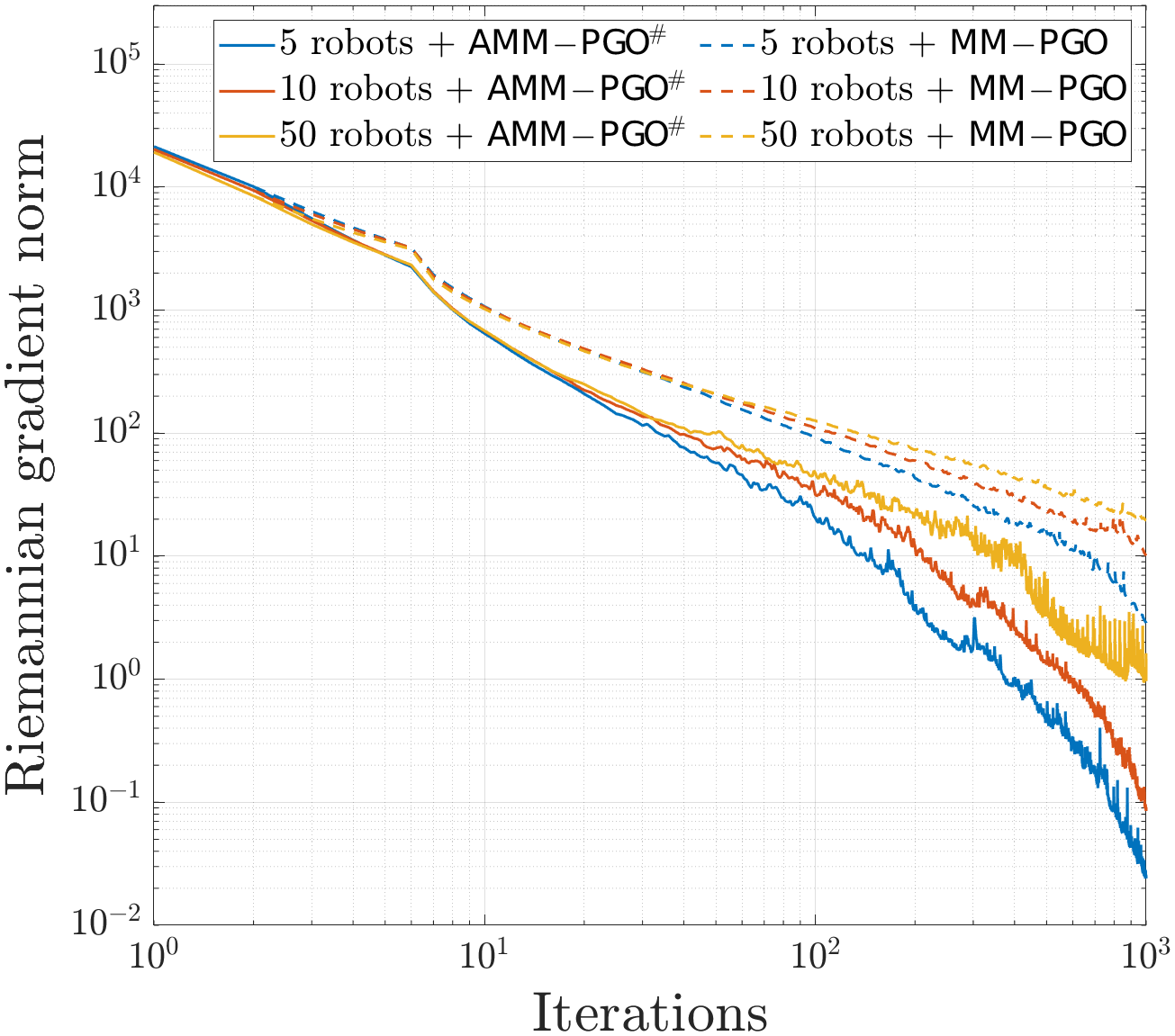}} &
		\hspace{-0.6em}\subfloat[][5 robots]{\includegraphics[trim =0mm 0mm 0mm 0mm,width=0.24\textwidth]{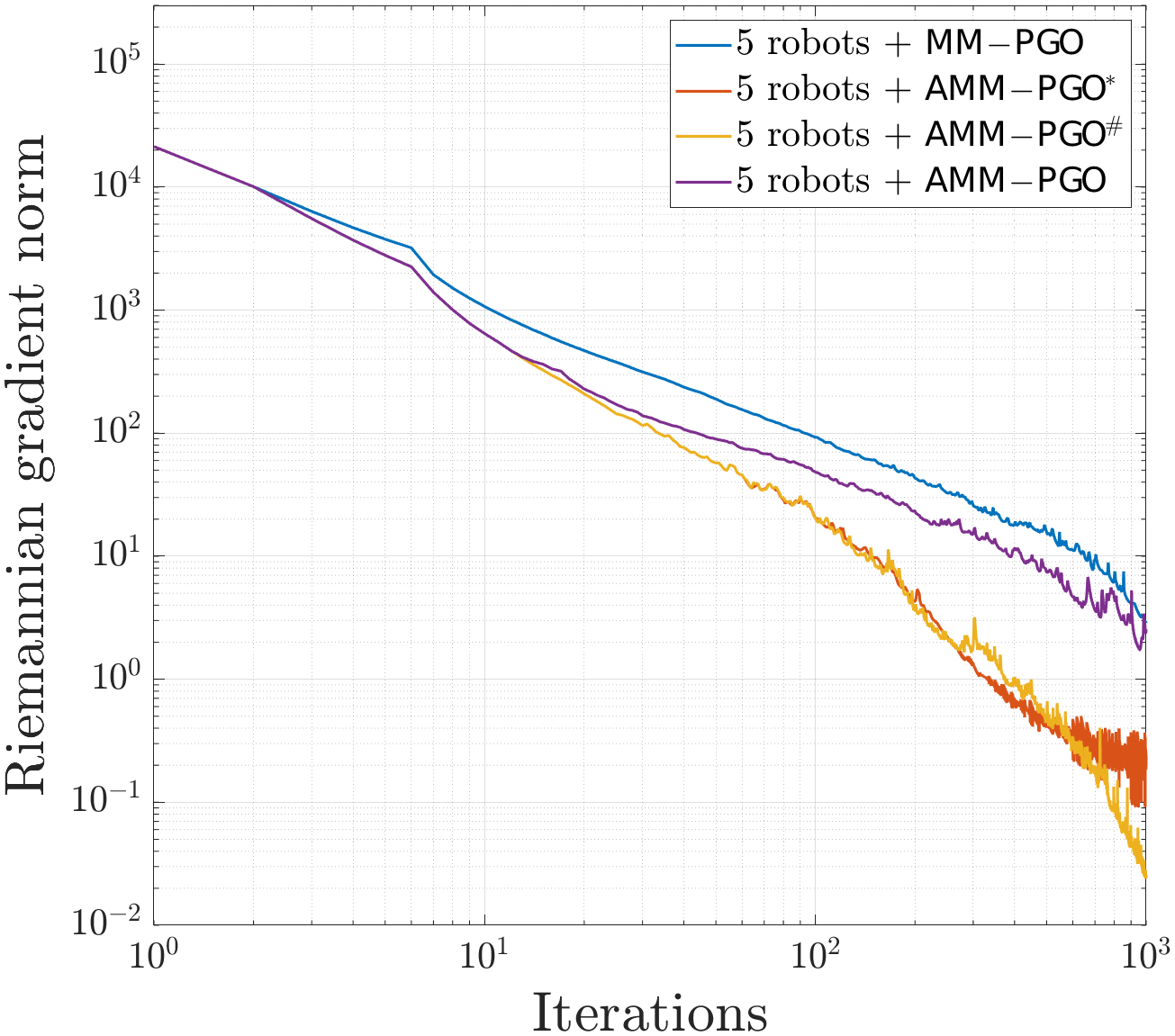}} &
		\hspace{-0.6em}\subfloat[][10 robots]{\includegraphics[trim =0mm 0mm 0mm 0mm,width=0.24\textwidth]{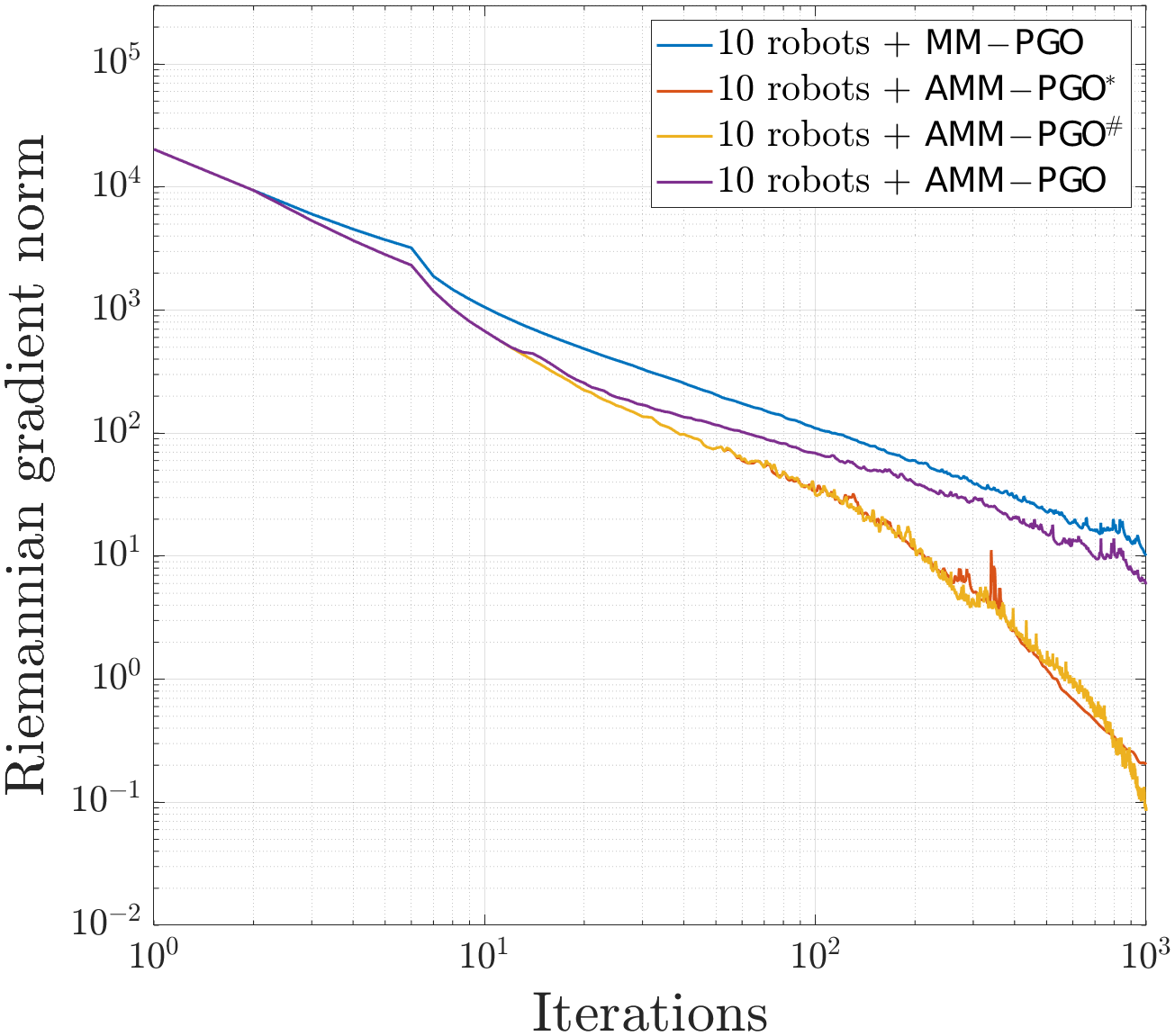}}&
		\hspace{-0.6em}\subfloat[][50 robots]{\includegraphics[trim =0mm 0mm 0mm 0mm,width=0.24\textwidth]{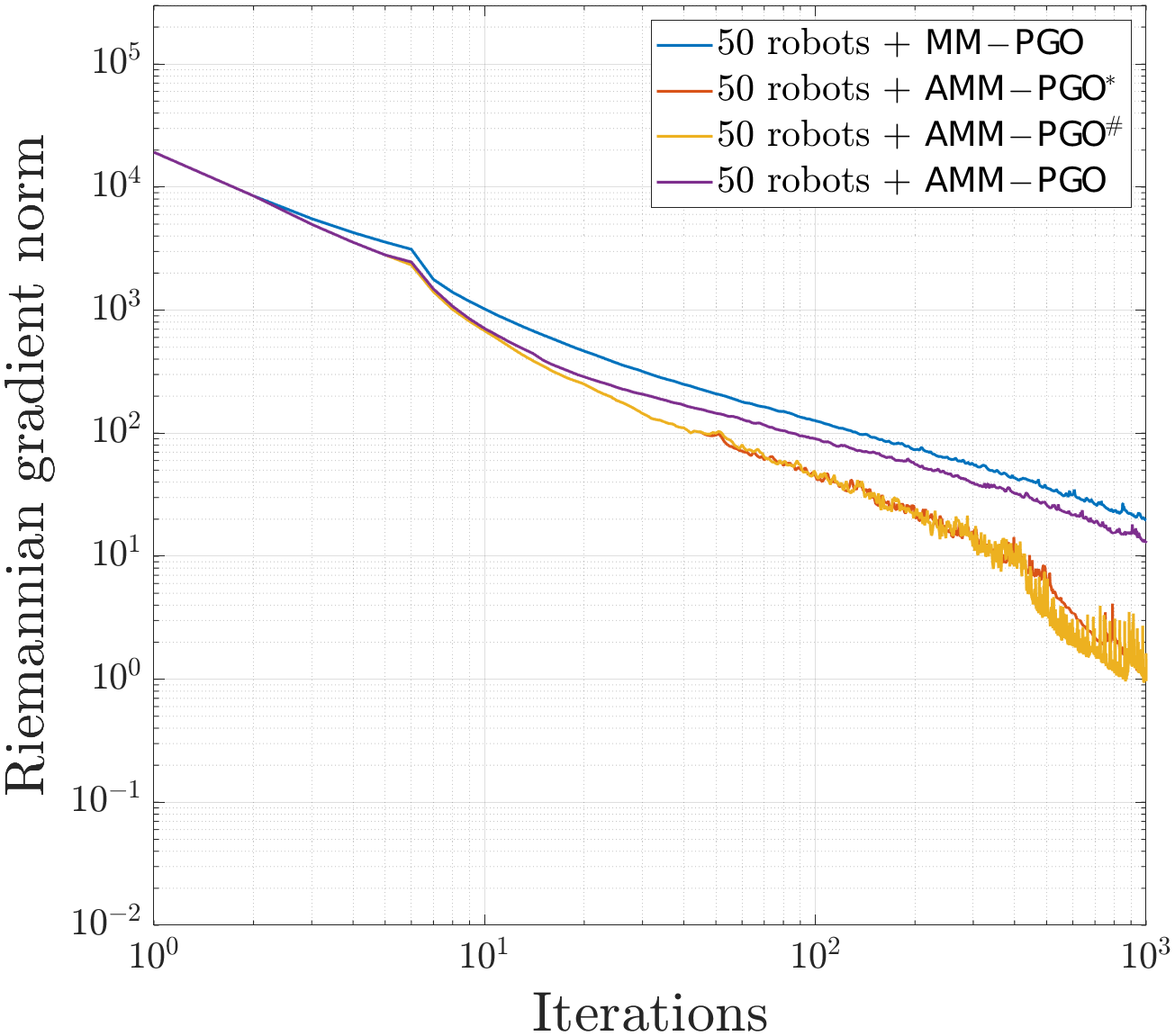}}
	\end{tabular}
	\caption{The Riemannian gradient norms of  $\mm$, $\ammc$, $\ammd$ and $\amm$ \cite{fan2020mm}  for distributed PGO with the \textbf{Welsch loss kernel} on $5$, $10$ and $50$ robots. The results are averaged over $20$  Monte Carlo runs.}\label{fig::cube_g_gm}
	\vspace{-1em}
\end{figure*}

%% file: table_cmp_method.tex
\begin{table}
	\centering
		\renewcommand{\arraystretch}{1.2}
	\setlength{\tabcolsep}{0.3em}
	\caption{An overview of the state-of-the-art algorithms for distributed and centralized PGO. Note that $\ammc$ and $\rbcdc$ require a master node for distributed PGO, and $\ammd$  is the only accelerated method with provable convergence for distributed PGO without master node.
	}\label{table::cmp_method}
	\begin{tabular}{|P{0.125\textwidth}|P{0.07\textwidth}|P{0.075\textwidth}|P{0.085\textwidth}|P{0.075\textwidth}|}
		\hline			
		{Method} &{Distributed} & {Accelerated}  &Masterless & {Converged}\\
		\hline
		\hline
		$\sesync$\cite{rosen2016se}& \xmark & N/A & N/A & \cmark \\
		\hline
		$\dgs$\cite{choudhary2017distributed} &\cmark & \xmark & \cmark & \xmark \\
		\hline
		$\rbcdc$\cite{tian2019distributed}&\cmark& \cmark & \xmark & \cmark\\
		\hline
		$\rbcdd$\cite{tian2019distributed}&\cmark & \cmark & \cmark & \xmark\\
		\hline
		$\mm$ & \cmark &\xmark & \cmark & \cmark \\
		\hline
		$\ammc$&\cmark & \cmark & \xmark & \cmark\\
		\hline
		$\ammd$&\cmark & \cmark & \cmark & \cmark \\
		\hline
	\end{tabular}
\end{table}

%% file: table_cmp_dataset.tex
\begin{table*}[t]
	\renewcommand{\arraystretch}{1.22}
		\setlength{\tabcolsep}{0.25em}
	\centering
	\caption{Results of distributed PGO on 2D and 3D SLAM benchmark datasets (see \datasetinfo). The distributed PGO has 10 robots  and is initialized with distributed Nesterov's accelerated chordal initialization \cite{fan2020mm}. We report the objective values of each method with 100, 250 and 1000 iterations. $F^{(\sk)}$ and $F^*$ are the objective value at iteration $\sk$ and globally optimal objective value, respectively. The best results are colored in {\color{red}red} and the second best in {\color{blue}blue} if no methods tie for the best.}	\label{table::comp}
	\begin{tabular}{|c||c|c||c|c|c||c|c|c|c|}
		\hline
		\multirow{3}{*}{Dataset}&\multirow{3}{*}{$F^{(0)}$} &\multirow{3}{*}{$F^{*}$} &\multirow{3}{*}{\;$\sk$\;}& \multicolumn{6}{c|}{$F^{(\sk)}$}\\
		\cline{5-10}
		& & & & \multicolumn{2}{c||}{Methods w/ Master Node} & \multicolumn{4}{c|}{Methods w/o Master Node}\\
		\cline{5-10}
		& & & & \;$\ammc$\;  & $\rbcdc$ \cite{tian2019distributed} & $\;\;\;\mm\;\;\;$  & \;$\ammd$\;  & $\;\;\;\dgs$ \cite{choudhary2017distributed} $\;$ & {$\rbcdd$ \cite{tian2019distributed}}  \\
		\hline
    \multicolumn{10}{|c|}{{2D SLAM Benchmark Datasets}}\\
		\hline
		\multirow{3}{*}{\sf ais2klinik} &\multirow{3}{*}{$3.8375\times 10^2$} &\multirow{3}{*}{$1.8850\times 10^2$}
    &100 & {\color{blue}$2.0372\times 10^{2}$} &$2.1079\times 10^{2}$ & $2.1914\times 10^{2}$ & {\color{red}$2.0371\times 10^{2}$} & $8.4701\times 10^{2}$ & $2.1715\times 10^{2}$  \\
		\cline{4-10}
		& &
    &250 & {\color{blue}$1.9447\times 10^{2}$} &$2.0077\times 10^{2}$ & $2.1371\times 10^{2}$ & {\color{red}$1.9446\times 10^{2}$} & $9.1623\times 10^{1}$ & $2.1084\times 10^{2}$  \\
		\cline{4-10}
		& &
    &1000& {\color{blue}$1.8973\times 10^{2}$} &$1.9074\times 10^{2}$ & $2.0585\times 10^{2}$ & {\color{red}$1.8936\times 10^{2}$} & $3.8968\times 10^{2}$ & $2.0253\times 10^{2}$  \\
		\hline
    \multirow{3}{*}{\sf city} &\multirow{3}{*}{$7.0404\times 10^2$} &\multirow{3}{*}{$6.3862\times 10^2$}
    &100 &\color{ao}$6.4327\times 10^2$ & $6.5138\times 10^2$ &$6.5061\times 10^2$ &\color{ao}$6.4327\times10^{2}$ &$7.7745\times10^{2}$ & $6.5396\times 10^2$ \\
		\cline{4-10}
    & &
    &250 &\color{ao}$6.3899\times 10^2$ & $6.4732\times 10^2$ &$6.4850\times 10^2$ &\color{ao}$6.3899\times 10^2$ &$7.0063\times 10^2$ & $6.5122\times 10^2$\\
		\cline{4-10}
    & & & 1000 & \color{ao}$6.3862\times 10^2$ & $6.3935\times 10^2$ &$6.4461\times 10^2$ &\color{ao}$6.3863\times 10^2$ &$6.5583\times 10^2$ & $6.4768\times 10^2$\\
		\hline
		\multirow{3}{*}{\sf CSAIL} &\multirow{3}{*}{$3.1719\times 10^1$} &\multirow{3}{*}{$3.1704\times 10^1$}
    &100 & \color{ao}$3.1704\times10^1$ & \color{ao}$3.1704\times10^1$ & $3.1706\times10^1$ & \color{ao}$3.1704\times10^1$ & $3.2479\times10^1$ & $3.1705\times10^1$ \\
		\cline{4-10}
		& &
    &250 & \color{ao}$3.1704\times10^1$ & \color{ao}$3.1704\times10^1$ & $3.1706\times10^1$ & \color{ao}$3.1704\times10^1$ & $3.1792\times10^1$ & \color{ao}$3.1704\times10^1$ \\
		\cline{4-10}
		& &
    &1000 & \color{ao}$3.1704\times10^1$ & \color{ao}$3.1704\times10^1$ & $3.1705\times10^1$ & \color{ao}$3.1704\times10^1$ & $3.1712\times10^1$ & \color{ao}$3.1704\times10^1$ \\
		\hline
		\multirow{3}{*}{\sf M3500} &\multirow{3}{*}{$2.2311\times 10^2$} &\multirow{3}{*}{$1.9386\times 10^2$}
    &100 & \color{red}$1.9446\times10^2$ & $1.9511\times10^2$ & $1.9560\times10^2$ & \color{blue}$1.9447\times10^2$ & $1.9557\times10^2$ & $1.9551\times10^2$ \\
		\cline{4-10}
		& &
    &250 & \color{red}$1.9414\times10^2$ & $1.9443\times10^2$ & $1.9516\times10^2$ & \color{red}$1.9414\times10^2$ & $1.9445\times10^2$ & $1.9511\times10^2$ \\
		\cline{4-10}
		& &
    &1000 & \color{red}$1.9388\times10^2$ & $1.9392\times10^2$ & $1.9461\times10^2$ & \color{red}$1.9388\times10^2$ & $1.9415\times10^2$ & $1.9455\times10^2$ \\
		\hline
		\multirow{3}{*}{\sf intel} &\multirow{3}{*}{$5.3269\times 10^1$} &\multirow{3}{*}{$5.2348\times 10^1$}
    &100 & \color{red}$5.2397\times10^1$ & $5.2496\times10^1$ & $5.2517\times 10^1$ & \color{red}$5.2397\times 10^1$ & $5.2541\times 10^1$ & $5.2526\times10^1$ \\
		\cline{4-10}
		& &
    &250 & \color{blue}$5.2352\times10^1$ & $5.2415\times10^1$ & $5.2483\times 10^1$ & \color{red}$5.2351\times 10^1$ & $5.2441\times 10^1$ & $5.2489\times10^1$ \\
		\cline{4-10}
		& &
    &1000 & \color{red}$5.2348\times10^1$ & $5.2349\times10^1$ & $5.2421\times 10^1$ & \color{red}$5.2348\times 10^1$ & $5.2381\times 10^1$ & $5.2425\times10^1$ \\
		\hline
		\multirow{3}{*}{\sf MITb} &\multirow{3}{*}{$8.8430\times 10^1$} &\multirow{3}{*}{$6.1154\times 10^1$}
    &100 & \color{blue}$6.1331\times10^1$ & $6.1518\times10^1$ & $6.3657\times10^1$ & \color{red}$6.1330\times10^1$ & $9.5460\times 10^1$ & $6.1997\times10^1$ \\
		\cline{4-10}
		& &
    &250 & \color{red}$6.1157\times10^1$ & $6.1187\times10^1$ & $6.2335\times10^1$ & \color{blue}$6.1165\times10^1$ & $7.8273\times 10^1$ & $6.1599\times10^1$ \\
		\cline{4-10}
		& &
    &1000 & \color{ao}$6.1154\times10^1$ & \color{ao}$6.1154\times10^1$ & $6.1454\times10^1$ & \color{ao}$6.1154\times10^1$ & $7.2450\times 10^1$ & $6.1209\times10^1$ \\
		\hline
    \multicolumn{10}{|c|}{{3D SLAM Benchmark Datasets}}\\
		\hline
    \multirow{3}{*}{\sf sphere} &\multirow{3}{*}{$1.9704\times 10^3$} &\multirow{3}{*}{$1.6870\times 10^3$}
    &100 & \color{ao}$1.6870\times 10^{3}$ & \color{ao}$1.6870\times 10^3$ & $1.6901\times 10^3$ & \color{ao}$1.6870\times 10^{3}$ & $1.6875\times 10^{3}$ & \color{ao}$1.6870\times 10^{3}$  \\
		\cline{4-10}
		& &
    &250 & \color{ao}$1.6870\times 10^{3}$ & \color{ao}$1.6870\times 10^3$ & $1.6874\times 10^3$ & \color{ao}$1.6870\times 10^{3}$ & $1.6872\times 10^{3}$ & \color{ao}$1.6870\times 10^{3}$  \\
		\cline{4-10}
		& &
    &1000 & \color{ao}$1.6870\times 10^{3}$ & \color{ao}$1.6870\times 10^3$ & \color{ao}$1.6870\times 10^3$ & \color{ao}$1.6870\times 10^{3}$ & $1.6872\times 10^{3}$ & \color{ao}$1.6870\times 10^{3}$  \\
		\hline
    \multirow{3}{*}{\sf torus} &\multirow{3}{*}{$2.4654\times 10^4$} &\multirow{3}{*}{$2.4227\times 10^4$}
    &100 & \color{ao}$2.4227\times 10^{4}$ & \color{ao}$2.4227\times 10^{4}$ & $2.4234\times 10^4$ & \color{ao}$2.4227\times 10^4$ & $2.4248\times 10^4$ & \color{ao}$2.4227\times 10^4$  \\
		\cline{4-10}
		& &
    &250 & \color{ao}$2.4227\times 10^{4}$ & \color{ao}$2.4227\times 10^{4}$ & \color{ao}$2.4227\times 10^4$ & \color{ao}$2.4227\times 10^4$ & $2.4243\times 10^4$ & \color{ao}$2.4227\times 10^4$  \\
		\cline{4-10}
		& &
    &1000 & \color{ao}$2.4227\times 10^{4}$ & \color{ao}$2.4227\times 10^{4}$ & \color{ao}$2.4227\times 10^4$ & \color{ao}$2.4227\times 10^4$ & $2.4236\times 10^4$ & \color{ao}$2.4227\times 10^4$  \\
		\hline
    \multirow{3}{*}{\sf grid} &\multirow{3}{*}{$2.8218\times 10^5$} &\multirow{3}{*}{$8.4319\times 10^4$}
    &100 & {\color{blue}$8.4323\times 10^{4}$} & \color{red}$8.4320\times 10^{4}$ & $1.0830\times 10^5$ & $8.4399\times 10^4$ & $1.4847\times 10^5$ & $8.4920\times 10^4$  \\
		\cline{4-10}
		& &
    &250 & {\color{ao}$8.4319\times 10^{4}$} & \color{ao}$8.4319\times 10^{4}$ & $8.6054\times 10^4$ & $8.4321\times 10^4$ & $1.4066\times 10^5$ & \color{ao}$8.4319\times 10^4$  \\
		\cline{4-10}
		& &
    &1000 & \color{ao}$8.4319\times 10^{4}$ & \color{ao}$8.4319\times 10^{4}$ & \color{ao}$8.4319\times 10^4$ & \color{ao}$8.4319\times 10^4$ & $1.4654\times 10^5$ & \color{ao}$8.4319\times 10^4$  \\
		\hline
    \multirow{3}{*}{\sf garage} &\multirow{3}{*}{$1.5470\times 10^0$} &\multirow{3}{*}{$1.2625\times 10^0$}
    &100 & \color{red}$1.3105\times 10^0$ & $1.3282\times 10^0$ & $1.3396\times 10^0$ & \color{red}$1.3105\times 10^0$ & $1.3170\times 10^0$ & $1.3364\times 10^0$  \\
		\cline{4-10}
		& &
    &250 & \color{blue}$1.2872\times 10^0$ & $1.3094\times 10^0$ & $1.3288\times 10^0$ & \color{blue}$1.2872\times 10^0$ & \color{red}$1.2867\times 10^0$ & $1.3276\times 10^0$  \\
		\cline{4-10}
		& &
    &1000 & \color{red}$1.2636\times 10^0$ & $1.2681\times 10^0$ & $1.3145\times 10^0$ & \color{red}$1.2636\times 10^0$ & $1.2722\times 10^0$ & $1.3124\times 10^0$  \\
		\hline
    \multirow{3}{*}{\sf cubicle} &\multirow{3}{*}{$8.3514\times 10^2$} &\multirow{3}{*}{$7.1713\times 10^2$}
    &100 & \color{ao}$7.1812\times 10^2$ & $7.2048\times 10^2$ & $7.2300\times 10^2$ & \color{ao}$7.1812\times 10^2$ & $7.3185\times 10^2$ & $7.2210\times 10^2$  \\
		\cline{4-10}
		& &
    &250 & \color{red}$7.1714\times 10^2$ & $7.1794\times 10^2$ & $7.2082\times 10^2$ & \color{blue}$7.1715\times 10^2$ & $7.2308\times 10^2$ & $7.2081\times 10^2$  \\
		\cline{4-10}
		& &
    &1000 & \color{ao}$7.1713\times 10^2$ & \color{ao}$7.1713\times 10^2$ & $7.2082\times 10^2$ & \color{ao}$7.1713\times 10^2$ & $7.2044\times 10^2$ & $7.1845\times 10^2$  \\
		\hline
    \multirow{3}{*}{\sf rim} &\multirow{3}{*}{$8.1406\times 10^4$} &\multirow{3}{*}{$5.4609\times 10^3$}
    &100 & \color{ao}$5.5044\times 10^3$ & $5.7184\times 10^3$ & $5.8138\times 10^3$ & \color{ao}$5.5044\times 10^3$ & $6.1840\times 10^3$ & $5.7810\times 10^3$  \\
		\cline{4-10}
		& &
    &250 & \color{ao}$5.4648\times 10^3$ & $5.5050\times 10^3$ & $5.7197\times 10^3$ & \color{ao}$5.4648\times 10^3$ & $6.1184\times 10^3$ & $5.7195\times 10^3$  \\
		\cline{4-10}
		& &
    &1000 & \color{ao}$5.4609\times 10^3$ & $5.4617\times 10^3$ & $5.5509\times 10^3$ & \color{ao}$5.4609\times 10^3$ & $6.0258\times 10^3$ & $5.5373\times 10^3$  \\
		\hline
	\end{tabular}
\vspace{-0.75em}
\end{table*}

%% file: fig_benchmark.tex
\begin{figure*}[p]
	\vspace{-1em}
	\centering
	\begin{tabular}{ccc}
		\subfloat[][\sf ais2klinik]{\includegraphics[trim =0mm 0mm 0mm 0mm,width=0.25\textwidth]{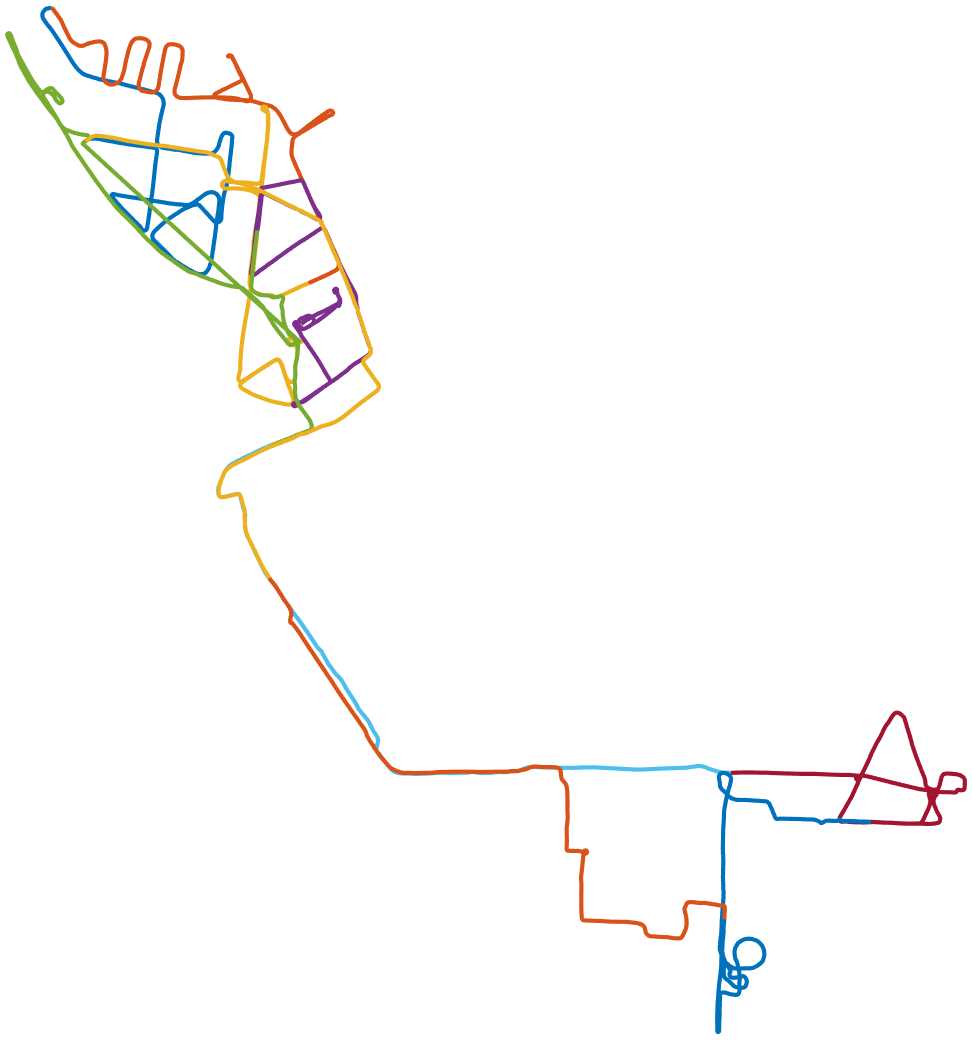}} &
		\hspace{1em}\subfloat[][\sf city]{\includegraphics[trim =0mm 0mm 0mm 0mm,width=0.25\textwidth]{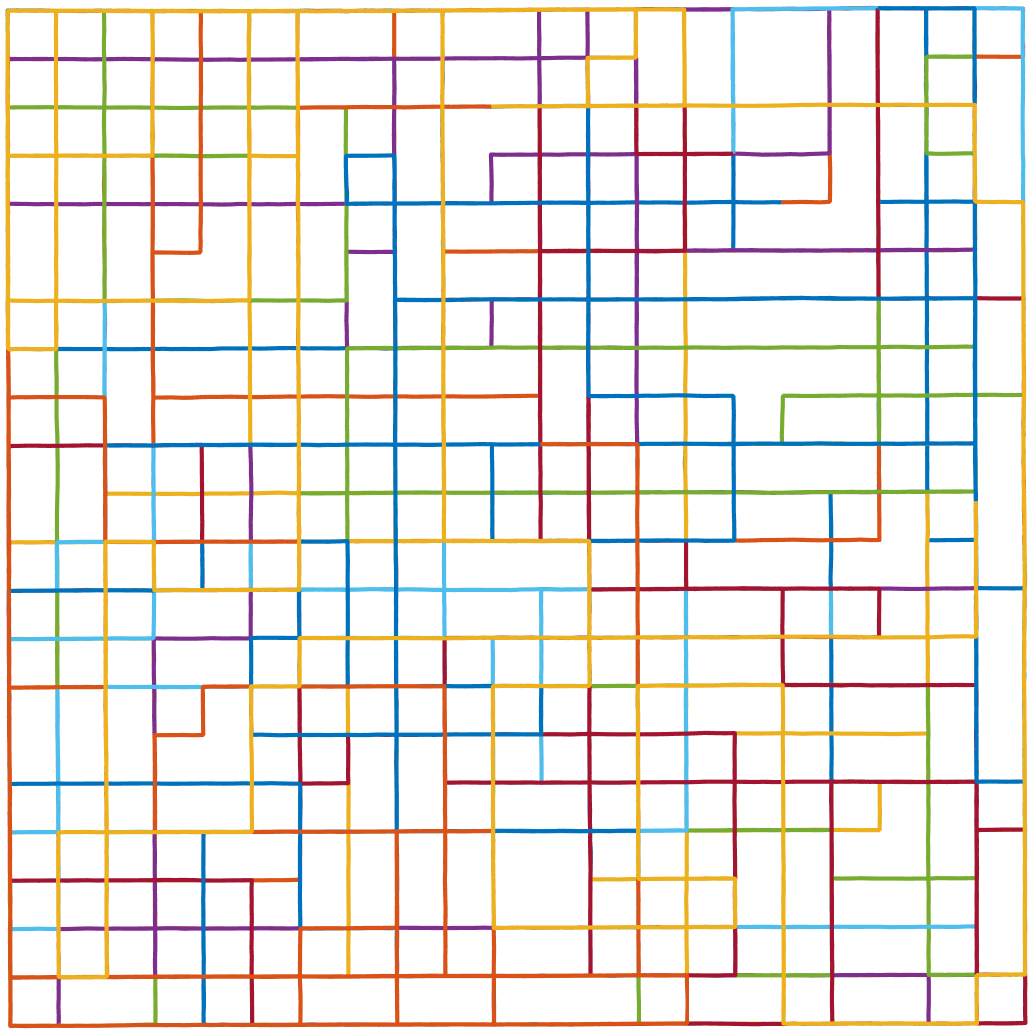}} &
		\hspace{1em}\subfloat[][\sf CSAIL]{\includegraphics[trim =0mm 0mm 0mm 0mm,width=0.29\textwidth]{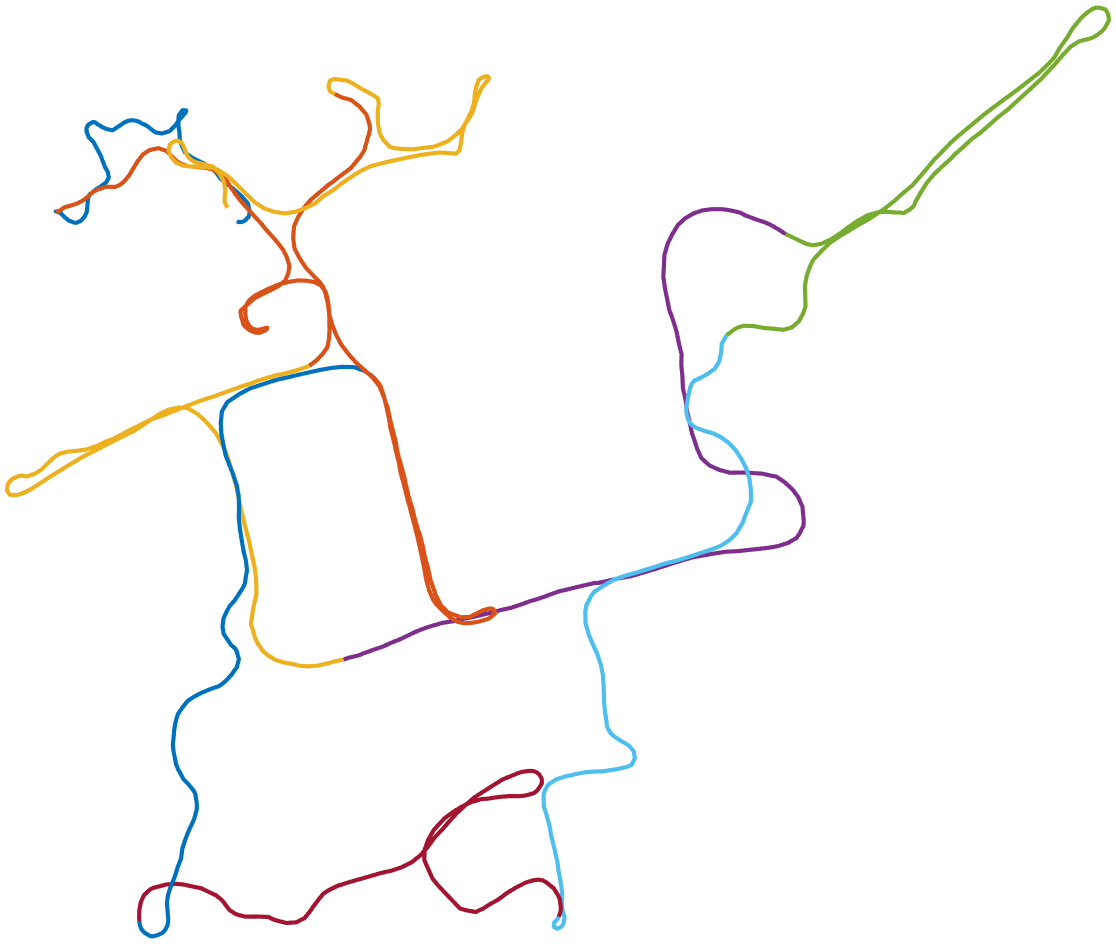}}\\[1em]
		\subfloat[][\sf M3500]{\includegraphics[trim =23mm 10mm 23mm 5mm,width=0.27\textwidth]{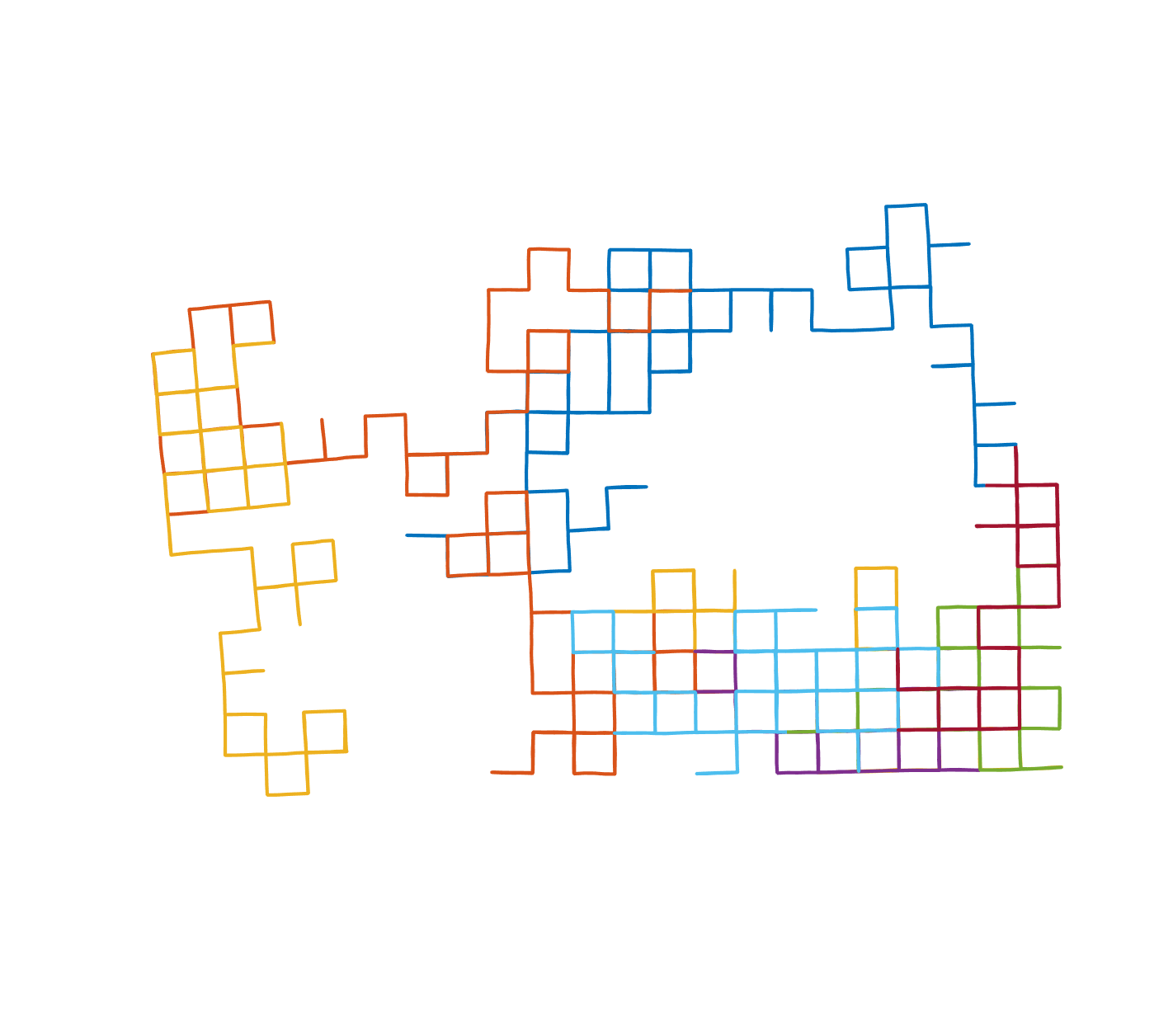}}&
		\hspace{1em}\subfloat[][\sf intel]{\includegraphics[trim =0mm 0mm 0mm 0mm,width=0.27\textwidth]{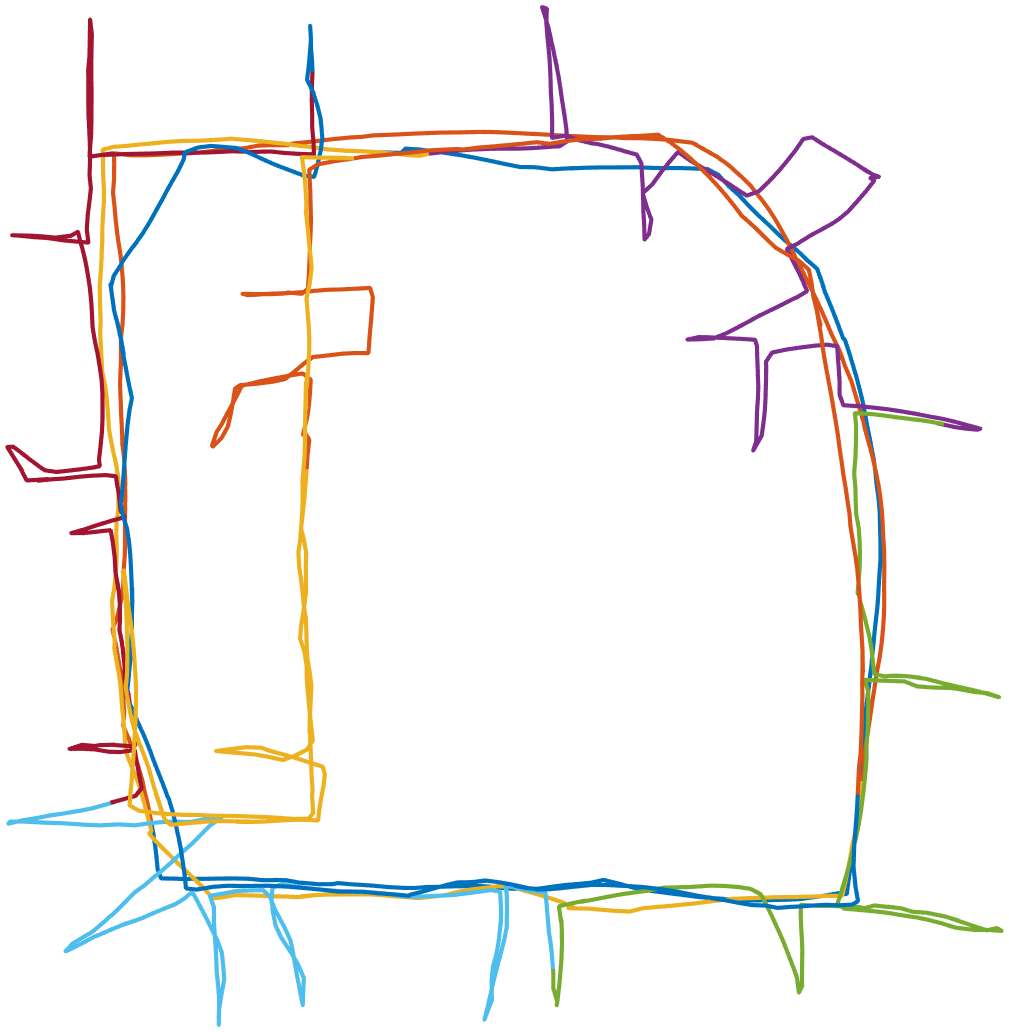}}&
		\subfloat[][\sf MITb]{\includegraphics[trim =0mm 0mm 0mm 0mm,width=0.27\textwidth]{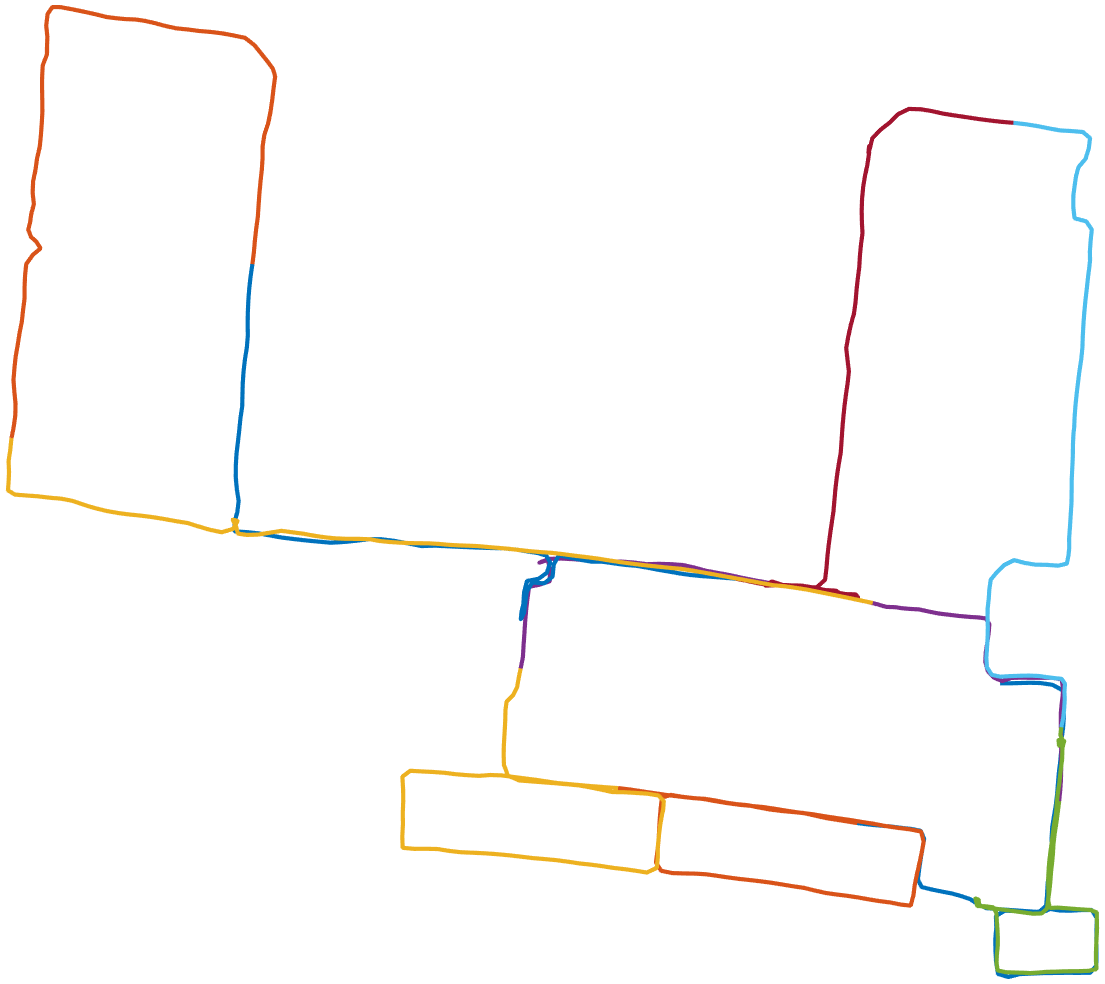}}
	\end{tabular}
	\caption{$\ammd$ results on the 2D SLAM benchmark datasets where the different colors denote the odometries of different robots. The distributed PGO has 10 robots  and is initialized with the distributed Nesterov's accelerated chordal initialization \cite{fan2020mm}. The number of iterations is 1000.}\label{fig::benchmark2D}
	\vspace{-0.25em}
	\centering
	\begin{tabular}{ccc}
		\subfloat[][\sf sphere]{\includegraphics[trim =0mm 0mm 0mm 0mm,width=0.25\textwidth]{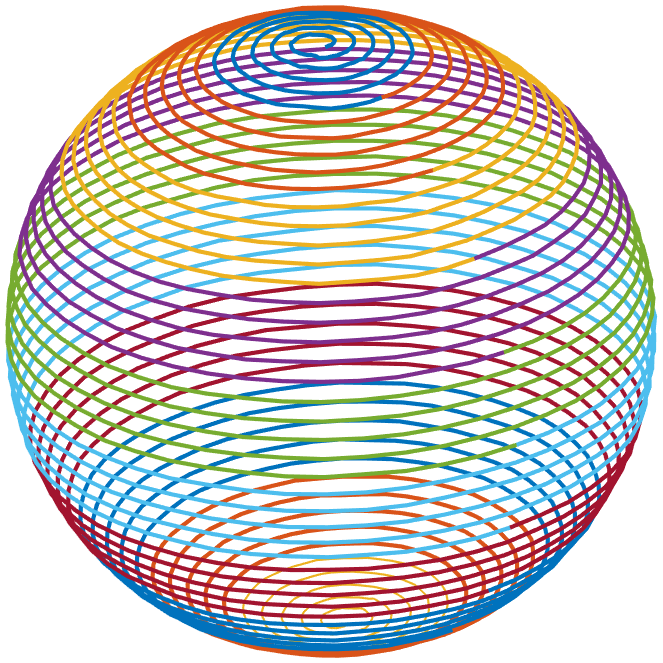}} &
		\subfloat[][\sf torus]{\includegraphics[trim =35mm 20mm 35mm 0mm,width=0.23\textwidth]{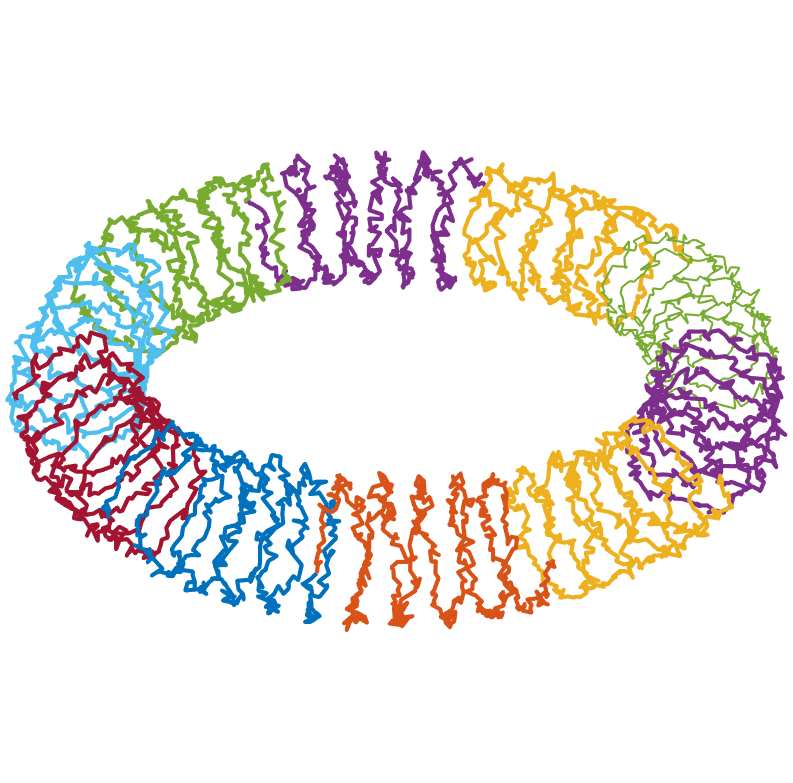}} &
		\subfloat[][\sf grid]{\includegraphics[trim =0mm 0mm 0mm 0mm,width=0.24\textwidth]{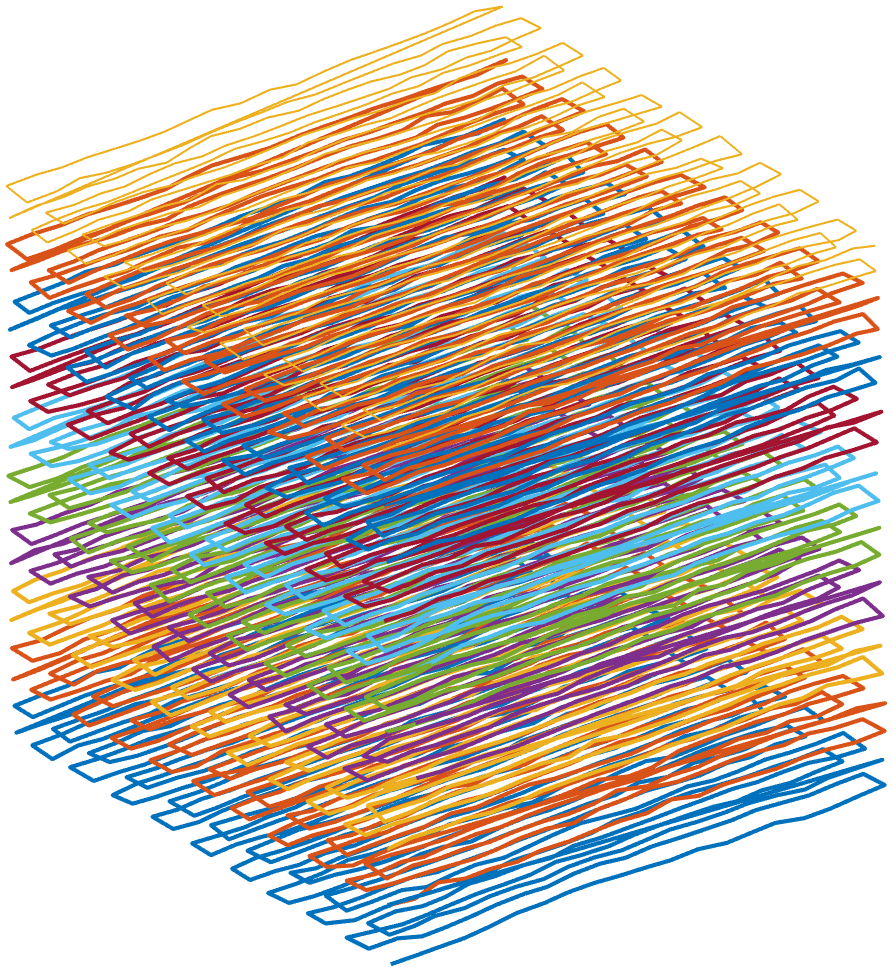}}\\[1em]
		\subfloat[][\sf garage]{\includegraphics[trim =0mm 30mm 0mm 20mm,width=0.30\textwidth]{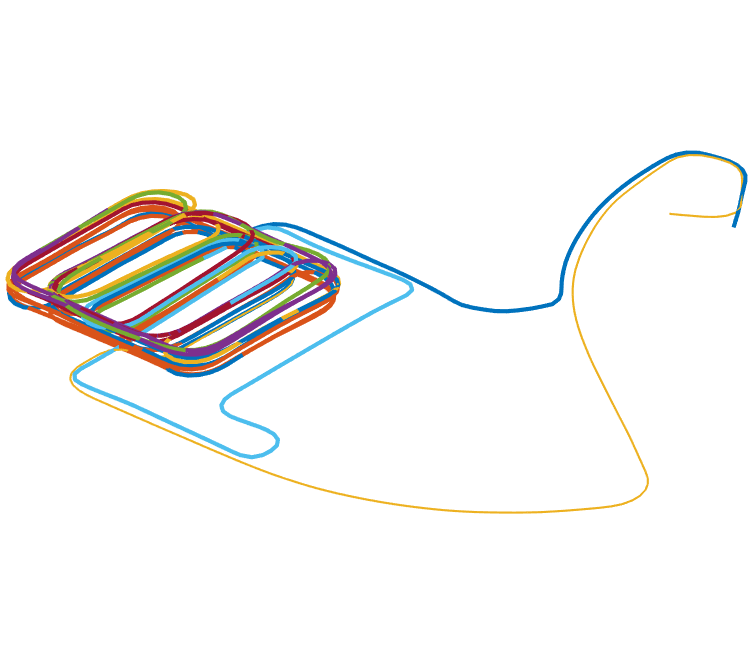}}&
		\subfloat[][\sf cubicle]{\includegraphics[trim =0mm 20mm 0mm 20mm,width=0.31\textwidth]{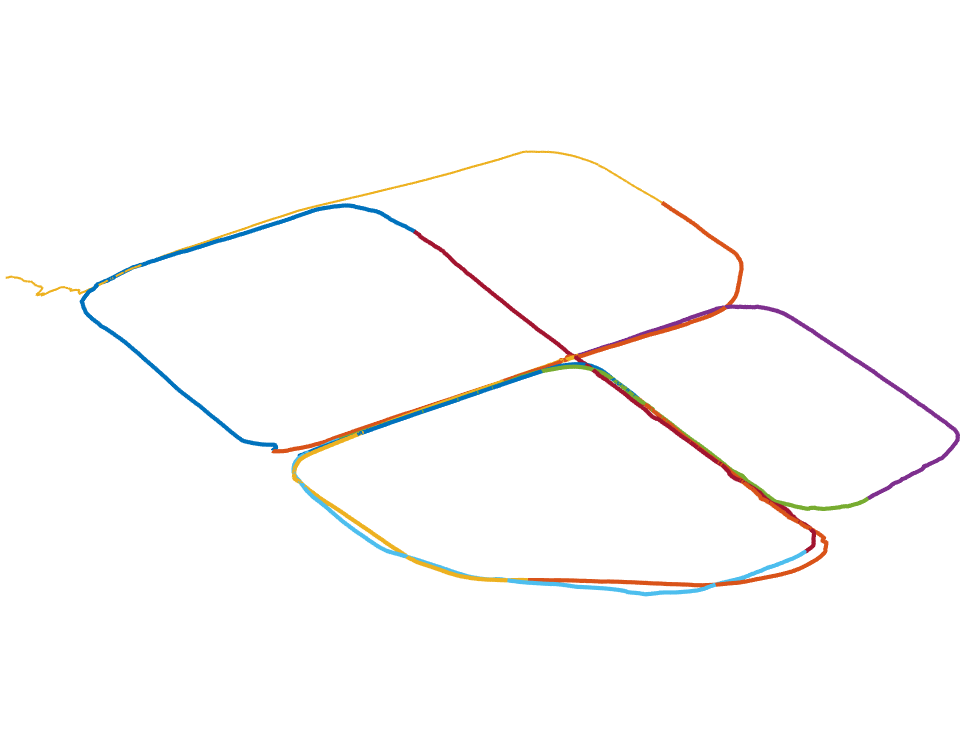}}&
		\subfloat[][\sf rim]{\includegraphics[trim =0mm 20mm 0mm 30mm,width=0.31\textwidth]{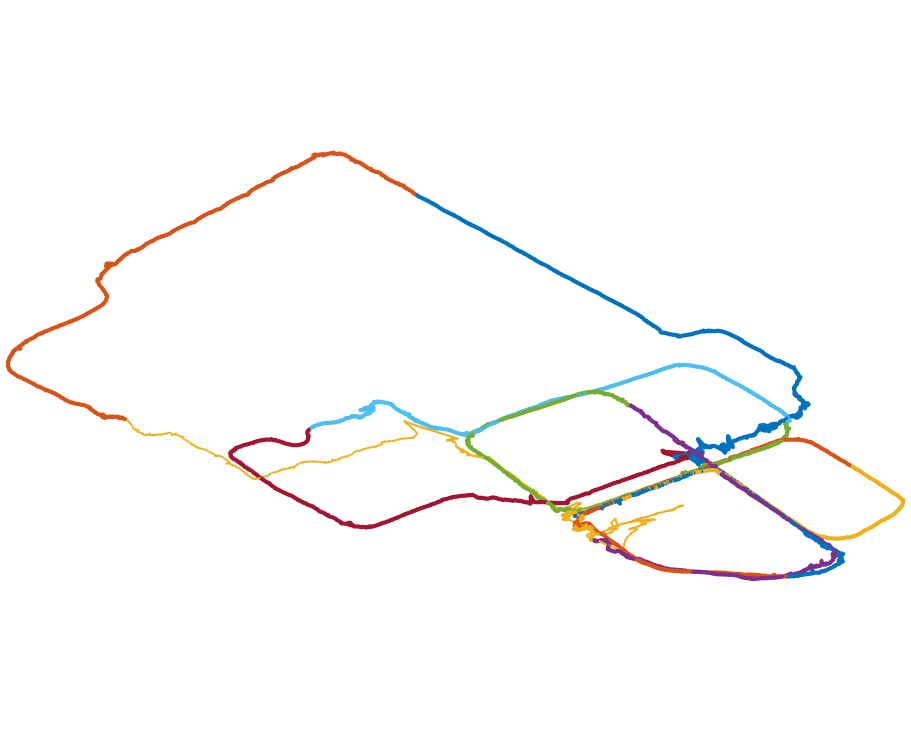}}
	\end{tabular}
	\caption{$\ammd$ results on the 3D SLAM benchmark datasets where the different colors denote the odometries of different robots. The distributed PGO has 10 robots  and is initialized with the distributed Nesterov's accelerated chordal initialization \cite{fan2020mm}. The number of iterations is 1000.}\label{fig::benchmark3D}
	\vspace{-1.25em}
\end{figure*}

%% file: fig_succ_iter.tex
\begin{figure*}[t]
	\centering
	\begin{tabular}{cccc}
		\hspace{-0.5em}\subfloat[][$\Delta=1\times10^{-2}$]{\includegraphics[trim =0mm 0mm 0mm 0mm,width=0.2425\textwidth]{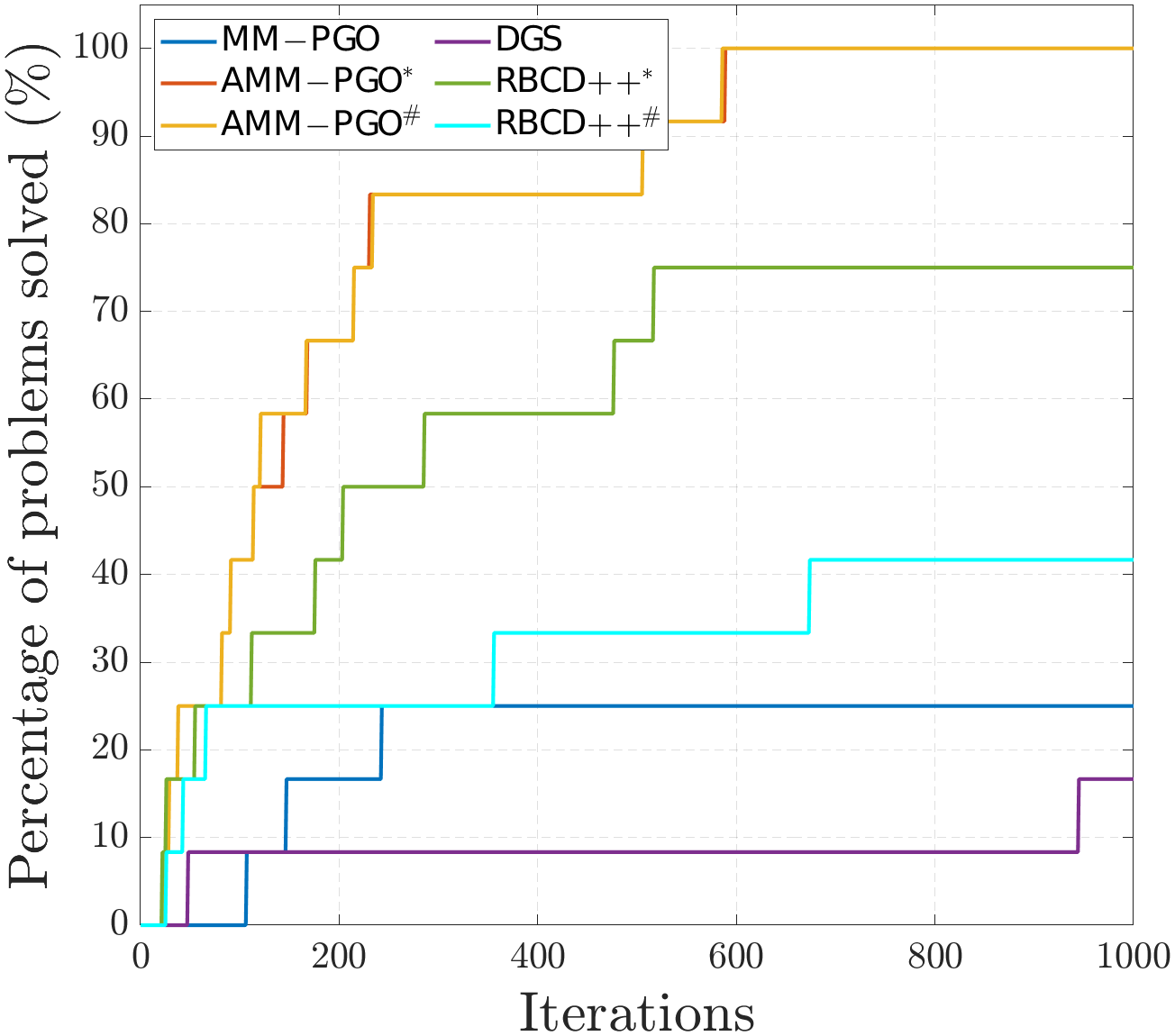}} &
		\hspace{-0.6em}\subfloat[][$\Delta=5\times10^{-3}$]{\includegraphics[trim =0mm 0mm 0mm 0mm,width=0.2425\textwidth]{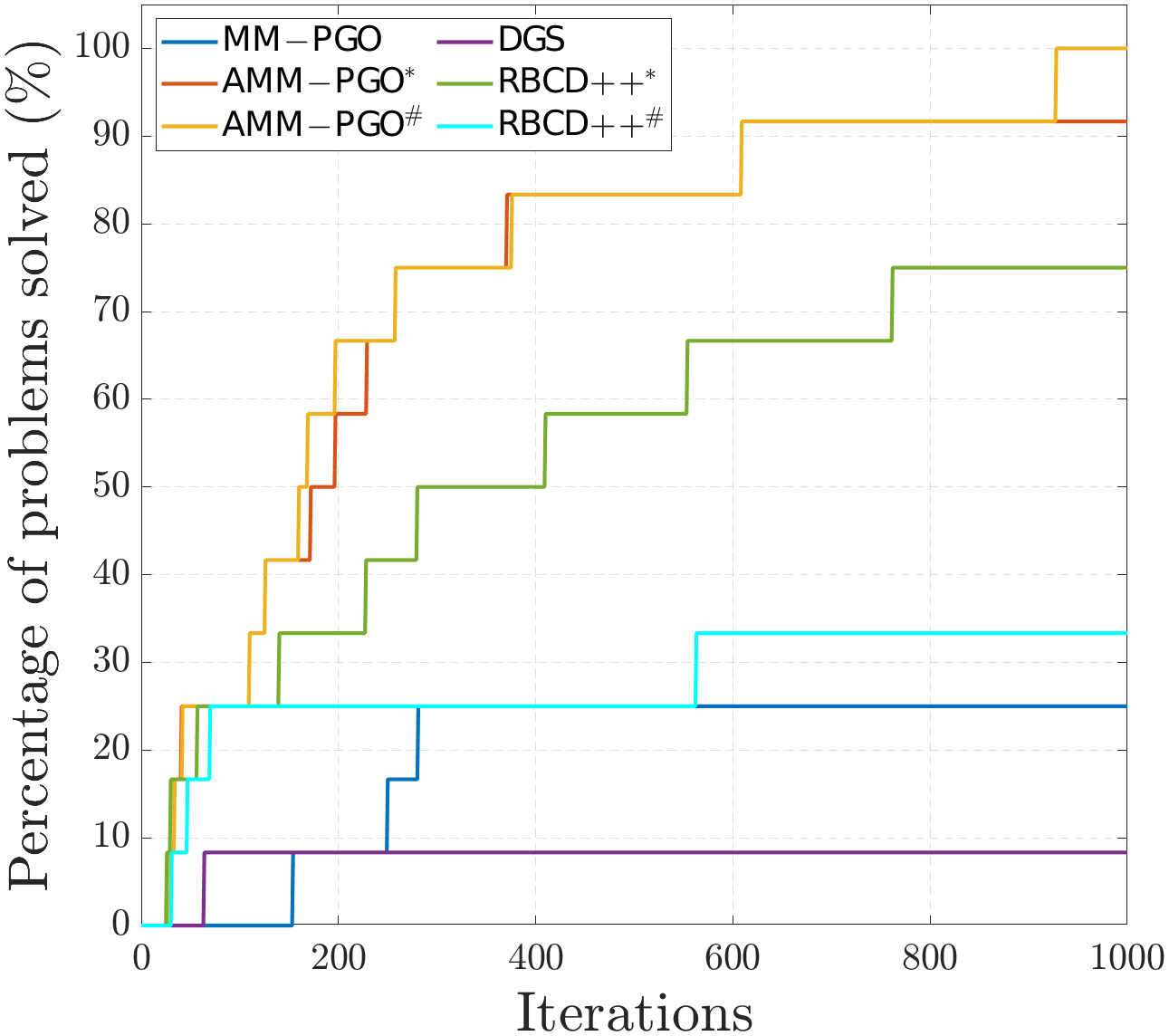}} &
		\hspace{-0.6em}\subfloat[][$\Delta=1\times10^{-3}$]{\includegraphics[trim =0mm 0mm 0mm 0mm,width=0.2425\textwidth]{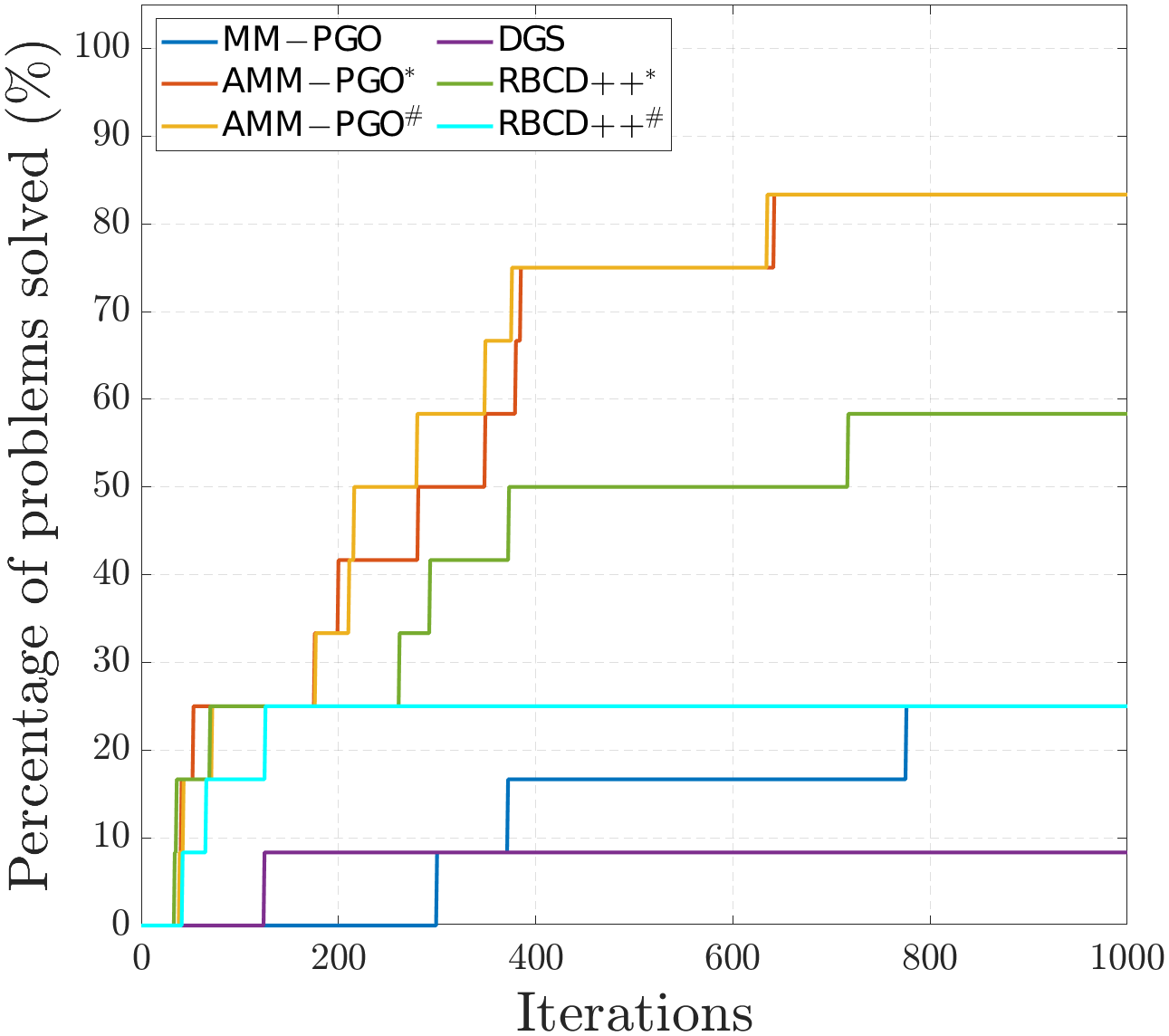}}&
		\hspace{-0.6em}\subfloat[][$\Delta=1\times10^{-4}$]{\includegraphics[trim =0mm 0mm 0mm 0mm,width=0.2425\textwidth]{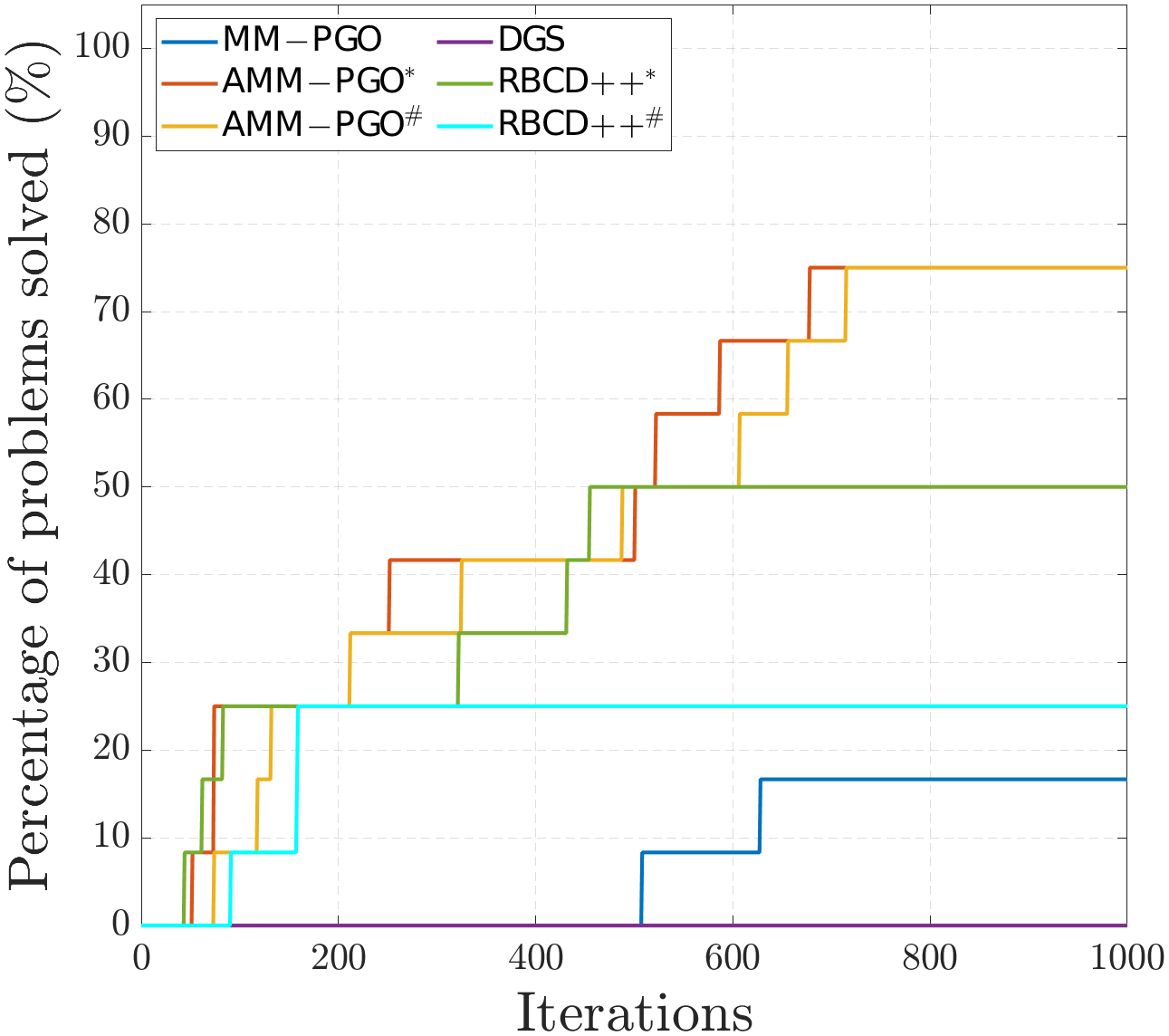}}
	\end{tabular}
	\caption{Performance profiles for $\mm$, $\ammc$, $\ammd$, $\dgs$ \cite{choudhary2017distributed}, $\rbcdc$ \cite{tian2019distributed} and $\rbcdd$ \cite{tian2019distributed} on 2D and 3D SLAM Benchmark datasets (see \datasetinfo). The performance is based on the number of iterations $\sk$ and the evaluation tolerances are $\Delta=1\times10^{-2}$, $5\times10^{-3}$, $1\times10^{-3}$, $1\times10^{-4}$. The distributed PGO has 10 robots (nodes) and is initialized with distributed Nesterov's accelerated chordal initialization \cite{fan2020mm}. Note that $\ammc$ and $\rbcdc$  require a master node, whereas $\mm$, $\ammd$, $\dgs$  and $\rbcdd$ do not.}\label{fig::succ_iter}
	\vspace{-1em}
\end{figure*}

%% file: fig_succ_time.tex
\begin{figure*}[t]
	\centering
	\begin{tabular}{cccc}
		\subfloat[][$\Delta=1\times10^{-2}$]{\includegraphics[trim =0mm 0mm 0mm 0mm,width=0.2425\textwidth]{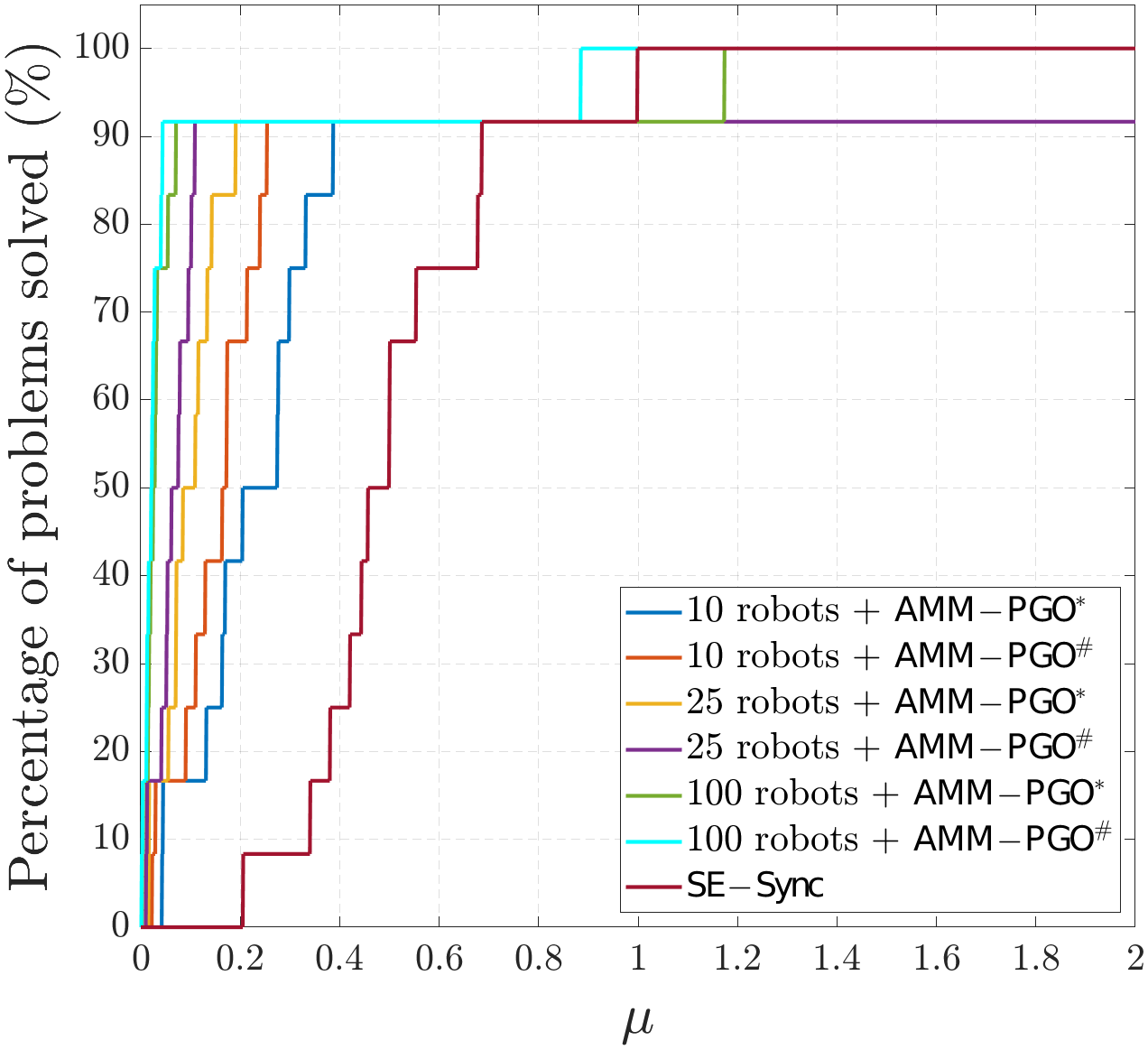}} &
		\subfloat[][$\Delta=1\times10^{-3}$]{\includegraphics[trim =0mm 0mm 0mm 0mm,width=0.2425\textwidth]{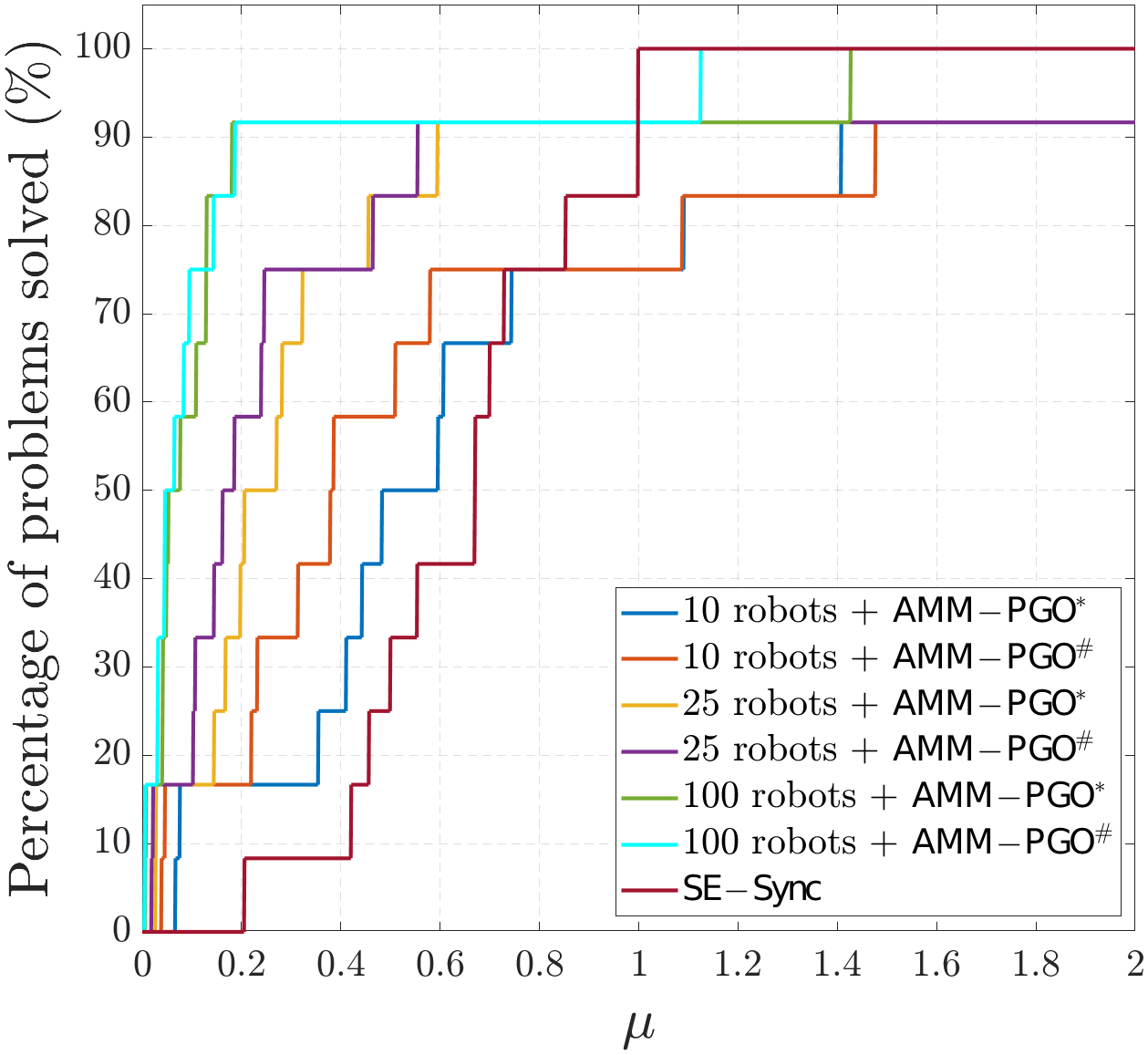}} &
		\subfloat[][$\Delta=1\times10^{-4}$]{\includegraphics[trim =0mm 0mm 0mm 0mm,width=0.2425\textwidth]{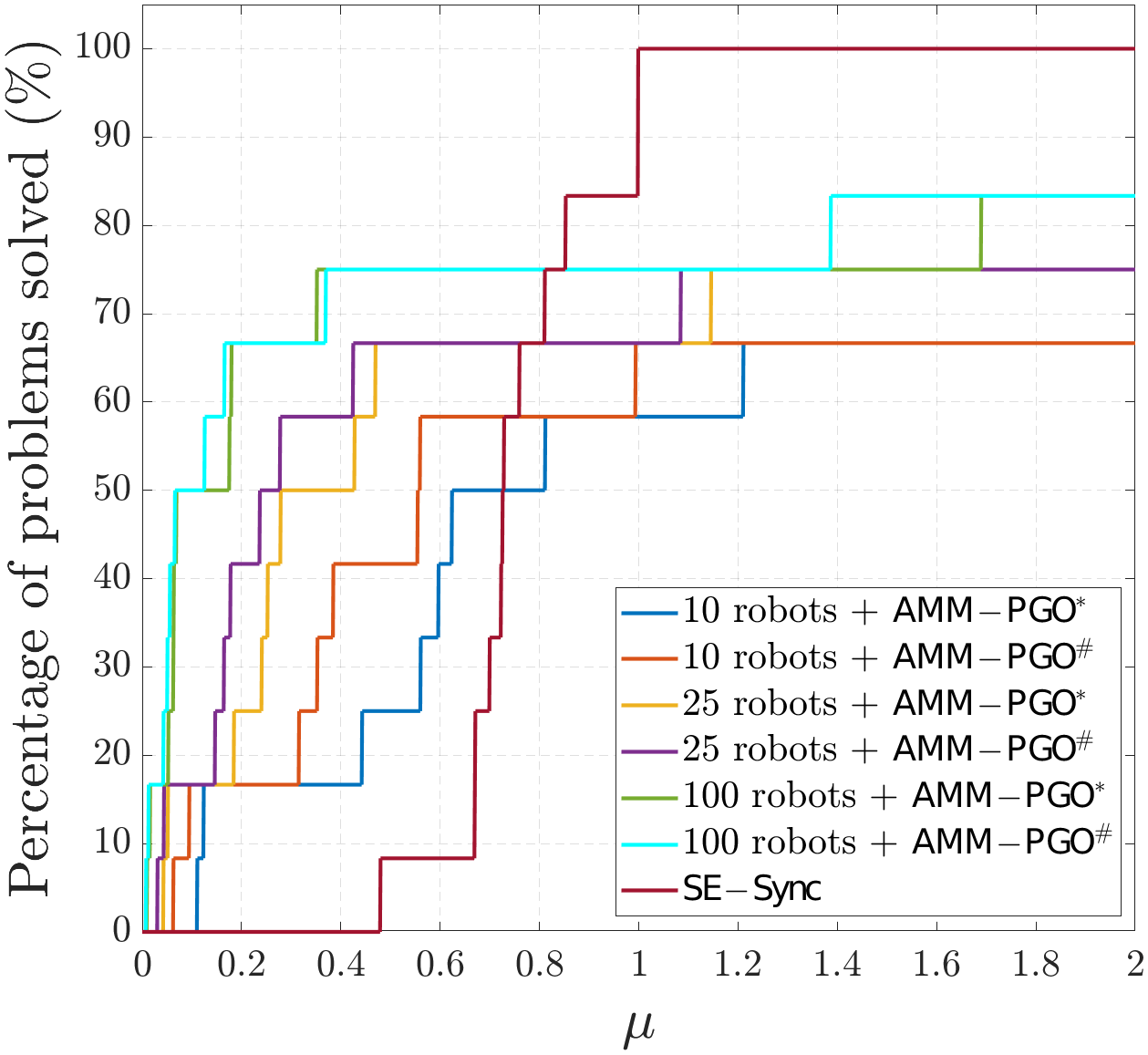}}&
		\subfloat[][$\Delta=1\times10^{-5}$]{\includegraphics[trim =0mm 0mm 0mm 0mm,width=0.2425\textwidth]{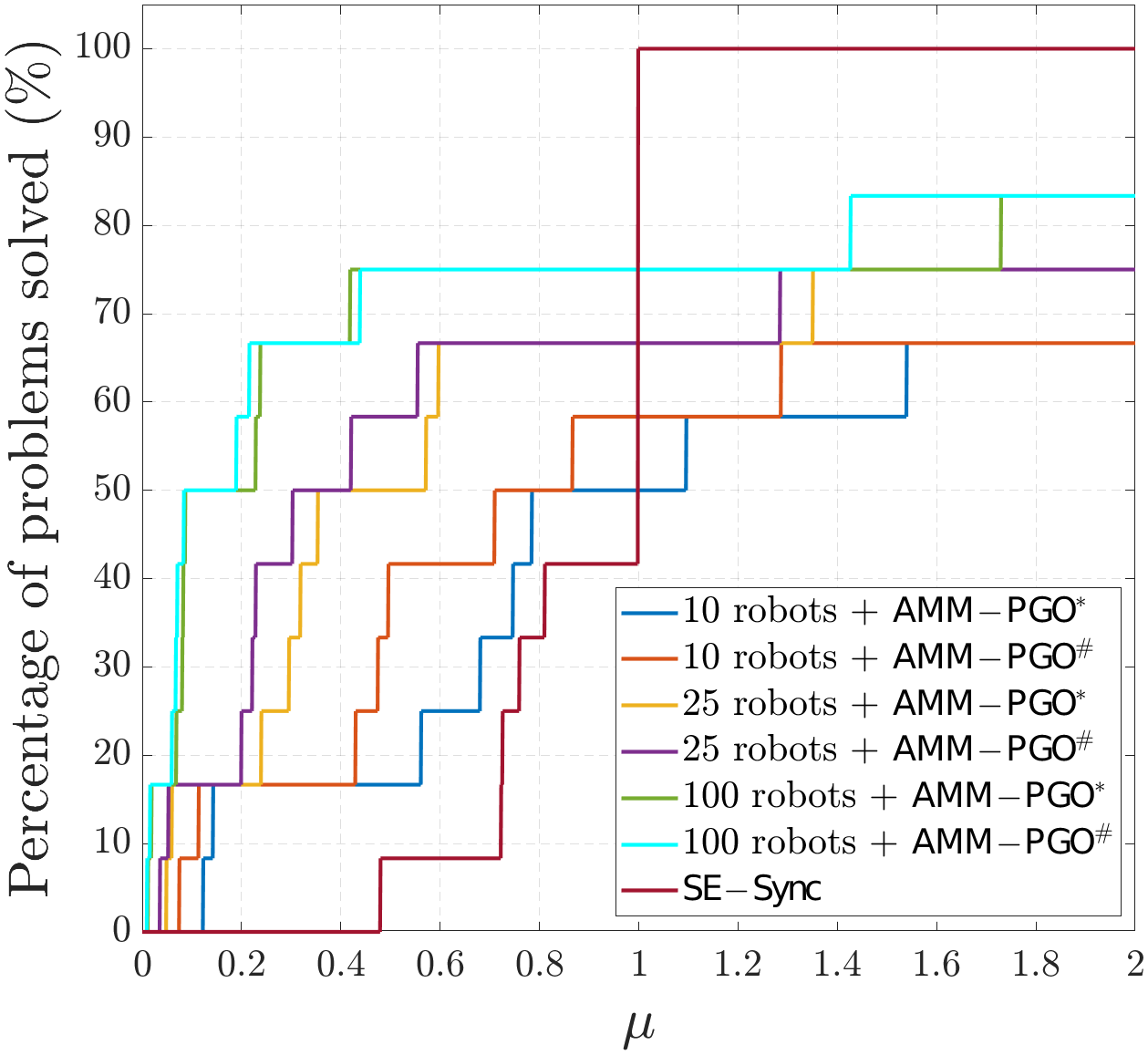}}
	\end{tabular}
	\caption{Performance profiles for $\ammc$, $\ammd$ and $\sesync$ \cite{rosen2016se} on 2D and 3D SLAM benchmark datasets (see \datasetinfo). The performance is based on the scaled average optimization time per node $\ratio\in[0,\,+\infty)$ with  tolerances $\Delta=1\times10^{-2}$, $1\times10^{-3}$, $1\times10^{-4}$, $1\times10^{-5}$. The distributed PGO has 10, 25 and 100 robots (nodes) and is initialized with the centralized  chordal initialization \cite{carlone2015initialization}. Note that $\sesync$  solves all the PGO problems globally at $\ratio =1$. }\label{fig::succ_time}
	\vspace{-1.em}
\end{figure*}

%% file: fig_outlier.tex
\begin{figure}[t]
	\vspace{-0.75em}
	\centering
	\begin{tabular}{cc}
		\hspace{-0.95em}\subfloat[][{\sf intel}]{\includegraphics[trim =0mm 0mm 0mm 0mm,width=0.24\textwidth]{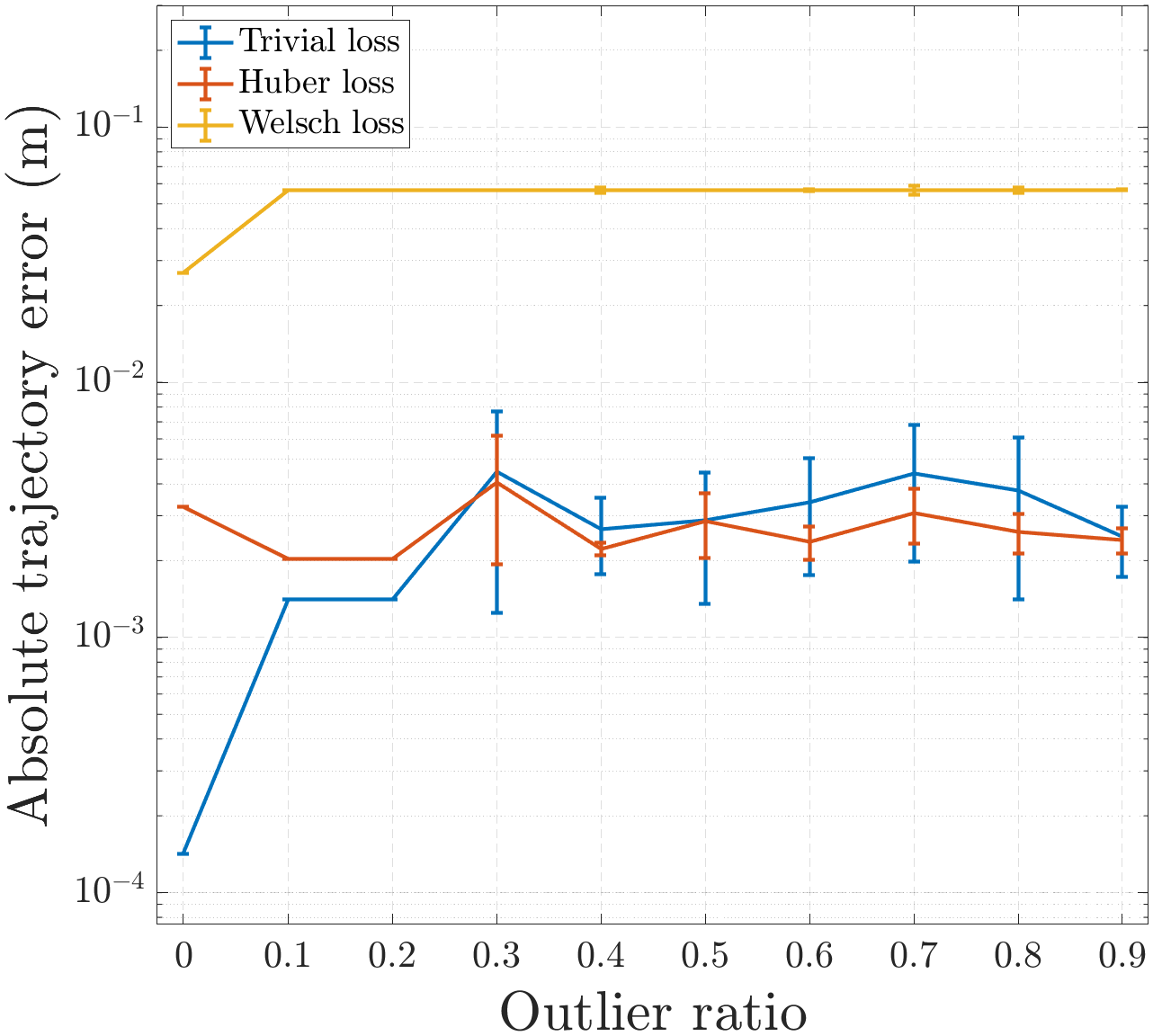}} &
		\hspace{-0.55em}\subfloat[][{\sf garage}]{\includegraphics[trim =0mm 0mm 0mm 0mm,width=0.2425\textwidth]{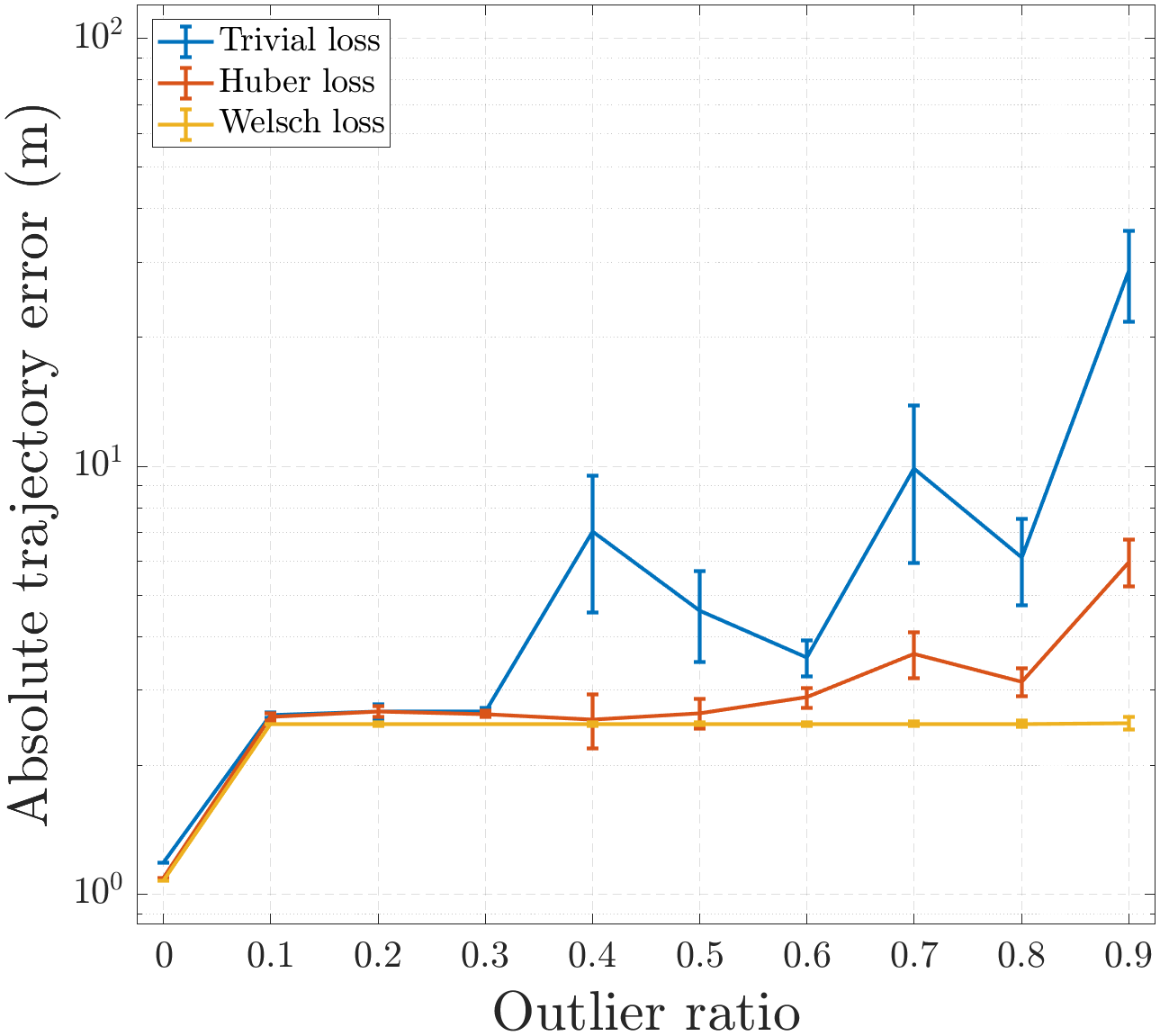}}
	\end{tabular}
	\caption{Absolute trajectory errors (ATE) of distributed PGO using $\ammd$ with the trivial, Huber and Welsch loss kernels on the 2D {\sf intel} and 3D {\sf garage} datasets.  The {\highlight outlier ratios} of inter-node loop closures are $0\sim 0.9$. The ATEs are computed against the outlier-free results of $\sesync$ \cite{rosen2016se} and  are averaged over $10$ Monte Carlo runs.   $\pcm$  \cite{mangelson2018pairwise} is used to initially reject spurious loop closures.   }\label{fig::outlier}
	\vspace{-.5em}
\end{figure}

%% file: fig_garage.tex
\begin{figure*}[t]
	\centering
	\begin{tabular}{cccc}
		\hspace{-0.5em}\subfloat[][$\sesync$]{\includegraphics[trim =0mm 45mm 0mm 40mm,width=0.21\textwidth]{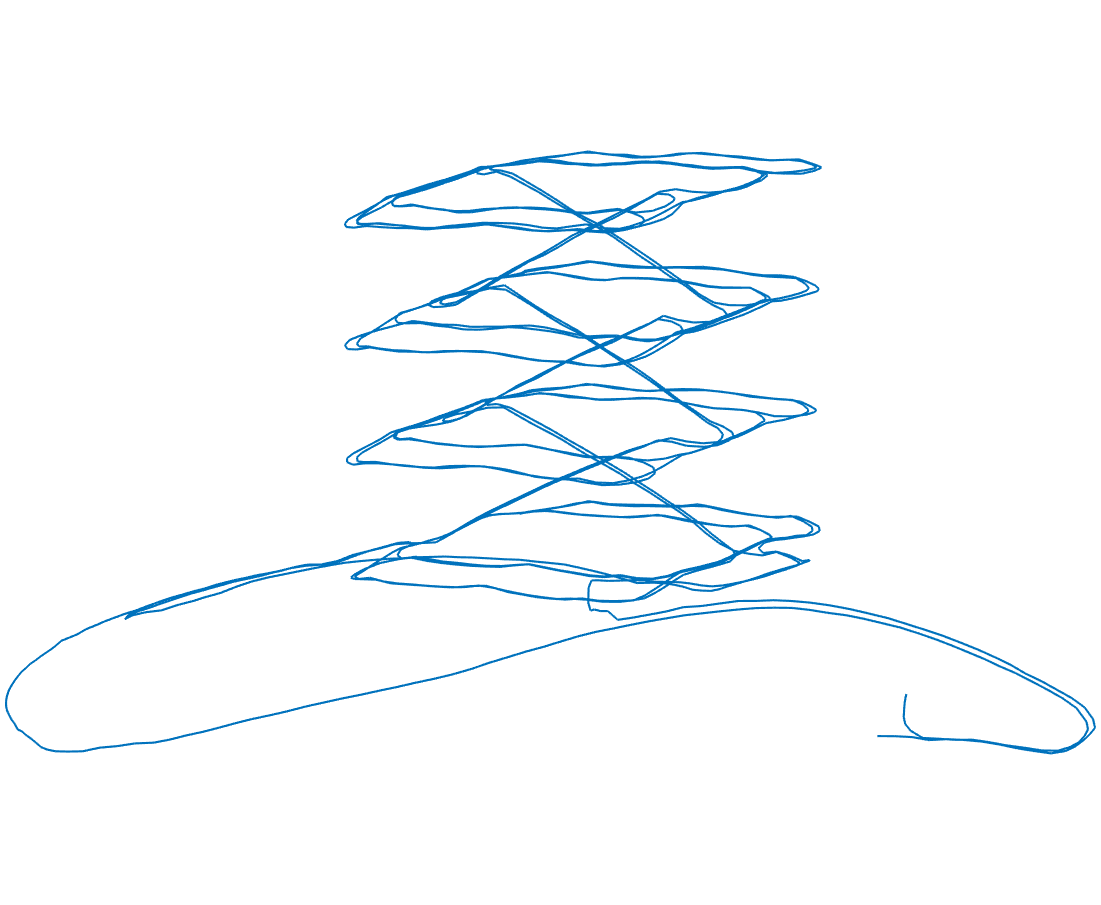}} &
		\hspace{-0.6em}\subfloat[][The trivial loss kernel]{\includegraphics[trim =0mm 45mm 0mm 40mm,width=0.21\textwidth]{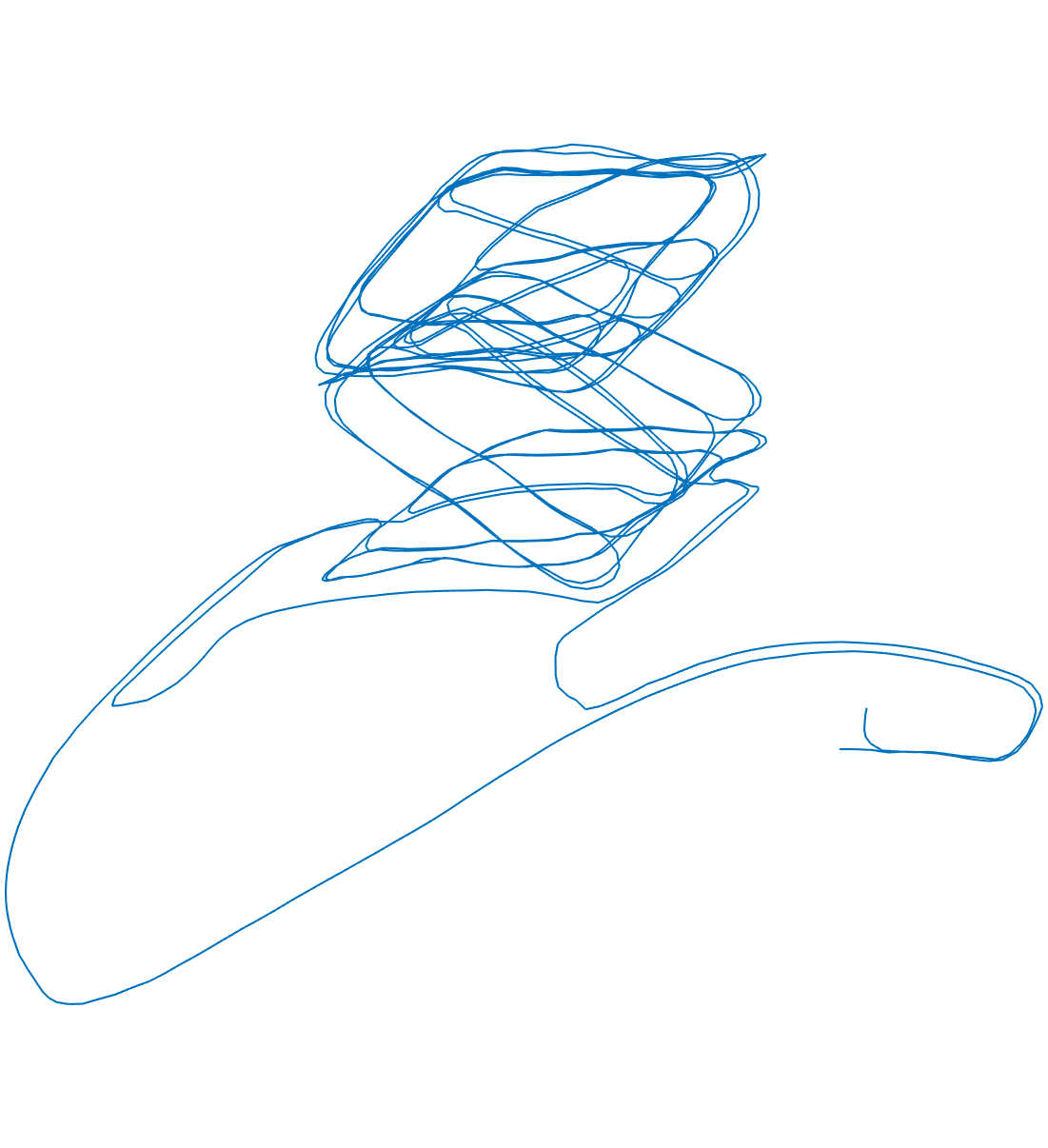}} &
		\hspace{-0.6em}\subfloat[][The Huber loss kernel]{\includegraphics[trim =0mm 45mm 0mm 40mm,width=0.21\textwidth]{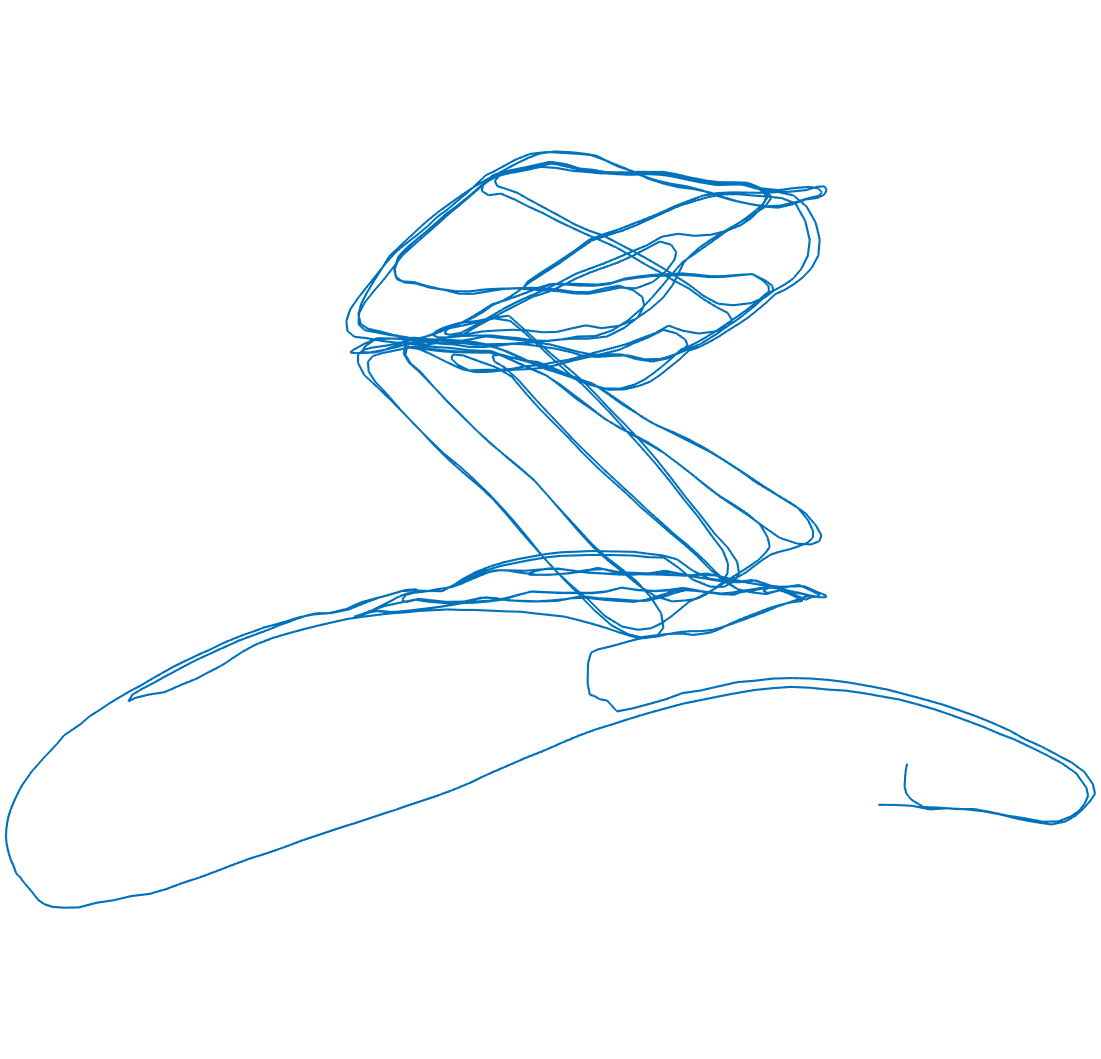}}&
		\hspace{-0.6em}\subfloat[][The Welsch loss kernel]{\includegraphics[trim =0mm 45mm 0mm 40mm,width=0.21\textwidth]{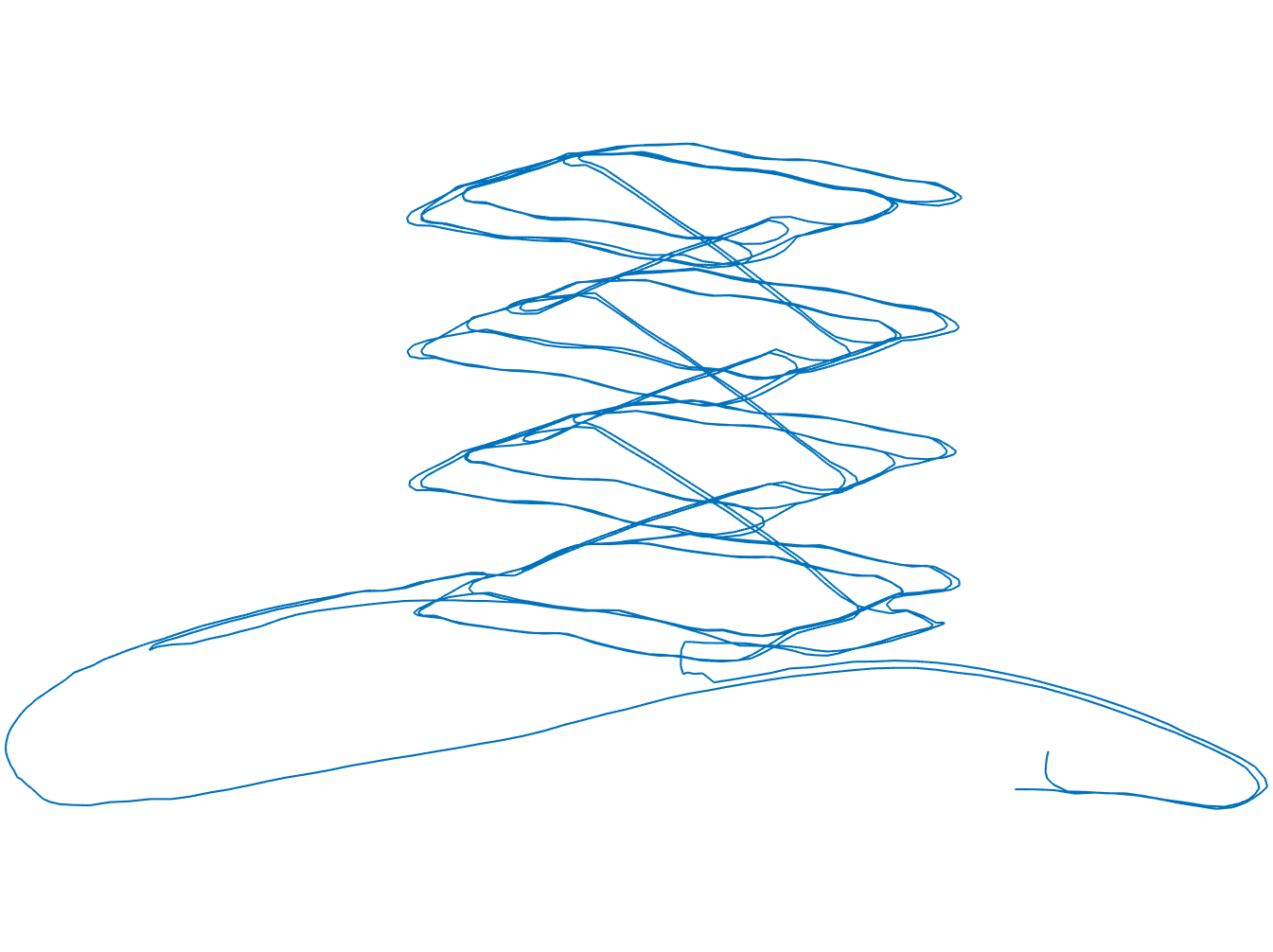}}
	\end{tabular}
	\caption{Qualitative comparisons of distributed PGO with the trivial, Huber and Welsch loss kernels for the {\sf garage} dataset with spurious inter-node loop closures. The outlier-free result of $\sesync$ \cite{rosen2016se} is shown in \cref{fig::outlier_garage}(a) for reference. The {\highlight outlier ratio} of inter-node loop closures is 0.6 and $\pcm$ \cite{mangelson2018pairwise} is used for initial outlier rejection. }\label{fig::outlier_garage}
	\vspace{-1.25em}
\end{figure*}

%% file: conclusions.tex
We presented majorization minimization (MM) methods for distributed PGO that has important applications in multi-robot SLAM. Our MM methods had provable convergence for a broad class of robust loss kernels in robotics and computer vision. Furthermore, we elaborated on the use of Nesterov's method and adaptive restart for acceleration and developed accelerated MM methods $\ammc$ and $\ammd$ without sacrifice of convergence guarantees. In particular, we designed a novel adaptive restart scheme making  $\ammd$  without a master node comparable to  $\ammc$  using a master node for information aggregation. The extensive experiments on numerous 2D and 3D SLAM datasets indicated that our MM methods outperformed existing state-of-the-art methods and  robustly handled distributed PGO with outlier inter-node loop closures.

Our MM methods for distributed PGO can be improved as  follows.  A more tractable and robust initialization technique is definitely beneficial to the accuracy and efficiency of distributed PGO. Even though our MM methods have  reliable performances against outliers, a more complete theoretical analysis for robust distributed PGO is still necessary. We might also extend our MM methods for differentiable distributed PGO \cite{pineda2022theseus}.  In addition, our MM methods can be implemented as local solvers for distributed certifiably correct PGO \cite{tian2019distributed}  to handle poor or random initialization. Since all the nodes are now assumed to be synchronized, it is necessary and useful to extend our MM methods for asynchronous distributed PGO. {\highlight Lastly, real multi-robot tests might make the results of our MM methods more convincing where not only the optimization time but also the communication overhead can be validated.}